\documentclass[runningheads]{llncs}
\usepackage{graphicx}
\usepackage{comment}
\usepackage{amsmath,amssymb} %
\usepackage{color}
\usepackage{subfig}
\usepackage{floatrow}
\usepackage[dvipsnames]{xcolor}
\usepackage[ruled]{algorithm2e}
\usepackage[utf8]{inputenc}
\usepackage[english]{babel}
\usepackage{lipsum}
\usepackage{siunitx}
\usepackage{pifont}
\usepackage{booktabs}
\newcolumntype{P}[1]{>{\centering\arraybackslash}p{#1}}
\newcolumntype{M}[1]{>{\centering\arraybackslash}m{#1}}
\usepackage{bbm}
\usepackage{multirow}

\newcommand{\etal}{\textit{et. al.}~}
\usepackage[utf8]{inputenc} %
\usepackage[T1]{fontenc}    %
\usepackage{url}            %
\usepackage{booktabs}       %
\usepackage{amsfonts}       %
\usepackage{nicefrac}       %
\usepackage{microtype}      %
\usepackage{tabularx}
\usepackage{paralist}

\usepackage{caption}
\usepackage[T1]{fontenc}
\usepackage{lmodern}

\usepackage[colorlinks=true,linkcolor=red]{hyperref}

\DeclareMathOperator*{\argmin}{arg\,min}

\definecolor{green}{RGB}{0,150,10}

\newcommand*{\pd}[3][]{\ensuremath{\frac{\partial^{#1} #2}{\partial #3}}}
\newcommand{\xmark}{\ding{54}}%

\newcommand{\figLabel}{Fig.~}
\newcommand{\eg}{\textit{e.g.~}}
\newcommand{\ie}{\textit{i.e.~}}
\newcommand{\eqLabel}[1]{{Eq (#1)}}
\newcommand{\secLabel}{Section~}

\newcommand{\mysection}[1]{\vspace{3pt}\noindent\textbf{#1.}}
\newcommand{\supp}{\textbf{supplement~}}
\newcommand{\specialcell}[2][c]{%
  \begin{tabular}[#1]{@{}c@{}}#2\end{tabular}}
\newcommand\blfootnote[1]{%
  \begingroup
  \renewcommand\thefootnote{}\footnote{#1}%
  \addtocounter{footnote}{-1}%
  \endgroup
}

\begin{document}
\pagestyle{headings}
\mainmatter
\def\ECCVSubNumber{11}  %

\title{Towards Analyzing Semantic Robustness of Deep Neural Networks} %

\titlerunning{Towards Analyzing Semantic Robustness of Deep Neural Networks}
\author{Abdullah Hamdi \and Bernard Ghanem}
\authorrunning{A. Hamdi, B. Ghanem}
\institute{King Abdullah University of Science and Technology (KAUST), Thuwal, Saudi Arabia\\
\email{\{abdullah.hamdi, bernard.ghanem\}@kaust.edu.sa}}
\maketitle

\begin{abstract}
Despite the impressive performance of Deep Neural Networks (DNNs) on various vision tasks, they still exhibit erroneous high sensitivity toward semantic primitives (\eg object pose). We propose a theoretically grounded analysis for DNNs robustness in the semantic space. We qualitatively analyze different DNNs semantic robustness by visualizing the DNN global behavior as semantic maps and observe interesting behavior of some DNNs. Since generating these semantic maps does not scale well with the dimensionality of the semantic space, we develop a bottom-up approach to detect robust regions of DNNs. To achieve this, we formalize the problem of finding robust semantic regions of the network as optimization of integral bounds and develop expressions for update directions of the region bounds. We use our developed formulations to quantitatively evaluate the semantic robustness of different famous network architectures. We show through extensive experimentation that several networks, though trained on the same dataset and while enjoying comparable accuracy, they do not necessarily perform similarly in semantic robustness. For example, InceptionV3 is more accurate despite being less semantically robust than ResNet50. We hope that this tool will serve as the first milestone towards understanding the semantic robustness of DNNs.
\end{abstract}
\linespread{0.98}

\section{Introduction} \label{sec:intro}
As a result of recent advances in machine learning and computer vision, deep neural networks (DNNS) have become an essential part of our lives.\blfootnote{The code is available at \url{https://github.com/ajhamdi/semantic-robustness}.} DNNs are used to suggest articles to read, detect people in surveillance cameras, automate big machines in factories, and even diagnose X-rays for patients in hospitals. So, what is the catch here? These DNNs struggle with a detrimental weakness on specific naive scenarios, despite having strong on-average performance. \figLabel{\ref{fig:intro_fig}} shows how a small perturbation in the view angle of the teapot object results in a drop in InceptionV3 \cite{inception} confidence score from 100\% to almost 0\%. The \emph{softmax} confidence scores are plotted against one semantic parameter (\ie the azimuth angle around the teapot) and it fails in such a simple task. Similar behaviors are consistently observed across different DNNs (trained on ImageNet \cite{IMAGENET}) as noted by other concurrent works \cite{strike}.
\begin{figure}[t]
\tabcolsep=0.03cm
   \begin{tabular}{ccc}
      \includegraphics[trim={1.5cm 1.5cm 1.5cm 1.5cm},clip, width = 2cm]{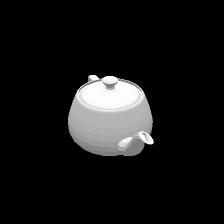} &
\includegraphics[trim={1.5cm 1.5cm 1.5cm 1.5cm},clip, width = 2cm]{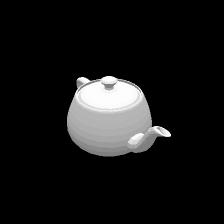}  &

      \includegraphics[trim={0.7cm 1.2cm 0.7cm 1.72cm},clip, width=0.65\linewidth]{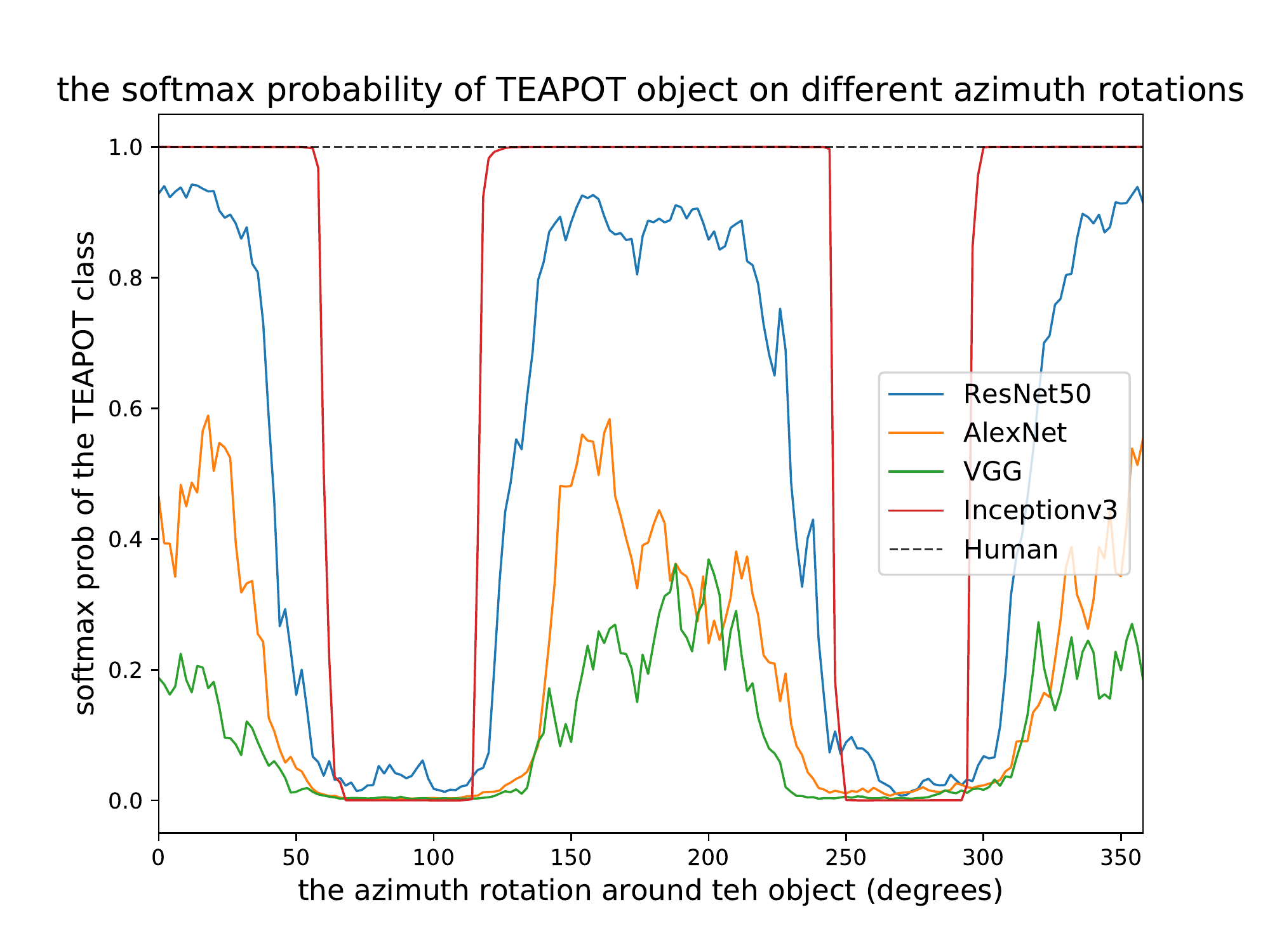}
       \\
       angle:\ang{68} & angle:\ang{55} &  \scriptsize The azimuth angle around the TEAPOT object (degrees) \\
       \textcolor{red}{teapot:0.01\%} &\textcolor{green}{teapot:99.89\%} & \\
\end{tabular}
   \caption{\small \textbf{Semantic Robustness of Deep Networks}. Trained neural networks can perform poorly when subject to small perturbations in the semantics of the image. (\textit{left}): We show how perturbing the azimuth viewing angle of a simple \emph{teapot} object can dramatically affect the score of a pretrained InceptionV3 \cite{inception} for the \emph{teapot} class. (\textit{right}): We plot the softmax confidence scores of different DNNs on the same \emph{teapot} object viewed from 360 degrees around the object. For comparison, lab researchers identified the object from all angles.}
   \label{fig:intro_fig}
\end{figure}

Furthermore, because DNNs are not easily interpretable, they work well without a complete understanding of \textit{why} they behave in such a manner. A whole research direction is dedicated to studying and analyzing DNNs. Examples of such analysis include activation visualization \cite{unn-visual1,unn-visual3,unn-visual2}, noise injection \cite{unn-robustness-noise1,unn-universal,unn-modar}, and studying the effect of image manipulation on DNNs \cite{unn-texture,unn-texture,unn-robustness-geometry}. By leveraging a differentiable renderer $\mathbf{R}$ and evaluating rendered images for different semantic parameters $\mathbf{u}$, we provide a new lens of semantic robustness analysis for such DNNs as illustrated in \figLabel{\ref{fig:nimi-pipeline}}. These Network Semantic Maps (NSM) demonstrate unexpected behavior of some DNNs, in which adversarial regions lie inside a very confident region of the semantic space. This constitutes a ``trap'' that is hard to detect without such analysis and can lead to catastrophic failure for the DNN.

Recent work in adversarial network attacks explores DNN sensitivity and performs gradient updates to derive targeted perturbations \cite{first-attack,fast-sign,carlini,projected-gradient}. In practice, such attacks are less likely to naturally occur than semantic attacks, such as changes in camera viewpoint and lighting conditions. The literature on semantic attacks is sparser, since they are more subtle and challenging to analyze \cite{renderer-attack,sada,physicalattack,strike}. This is due to the fact that we are unable to distinguish between failure cases that result from the network structure, and learning, or from the data bias \cite{bias}. Current methods for adversarial semantic attacks either work on individual examples \cite{strike}, or try to find distributions but rely on sampling methods, which do not scale with dimensionality \cite{sada}. We present a novel approach to find robust/adversarial regions in the n-dimensional semantic space. The proposed method scales better than sampling-based methods \cite{sada}. We use this method to quantify semantic robustness of popular DNNs on a collected dataset.

\mysection{Contributions}
\textbf{(1)} We analyze popular deep networks from a semantic lens showing unexpected behavior in the 1D and 2D semantic space. \textbf{(2)} We develop a novel bottom-up approach to detect robust/adversarial regions in the semantic space of a DNN, which scales well with increasing dimensionality. The method specifically optimizes for the region's bounds in semantic space (around a point of interest), such that the continuous region offers semantic parameters that confuse the network. \textbf{(3)} We develop a new metric to quantify semantic robustness of DNNs that we dub Semantic Robustness Volume Ratio (SRVR), and we use it to benchmark popular DNNs on a collected dataset. 

\begin{figure}[t]
    \centering
  \includegraphics[width=0.8\linewidth]{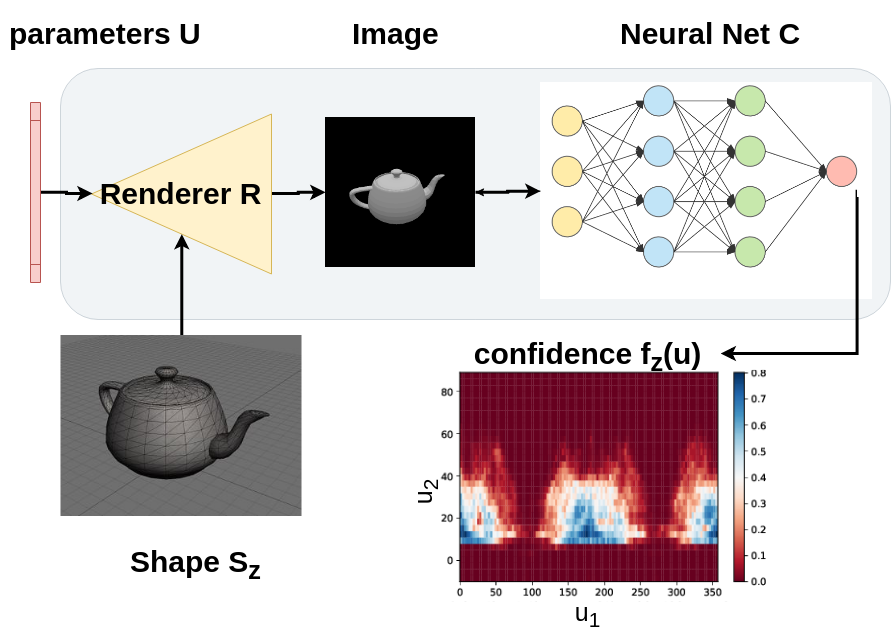}
  \caption{\small \textbf{Analysis Pipeline}: We leverage neural mesh renderer $\mathbf{R}$ \cite{neral-renderer} to render shape $\mathbf{S}_{z}$ of class $z$ according to semantic scene parameters $\mathbf{u}$. The resulting image is passed to trained network $\mathbf{C}$ that is able to identify the class $z$. The behaviour of the softmax score at label $z$ (dubbed $f_{z}(\mathbf{u})$) is analyzed for different parameters $\mathbf{u}$ and for the specific shape $\mathbf{S}_{z}$. In our experiments, we pick $u_1$ and $u_2$ to be the azimuth and elevation angles, respectively.}
  \label{fig:nimi-pipeline}
\end{figure}
\section{Related Work}
\subsection{Understanding Deep Neural Networks}
 There are different lenses to analyze DNNs depending on the purpose of analysis. A popular line of work tries to visualize the network hidden layers by inverting the activations to get a visual image that represents a specific activation  \cite{unn-visual1,unn-visual2,unn-visual3}. Others observe the behavior of these networks under injected noise \cite{unn-robustness-noise1,unn-robustness-noise2,unn-robustness-noise3,unn-robustness-noise4,unn-robustness-noise5,unn-modar}. Geirhos \etal show that changing the texture of the object while keeping the borders can hugely deteriorate the recognizability of the object by the DNN \cite{unn-texture}. More closely related to our work is the work of Fawzi \etal, which shows that geometric changes in the image greatly affect the performance of the classifier \cite{unn-robustness-geometry}. The work of Engstorm \etal \cite{spatial-robustness} studies the robustness of a network under natural 2D transformations (\ie translation and planar rotation).
\subsection{Adversarial Attacks on Deep Neural Networks}
\mysection{Pixel-based Adversarial Attacks}
The way that DNNs fail for some noise added to the image motivated the adversarial attacks literature. Several works formulate attacking neural networks as an optimization on the input pixels \cite{first-attack,fast-sign,projected-gradient,deepfool,carlini}. However, all these methods are limited to pixel perturbations and only fool classifiers, while we consider more general cases of attacks, \eg changes in camera viewpoint to fool a DNN by finding adversarial regions. %
Most attacks are white-box attacks, in which the attack algorithm has access to network gradients. Another direction of adversarial attacks treats the classifier as a black-box, where the adversary can only probe the network and get a score from the classifier without backpropagating through the DNN \cite{reduc-black,zeroth-order-attack}. We formulate the problem of finding the robust/adversarial region as an optimization of the corners of a hyper-rectangle in the semantic space for both black-box and white-box attacks.

\mysection{Semantic Adversarial Attacks}
Other works tried to move away from pixel perturbation to semantic 3D scene parameters and 3D attacks \cite{physicalattack,sada,strike,renderer-attack,advpc}. Zeng \etal \cite{physicalattack} generate attacks on deep classifiers by perturbing scene parameters like lighting and surface normals. Hamdi \etal propose generic adversarial attacks that incorporate semantic and pixel attacks, in which an adversary is sampled from some latent distribution that is produced from a GAN trained on example semantic adversaries \cite{sada}. However, their work used a sampling-based approach to learn these adversarial regions, which does not scale with the dimensionality of the problem. Another recent work by Alcorn \etal \cite{strike} tries to fool trained DNNs by changing the pose of the object. They used the Neural Mesh renderer (NMR) by Kato \etal \cite{vig-nmr} to allow for a fully differentiable pipeline that performs adversarial attacks based on the gradients to the parameters. Our work differs in that we use NMR to obtain gradients to the parameters $\mathbf{u}$ not to attack the model, but to detect and quantify the robustness of different networks as shown in \secLabel{\ref{sec:application}}. Furthermore, Dreossi \etal \cite{semantic-deep-learning} used adversarial training in the semantic space for self-driving, whereas Liu \etal \cite{renderer-attack} proposed a differentiable renderer to perform parametric attacks and the \textit{parametric-ball} as an evaluation metric for physical attacks. The work by Shu \etal \cite{adv-examiner} used an RL agent and a Bayesian optimizer to asses the DNNs behaviour under good/bad physical parameters for the network. While we share similar insights as \cite{adv-examiner}, we try to study the global behaviour of DNNs as collections of regions, whereas \cite{adv-examiner} tries to find individual points that pose difficulty for the DNN.    
\subsection{Optimizing Integral Bounds}
\mysection{Naive Approach} 
To develop an algorithm for robust region finding, we adopt an idea from weakly supervised activity detection in videos by Shou \etal \cite{ioc}. The idea is to find bounds that maximize the inner average of a continuous function while minimizing the outer average in a region. This is achieved because optimizing the bounds to exclusively maximize the area  can lead to diverging bounds of $\{-\infty,\infty\}$. To solve the issue of diverging bounds, the following naive formulation is to simply regularize the loss by adding a penalty on the region size. The expressions for the loss of $n$=1 dimension is: $L = -\text{Area}_{\text{in}} + \frac{\lambda}{2} \left| b-a\right|_{2}^{2} = \int_{a}^{b} f(u)du ~+ \frac{\lambda}{2} \left| b-a\right|_{2}^{2}$, where $f: \mathbb{R}^{1} \rightarrow (0,1)$ is the function of interest and $(a,b)$ are the left and right bounds respectively and $\lambda$ is a hyperparameter. The update directions to minimize the loss are: $\pd{L}{a} = f(a) - \lambda (b-a)~;~\pd{L}{b} = - f(b) + \lambda (b-a)$. The regularizer will prevent the region from growing to $\infty$ and the best bounds will be found if the loss is minimized with gradient descent or any similar approach.

\mysection{Trapezoidal Approximation}
To extend the naive approach to $n$-dimensions, we face more integrals in the update directions (hard to compute). Therefore, we deploy the following first-order trapezoid approximation of definite integrals. The Newton-Cortes formula for numerical integration \cite{numerical} states that: $\int_{a}^{b}f(u)du \approx (b-a)\frac{f(a)+f(b)}{2}$.
An asymptotic error estimate is given by $-\frac { ( b - a ) ^ { 2 } } { 48 } \left[ f ^ { \prime } ( b ) - f ^ { \prime } ( a ) \right] + \mathcal{O} \left( \frac{1}{8} \right)$. So, as long as the derivatives are bounded by some Lipschitz constant $\mathbb{L}$, then the error becomes bounded such that $|\text{error}| \leq \mathbb{L}( b - a ) ^ { 2 }  $. 

\begin{figure}[t]
\centering
 \begin{tabular}{cc}
 \small Network Confidence under 1 Parameter &  \small Network Confidence under 2 Parameters \\
     \includegraphics[trim={1cm 0 1.0cm 1.1cm},clip,width=.48\linewidth]{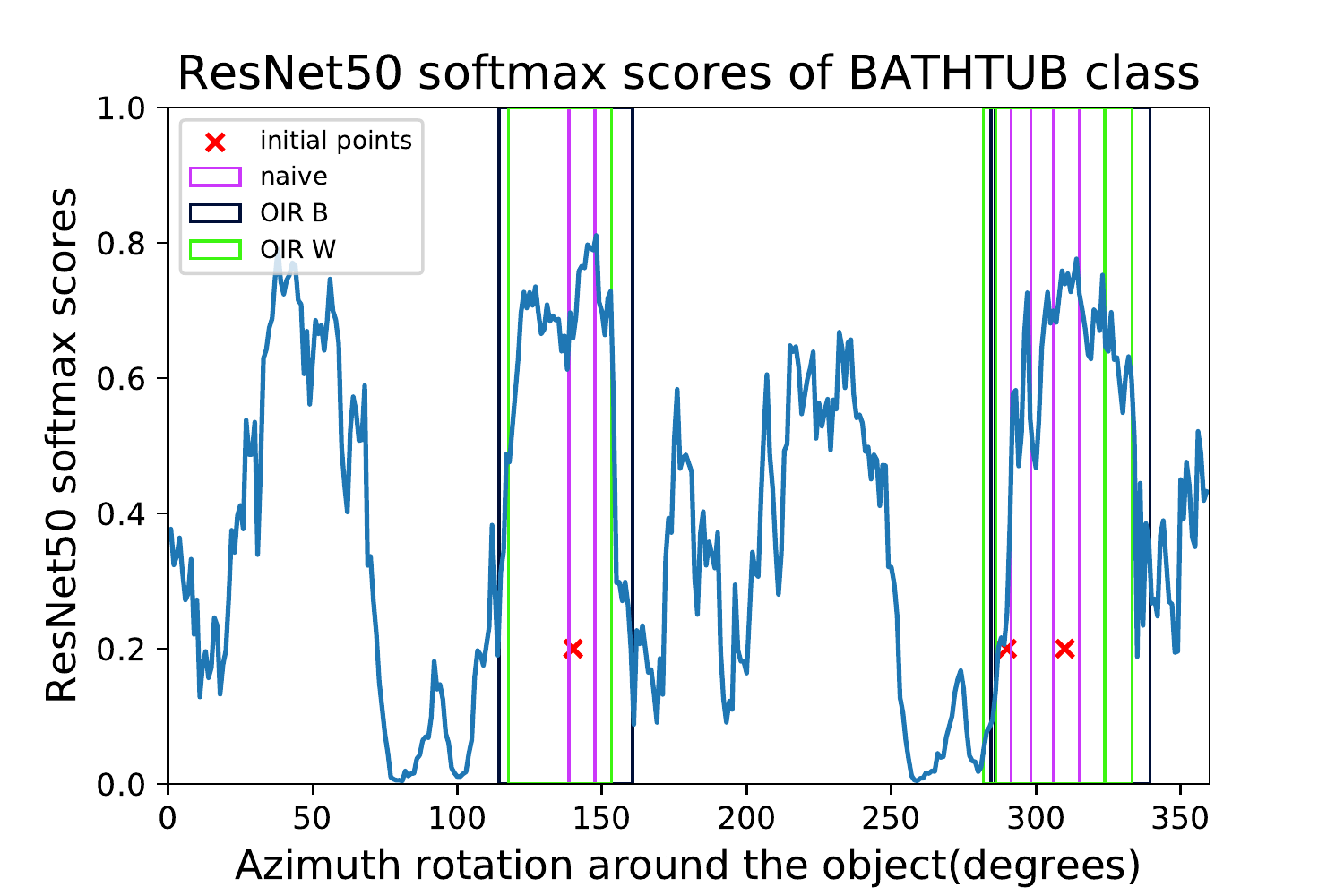} 
  &\includegraphics[trim={2.2cm 0 1.2cm 1.1cm},clip,width=.51\linewidth]{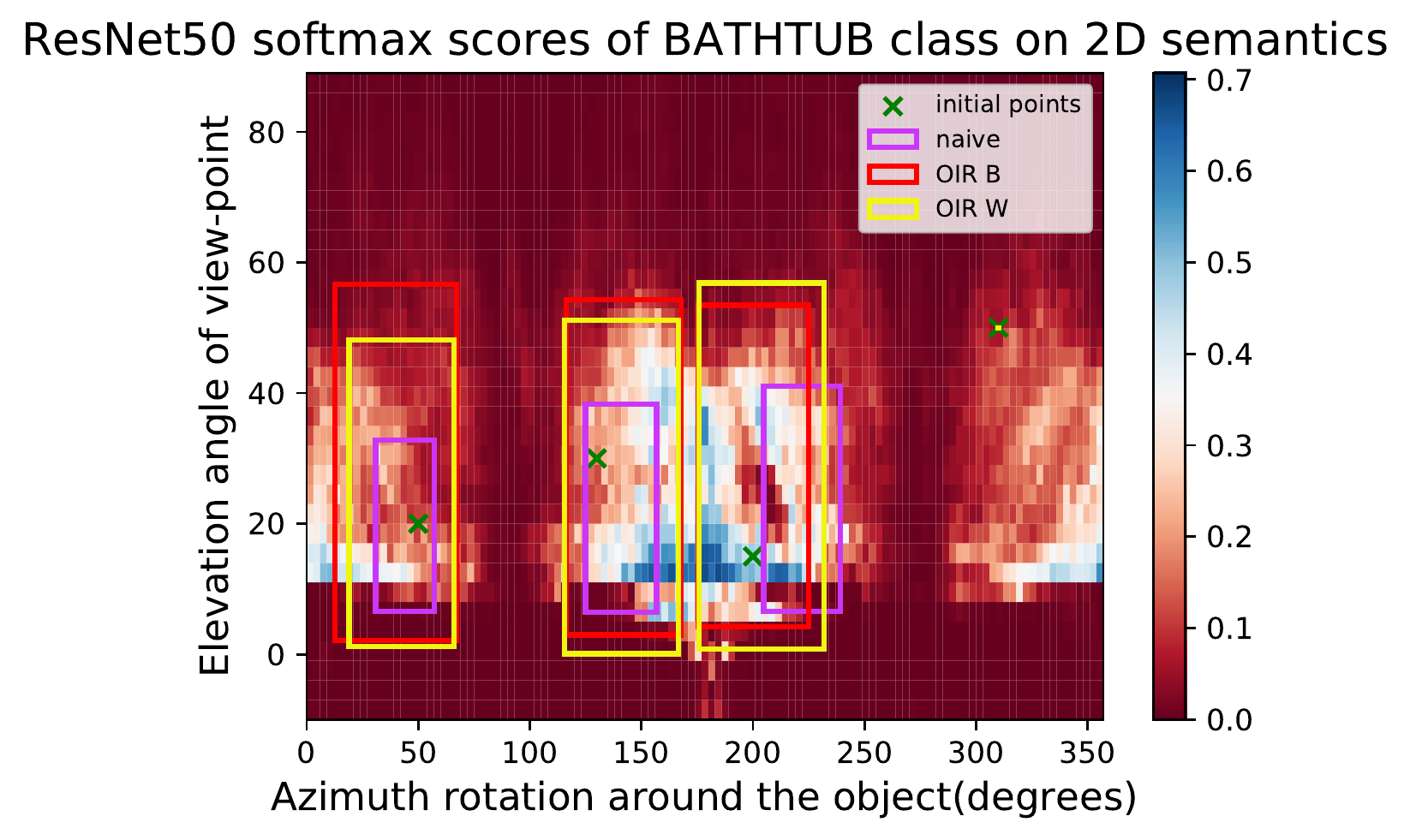}
 \end{tabular}

\caption{ \small \textbf{Semantic Robust Region Finding}: We find robust regions of semantic parameters for Rsnet50 \cite{resnet} and for a \emph{bathtub} object by the three bottom-up formulations (naive , OIR\_W , and OIR\_B). (\textit{left}): Semantic space is 1D (azimuth angle of camera) with three initial points. (\textit{right}): Semantic space is 2D  (azimuth angle and elevation angle of camera) with four initial points. We note that the naive approach usually predicts smaller regions, while the OIR formulations find more comprehensive regions.}
\label{fig:operator}
\end{figure}

\begin{figure}[h]
\centering
\tabcolsep=0.08cm
  \begin{tabular}{c|c|c|c}
  \textbf{AlexNet}\cite{AlexNet} & \textbf{VGG}\cite{vgg} & \textbf{Resnet50}\cite{resnet} & \textbf{InceptionV3}\cite{inception} \\
  \includegraphics[trim={3cm 1.4cm 2.8cm 1.2cm},clip, width=.24\linewidth]{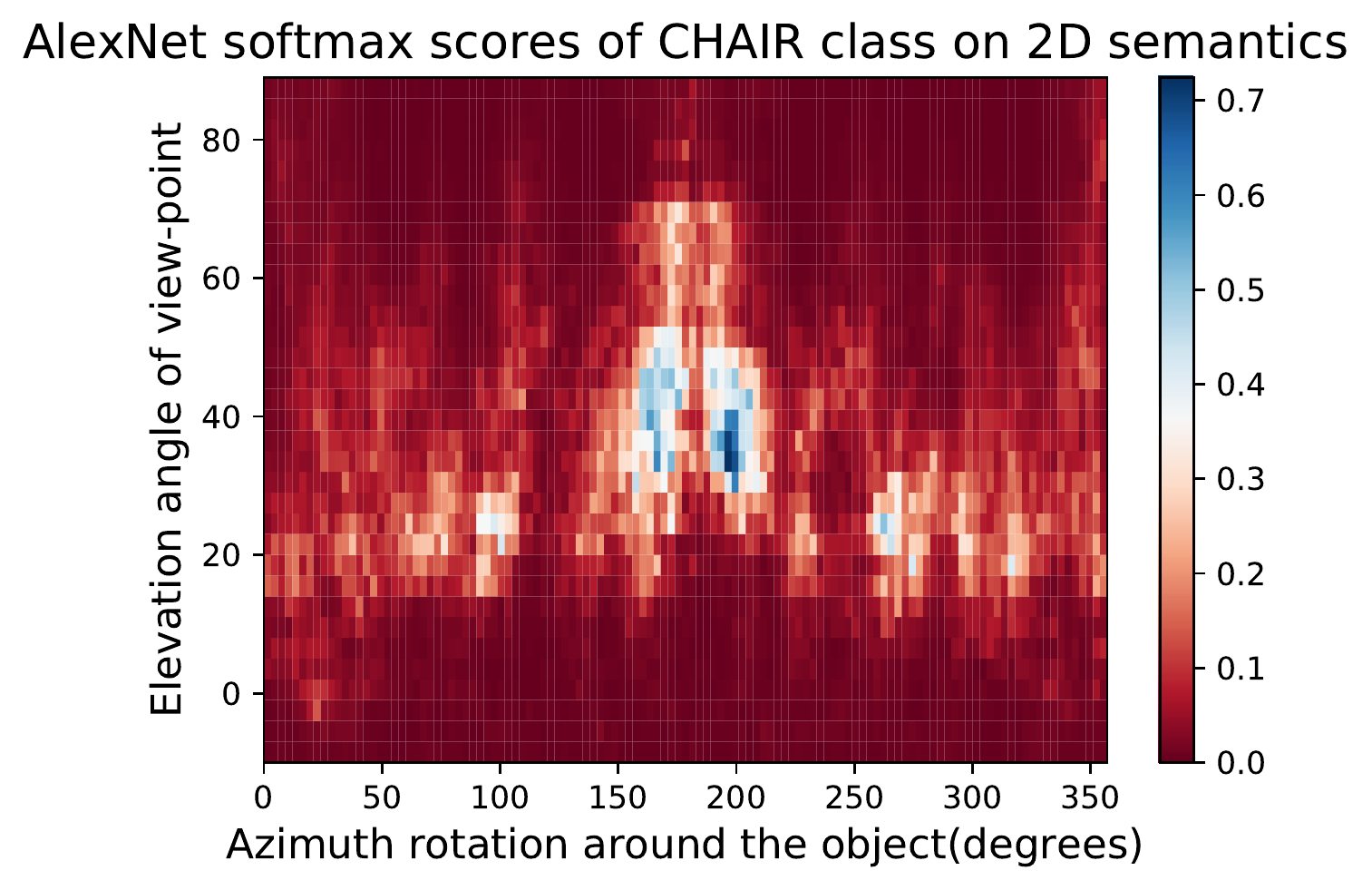}&
  \includegraphics[trim={3cm 1.8cm 2.8cm 1.2cm},clip, width=.23\linewidth]{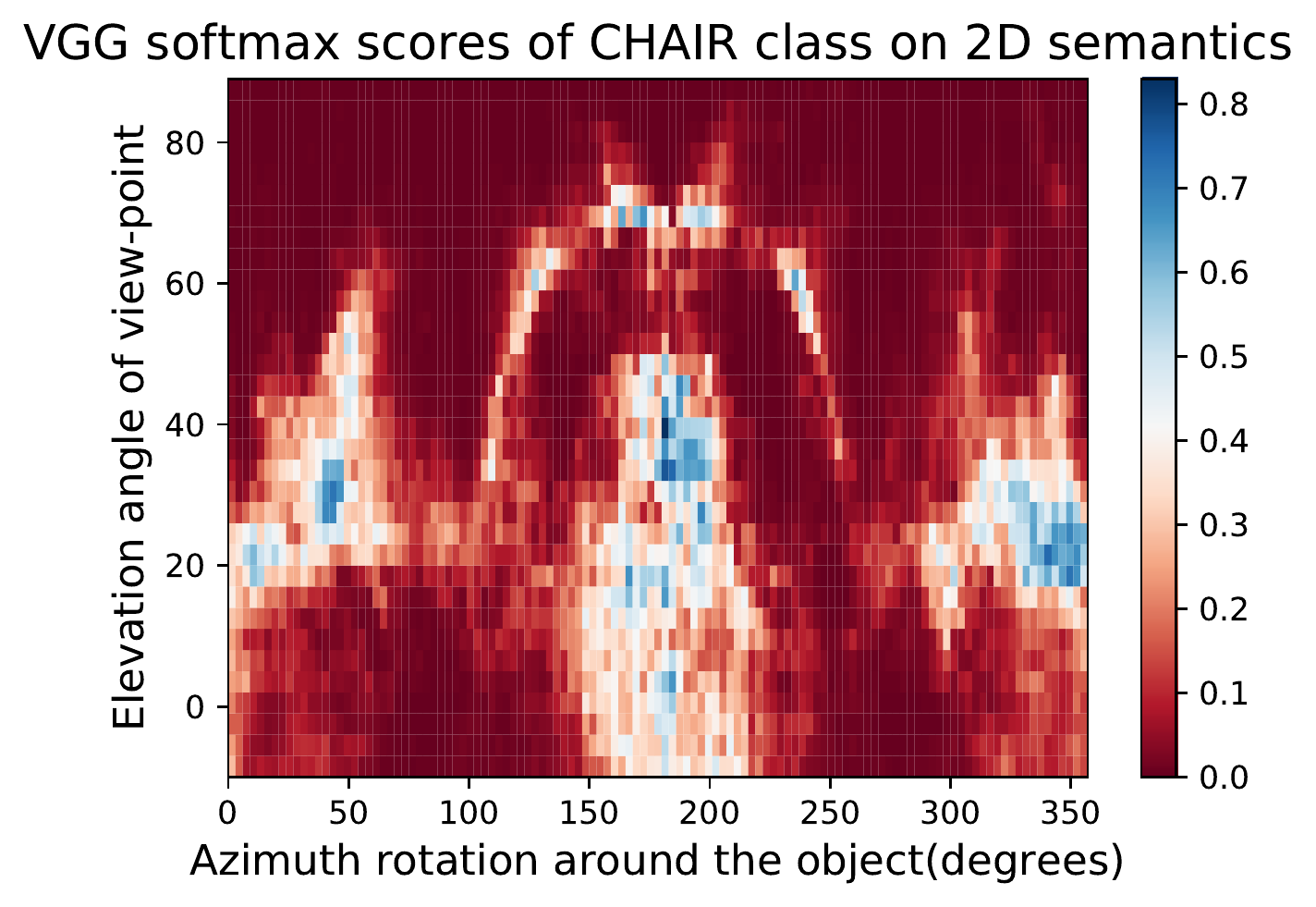}&
  \includegraphics[trim={3.2cm 1.4cm 2.8cm 1.4cm},clip, width=.257\linewidth]{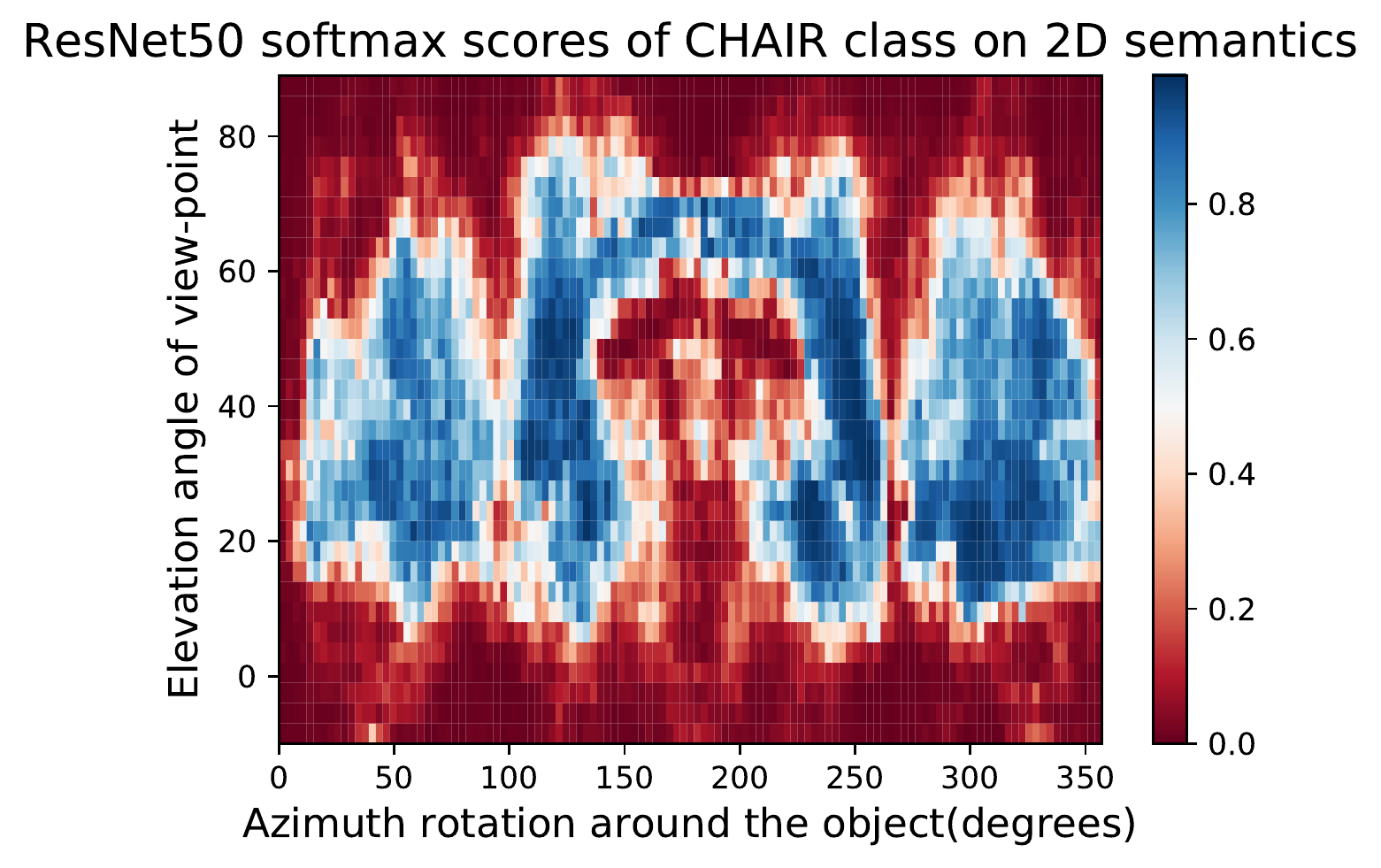}&
  \includegraphics[trim={3.5cm 1.4cm 3cm 1.4cm},clip, width=.259\linewidth]{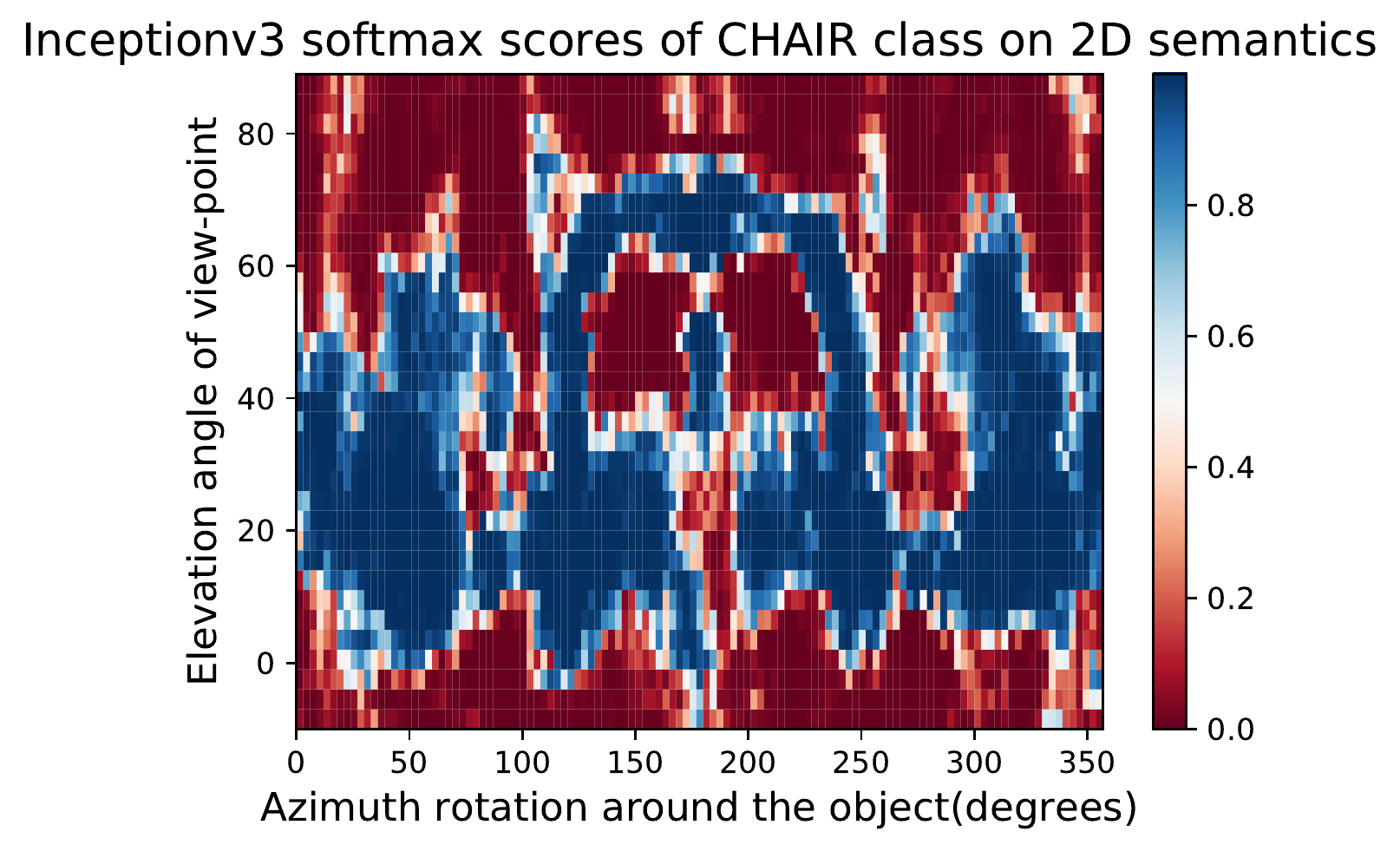} \\ \hline
  \includegraphics[trim={3cm 1.4cm 2.8cm 1.2cm},clip, width=.24\linewidth]{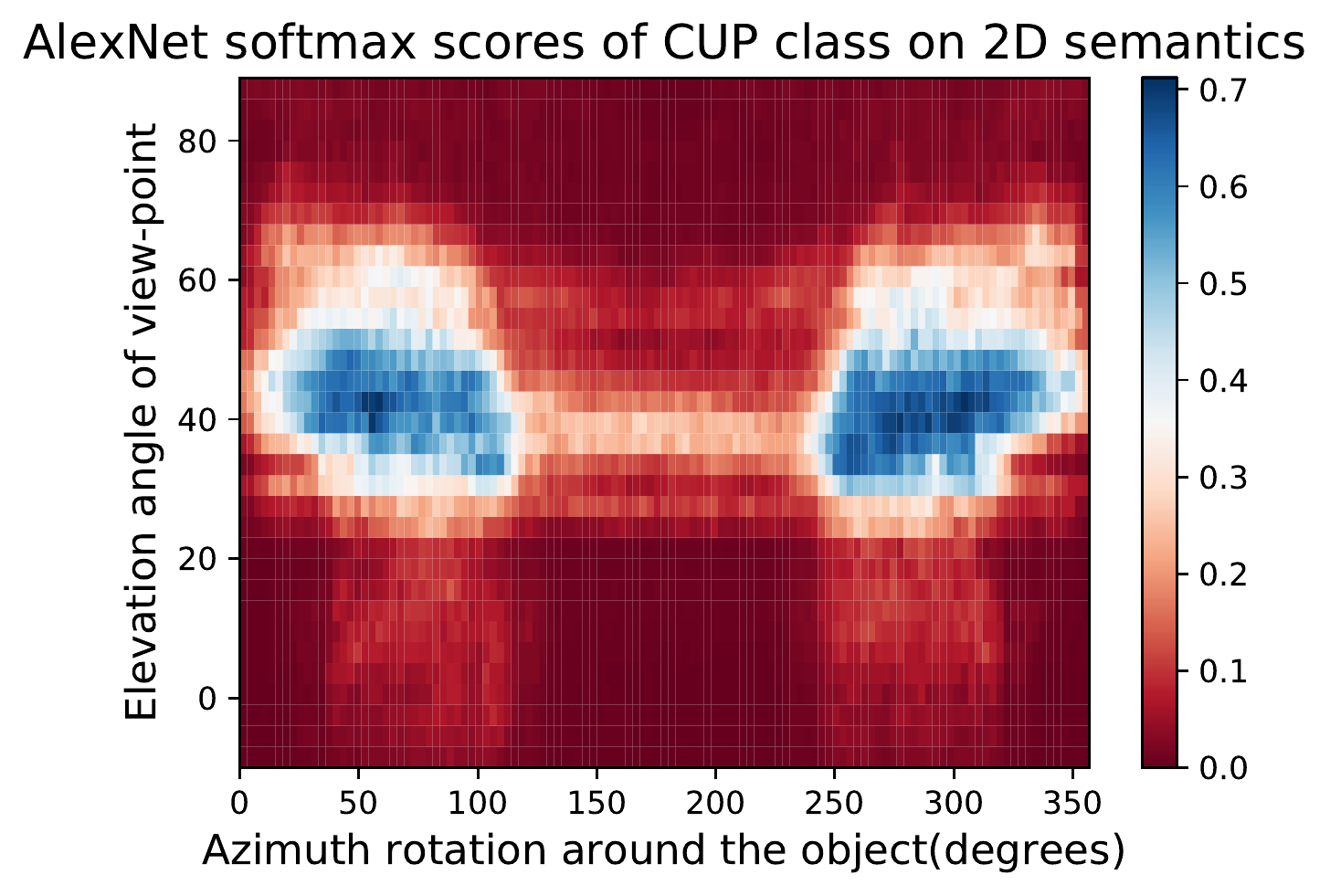}&
  \includegraphics[trim={3cm 1.85cm 2.8cm 1.2cm},clip, width=.23\linewidth]{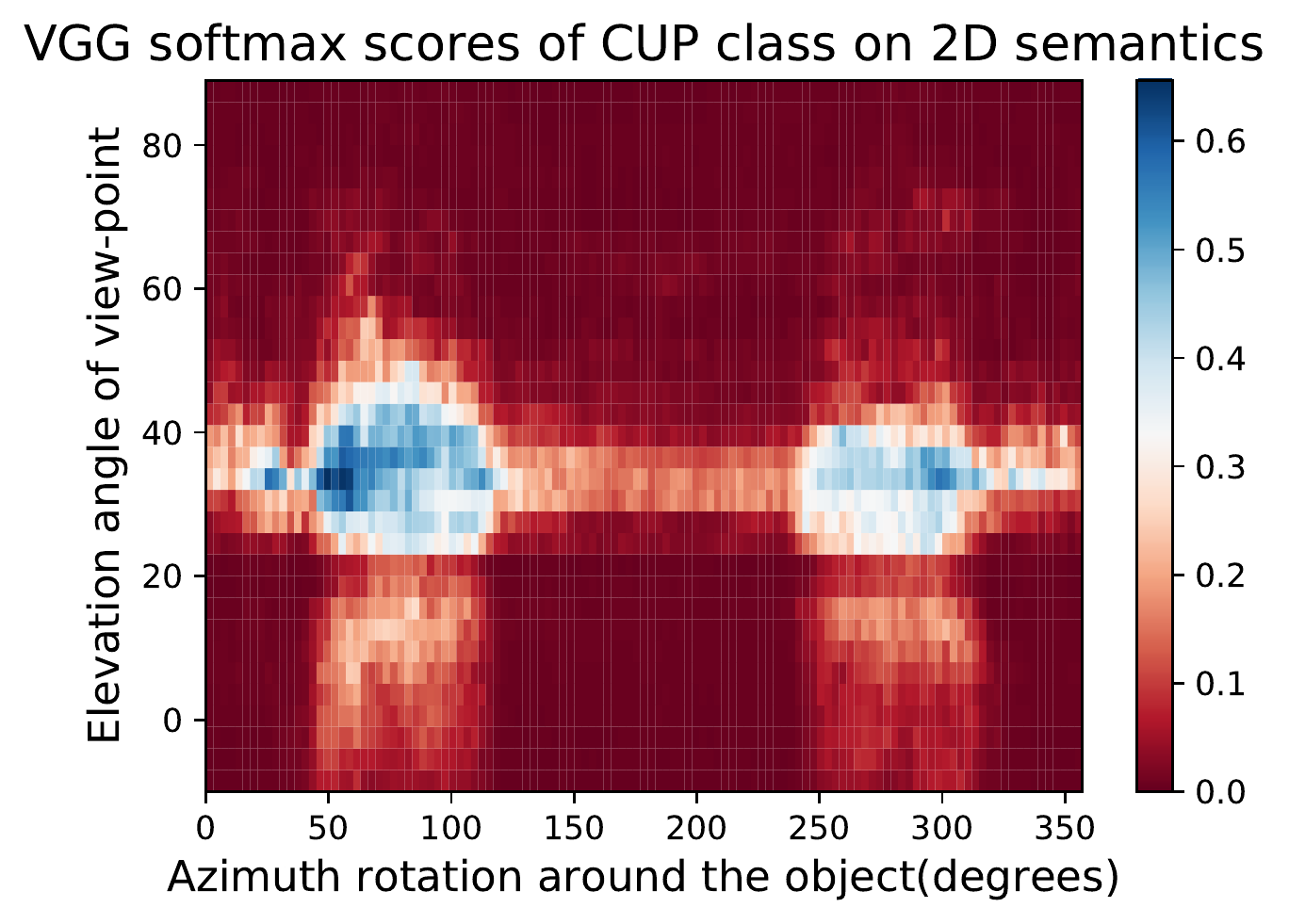}&
  \includegraphics[trim={3cm 1.4cm 2.8cm 1.2cm},clip, width=.25\linewidth]{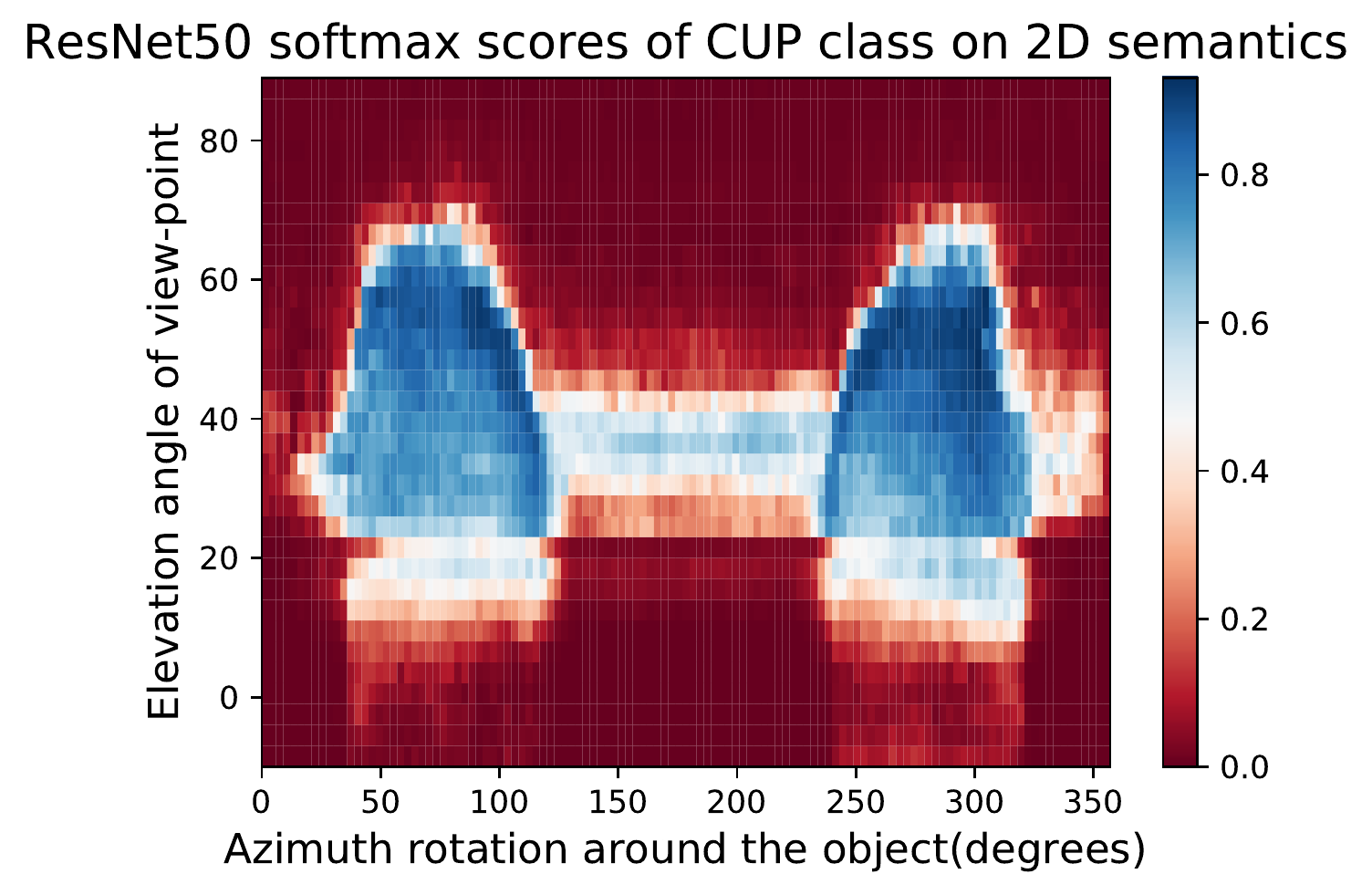} &
  \includegraphics[trim={3.2cm 1.4cm 2.8cm 1.2cm},clip, width=.259\linewidth]{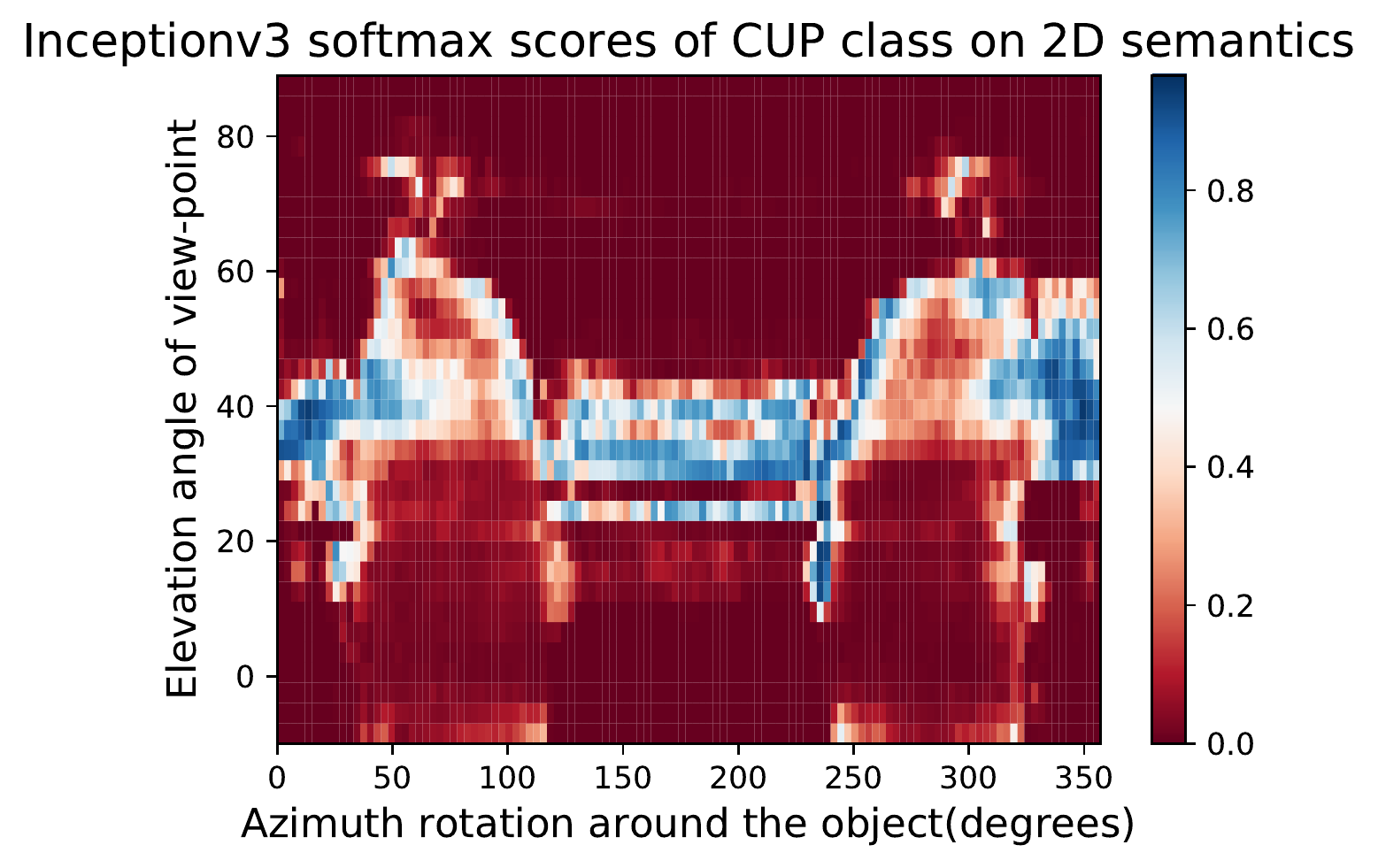} \\
   \end{tabular}
\caption{ \small \textbf{Network Semantic Maps}: We plot the 2D semantic maps (as in \figLabel{\ref{fig:operator} \textit{right}}) of four different networks on two shapes of a \emph{chair} class (\textit{top}) and \emph{cup} class (\textit{bottom}). %
InceptionV3 is very confident about its decision, but at the cost of creating semantic ``traps", where sharp performance degradation happens in the middle of a robust region. This behaviour is more apparent for complex shapes (\eg the chair in \textit{top} row).}
\label{fig:NMS}
\end{figure}
\section{Methodology} \label{sec:methodology}
Typical adversarial pixel attacks involve a neural network $\mathbf{C}$ (\eg classifier or detector) that takes an image $\mathbf{x} \in [0,1]^{d}$ as input and outputs a multinoulli distribution over $K$ class labels with softmax values $[l_{1}, l_{2}, ... ,l_{K}]$, where $l_{j}$ is the softmax value for class $j$. The adversary (attacker) tries to produce a perturbed image $\mathbf{x'} \in [0,1]^{d}$ that is as close as possible to $\mathbf{x}$, such that $\mathbf{x}$ to $\mathbf{x'}$ have different class predictions through $\mathbf{C}$.  %
In this work, we consider a more general case where we are interested in parameters $\mathbf{u} \in \Omega \subset \mathbb{R}^{n}$, a latent parameter that generates the image via a scene generator (\eg a renderer function $\mathbf{R}$). This generator/renderer takes the parameter $\mathbf{u}$ and an object shape $\mathbf{S}$ of a class that is identified by $\mathbf{C}$. $\Omega$ is the continuous semantic space for the parameters that we intend to study. The renderer creates the image $\mathbf{x} \in \mathbb{R}^{d}$, and then we study the behavior of a classifier $\mathbf{C}$ of that image across multiple shapes and multiple popular DNN architectures. Now, this function of interest is defined as follows:
    \begin{equation}
\begin{aligned} 
 f(\mathbf{u}) = \mathbf{C}_{z}(\mathbf{R}(\mathbf{S}_{z},\mathbf{u})) ~, ~~ 0\leq  f(\mathbf{u}) \leq 1
\label{eq:f}
\end{aligned}
\end{equation}
where $z$ is a class label of interest to study and we observe the network score for that class by rendering a shape $\mathbf{S}_{z}$ of the same class. The shape and class labels are constants and only the parameters $\mathbf{u}$ vary for $f$ during analysis. 

\subsection{Region Finding as an Operator} \label{sec:operator}
We can visualize the function in \eqLabel{\ref{eq:f}} for any shape $\mathbf{S}_{z}$ as long as the DNN can identify the shape at some region in the semantic space $\Omega$ of interest, as we show in \figLabel{\ref{fig:intro_fig}}. However, plotting these figures is expensive and the complexity of plotting them increases exponentially with a big base. The complexity of plotting these plots of semantic maps, which we call Network Semantic Maps (NSM), is $N$ for $n=1$, where $N$ is the number of samples needed for that dimension to be fully characterized. The complexity is $N^{2}$ for $n=2$, and we can see that for a general dimension $n$, the complexity of plotting the NMS to adequately fill the semantic space $\Omega$ is $N^{n}$. This number is intractable even if we have only moderate dimensionality. To tackle this issue, we use a bottom-up approach to detect regions around some initial parameters $\mathbf{u}_{0}$, instead of sampling in the entire space of parameters $\Omega$.
Explicitly, we define region finding as an operator $\mathbf{\Phi}$ that takes the function of interest in \eqLabel{\ref{eq:f}}, initial point in the semantic space $\mathbf{u}_{0} \in \Omega $, and a shape $\mathbf{S}_{z}$ of some class $z$. The operator will return the hyper-rectangle $\mathbb{D} \subset \Omega $, where the DNN is robust in the region and does not sharply drop the score of the intended class. It also keeps identifying the shape with label $z$ as illustrated in \figLabel{\ref{fig:NMS}}. The robust-region-finding operator is then defined as follows:
\begin{equation}
\begin{aligned} 
& \mathbf{\Phi}_{\text{robust}}(f(\mathbf{u}),\mathbf{S}_{z},\mathbf{u}_{0}) = \mathbb{D} = \{\mathbf{u}: \mathbf{a} \leq \mathbf{u} \leq \mathbf{b}\} \\
  \text{s.t.}&~~ \mathbb{E}_{\mathbf{u}\sim \mathbb{D}} [f(\mathbf{u})] \ge 1-\epsilon_{m}~, ~~ \mathbf{u}_{0} \in \mathbb{D} ~, ~ \text{VAR}[f(\mathbf{u})] \le \epsilon_{v}
\label{eq:phi-rob}
\end{aligned}
\end{equation}
where the left and right bounds of $\mathbb{D}$ are $\mathbf{a} = [a_{1},a_{2},...,a_{n}]$ and $\mathbf{b} = [b_{1},b_{2},...,b_{n}]$, respectively. The two small thresholds $(\epsilon_{m},\epsilon_{v})$ are needed to ensure high performance and low variance of the DNN  in that robust region. We can define the complementary operator, which finds adversarial regions as:
\begin{equation}
\begin{aligned} 
& \mathbf{\Phi}_{\text{adv}}(f(\mathbf{u}),\mathbf{S}_{z},\mathbf{u}_{0}) = \mathbb{D} = \{\mathbf{u}: \mathbf{a} \leq \mathbf{u} \leq \mathbf{b}\} \\
 & \text{s.t.}~~ \mathbb{E}_{\mathbf{u}\sim \mathbb{D}} [f(\mathbf{u})] \leq \epsilon_{m}~, ~~ \mathbf{u}_{0} \in \mathbb{D} ~, ~ \text{VAR}[f(\mathbf{u})] \ge \epsilon_{v}
\label{eq:phi-adv}
\end{aligned}
\end{equation}
We can clearly show  that $\mathbf{\Phi}_{\text{adv}}$ and $\mathbf{\Phi}_{\text{robust}}$ are related:
\begin{equation}
\begin{aligned} 
& \mathbf{\Phi}_{\text{adv}}(f(\mathbf{u}),\mathbf{S}_{z},\mathbf{u}_{0}) = \mathbf{\Phi}_{\text{robust}}(1-f(\mathbf{u}),\mathbf{S}_{z},\mathbf{u}_{0})
\label{eq:phi-adv-robust}
\end{aligned}
\end{equation}
So, we can just focus our attention on $\mathbf{\Phi}_{\text{robust}}$ to find robust regions, and the adversarial regions follow directly from \eqLabel{\ref{eq:phi-adv-robust}}. We need to ensure that $\mathbb{D}$ has a positive size: $\mathbf{r} =  \mathbf{b} -  \mathbf{a} > \mathbf{0} $. The volume of $\mathbb{D}$ normalized by the exponent of dimension $n$ is expressed as follows: %
\begin{equation}
\begin{aligned} 
\text{volume}(\mathbb{D}) = \triangle = \frac{1}{2^{n}}\prod_{i=1}^{n}\mathbf{r}_{i} 
\label{eq:n-vol}
\end{aligned}
\end{equation}
The region $\mathbb{D}$ can also be defined in terms of the matrix $\mathbf{D}$ of all the corner points $\{\mathbf{d}^{i}\}_{i=1}^{2^{n}}$ as follows:
\begin{equation}
\begin{aligned} 
\text{corners}&(\mathbb{D}) = \mathbf{D}_{n\times 2^{n}} = \left[\mathbf{d}^{1} | \mathbf{d}^{2} |.. | \mathbf{d}^{2^{n}}\right] = \mathbf{1}^{T}\mathbf{a}~ +~ \mathbf{M}^{T} \odot (\mathbf{1}^{T}\mathbf{r}) \\
\mathbf{M}_{n\times 2^{n}} = &\left[ \mathbf{m}^{0}| \mathbf{m}^{1} |.. | \mathbf{m}^{2^{n}-1} \right], ~ \text{where}~~ \mathbf{m}^{i} = \text{binary}_{n}(i)
\label{eq:n-corners}
\end{aligned}
\end{equation}
and $\mathbf{1}$ is the all-ones vector of size $2^{n}$, $\odot$ is the Hadamard (element-wise) product of matrices, and $\mathbf{M}$ is a constant  masking matrix defined as the permutation matrix of binary numbers of $n$ bits that range from 0 to $2{n} - 1 $
\subsection{Deriving Update Directions}
\mysection{Extending Naive to $n$-dimensions}
We start by defining the function vector $\mathbf{f}_{\mathbb{D}}$ of all function evaluations at all corner points of $\mathbb{D}$.
\begin{equation}
\begin{aligned} 
\mathbf{f}_{\mathbb{D}} &= \left[f(\mathbf{d}^{1}), f(\mathbf{d}^{2}),...,f(\mathbf{d}^{2^{n}}) \right]^{T} , ~ \mathbf{d^{i}} = \mathbf{D}_{:,i}
\label{eq:n-function}
\end{aligned}
\end{equation}
Then, using Trapezoid approximation and Leibniz rule of calculus, the loss expression and the update directions become as follows:
\begin{equation}
\begin{aligned} 
L(\mathbf{a},\mathbf{b}) &= - \idotsint_\mathbb{D} f(u_1,\dots,u_n) \,du_1 \dots du_n  + \frac{\lambda}{2} \left| \mathbf{r}\right|^{2}
~\approx~ -\triangle\mathbf{1}^{T}\mathbf{f}_{\mathbb{D}} ~+~ \frac{\lambda}{2} \left| \mathbf{r}\right|^{2}\\ 
\nabla_{\mathbf{a}}L  &\approx~ 2\triangle\text{diag}^{-1}(\mathbf{r}) \overline{\mathbf{M}}\mathbf{f}_{\mathbb{D}} + \lambda \mathbf{r}
~~~;~~~ \nabla_{\mathbf{b}}L  \approx~ -2\triangle\text{diag}^{-1}(\mathbf{r}) \mathbf{M}\mathbf{f}_{\mathbb{D}} - \lambda \mathbf{r}
\label{eq:n-loss-update-naive}
\end{aligned}
\end{equation}
We show all the derivations for $n\in\{1,2\}$ and for general $n$ in the \supp.

\mysection{Outer-Inner Ratio Loss (OIR)}
We introduce an outer region $(A,B)$ with that contains the small region $(a,b)$. We follow the following assumption to ensure that the outer area is always positive: $A =  a - \alpha \frac{b-a}{2} ;  B =  b + \alpha \frac{b-a}{2}$. Here, $\alpha$ is the small boundary factor of the outer area to the inner.
We formulate the problem as a ratio of outer over inner areas and we try to make this ratio ($L =  \frac{\text{Area}_{\text{out}}}{\text{Area}_{\text{in}}}  $) as close as possible to 0 . 
We utilize the Dinkelbach technique for solving non-linear fractional programming problems \cite{dinckl} to transform $L$ as follows.
\begin{equation}
\begin{aligned} 
L &= \frac{\text{Area}_{\text{out}}}{\text{Area}_{\text{in}}} ~ =~ \text{Area}_{\text{out}} ~-~ \lambda ~ \text{Area}_{\text{in}} \\
&= \int_{A}^{B}f(a)du ~ -~ \int_{a}^{b}f(a)du ~ -~ \lambda ~\int_{a}^{b}f(a)du
\label{eq:loss-oir}
\end{aligned}
\end{equation}
where $\lambda^{*} = \frac{\text{Area}_{\text{out}}^{*}}{\text{Area}_{\text{in}}^{*}}$ is the Dinkelbach factor that is the best objective ratio.

\mysection{Black-Box (OIR\_B)}
Here we set $\lambda = 1$ to simplify the problem. This yields the following expression of the loss $L =  \text{Area}_{\text{out}} - \text{Area}_{\text{in}} =  \int_{A}^{B}f(u)du  - 2\int_{a}^{b}f(u)du$, which is similar to the area contrastive loss in \cite{ioc}. The update rules would be $\pd{L}{a} = -(1+ \frac{\alpha}{2})f(A) - \frac{\alpha}{2}f(B) + 2f(a) ~;~ \pd{L}{b} = (1+ \frac{\alpha}{2})f(B) + \frac{\alpha}{2}f(A) - 2f(b)$. To extend to $n$-dimensions, we define an outer region $\mathbb{Q}$ that includes the smaller region $\mathbb{D} $ and defined as: $\mathbb{Q} = \{\mathbf{u}: \mathbf{a} - \frac{\alpha}{2}\mathbf{r} \leq \mathbf{u} \leq \mathbf{b} + \frac{\alpha}{2}\mathbf{r} \}$, where $(\mathbf{a},\mathbf{b},\mathbf{r})$ are defined as before, while $\alpha$ is defined as the boundary factor of the outer region for all the dimensions. The inner region $\mathbb{D}$ is defined as in \eqLabel{\ref{eq:n-corners}}, while the outer region can be defined in terms of the corner points as follows:
\begin{equation}
\begin{aligned} 
\text{corners}(\mathbb{Q}) &= \mathbf{Q}_{n\times 2^{n}} = \left[\mathbf{q}^{1} | \mathbf{q}^{2} |.. | \mathbf{q}^{2^{n}}\right] \\
\mathbf{Q} = \mathbf{1}^{T}(\mathbf{a}& - \frac{\alpha}{2}\mathbf{r})~ +~ (1+\alpha)\mathbf{M}^{T} \odot (\mathbf{1}^{T}\mathbf{r}) 
\label{eq:n-corners2}
\end{aligned}
\end{equation}
Let $\mathbf{f}_{\mathbb{D}}$ be a function vector as in \eqLabel{\ref{eq:n-function}} and $\mathbf{f}_{\mathbb{Q}}$ be another function vector evaluated at all possible outer corner points: $\mathbf{f}_{\mathbb{Q}} = \left[f(\mathbf{q}^{1}), f(\mathbf{q}^{2}),...,f(\mathbf{q}^{2^{n}}) \right]^{T} , ~ \mathbf{q^{i}} = \mathbf{Q}_{:,i}$. 

Now, the loss and update directions for the $n$-dimensional case becomes:
\begin{equation}
\begin{aligned} 
L(\mathbf{a},\mathbf{b}) &= \idotsint_\mathbb{Q} f(u_1,\dots,u_n) \,du_1 \dots du_n - 2 \idotsint_\mathbb{D} f(u_1,\dots,u_n) \,du_1 \dots du_n \\ 
&\approx \triangle\left((1+\alpha)^{n} \mathbf{1}^{T}\mathbf{f}_{\mathbb{Q}} ~-~ 2 ~ \mathbf{1}^{T}\mathbf{f}_{\mathbb{D}}\right)\\ 
\nabla_{\mathbf{a}}L  ~\approx~ &2\triangle\text{diag}^{-1}(\mathbf{r}) \left(2\overline{\mathbf{M}}\mathbf{f}_{\mathbb{D}} ~-~ \overline{\mathbf{M}}_{\mathbb{Q}}\mathbf{f}_{\mathbb{Q}}  \right); ~~
\nabla_{\mathbf{b}}L  ~\approx~ 2\triangle\text{diag}^{-1}(\mathbf{r}) \left(-2\mathbf{M}\mathbf{f}_{\mathbb{D}} ~+~ \mathbf{M}_{\mathbb{Q}}\mathbf{f}_{\mathbb{Q}}  \right),
\label{eq:n-loss-update-outer}
\end{aligned}
\end{equation}
where diag(.) is the diagonal matrix of the vector argument or the diagonal vector of the matrix argument. $\overline{\mathbf{M}}_{\mathbb{Q}}$ is the outer region scaled mask defined as follows:
\begin{equation}
\begin{aligned} 
\overline{\mathbf{M}}_{\mathbb{Q}} = ~  (1+\alpha)^{n-1}\left((1+\frac{\alpha}{2})\overline{\mathbf{M}}+\frac{\alpha}{2}\mathbf{M}\right) ~;~
\mathbf{M}_{\mathbb{Q}} = ~  (1+\alpha)^{n-1}\left((1+\frac{\alpha}{2})\mathbf{M}+\frac{\alpha}{2}\overline{\mathbf{M}}\right)
\label{eq:n-mask-outer}
\end{aligned}
\end{equation}
\mysection{White-Box OIR (OIR\_W)}
 Here, we present the white-box formulation of Outer-Inner-Ratio. This requires access to the gradient of the function $f $ in order to update the current estimates of the bound. As we show in \secLabel{\ref{sec:experiments}}, access to gradients enhance the quality of the detected regions.
 To derive the formulation, We set $\lambda = \frac{\alpha}{\beta}$ in \eqLabel{\ref{eq:loss-oir}}, where $\alpha$ is the small boundary factor of the outer area and $\beta$ is the gradient emphasis factor. Hence, the objective in \eqLabel{\ref{eq:loss-oir}} becomes:
\begin{equation}
\begin{aligned} 
\argmin_{a,b} L = &\argmin_{a,b}~ \text{Area}_{\text{out}} - \lambda ~ \text{Area}_{\text{in}} \\
=& \argmin_{a,b}~ \int_{A}^{a}f(u)du + \int_{b}^{B}f(u)du - \frac{\alpha}{\beta} \int_{a}^{b}f(u)du \\
=& \argmin_{a,b}~ \frac{\beta}{\alpha} \int_{a - \alpha \frac{b-a}{2}}^{b + \alpha \frac{b-a}{2} }f(u)du  - (1+\frac{\beta}{\alpha})\int_{a}^{b}f(u)du \\
\pd{L}{a} =& \frac{\beta}{\alpha}\left(f(a) - f\left(a - \alpha \frac{b-a}{2}\right) \right) \\&~- \frac{\beta}{2}f\left(b + \alpha \frac{b-a}{2}\right) ~ - \frac{\beta}{2}f\left(a - \alpha \frac{b-a}{2}\right) + f(a) 
\label{eq:loss-oir2}
\end{aligned}
\end{equation}
Now, since $\lambda^{*}$ should be small for the optimal objective as $\lambda \rightarrow 0 ,~ \alpha \rightarrow 0$ and hence the derivative in \eqLabel{\ref{eq:loss-oir2}} becomes the following:
\begin{equation}
\begin{aligned} 
\lim_{\alpha \to 0}~\pd{L}{a} &= \frac{\beta}{2}\left((b-a)f^{\prime}(a) + f(b)\right) ~+~ (1-\frac{\beta}{2})f(a) \\
\lim_{\alpha \to 0}~\pd{L}{b} &= \frac{\beta}{2}\left((b-a)f^{\prime}(b) + f(a)\right) ~-~ (1-\frac{\beta}{2})f(b) 
\label{eq:update-oir-3}
\end{aligned}
\end{equation}
We can see that the update rule for $a$ and $b$ depends on the function value \textbf{and} the derivative of $f$ at the boundaries $a$ and $b$ respectively, with $\beta$ controlling the dependence.
If $\beta \rightarrow 0 $, the update directions in \eqLabel{\ref{eq:update-oir-3}} collapse to the unregularized naive update. To extend to $n$-dimensions, we have to define a term that involves the gradient of the function, \ie the all-corners gradient matrix  $\mathbf{G}_{\mathbb{D}}$.
\begin{equation}
\begin{aligned} 
\mathbf{G}_{\mathbb{D}} &= \left[\nabla f(\mathbf{d}^{1})~|~\nabla f(\mathbf{d}^{2})~|~...~|~\nabla f(\mathbf{d}^{2^{n}}) \right]^{T} 
\label{eq:n-gradient}
\end{aligned}
\end{equation}
\begin{algorithm}[t] 
\caption{Robust $n$-dimensional Region Finding for Black-Box DNNs by Outer-Inner Ratios}\label{alg: black}
\small
\SetAlgoLined
  \textbf{Requires: } Semantic Function of a DNN $f(\mathbf{u})$ in \eqLabel{\ref{eq:f}}, initial semantic parameter $\mathbf{u}_{0}$, number of iterations T , learning rate $\eta$ , object shape $\mathbf{S}_{z}$ of class label $z$, boundary factor $\alpha$, Small $\epsilon$ \\
   Form constant binary matrices $\mathbf{M}, \overline{\mathbf{M}},\mathbf{M}_{\mathbb{Q}},\overline{\mathbf{M}_{\mathbb{Q}}}, \mathbf{M}_{\mathbb{D}},\overline{\mathbf{M}_{\mathbb{D}}} $ \\
   Initialize bounds $\mathbf{a}_{0}\leftarrow \mathbf{u}_{0} - \epsilon \mathbf{1} $, $\mathbf{b}_{0} \leftarrow \mathbf{u}_{0}+- \epsilon \mathbf{1}$ \\
    $\mathbf{r}_{0} \leftarrow \mathbf{a}_{0}-\mathbf{b}_{0} $ , update region volume $ \triangle_{0} $ as in \eqLabel{\ref{eq:n-vol}}\\
  \For{$t \leftarrow 1$ \KwTo $T$}{
   form the all-corners function vectors ${f}_{\mathbb{D}},{f}_{\mathbb{Q}}$ as in \eqLabel{\ref{eq:n-function},\ref{eq:n-corners2}}\\
    $\nabla_{\mathbf{a}}L  \leftarrow 2\triangle_{t-1}\text{diag}^{-1}(\mathbf{r}_{t-1}) \left(2\overline{\mathbf{M}}\mathbf{f}_{\mathbb{D}} ~-~ \overline{\mathbf{M}}_{\mathbb{Q}}\mathbf{f}_{\mathbb{Q}}  \right)$ \\
$\nabla_{\mathbf{b}}L  \leftarrow 2\triangle_{t-1}\text{diag}^{-1}(\mathbf{r}_{t-1}) \left(-2\mathbf{M}\mathbf{f}_{\mathbb{D}} ~+~ \mathbf{M}_{\mathbb{Q}}\mathbf{f}_{\mathbb{Q}}  \right)$\\
    update bounds: $\mathbf{a}_{t}\leftarrow \mathbf{a}_{t-1} - \eta \nabla_{\mathbf{a}}L$, $\mathbf{b}_{t}\leftarrow \mathbf{b}_{t-1} - \eta \nabla_{\mathbf{b}}L$ \\
     $\mathbf{r}_{t} \leftarrow \mathbf{a}_{t}-\mathbf{b}_{t} $ , update region volume $ \triangle_{t} $ as in \eqLabel{\ref{eq:n-vol}}
    }
    \textbf{Returns: }robust region bounds: $ \mathbf{a}_{T},\mathbf{b}_{T}$ .
    
\end{algorithm}
\begin{algorithm}[h] 
\caption{Robust $n$-dimensional Region Finding for White-Box DNNs by Outer-Inner Ratios}\label{alg: white}
\small
\SetAlgoLined
  \textbf{Requires: }  Semantic Function of a DNN $f(\mathbf{u})$ in \eqLabel{\ref{eq:f}}, initial semantic parameter $\mathbf{u}_{0}$, , learning rate $\eta$ , object shape $\mathbf{S}_{z}$ of class label $z$, emphasis factor $\beta$, Small $\epsilon$ \\
   Form constant binary matrices $\mathbf{M}, \overline{\mathbf{M}}, \mathbf{M}_{\mathbb{D}},\overline{\mathbf{M}_{\mathbb{D}}} $ \\
   Initialize bounds $\mathbf{a}_{0}\leftarrow \mathbf{u}_{0} - \epsilon \mathbf{1} $, $\mathbf{b}_{0} \leftarrow \mathbf{u}_{0}+- \epsilon \mathbf{1}$ \\
    $\mathbf{r}_{0} \leftarrow \mathbf{a}_{0}-\mathbf{b}_{0} $ , update region volume $ \triangle_{0} $ as in as in \eqLabel{\ref{eq:n-vol}}\\
  \For{$t \leftarrow 1$ \KwTo $T$}{
   form the all-corners function vector ${f}_{\mathbb{D}}$ as in \eqLabel{\ref{eq:n-function}}\\
   form the all-corners gradients matrix $\mathbf{G}_{\mathbb{D}}$ as in \eqLabel{\ref{eq:n-gradient}}\\ 
   form the gradient selection vectors $\mathbf{s} ,\overline{\mathbf{s}}$ as in \eqLabel{\ref{eq:n-update-grad-selection}}
   $\nabla_{\mathbf{a}}L \leftarrow  \triangle_{t-1} \left(\text{diag}^{-1}(\mathbf{r}_{t-1})\overline{\mathbf{M}}_{\mathbb{D}}\mathbf{f}_{\mathbb{D}} + \beta\text{diag}(\overline{\mathbf{M}}\mathbf{G}_{\mathbb{D}}+ \beta \overline{\mathbf{s}}  \right)$  \\
$\nabla_{\mathbf{b}}L  \leftarrow  \triangle_{t-1} \left(- \text{diag}^{-1}(\mathbf{r}_{t-1})\mathbf{M}_{\mathbb{D}}\mathbf{f}_{\mathbb{D}} + \beta\text{diag}(\mathbf{M}\mathbf{G}_{\mathbb{D}})+ \beta \mathbf{s}  \right)$\\
    update bounds: $\mathbf{a}_{t}\leftarrow \mathbf{a}_{t-1} - \eta \nabla_{\mathbf{a}}L$, $\mathbf{b}_{t}\leftarrow \mathbf{b}_{t-1} - \eta \nabla_{\mathbf{b}}L$ \\
     $\mathbf{r}_{t} \leftarrow \mathbf{a}_{t}-\mathbf{b}_{t} $ , update region volume $ \triangle_{t} $ as in \eqLabel{\ref{eq:n-vol}}
    }
    \textbf{Returns: }robust region bounds: $ \mathbf{a}_{T},\mathbf{b}_{T}$ .
\end{algorithm}
Now, the loss and update directions are given as follows.
\begin{equation}
\begin{aligned} 
L(\mathbf{a},\mathbf{b})  \approx&~~ \frac{(1+\alpha)^{n} \mathbf{1}^{T}\mathbf{f}_{\mathbb{Q}}}{\mathbf{1}^{T}\mathbf{f}_{\mathbb{D}}} ~-~ 1\\ 
\nabla_{\mathbf{a}}L  \approx & ~\triangle  \left(\text{diag}^{-1}(\mathbf{r})\overline{\mathbf{M}}_{\mathbb{D}}\mathbf{f}_{\mathbb{D}} ~+~ \beta\text{diag}(\overline{\mathbf{M}}\mathbf{G}_{\mathbb{D}})~+ \beta \overline{\mathbf{s}}  \right)  \\
\nabla_{\mathbf{b}}L  \approx & ~\triangle  \left(- \text{diag}^{-1}(\mathbf{r})\mathbf{M}_{\mathbb{D}}\mathbf{f}_{\mathbb{D}} ~+~ \beta\text{diag}(\mathbf{M}\mathbf{G}_{\mathbb{D}})~+ \beta \mathbf{s}  \right)
\label{eq:n-loss-update-grad}
\end{aligned}
\end{equation}
where the mask is the special mask 
\begin{equation}
\begin{aligned} 
 \overline{\mathbf{M}}_{\mathbb{D}} =&~ \left( \gamma_n \overline{\mathbf{M}} ~-~\beta \mathbf{M}  \right) ~;~
  \mathbf{M}_{\mathbb{D}} =~ \left( \gamma_n \mathbf{M} ~-~\beta \overline{\mathbf{M}}  \right) ~;~
  \gamma_n =~ 2~-~\beta(2n-1) 
\label{eq:n-mask-grad}
\end{aligned}
\end{equation}
$\mathbf{s}$ is a weighted sum of the gradient from other dimensions ($i \neq k$) contributing to the update direction of dimension $k$, where $k \in \{1 , 2 ,...,n\}$.
\begin{equation}
\begin{aligned} 
\mathbf{s}_{k} &= \frac{1}{\mathbf{r}_{k}}\sum_{i=1, i\neq k }^{n}\mathbf{r}_{i}( (\overline{\mathbf{M}}_{i,:} - \mathbf{M}_{i,:})\odot \overline{\mathbf{M}}_{k,:} ) \mathbf{G}_{:,i} \\
\overline{\mathbf{s}}_{k} &= \frac{1}{\mathbf{r}_{k}}\sum_{i=1, i\neq k }^{n}\mathbf{r}_{i}( ( \mathbf{M}_{i,:} - \overline{\mathbf{M}}_{i,:} )\odot \mathbf{M}_{k,:} ) \mathbf{G}_{:,i}
\label{eq:n-update-grad-selection}
\end{aligned}
\end{equation}
Algorithms \ref{alg: black}, and \ref{alg: white} summarize the techniques explained above, which we implement in \secLabel{\ref{sec:experiments}}. The derivation of the 2-dimensional case and $n$-dimensional case of the OIR formulation, as well as other unsuccessful formulations are all included in the \supp\hspace{-4pt}.

\section{Experiments} \label{sec:experiments}
\subsection{Setup and Data} \label{sec:setup}
In this paper, we chose the semantic parameters $\mathbf{u}$ to be the azimuth rotations of the viewpoint and the elevation angle from the horizontal plane, where the object is always at the center of the rendering. This is common practise in the literature \cite{sada,semantic-attack}. We use 100 shapes from 10 different classes from ShapeNet \cite{shapenet}, the largest dataset for 3D models that are normalized from the semantic lens. We pick these 100 shapes specifically such that: (1) the class label is available in ImageNet \cite{IMAGENET} and that ImageNet classifiers can identify the exact class, and (2) the selected shapes are identified by the classifiers at some part of the semantic space. To do this, we measured the average score in the space and accepted the shape only if its average Resnet softmax score is 0.1. To render the images, we use a differentiable renderer NMR \cite{vig-nmr}, which allows obtaining the gradient to the semantic input parameters. The networks of interest are Resnet50 \cite{resnet}, VGG \cite{vgg}, AlexNet \cite{AlexNet}, and InceptionV3 \cite{inception}. We use the official PyTorch implementation for each network \cite{paszke2017pytorch}.

\subsection{Mapping the Networks} \label{sec:maps}
Similar to \figLabel{\ref{fig:intro_fig}}, we map the networks for all 100 shapes on the first semantic parameter (the azimuth rotation), as well as the joint (azimuth and elevation). We show these results in \figLabel{\ref{fig:NMS}}. The ranges for the two parameters were [\ang{0},\ang{360}],[\ang{-10},\ang{90}], with a 3$\times$3 grid. The total number of network evaluations is 4K forward passes from each network for every shape (total of 1.6M forward passes). We show all of the remaining results in the \supp\hspace{-4pt}.

\subsection{Growing Semantic Robust Regions} \label{sec:regions}
We implement the three bottom-up approaches in Table \ref{tbl:complexity} and Algorithms \ref{alg: black} and \ref{alg: white}.
The hyper-parameters were set to $\eta =0.1, \alpha=0.05,  \beta=0.0009, \lambda =0.1, T =800$. We can observe in \figLabel{\ref{fig:operator}} that multiple initial points inside the same robust region converge to the same boundary. One key difference to be noted between the naive approach in \eqLabel{\ref{eq:n-loss-update-naive}} and the OIR formulations in \eqLabel{\ref{eq:n-loss-update-outer},\ref{eq:n-loss-update-grad}} is that the naive approach fails to capture robust regions in some scenarios and fall for trivial regions (see \figLabel{\ref{fig:operator}}).
\begin{table}[t]
\small
\tabcolsep=0.09cm
\centering
\begin{tabular}{c|c|cc} 
\toprule
Deep Networks & SRVR & Top-1 error & Top-5 Error \\ 
\midrule
AlexNet \cite{AlexNet} & 8.87\% &43.45   &  20.91 \\
VGG-11 \cite{vgg}& 9.72\%  &30.98  &  11.37   \\
ResNet50 \cite{resnet} &  \textbf{16.79}\% & 23.85      & 7.13  \\
Inceptionv3 \cite{inception}&  7.92\% & \textbf{22.55}  & \textbf{6.44}  \\
\bottomrule
\end{tabular}
\caption{\small \textbf{Benchmarking popular DNNs in Semantic Robustness vs error rate}. We develop the Semantic Robustness Volume Ratio (SRVR) metric to quantify and compare the semantic robustness of well-known DNNs in \secLabel{\ref{sec:application}}. We see that semantic robustness does not necessarily depend on the accuracy of the DNN. This motivates studying it as an independent metric from the classification accuracy. Results are reported based on the official PyTorch implementations of these networks \cite{paszke2017pytorch}.}
\label{tbl:benchmarking}
\end{table}

\begin{table}[t]
\footnotesize
\setlength{\tabcolsep}{6pt} %
\renewcommand{\arraystretch}{1} %
\centering
\resizebox{\hsize}{!}{
\begin{tabular}{c||c|c|c|c|c|c} 
\toprule
\specialcell{\textbf{Analysis}\\ \textbf{Approach}} & \textbf{Paradigm}& \specialcell{\textbf{Total} \\ \textbf{Sampling} \\ \textbf{complexity} } & \specialcell{\textbf{Black} \\\textbf{-box}\\ \textbf{Functions} } & \specialcell{\textbf{Forward} \\ \textbf{pass} \\\textbf{/step} } & \specialcell{\textbf{Backward} \\ \textbf{pass} \\\textbf{/step} } & \specialcell{\textbf{Identification} \\ \textbf{Capabaility} } \\
\midrule
\textbf{Grid Sampling} &top-down &\specialcell{ $\mathcal{O}(N^{n})$\\$ N \gg 2$  }  & \textcolor{green}{\checkmark} & - & - & \specialcell{Fully identifies the \\semantic map of DNN}\\ \hline
\textbf{Naive} & bottom-up &$\mathcal{O}(2^{n})$  & \textcolor{green}{\checkmark}& $2^{n}$ &0 & \specialcell{finds strong robust\\ regions only around $\mathbf{u}_{0}$}  \\ \hline
\textbf{OIR\_B}& bottom-up &$\mathcal{O}(2^{n+1})$ & \textcolor{green}{\checkmark} & $2^{n+1}$ &0 & \specialcell{finds strong and\\ week robust regions\\ around $\mathbf{u}_{0}$ }  \\ \hline
\textbf{OIR\_W}& bottom-up &$\mathcal{O}(2^{n})$ & \textcolor{red}{\xmark}  & $2^{n}$& $2^{n}$ & \specialcell{finds strong and\\ week robust regions\\ around $\mathbf{u}_{0}$ }  \\ 
 \bottomrule
\end{tabular}
}
\caption{\small \textbf{Semantic Analysis Techniques}: We compare different approaches to analyse the semantic robustness of DNNs.}
\label{tbl:complexity}
\end{table}
\subsection{Applications} \label{sec:application}
\mysection{Quantifying Semantic Robustness}\\
Looking at these NSM can lead to insights about the network, but we would like to develop a systemic approach to quantify the robustness of these DNNs. To do this, we develop the Semantic Robustness Volume Ratio (SRVR) metric. The SRVR of a network is the ratio between the expected size of the robust region obtained by Algorithms \ref{alg: black} and \ref{alg: white} over the nominal total volume of the semantic map of interest. Explicitly, the SRVR of network $\mathbf{C}$ for class label $z$ is defined as follows:

\begin{equation}
\begin{aligned} 
 \text{SRVR}_{z} =  \frac{\mathbb{E}[\text{Vol}(\mathbb{D})]}{\text{Vol}(\Omega)} = \frac{\mathbb{E}_{\mathbf{u}_{0}\sim \Omega,\mathbf{S}_{z}\sim \mathbb{S}_{z}} [\text{Vol}(\mathbf{\Phi}(f,\mathbf{S}_{z},\mathbf{u}_{0}))]}{\text{Vol}(\Omega)},
\label{eq:SRVR}
\end{aligned}
\end{equation}
where $f,\Phi$ are defined in \eqLabel{\ref{eq:f},\ref{eq:phi-rob}} respectively. We take the average volume of all the adversarial regions found for multiple initializations and multiple shapes of the same class $z$. Then, we divide by the volume of the entire space. This provides a percentage of how close the DNN is from the ideal behaviour of identifying the object robustly in the entire space. The SRVR metric is not strict in its value, since the analyzer defines the semantic space of interest and the shapes used. However, comparing SRVR scores among DNNs is of extreme importance, as this relative analysis conveys insights about a network that might not be evident by only observing the accuracy of this network. For example, we can see in Table \ref{tbl:benchmarking} that while InceptionV3 \cite{inception} is the best in terms of accuracy, it lags behind Resnet50 \cite{resnet} in terms of semantic robustness. This observation is also consistent with the qualitative NSMs in \figLabel{\ref{fig:NMS}}, in which we can see that while Inception is very confident, it can fail completely inside these confident regions. Note that the reported SRVR results are averaged over all 10 classes and over all 100 shapes. We use 4 constant initial points for all experiments and the semantic parameters are the azimuth and elevation as in \figLabel{\ref{fig:operator}} and \ref{fig:NMS}. As can be seen in \figLabel{\ref{fig:operator}}, different methods predict different regions, so we take the average size of the the three methods used (naive, OIR\_W, and OIR\_B) to give an overall estimate of the volume used in the SRVR results reported in Table \ref{tbl:benchmarking}.

\begin{figure}[!t]
  \includegraphics[width=0.7\linewidth]{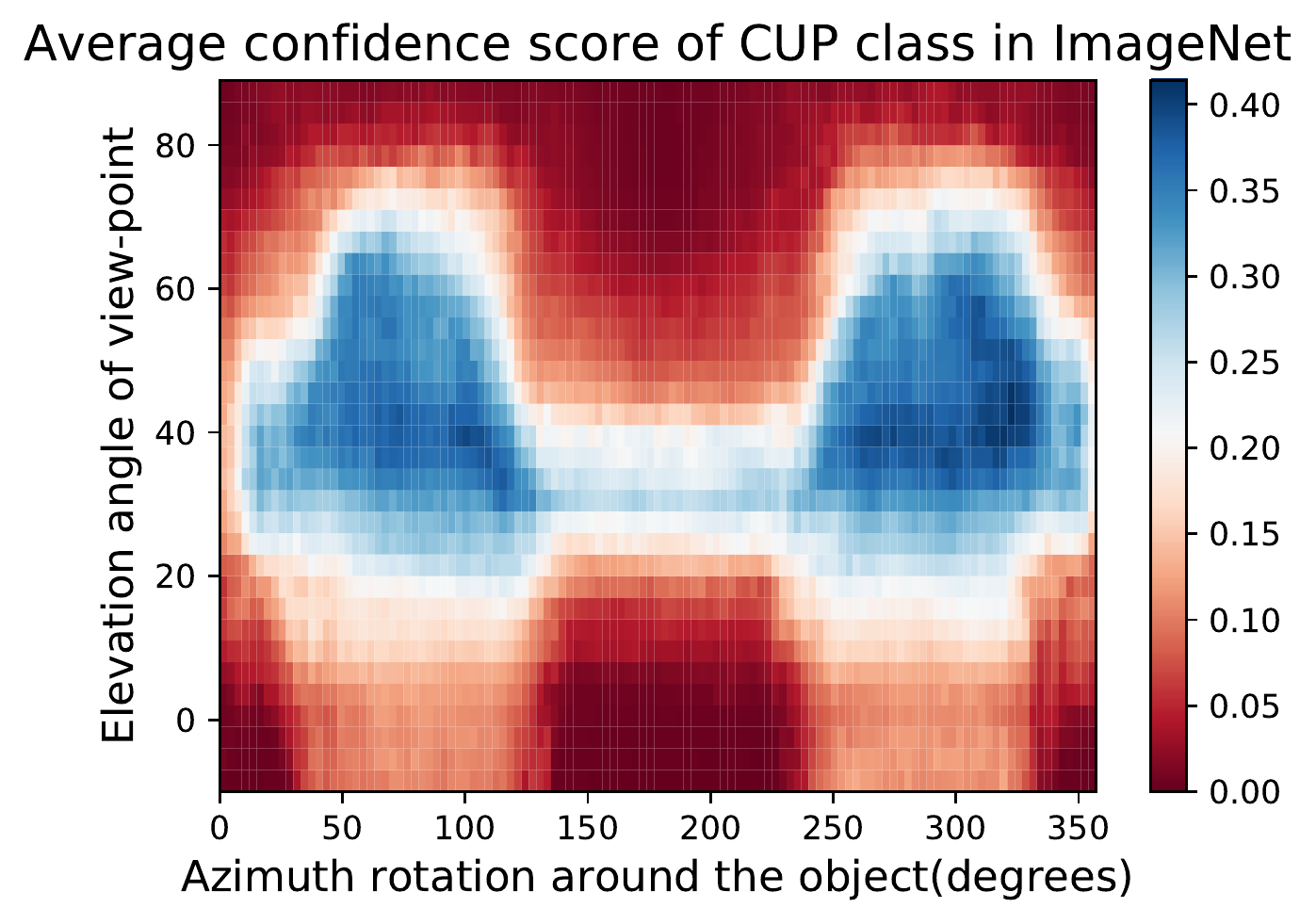}
  \caption{\small \textbf{Semantic Bias in ImageNet}. By taking the average semantic maps over 10 shapes of \emph{cup} class and over different networks, we can visualize the bias of the training data. The angles of low score are probably not well-represented in ImageNet \cite{IMAGENET}.}
  \label{fig:2d-semantic}
\end{figure}
\mysection{Finding Semantic Bias in the Data}\label{sec:data-bias}\\
While observing the above figures uncovers interesting insights about the DNNs and the training data of ImageNet \cite{IMAGENET}, it does not allow to make a conclusion about the network nor about the data. Therefore, we can average these semantic maps of these networks to factor out the effect of the network structure and training and maintain only the effect of training data. We show two such maps called Data Semantic Maps (DSMs). We observe that networks have holes in their semantic maps and these holes are shared among DNNs, indicating bias in the data. Identifying this gap in the data can help train a more semantically robust network. This can be done by leveraging adversarial training on these data-based adversarial regions as performed in the adversarial attack literature \cite{fast-sign}. \figLabel{\ref{fig:2d-semantic}} shows an example of this semantic map, which points to the possibility that ImageNet \cite{IMAGENET} may not contain angles of the easy \emph{cup} class.

\section{Analysis} \label{sec:analysis}
First, we observe in \figLabel{\ref{fig:NMS}} that some red areas (adversarial regions the DNN can not identify the object within) are surrounded by blue areas (robust regions the DNN can identify the object within). These ``semantic traps'' are dangerous in ML, since they are hard to identify (without NSM) and they can cause failure cases for some models. These traps can be attributed to either the model architecture, training, and loss, or bias in the dataset, on which the model was trained (\ie ImageNet \cite{IMAGENET}. 

Note that plotting NMS is extremely expensive even for a moderate dimensionality \eg $n=8$. For example, for the plot in \figLabel{\ref{fig:intro_fig}}, we use $N=180$ points in the range of $360$ degrees. If all the other dimensions require the same number of samples for their individual range, the total joint  space requires $180^{n} = 180^{8} = 1.1 \times 10^{18}$ samples, which is enormous. Evaluating the DNN for that many forward passes is intractable. Thus, we follow a bottom-up approach instead, where we start from one point in the semantic space $\mathbf{u}_{0}$ and we grow an $n$-dimensional hyper-rectangle around that point to find the robust ``neighborhood'' of that point for this specific DNN. 
Table \ref{tbl:complexity} compares different analysis approaches for semantic robustness of DNNs.

\section{Conclusion}%
We analyse DNN robustness with a semantic lens and show how more confident networks tend to create adversarial semantic regions inside highly confident regions. We developed a bottom-up approach to semantically analyse networks  by growing adversarial regions. This approach scales well with dimensionality, and we use it to benchmark the semantic robustness of several well-known and popular DNNs.

\mysection{Acknowledgments}
This work was supported by the King Abdullah University of Science and Technology (KAUST) Office of Sponsored Research under Award No. OSR-CRG2018-3730 
\clearpage
\bibliographystyle{splncs04}
\bibliography{egbib}
\clearpage
\appendix

\section{Analyzing Deep Neural Networks}
Here we visualize different Network semantic maps generated during our analysis. In the 1D case we fix the elevation of the camera to a nominal angle of $\ang{35}$ and rotate around the object. In the 2D case, we change both the elevation and azimuth around the object. These maps can be generated to any type of semantic parameters that affect the generation of the image , and not viewing angle. 
\subsection{Networks Semantic Maps (1D)}
In \figLabel{\ref{fig:nsm1d-1},\ref{fig:nsm1d-2}} we visualize the 1D semantic maps of rotating around the object and recording different DNNs performance and averaging the profile over 10 different shapes per class.

\begin{figure*}[h]
\centering
\tabcolsep=0.03cm
   \begin{tabular}{c|c}
\includegraphics[width = 0.49\linewidth]{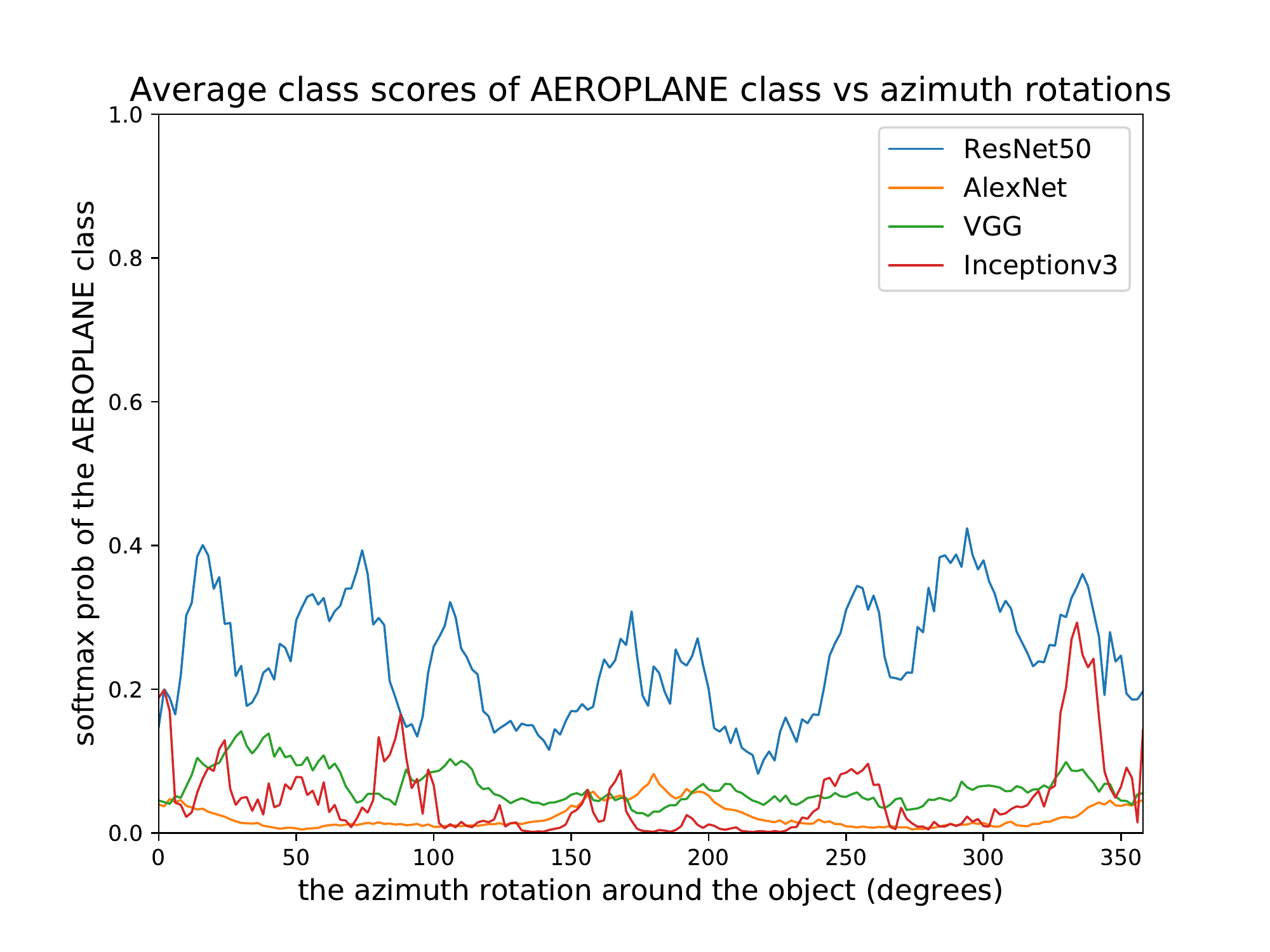}&
\includegraphics[width = 0.49\linewidth]{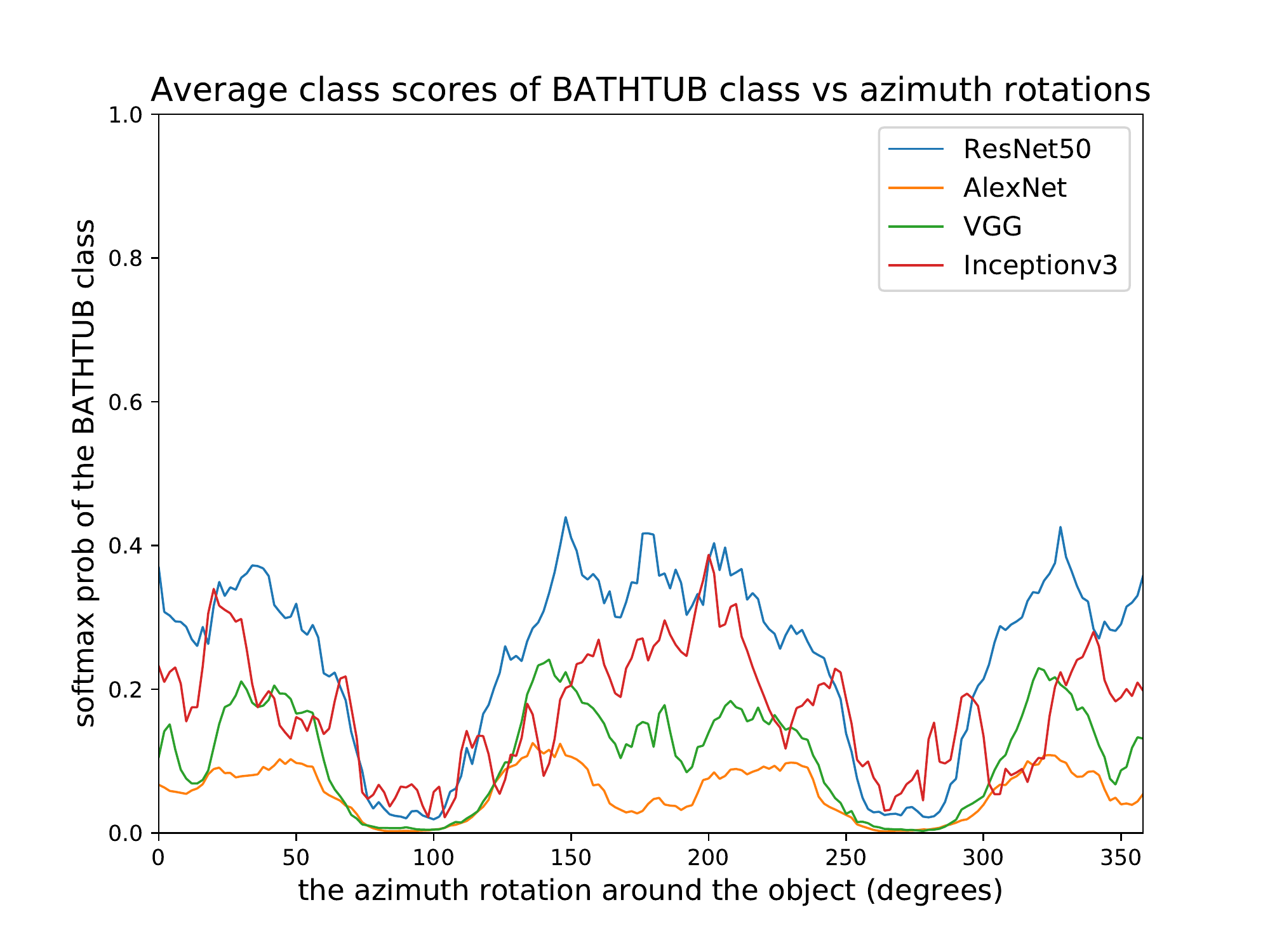}\\ \hline
\includegraphics[width = 0.49\linewidth]{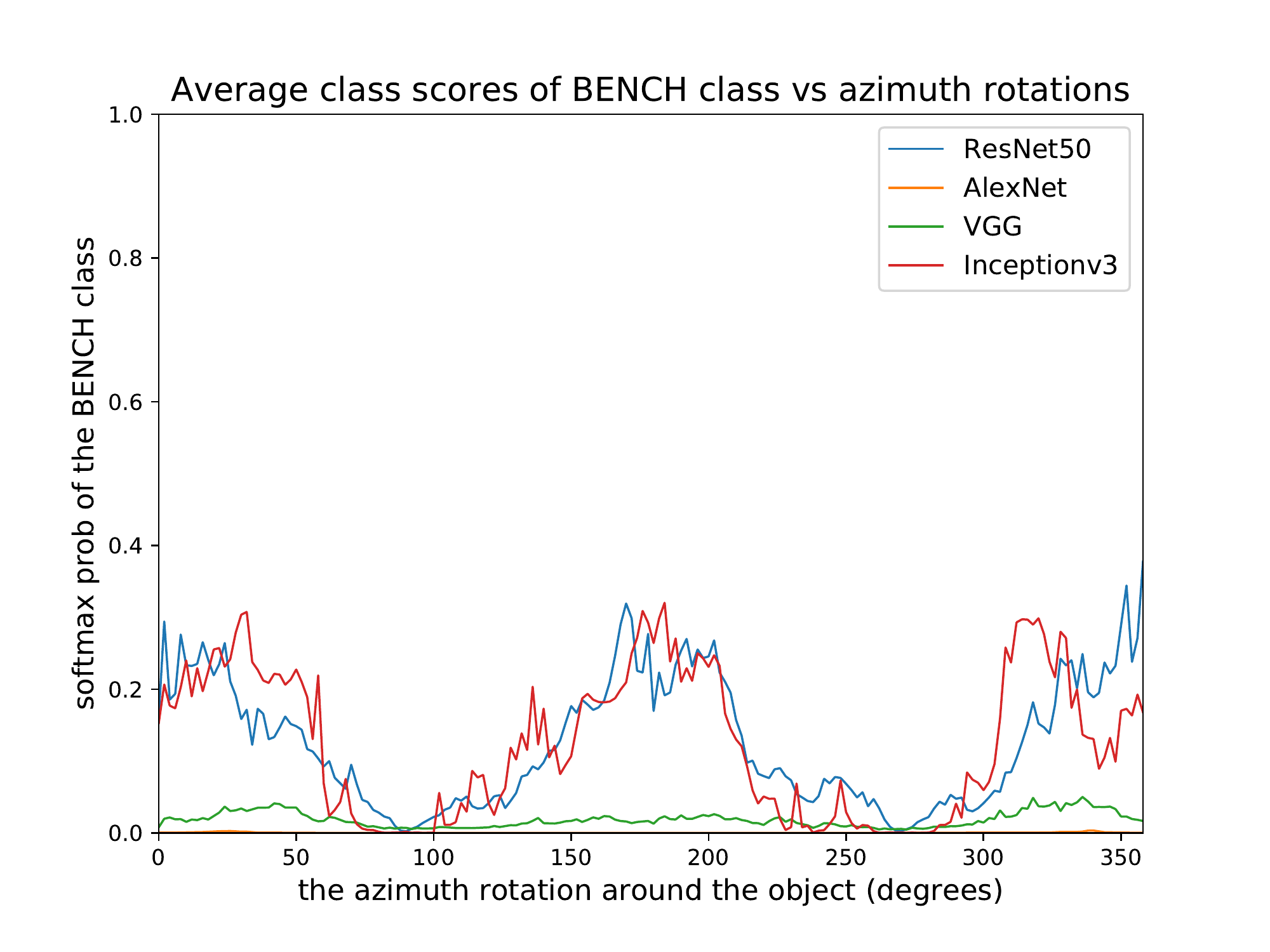}&
\includegraphics[width = 0.49\linewidth]{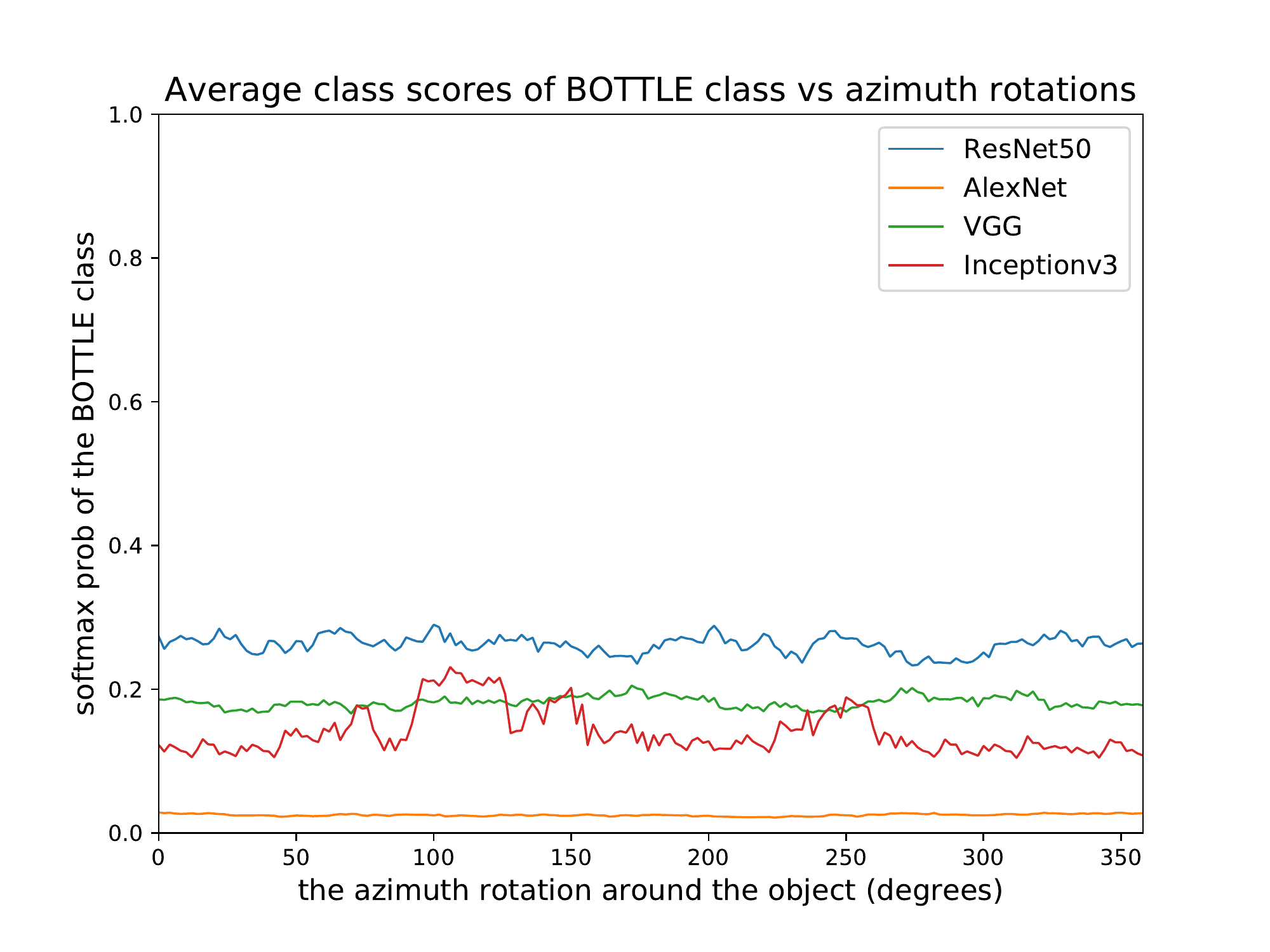}
\end{tabular}
   \caption{\small \textbf{1D Network Semantic Maps NMS-I}: visualizing 1D Semantic Robustness profile for different networks averaged over 10 different shapes. Observe that different DNNs profiles differ depending on the training , accuracy , and network architectures that all result in a unique ''signatures" for the DNN on that class. The correlation between the DNN profiles is due to the common data bias in ImageNet.}
   \vspace{-8pt}
   \label{fig:nsm1d-1}
\end{figure*}

\begin{figure*}[h]
\centering
\tabcolsep=0.03cm
   \begin{tabular}{c|c}
\includegraphics[width = 0.49\linewidth]{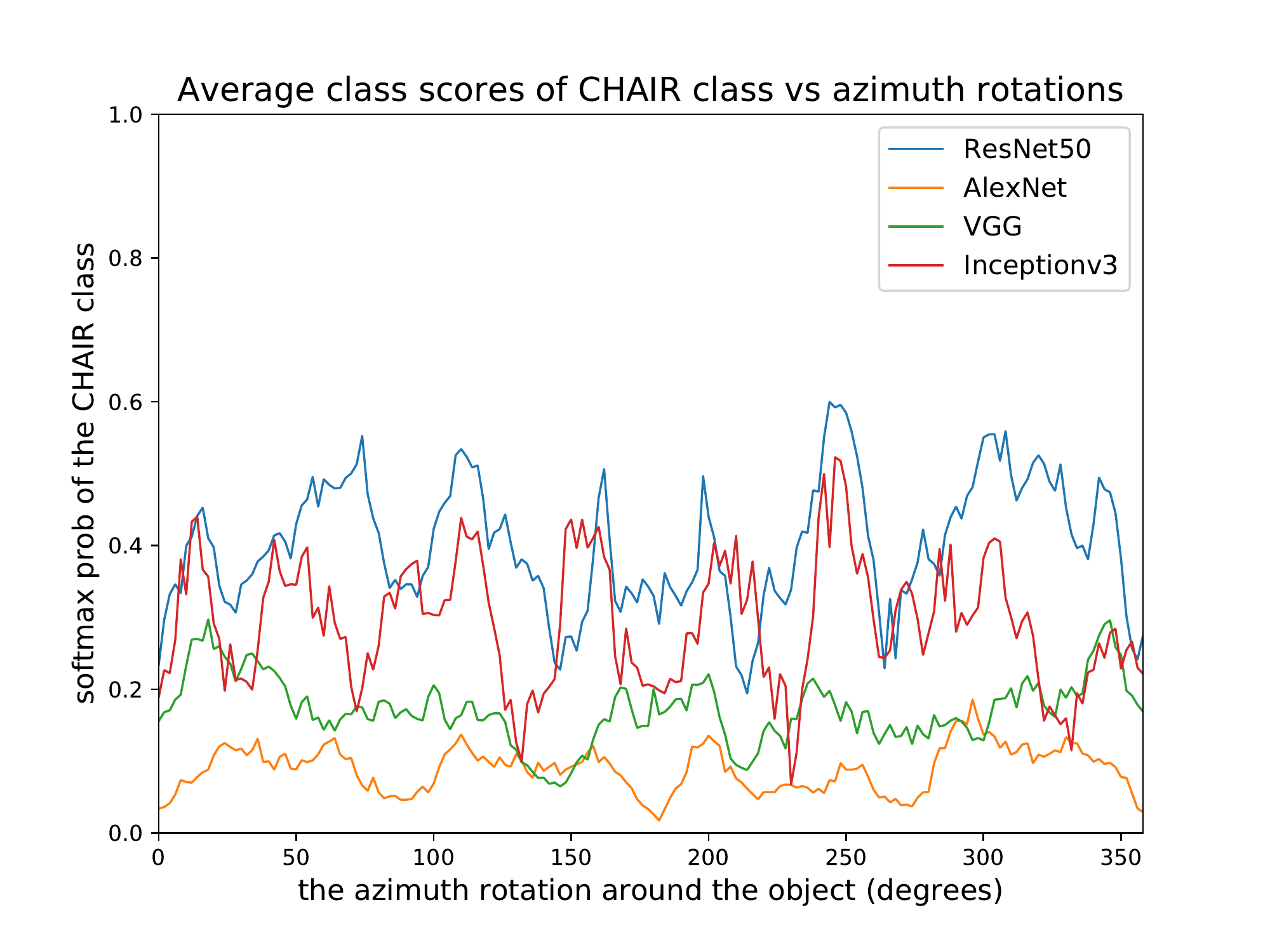}&
\includegraphics[width = 0.49\linewidth]{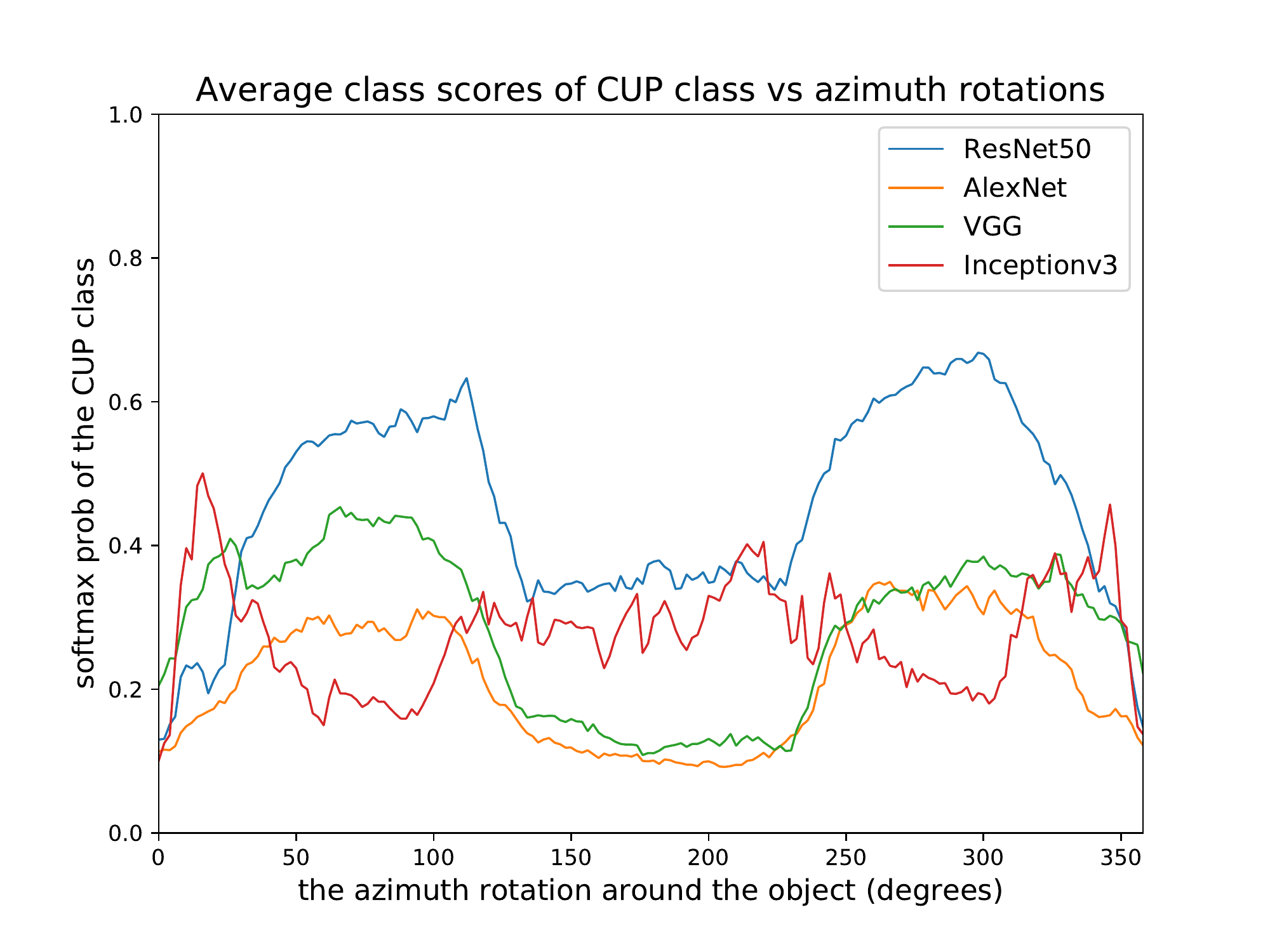}\\\hline
\includegraphics[width = 0.49\linewidth]{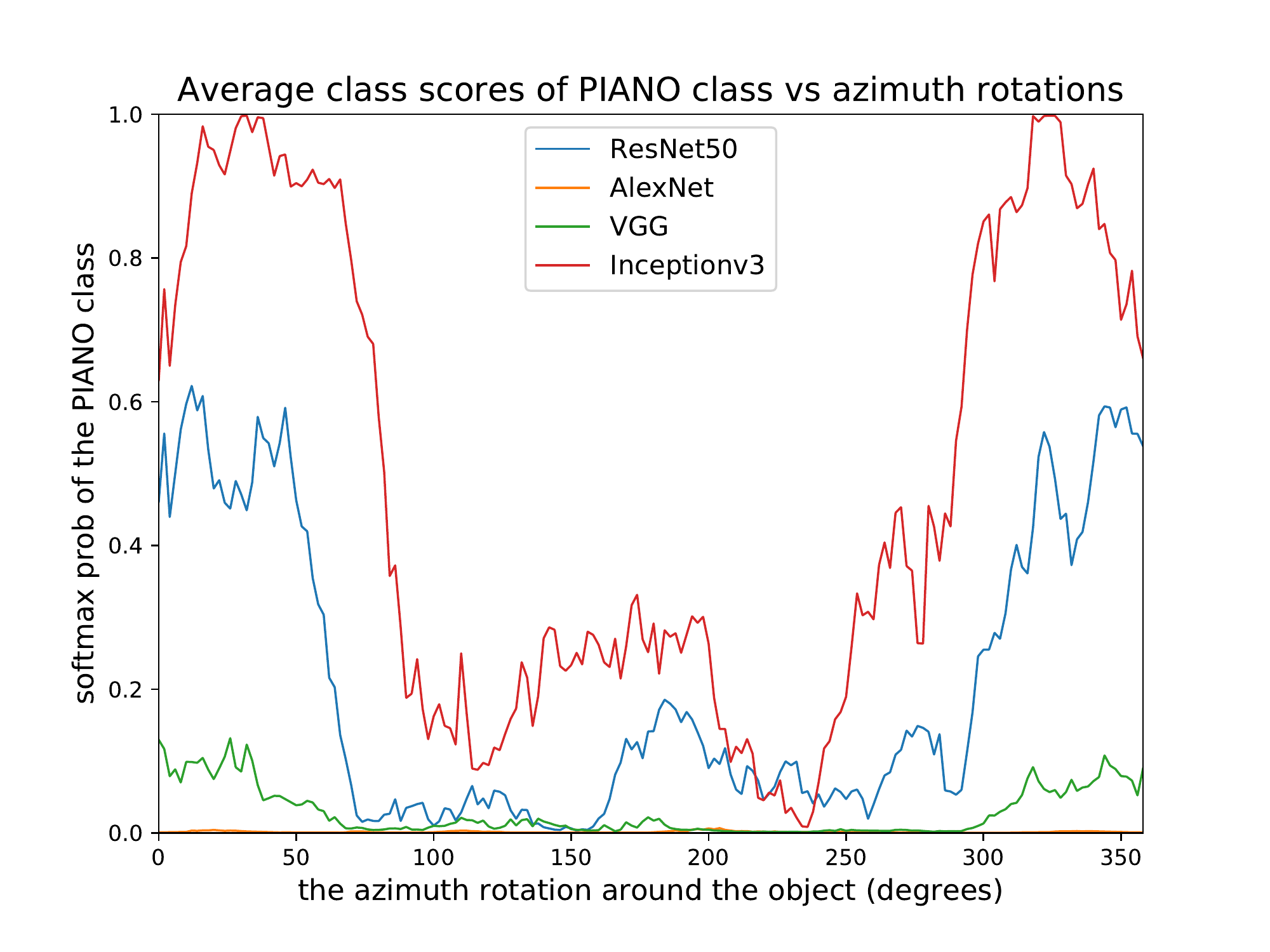}&
\includegraphics[width = 0.49\linewidth]{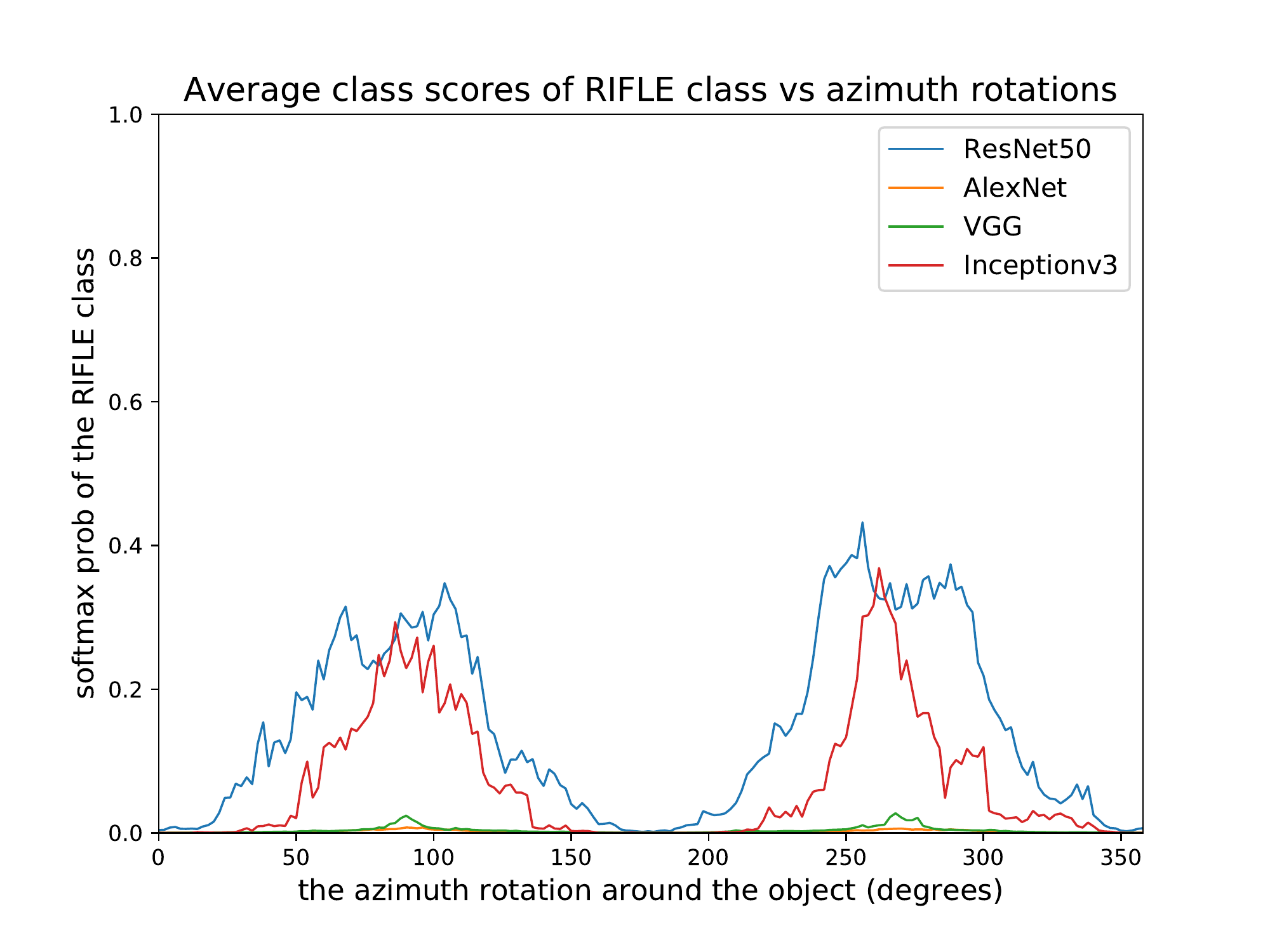}\\\hline
\includegraphics[width = 0.49\linewidth]{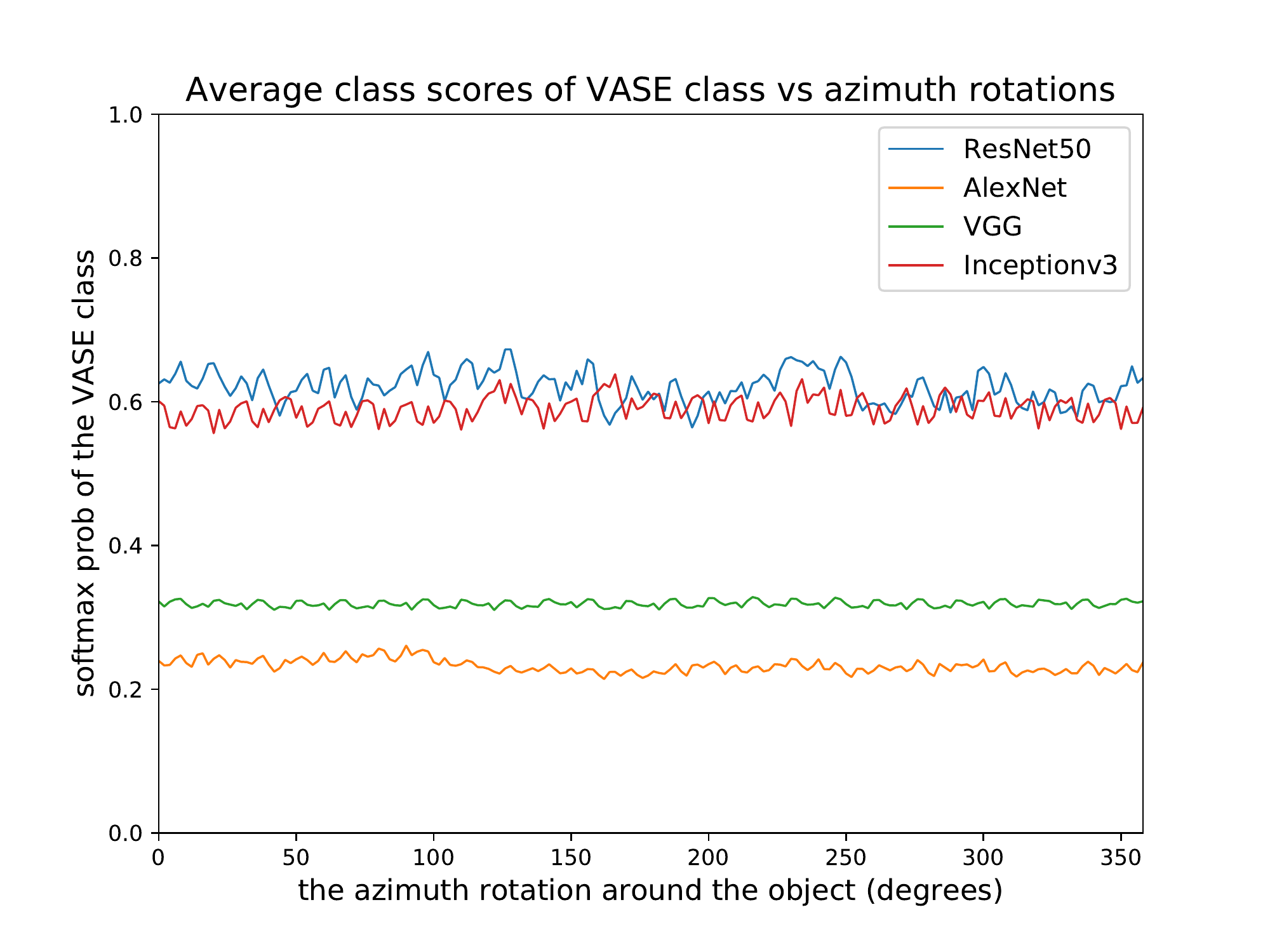}&
\includegraphics[width = 0.49\linewidth]{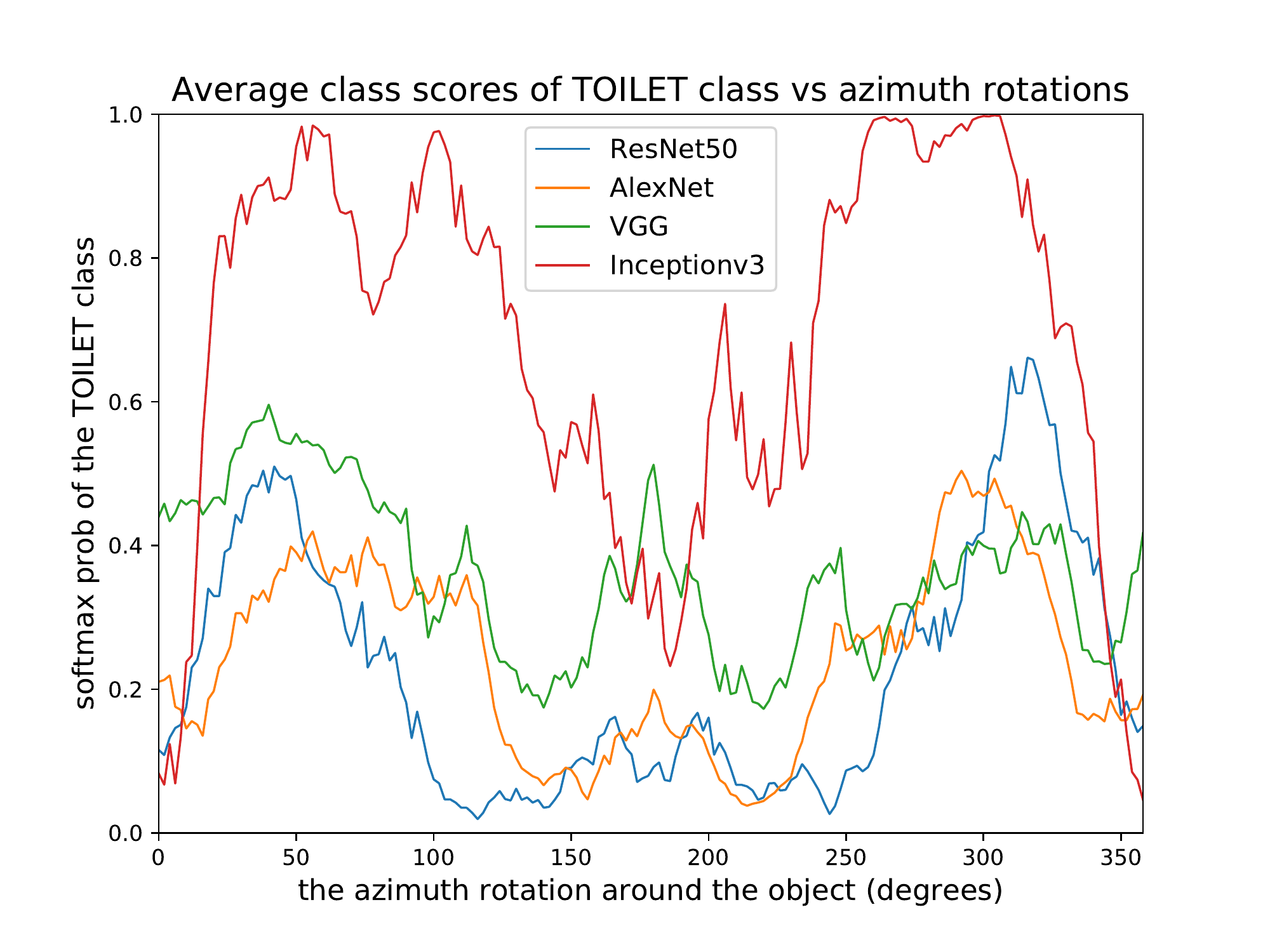}
\end{tabular}
   \caption{\small \textbf{1D Network Semantic Maps NMS-II}: visualizing 1D Semantic Robustness profile for different networks averaged over 10 different shapes. Observe that different DNNs profiles differ depending on the training , accuracy , and network architectures that all result in a unique ''signatures" for the DNN on that class. The correlation between the DNN profiles is due to the common data bias in ImageNet.}
   \vspace{-8pt}
   \label{fig:nsm1d-2}
\end{figure*}

\subsection{Networks Semantic Maps (2D)}
In \figLabel{\ref{fig:nsm2d-1},\ref{fig:nsm2d-2}} we visualize the 2D semantic maps of elevation angles and rotating around the object and recording different DNNs performance and averaging the maps over 10 different shapes per class.
\begin{figure*}[h]
\centering
\tabcolsep=0.03cm
   \begin{tabular}{||c|c|c|c||} \hline
   \textbf{AlexNet} & \textbf{VGG} &\textbf{ResNet50} & \textbf{InceptionV3} \\ \hline
\includegraphics[width = 0.24\linewidth]{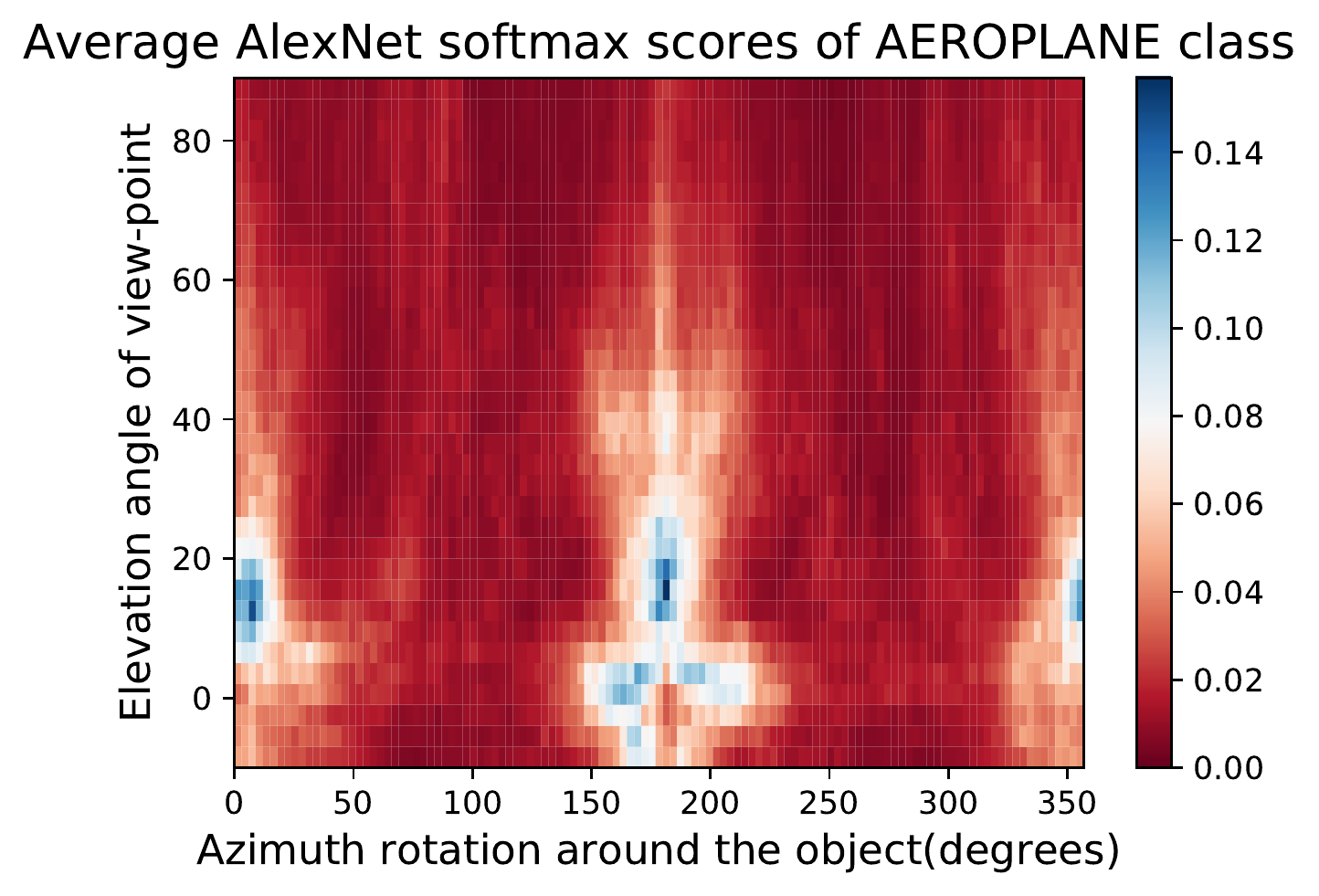}&
\includegraphics[width = 0.24\linewidth]{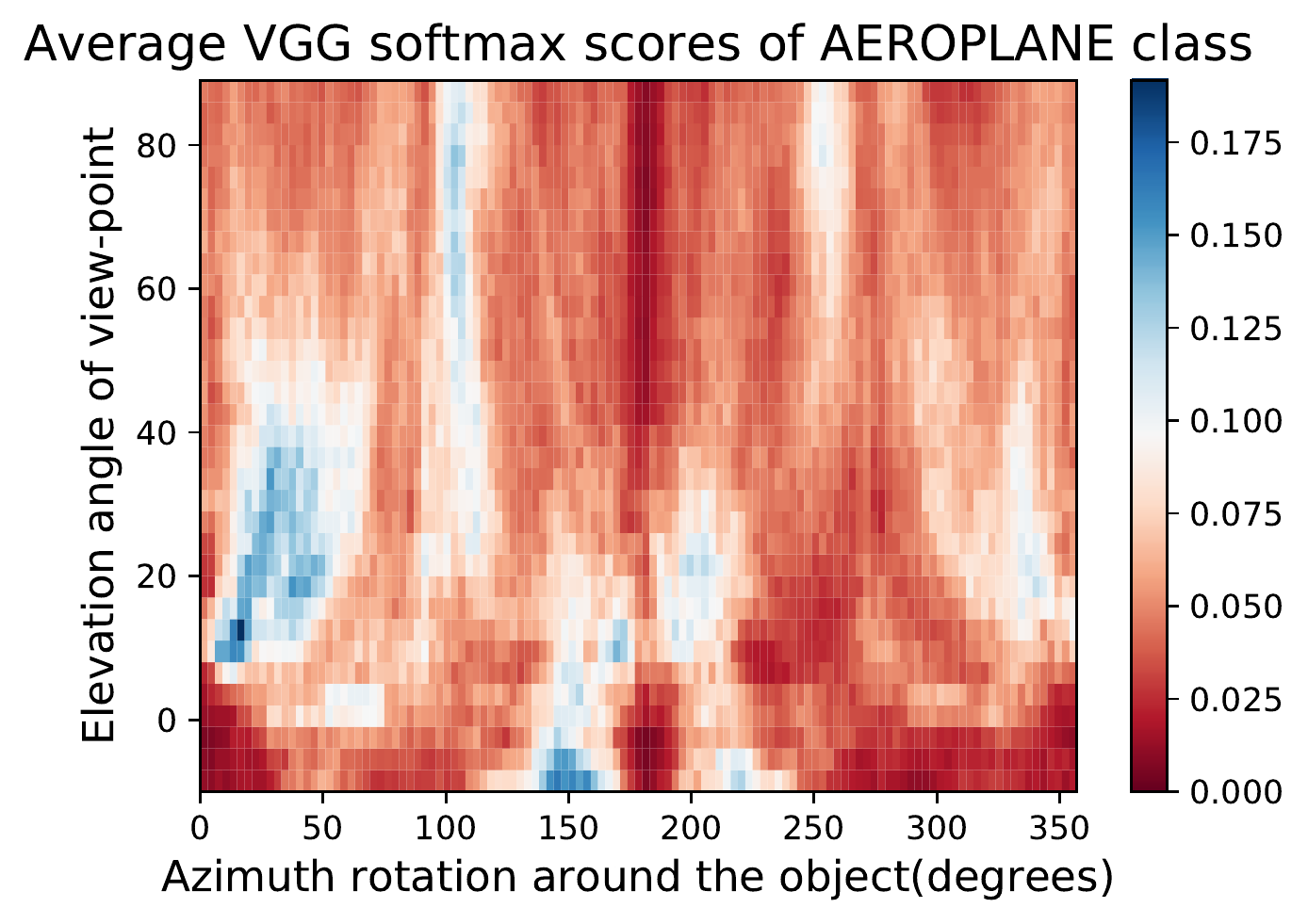}&
\includegraphics[width = 0.24\linewidth]{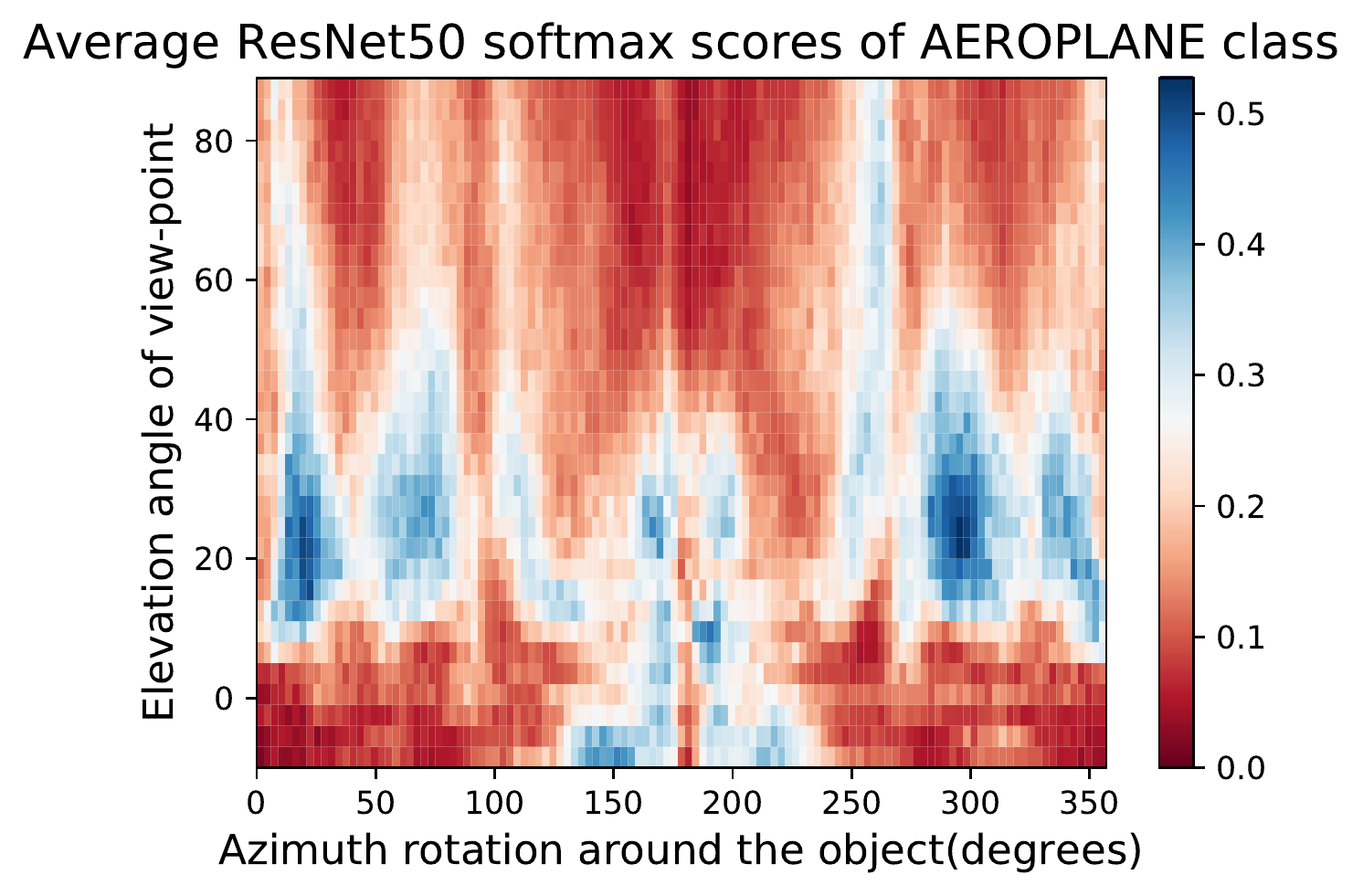}&
\includegraphics[width = 0.24\linewidth]{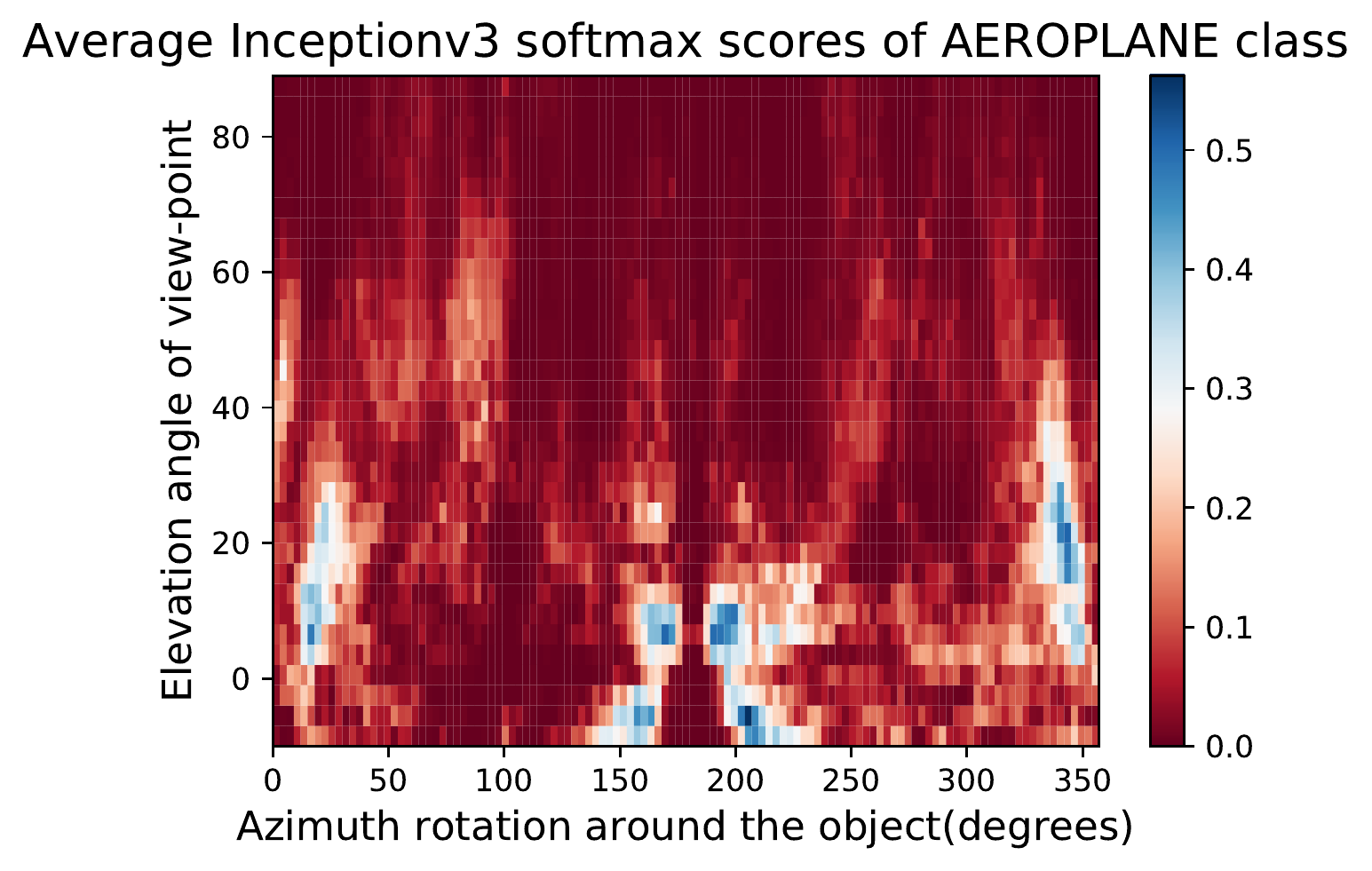}\\ \hline
\includegraphics[width = 0.24\linewidth]{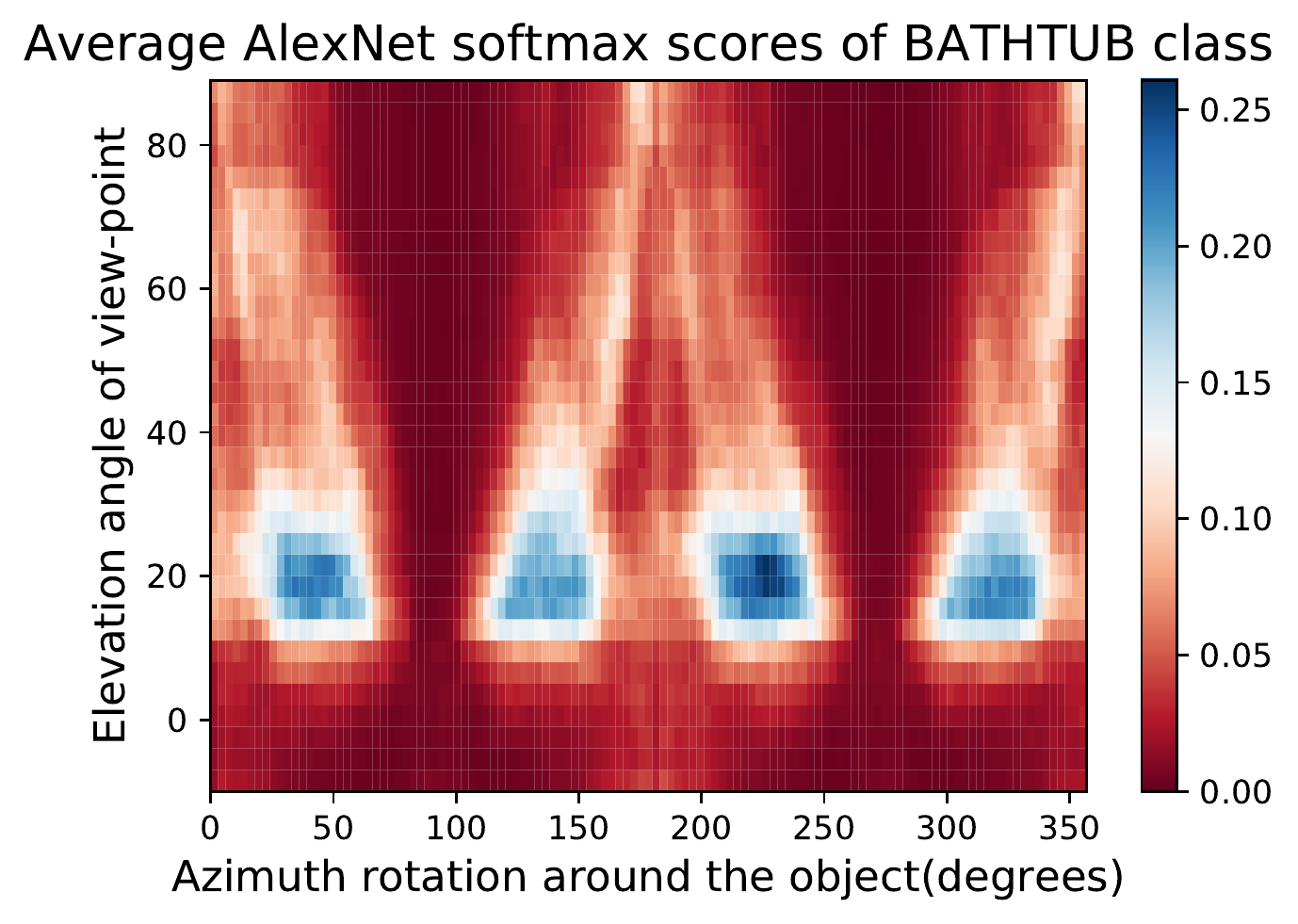}&
\includegraphics[width = 0.24\linewidth]{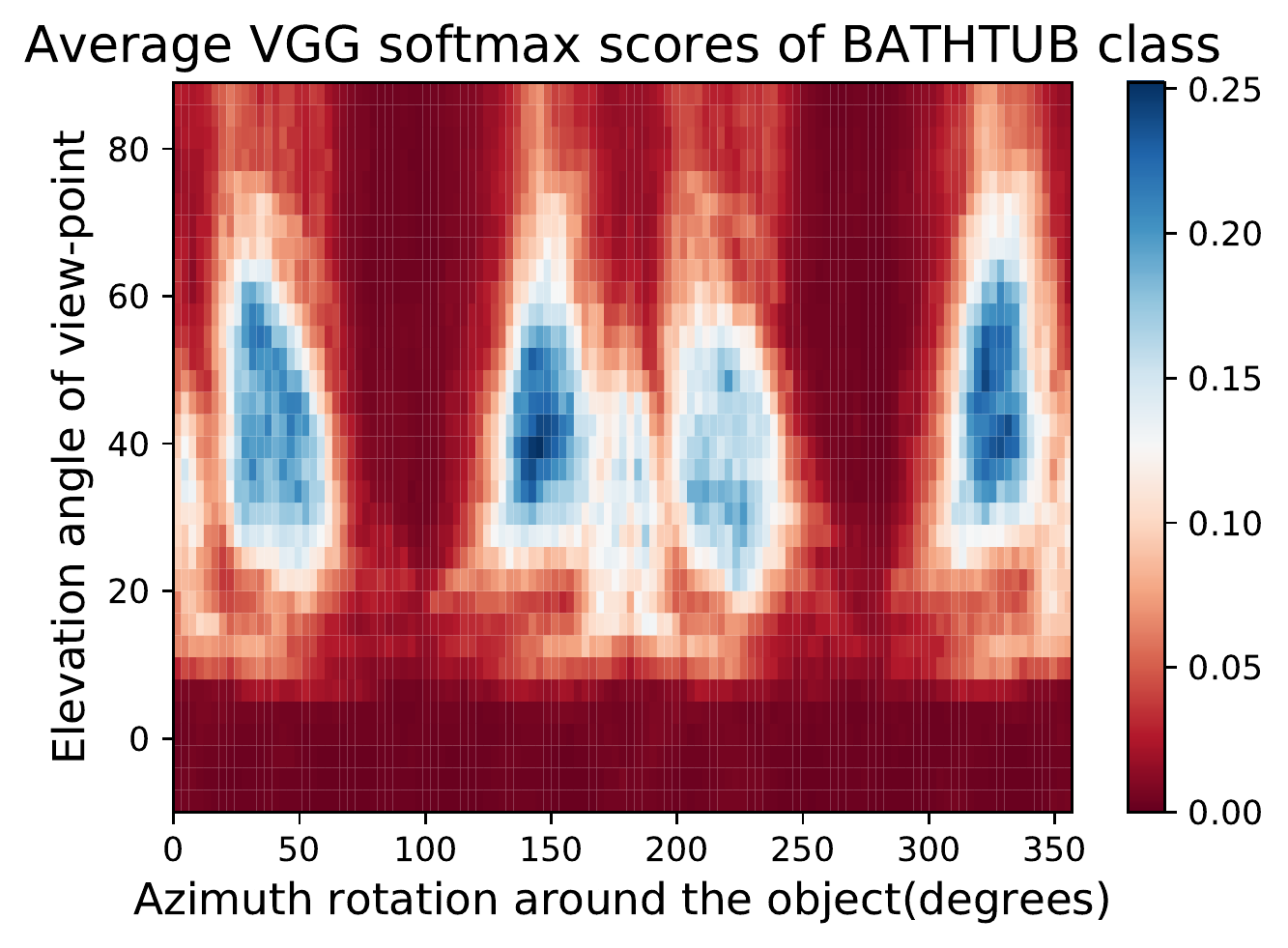}&
\includegraphics[width = 0.24\linewidth]{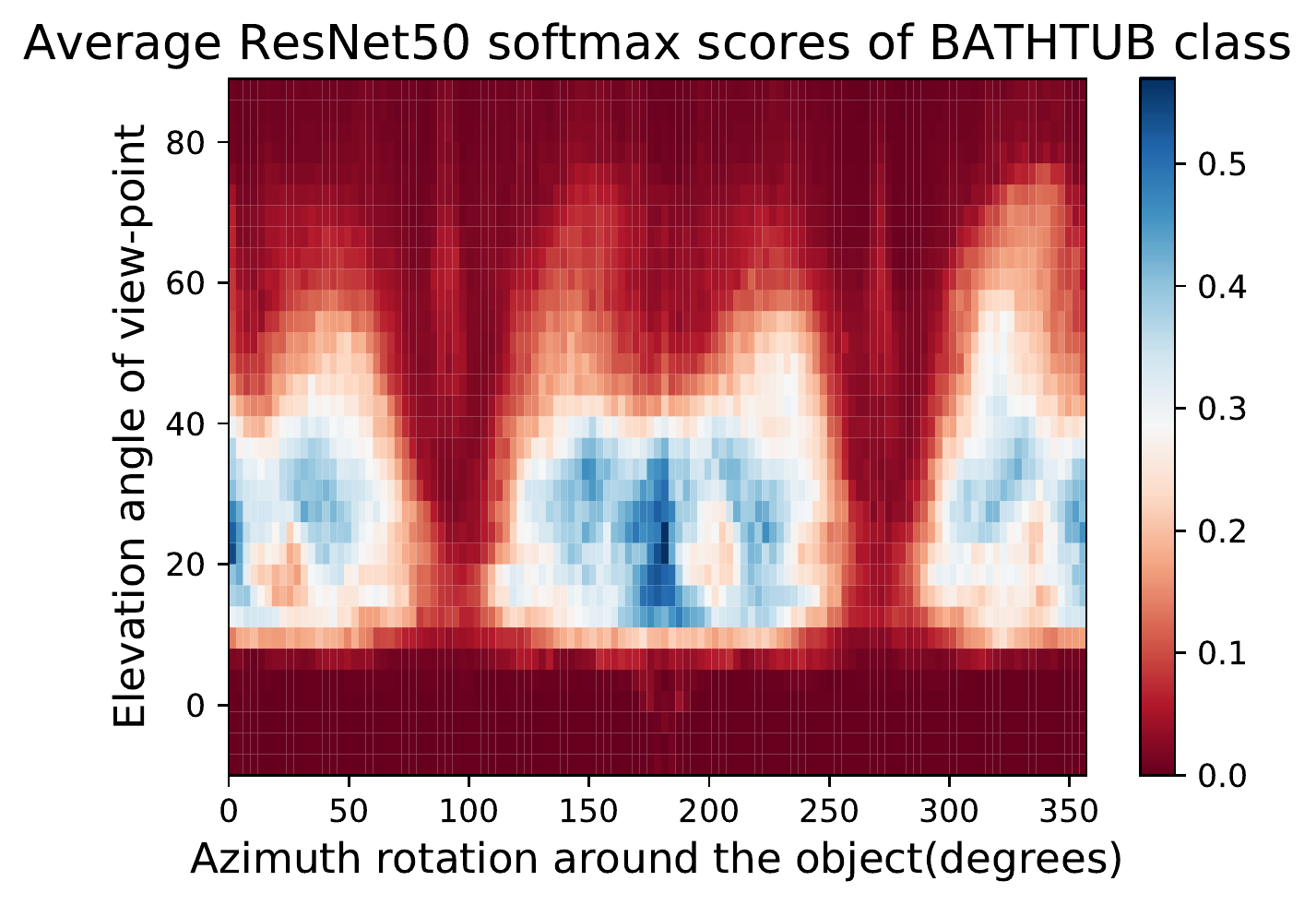}&
\includegraphics[width = 0.24\linewidth]{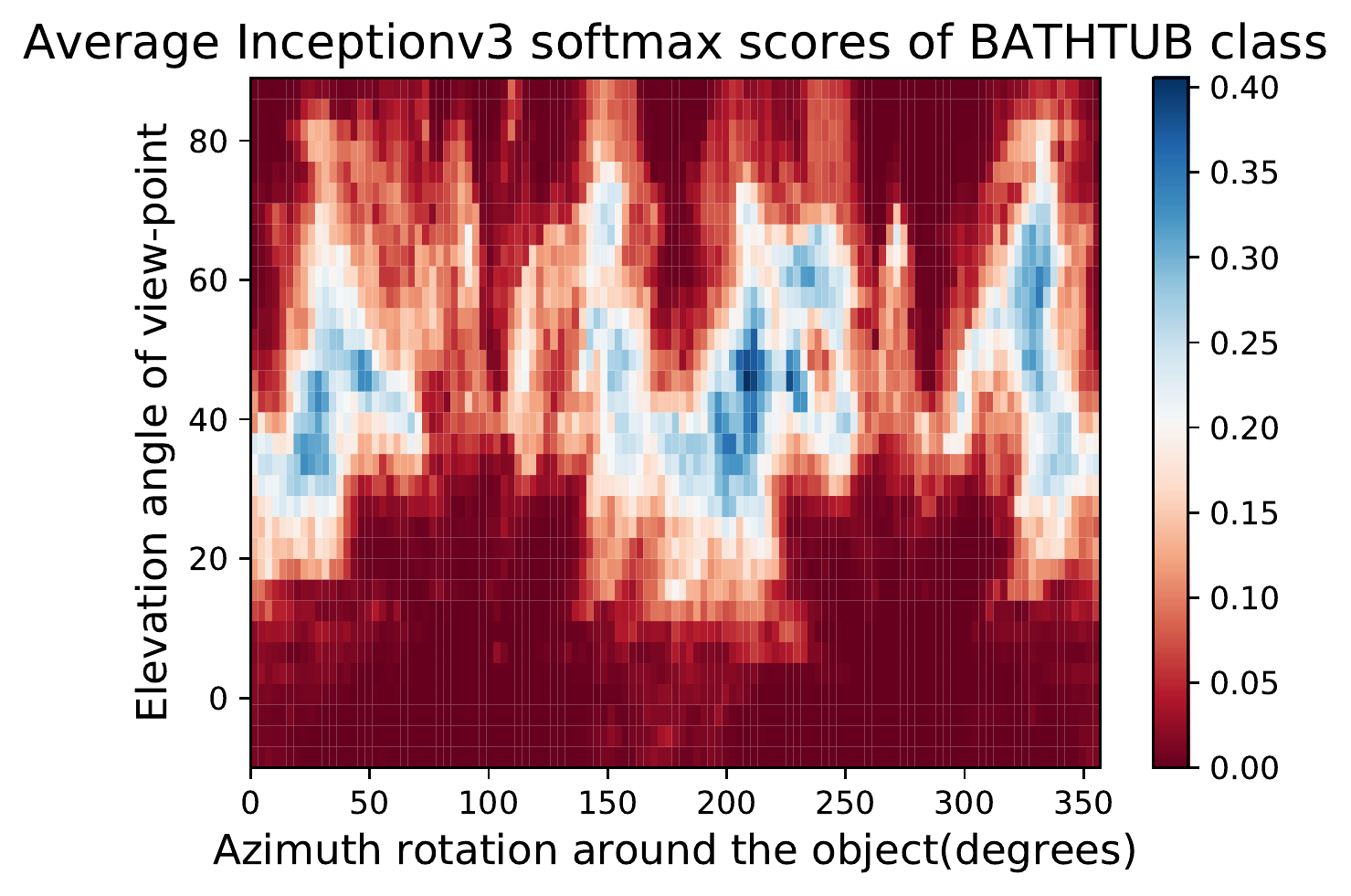}\\ \hline
\includegraphics[width = 0.24\linewidth]{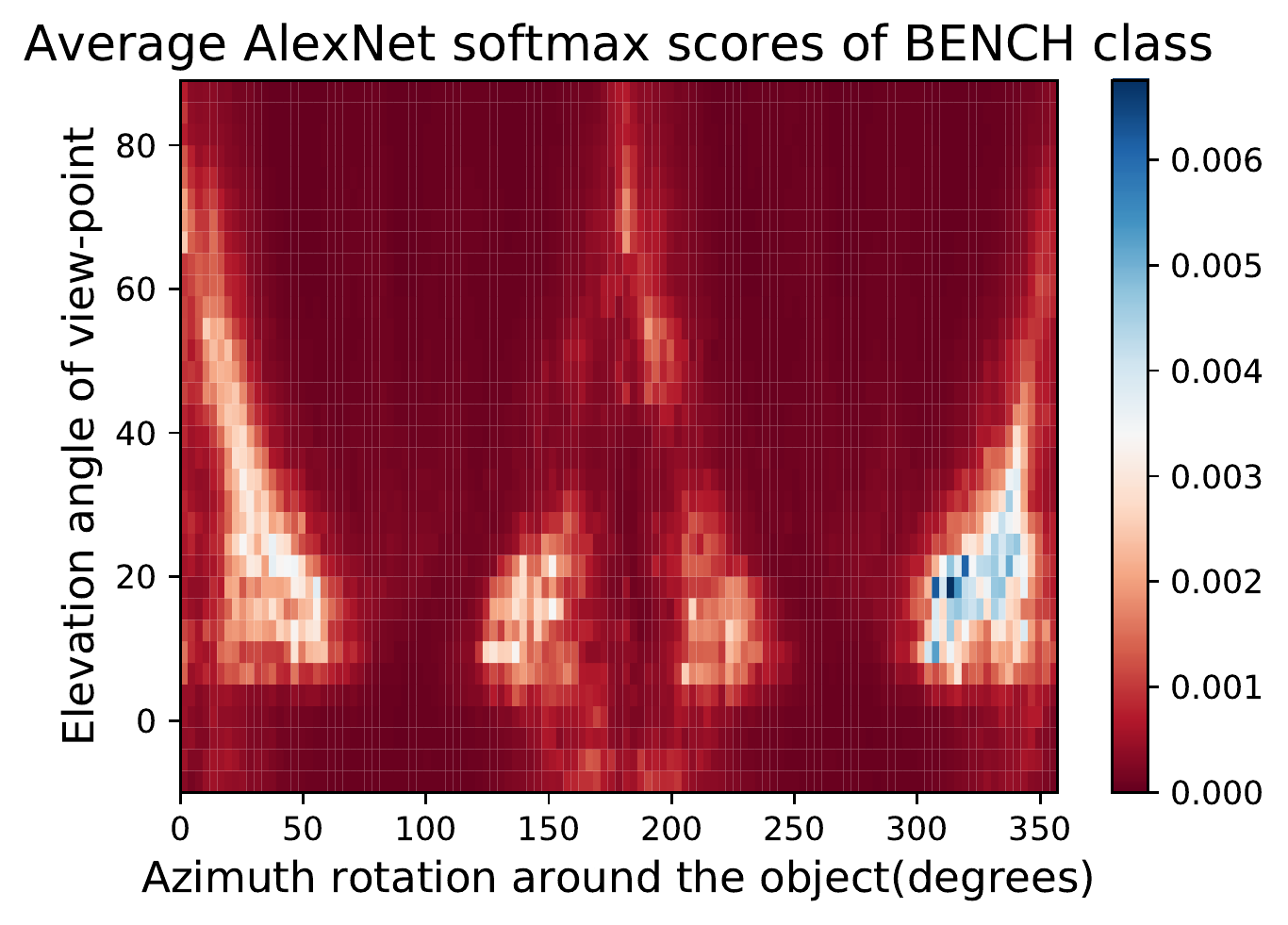}&
\includegraphics[width = 0.24\linewidth]{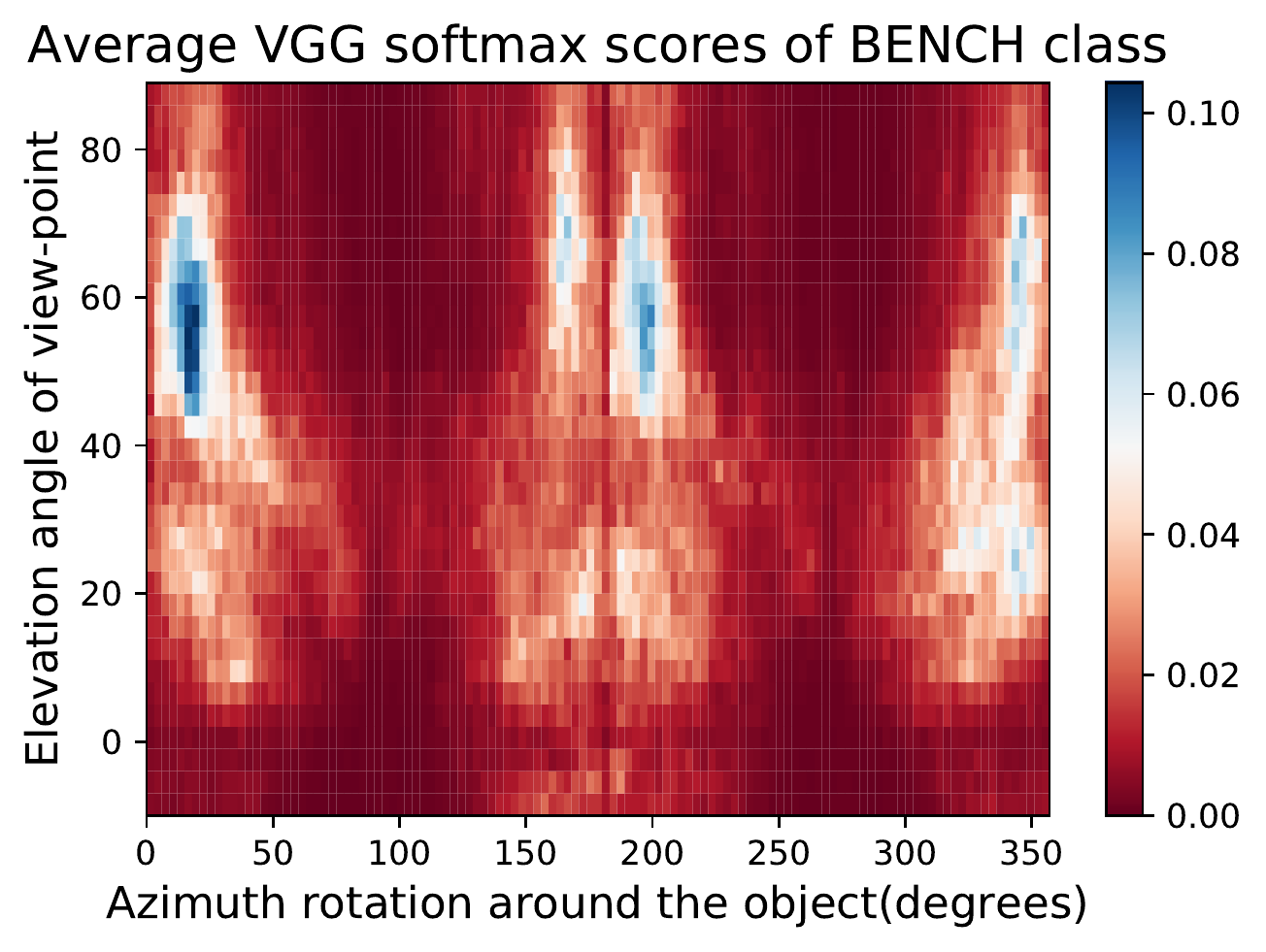}&
\includegraphics[width = 0.24\linewidth]{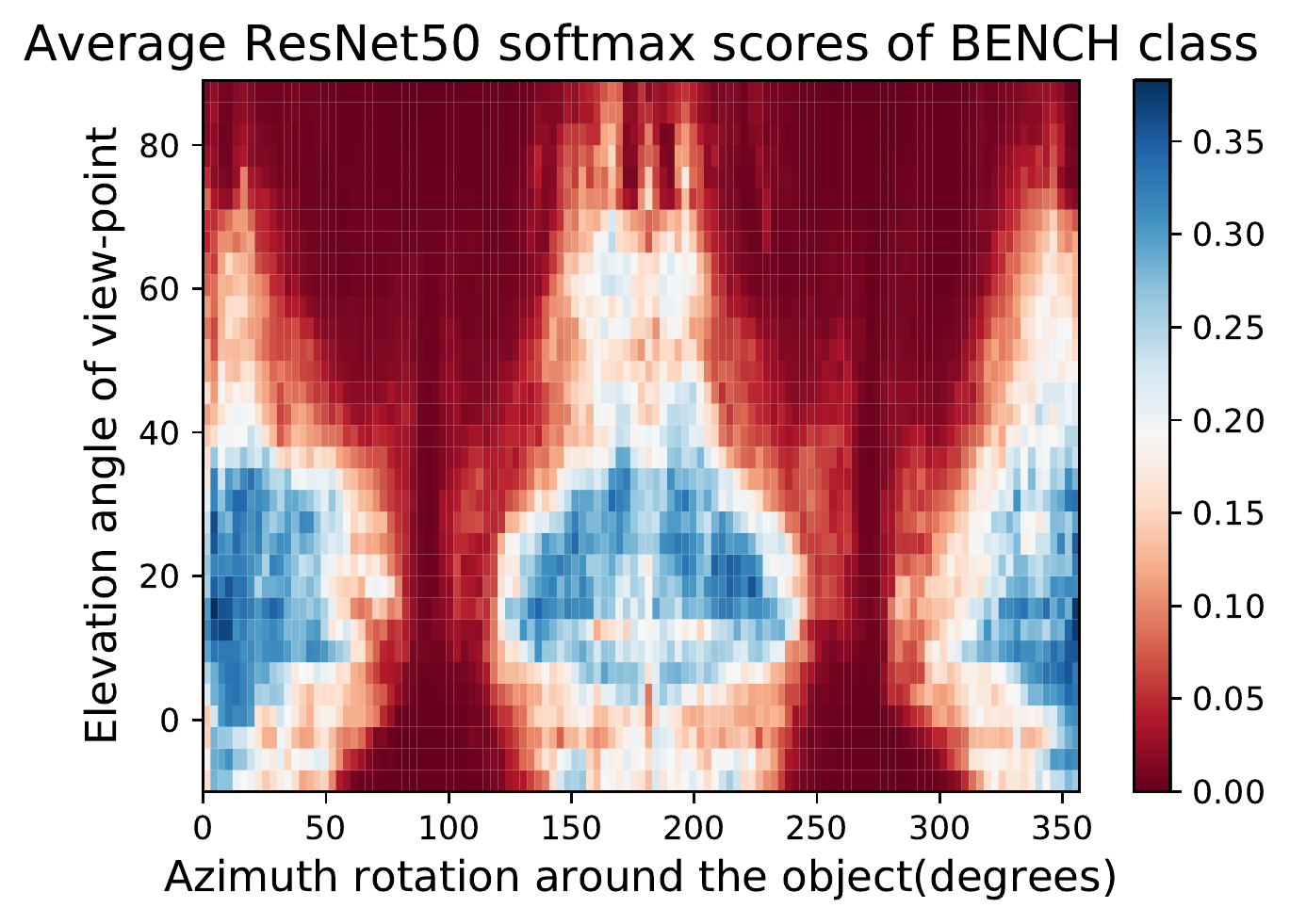}&
\includegraphics[width = 0.24\linewidth]{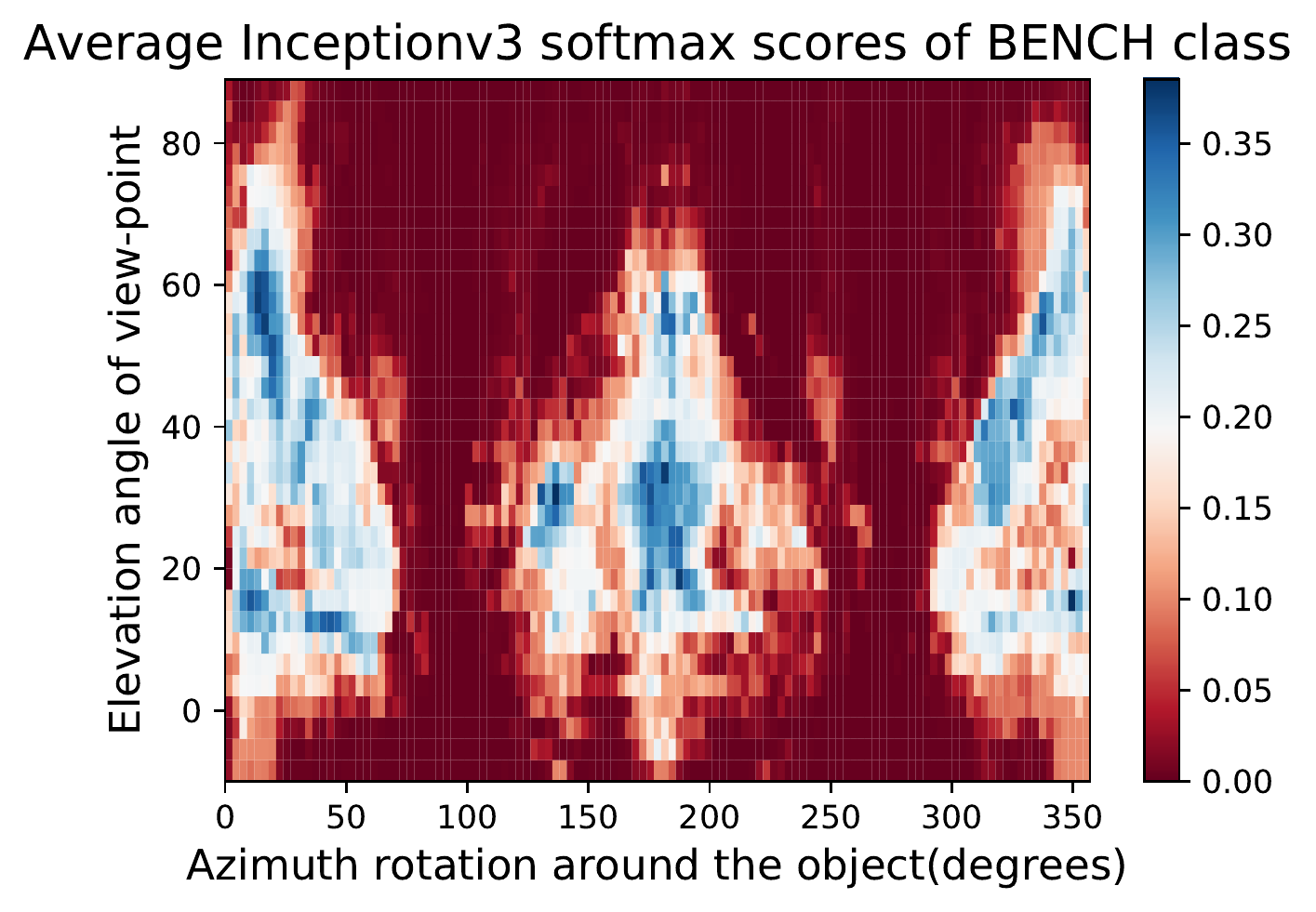}\\ \hline
\includegraphics[width = 0.24\linewidth]{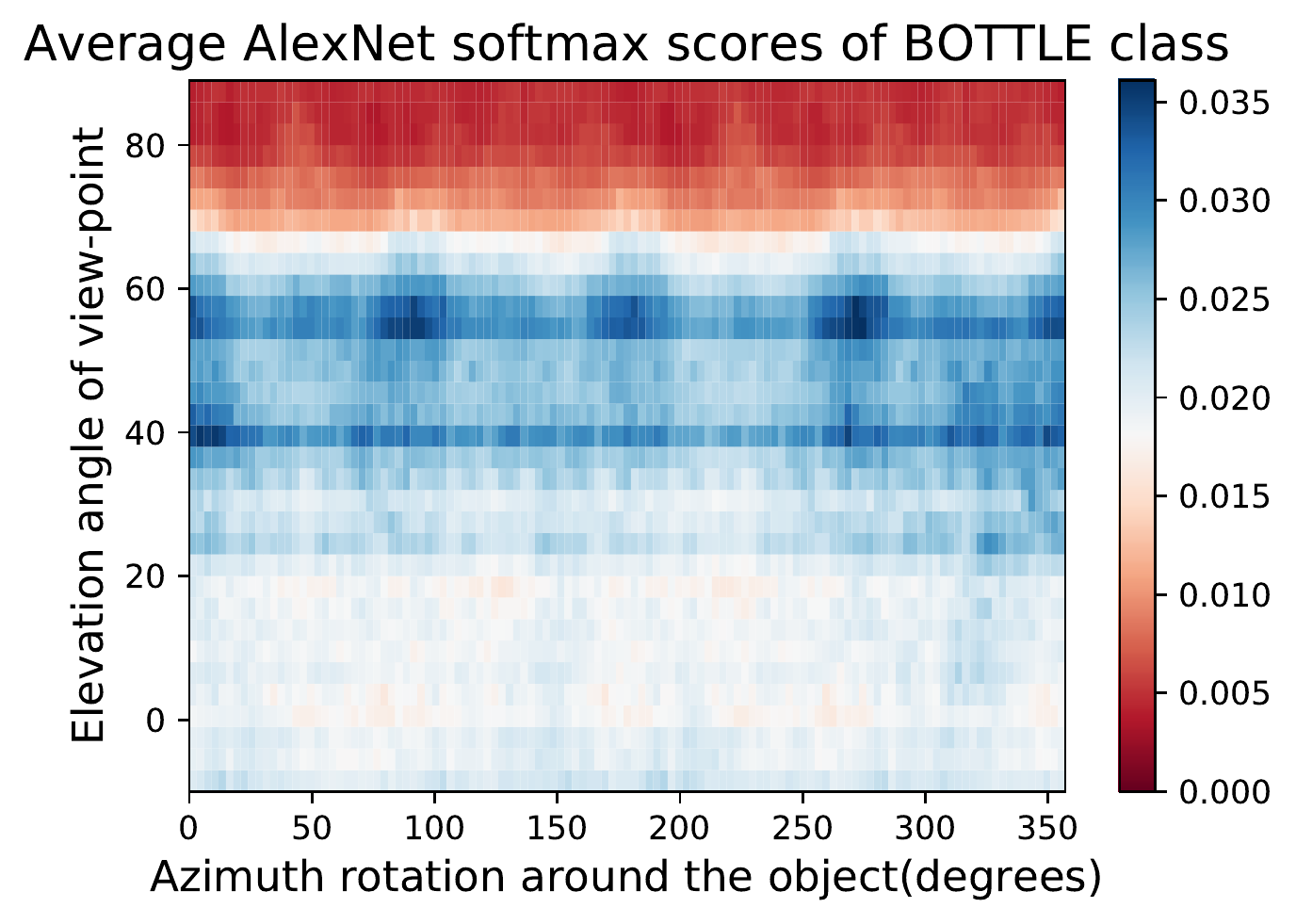}&
\includegraphics[width = 0.24\linewidth]{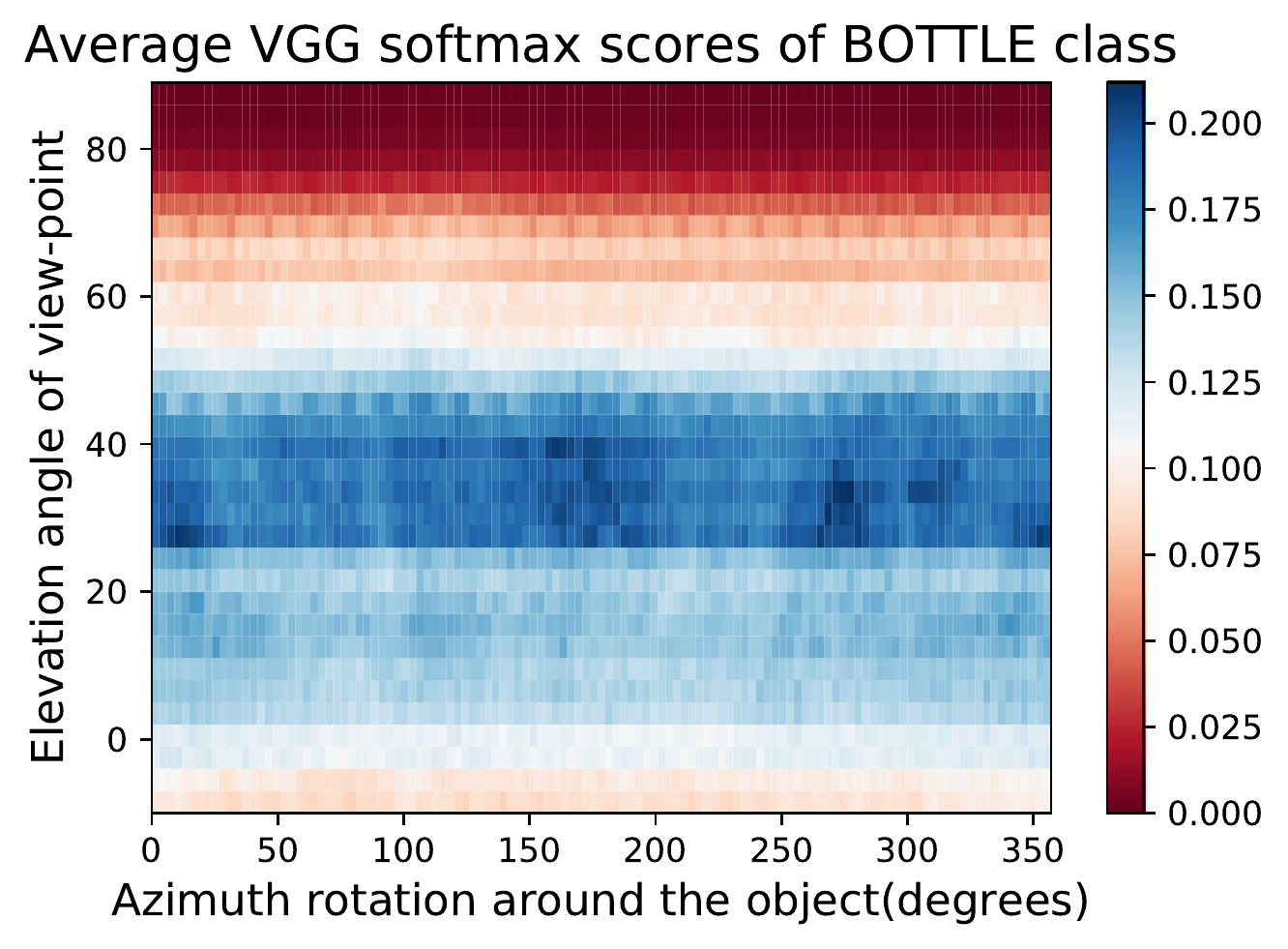}&
\includegraphics[width = 0.24\linewidth]{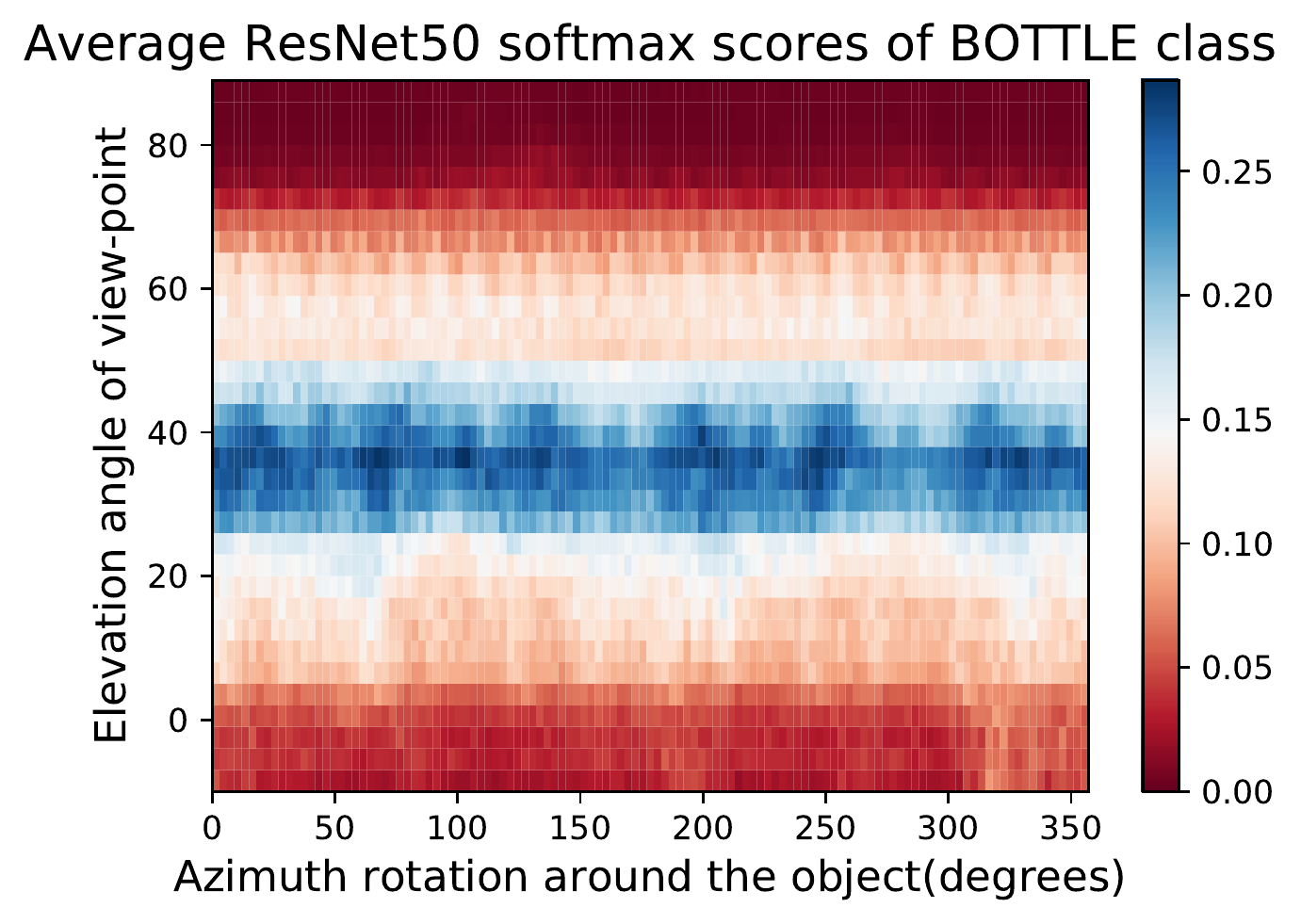}&
\includegraphics[width = 0.24\linewidth]{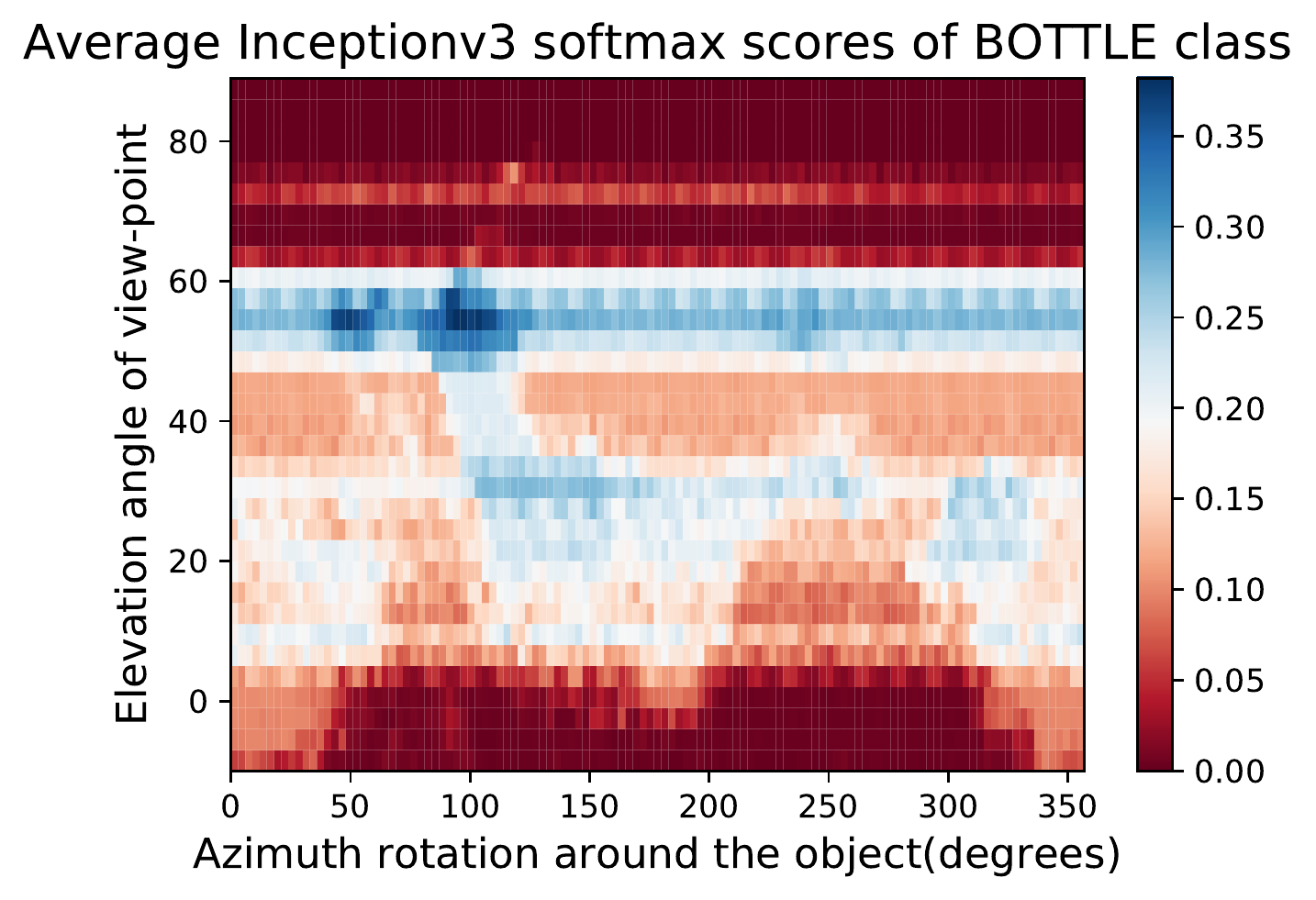} \\\hline
\includegraphics[width = 0.24\linewidth]{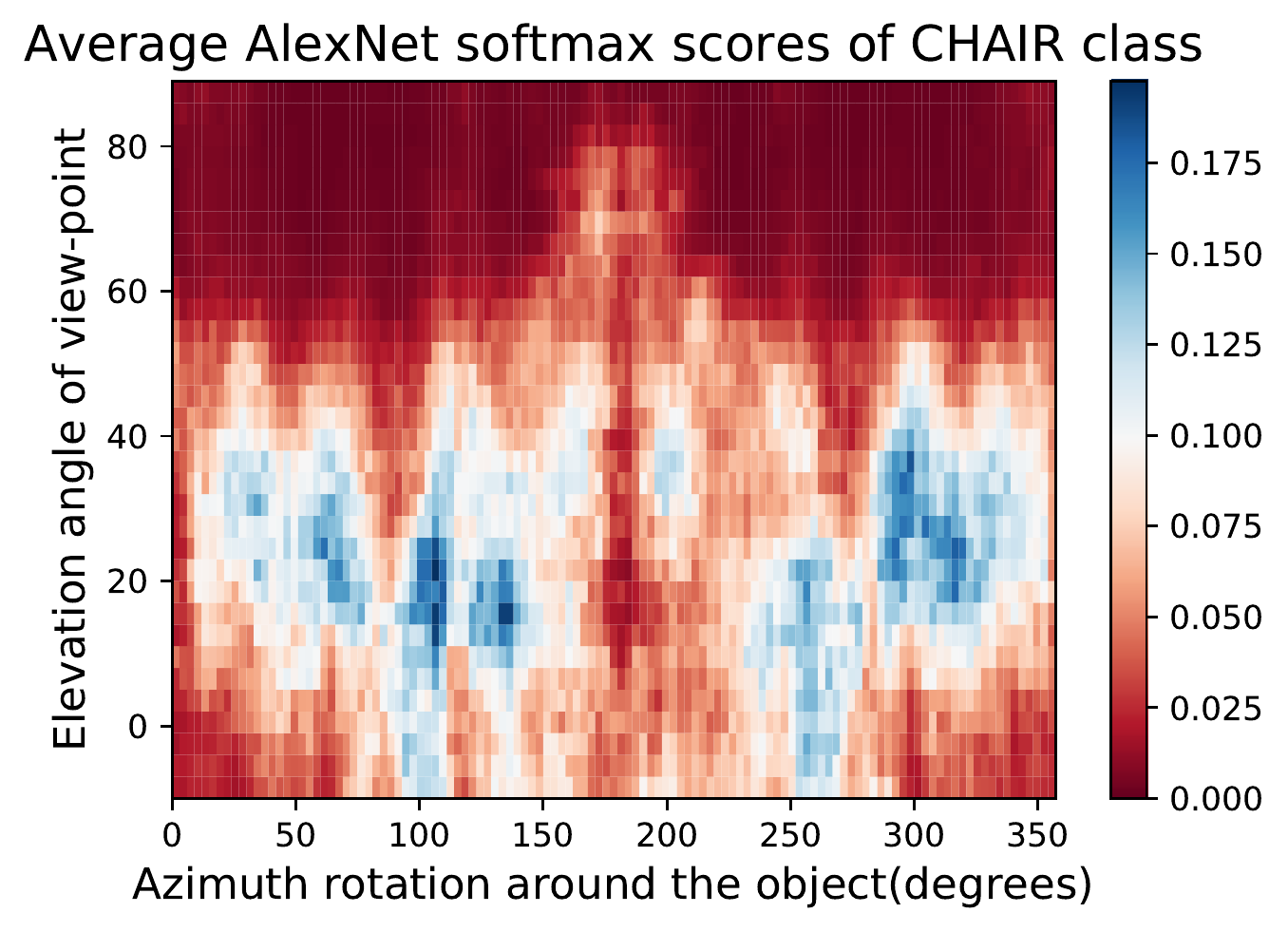}&
\includegraphics[width = 0.24\linewidth]{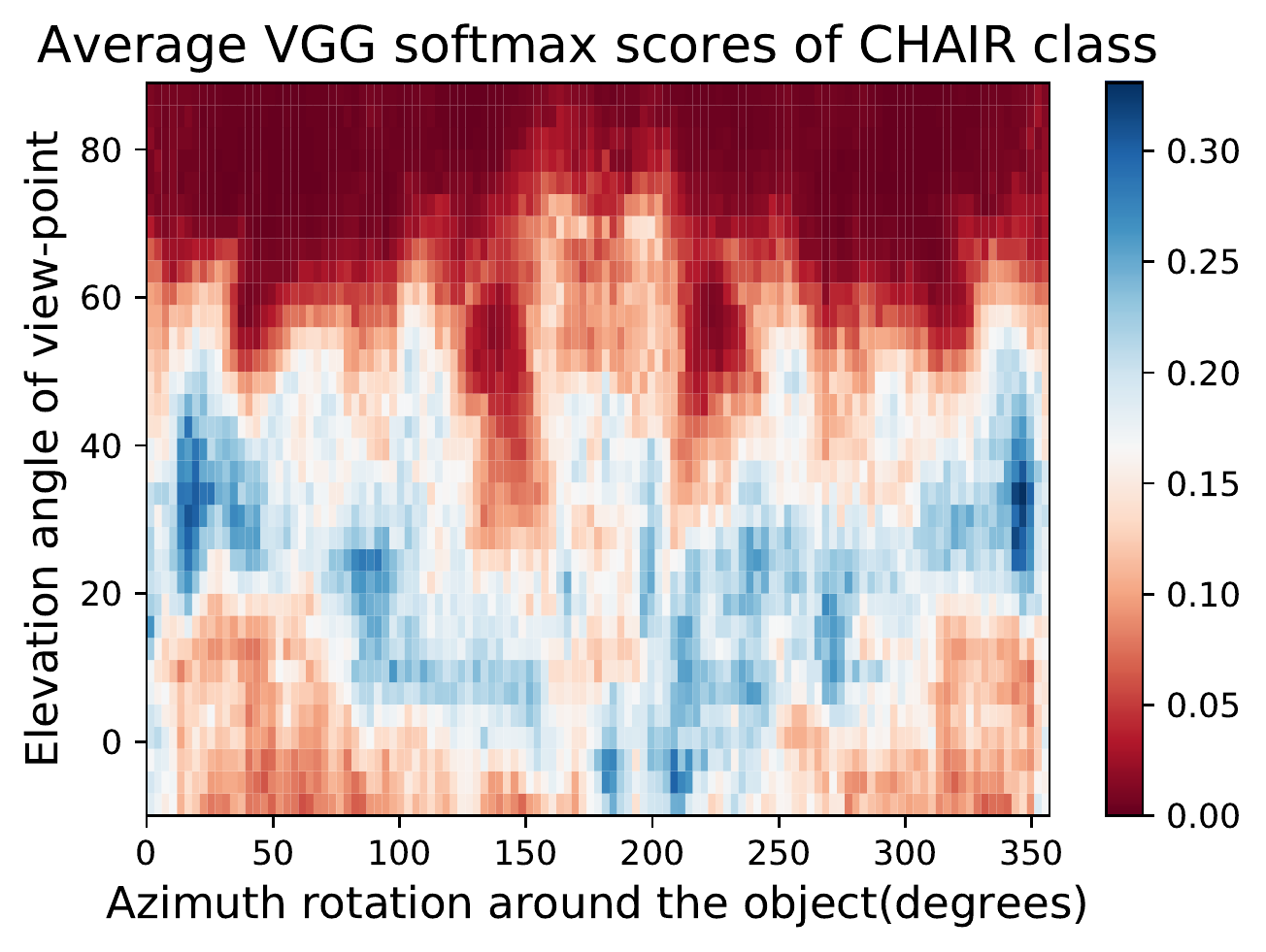}&
\includegraphics[width = 0.24\linewidth]{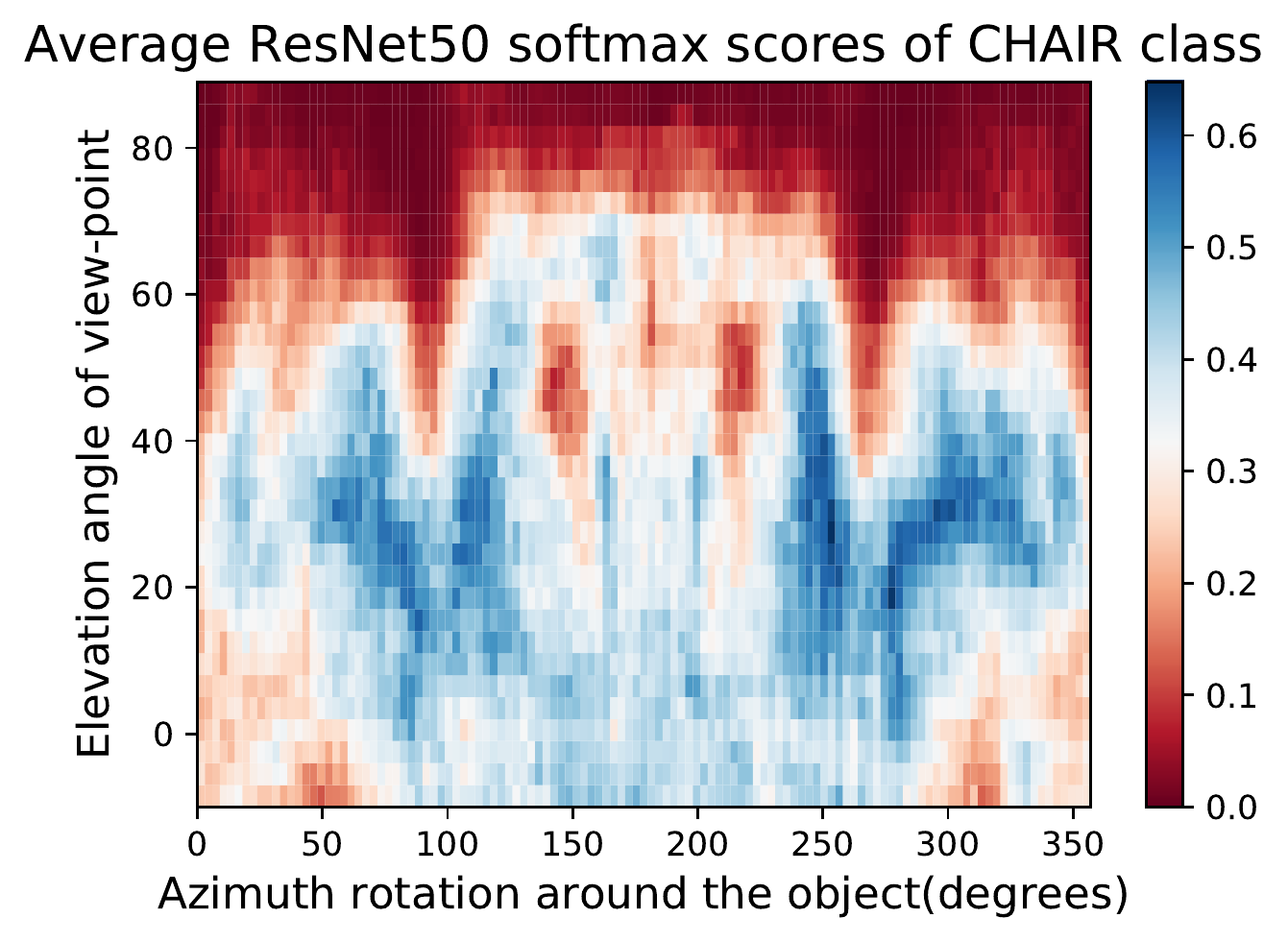}&
\includegraphics[width = 0.24\linewidth]{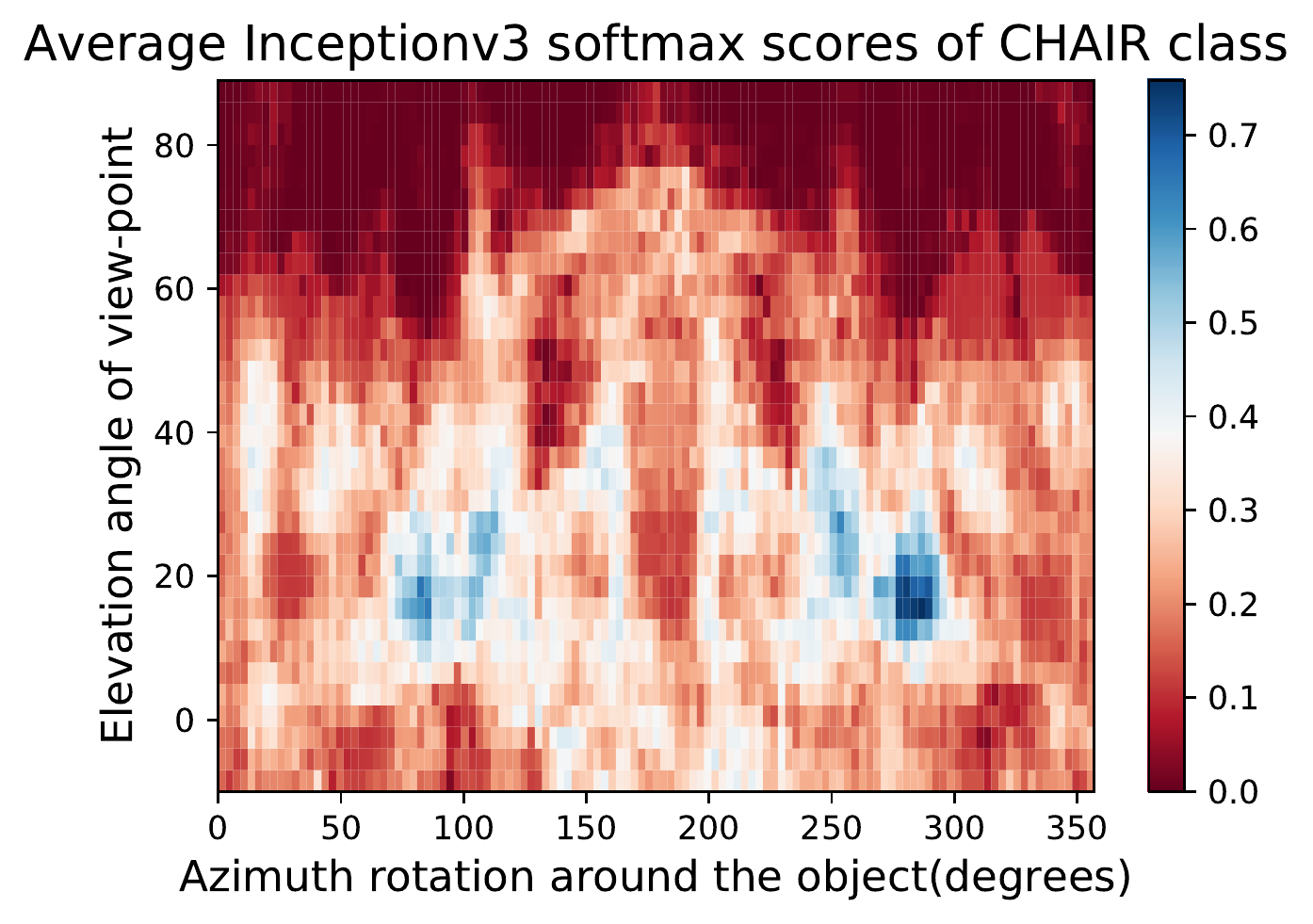} \\
\hline
\end{tabular}
   \caption{\small \textbf{2D Network Semantic Maps NMS-I}. Visualizing 2D Semantic Robustness profile for different networks averaged over 10 different shapes. Every row is different class. observe that different DNNs profiles differ depending on the training , accuracy , and network architectures that all result in a unique ''signatures" for the DNN on that class.}
   \vspace{-8pt}
   \label{fig:nsm2d-1}
\end{figure*}

\begin{figure*}[h]
\centering
\tabcolsep=0.03cm
   \begin{tabular}{||c|c|c|c||} \hline
   \textbf{AlexNet} & \textbf{VGG} &\textbf{ResNet50} & \textbf{InceptionV3} \\ \hline
\includegraphics[width = 0.24\linewidth]{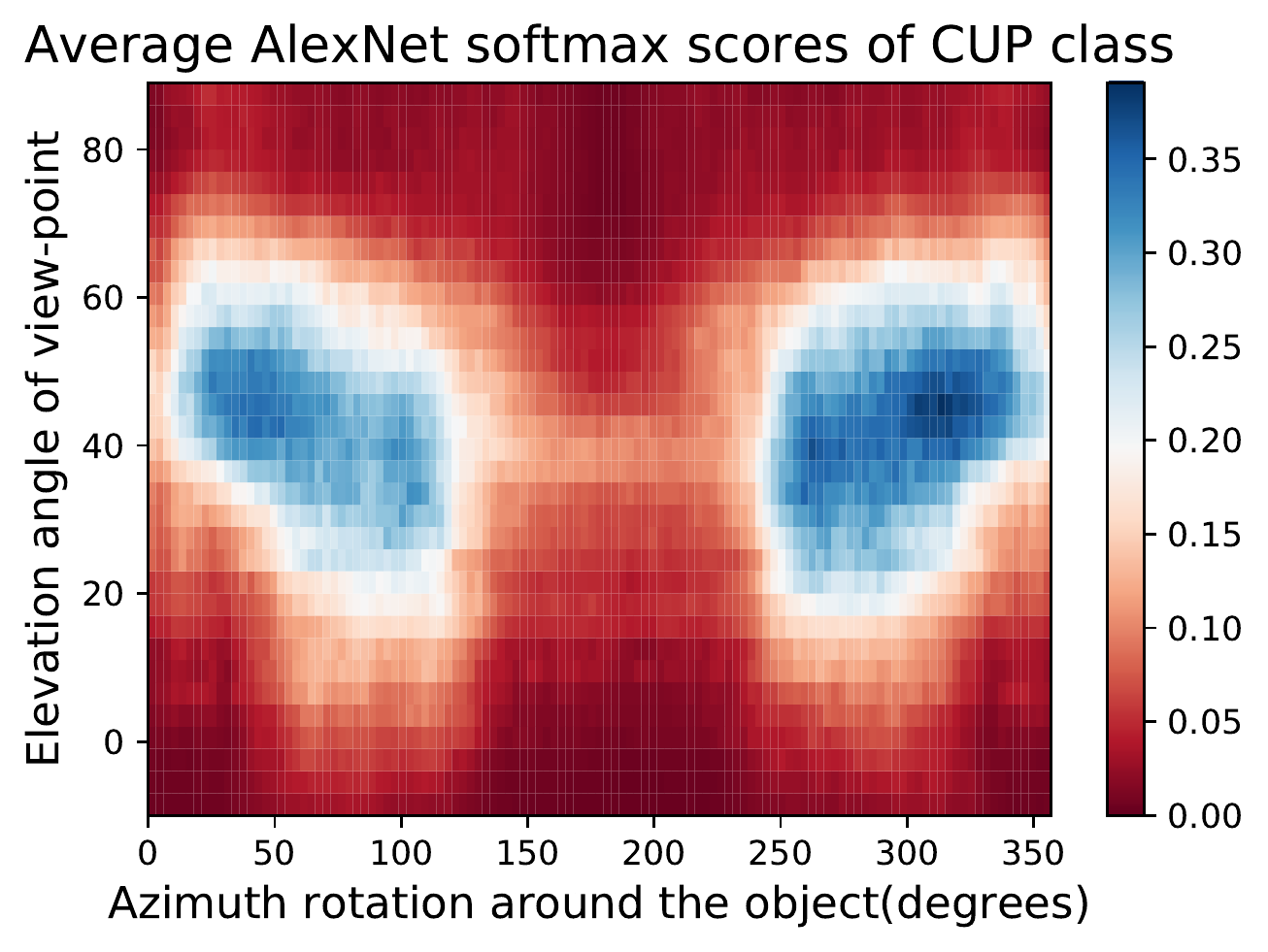}&
\includegraphics[width = 0.24\linewidth]{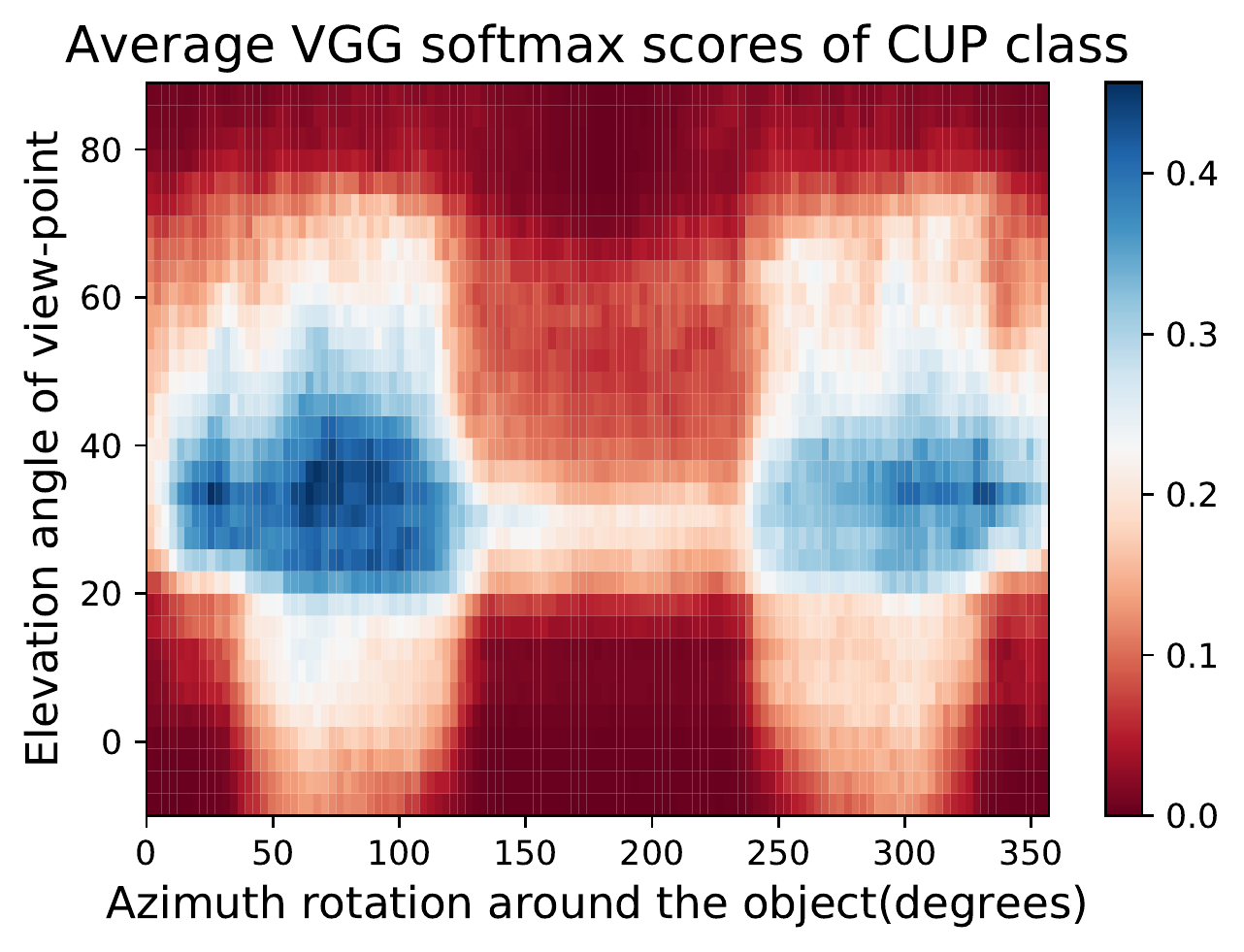}&
\includegraphics[width = 0.24\linewidth]{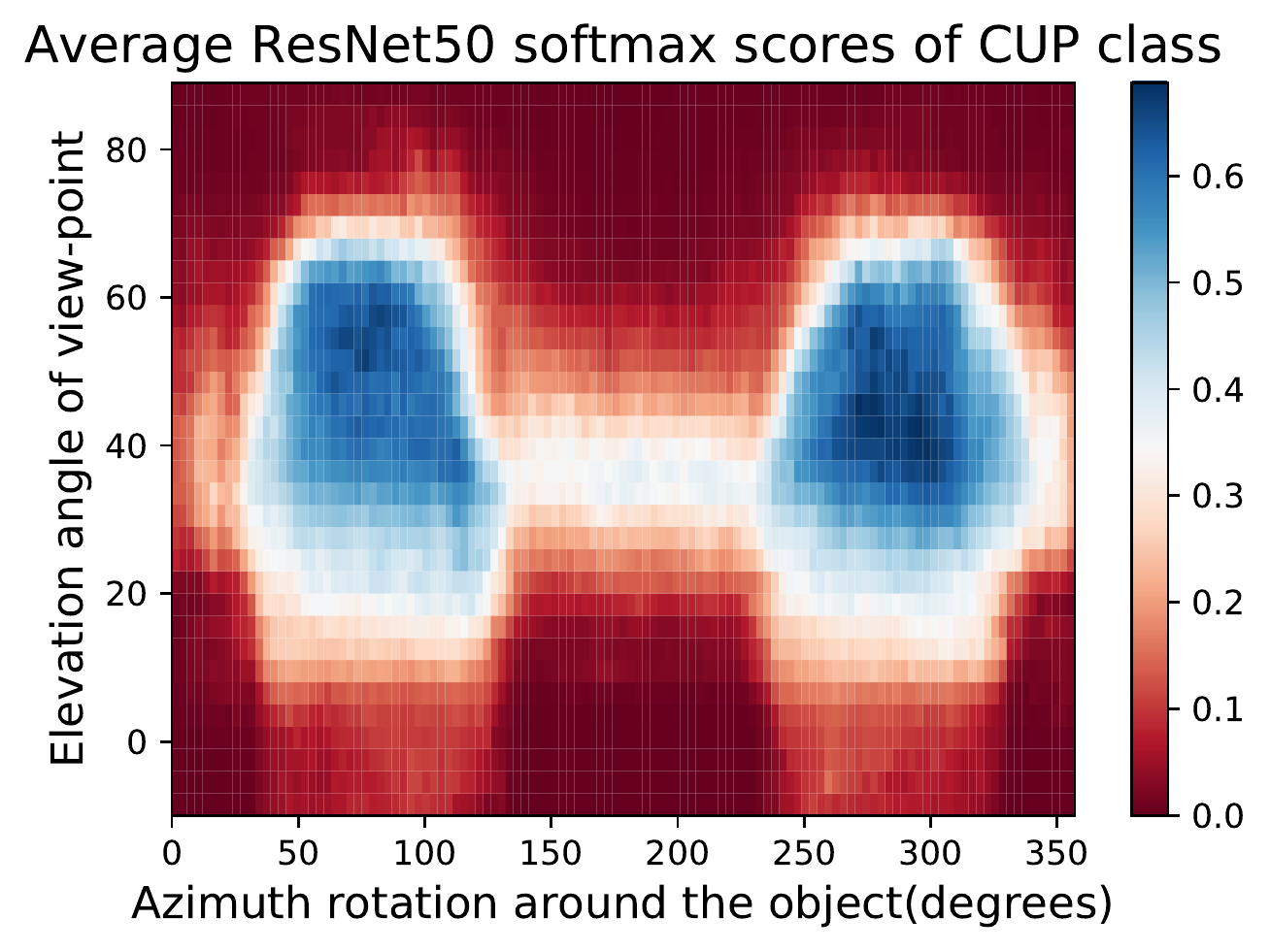}&
\includegraphics[width = 0.24\linewidth]{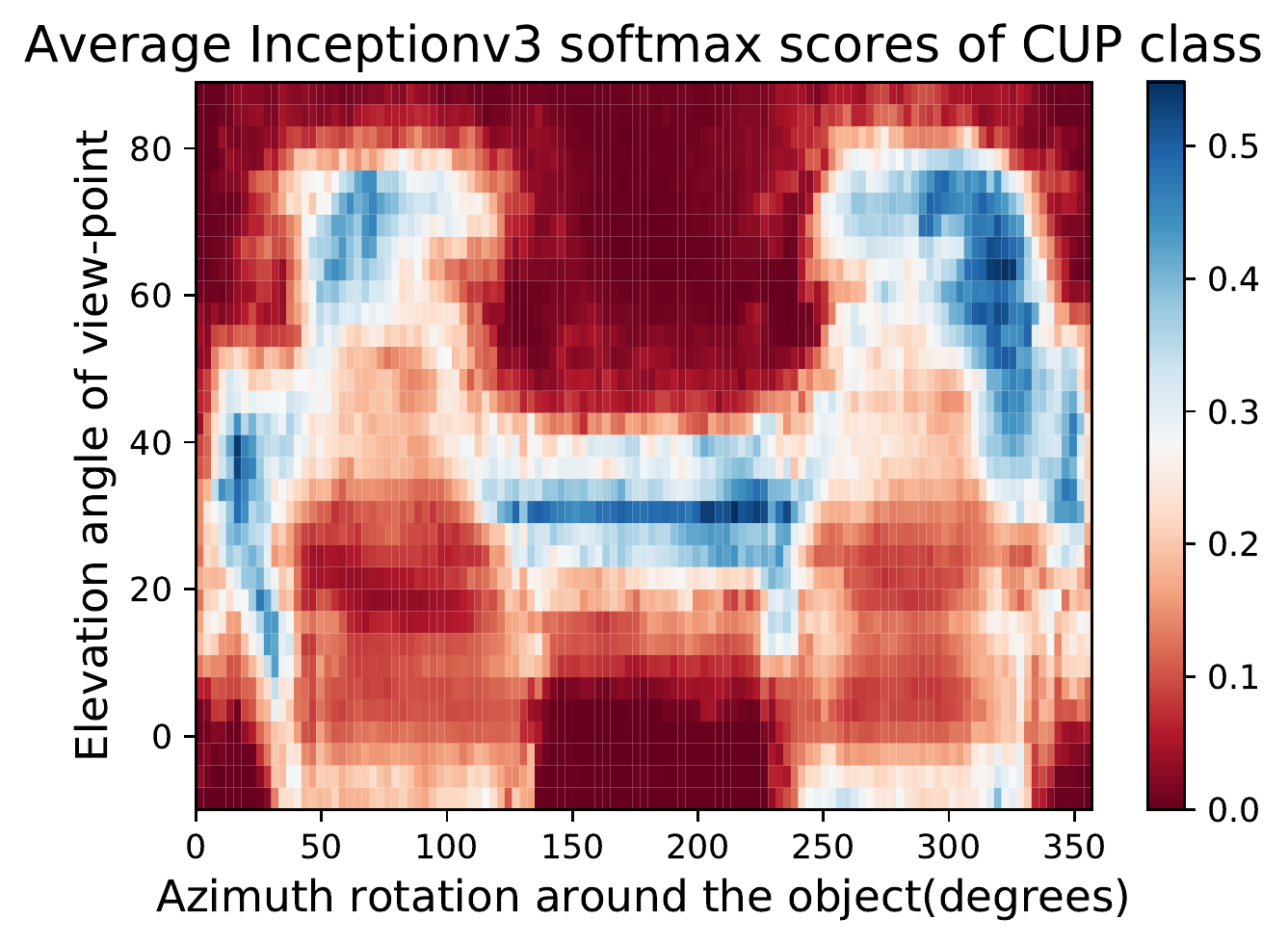}\\ \hline
\includegraphics[width = 0.24\linewidth]{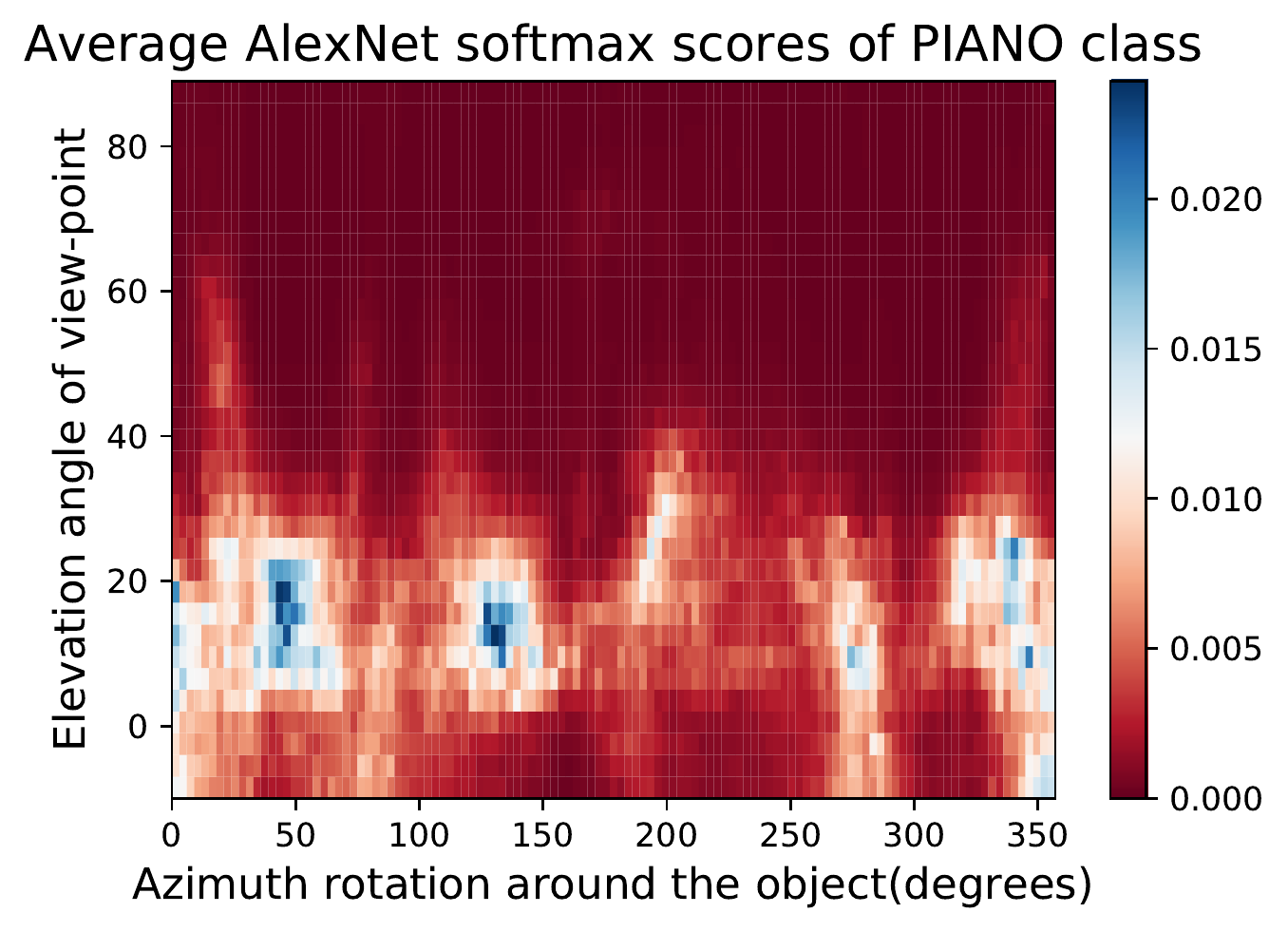}&
\includegraphics[width = 0.24\linewidth]{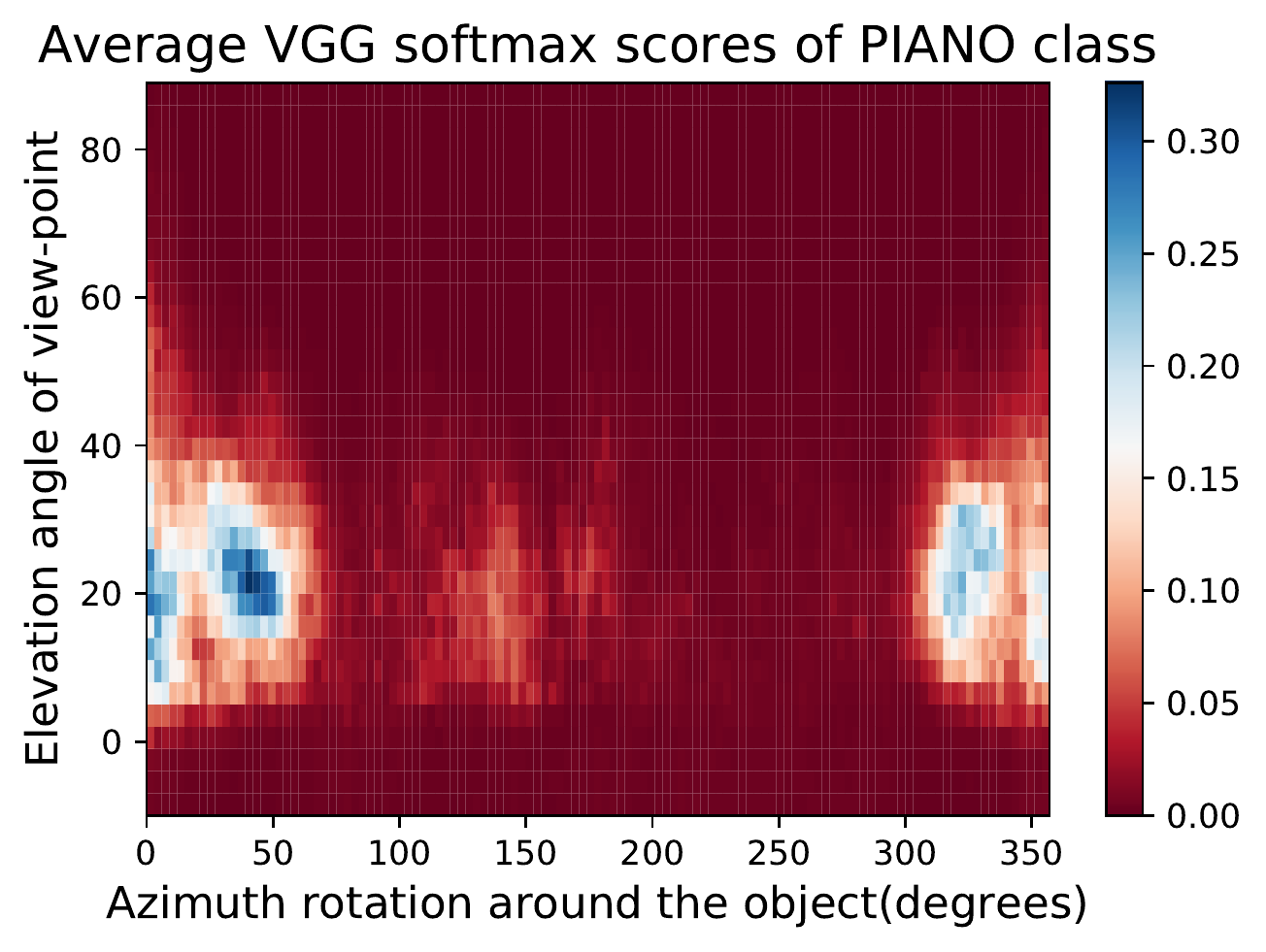}&
\includegraphics[width = 0.24\linewidth]{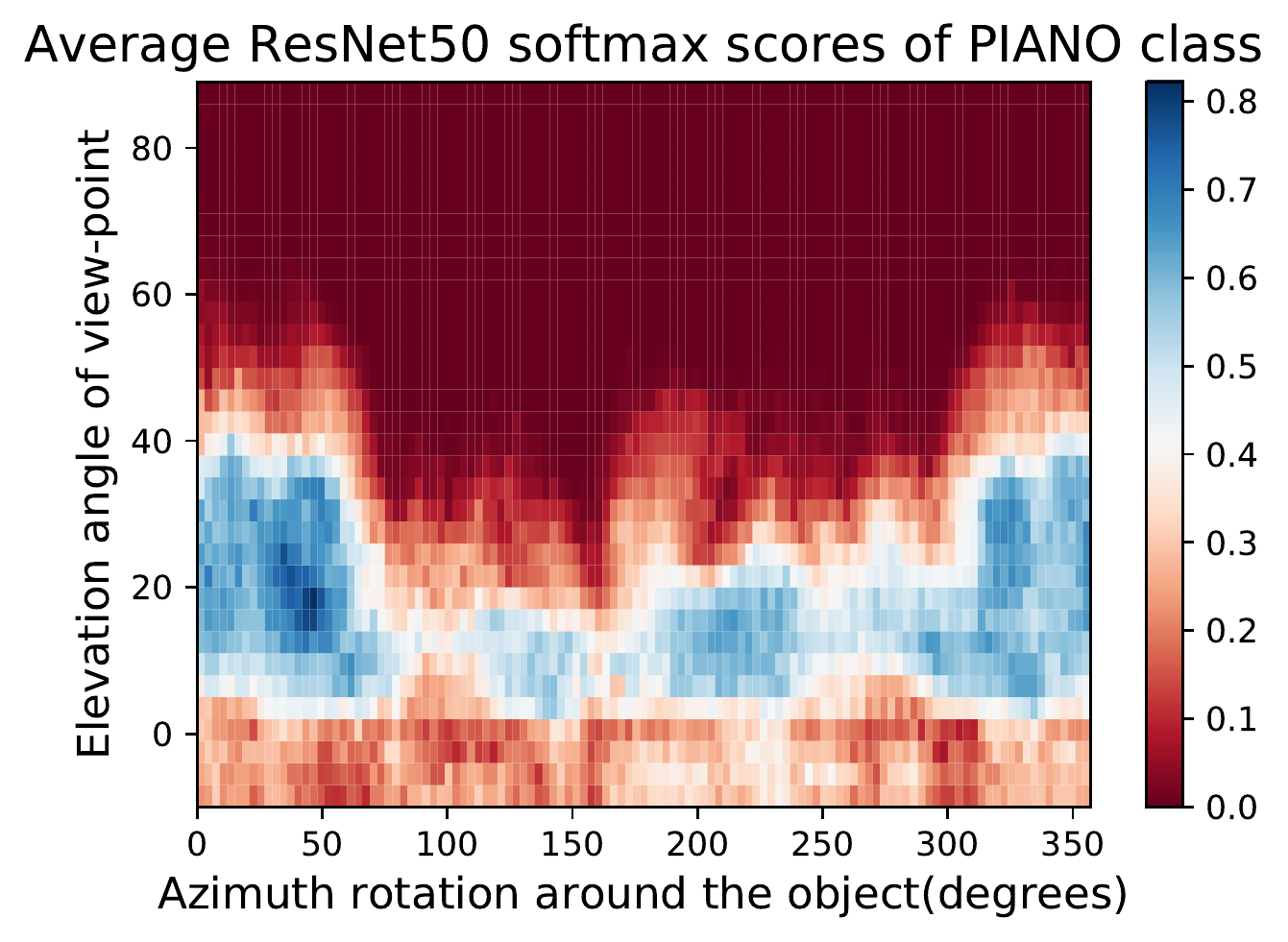}&
\includegraphics[width = 0.24\linewidth]{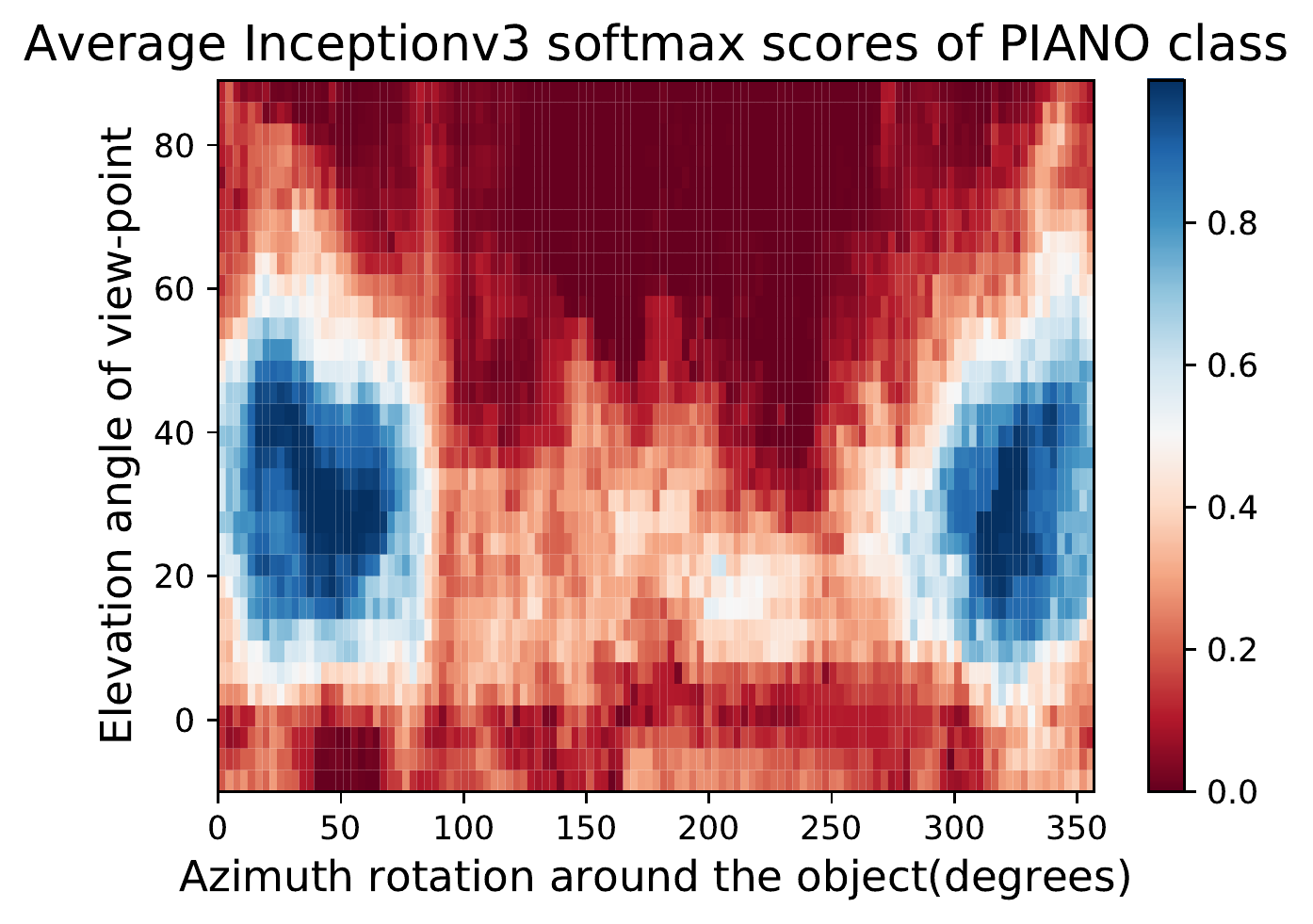}\\ \hline
\includegraphics[width = 0.24\linewidth]{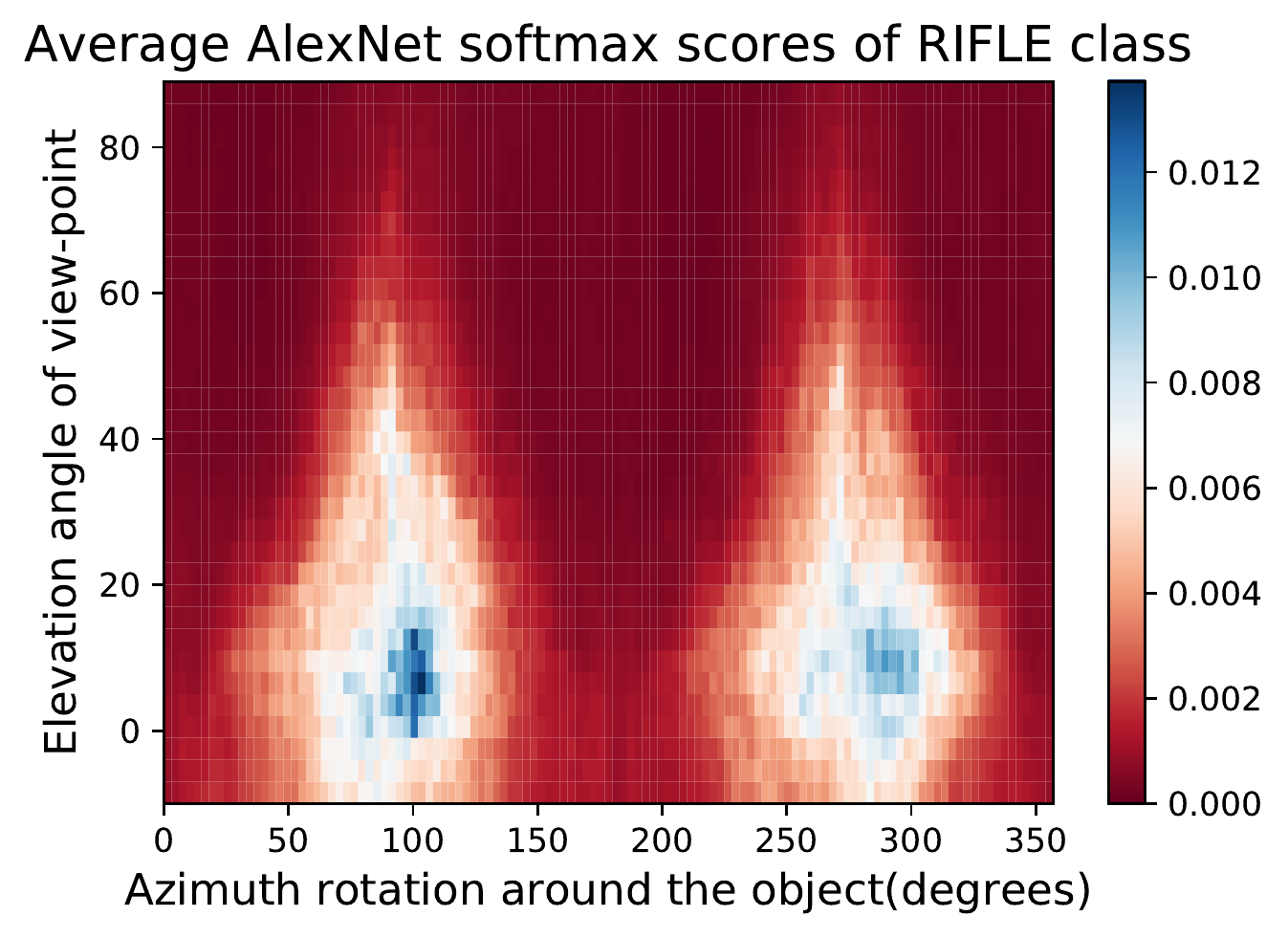}&
\includegraphics[width = 0.24\linewidth]{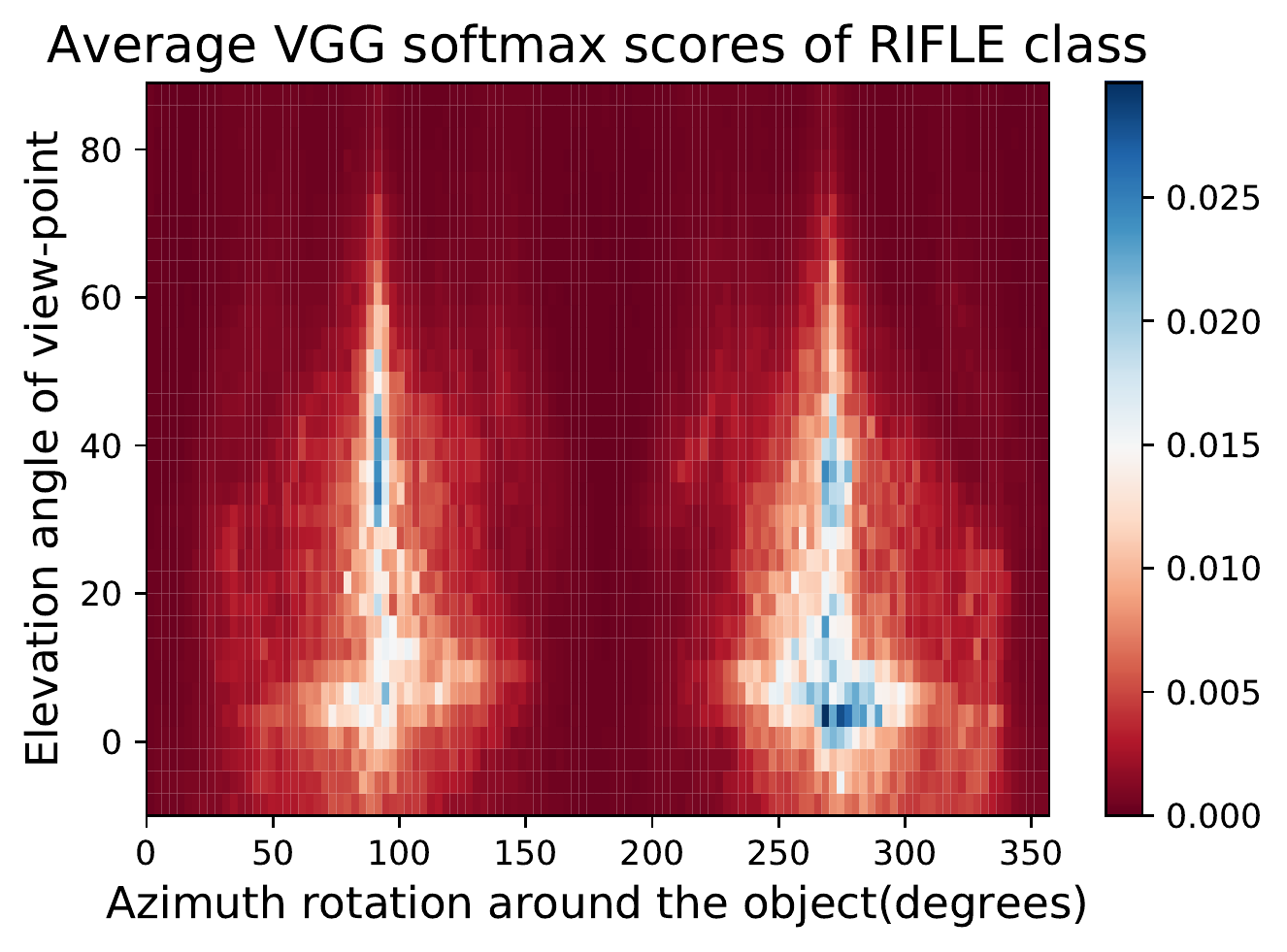}&
\includegraphics[width = 0.24\linewidth]{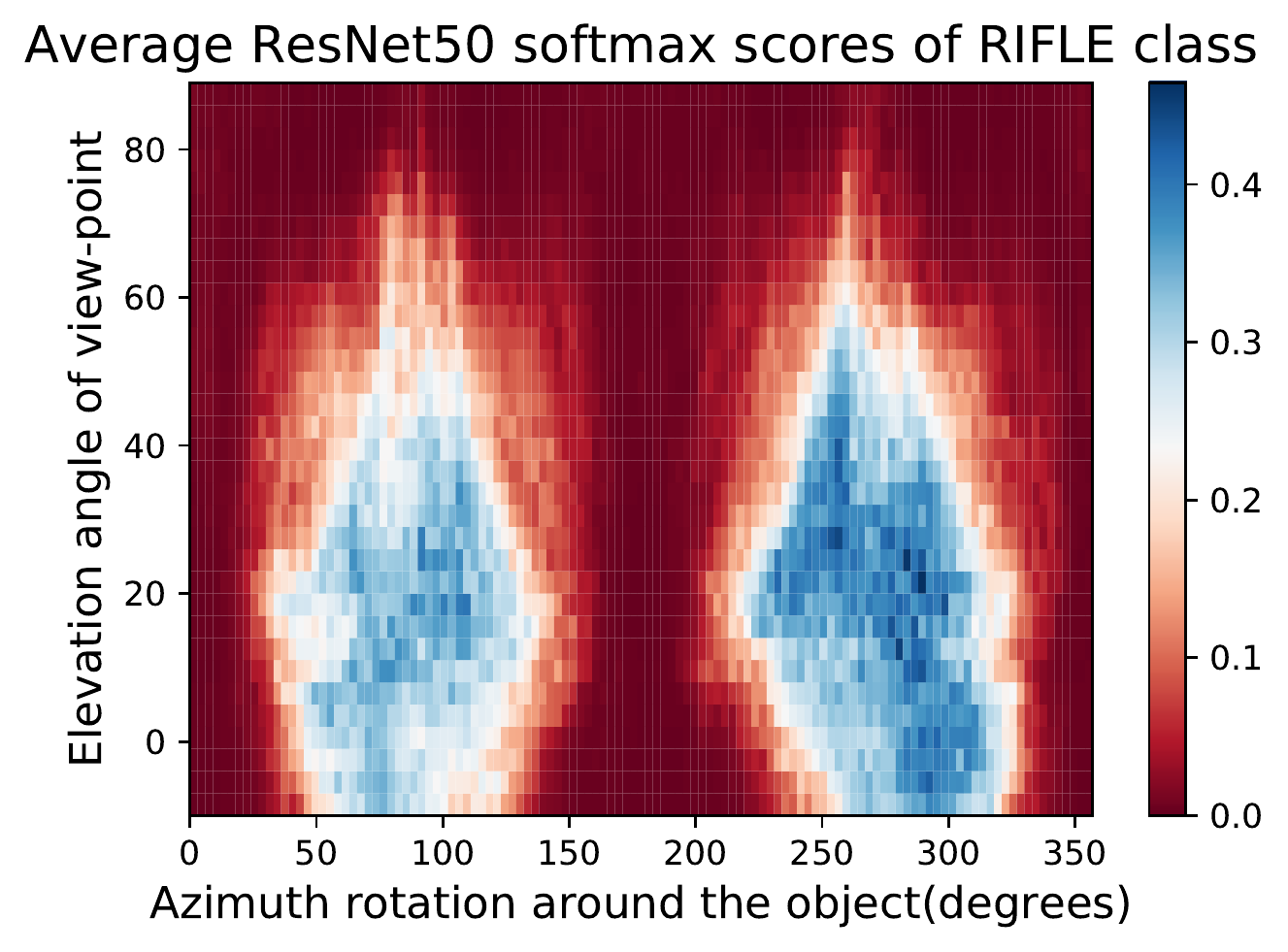}&
\includegraphics[width = 0.24\linewidth]{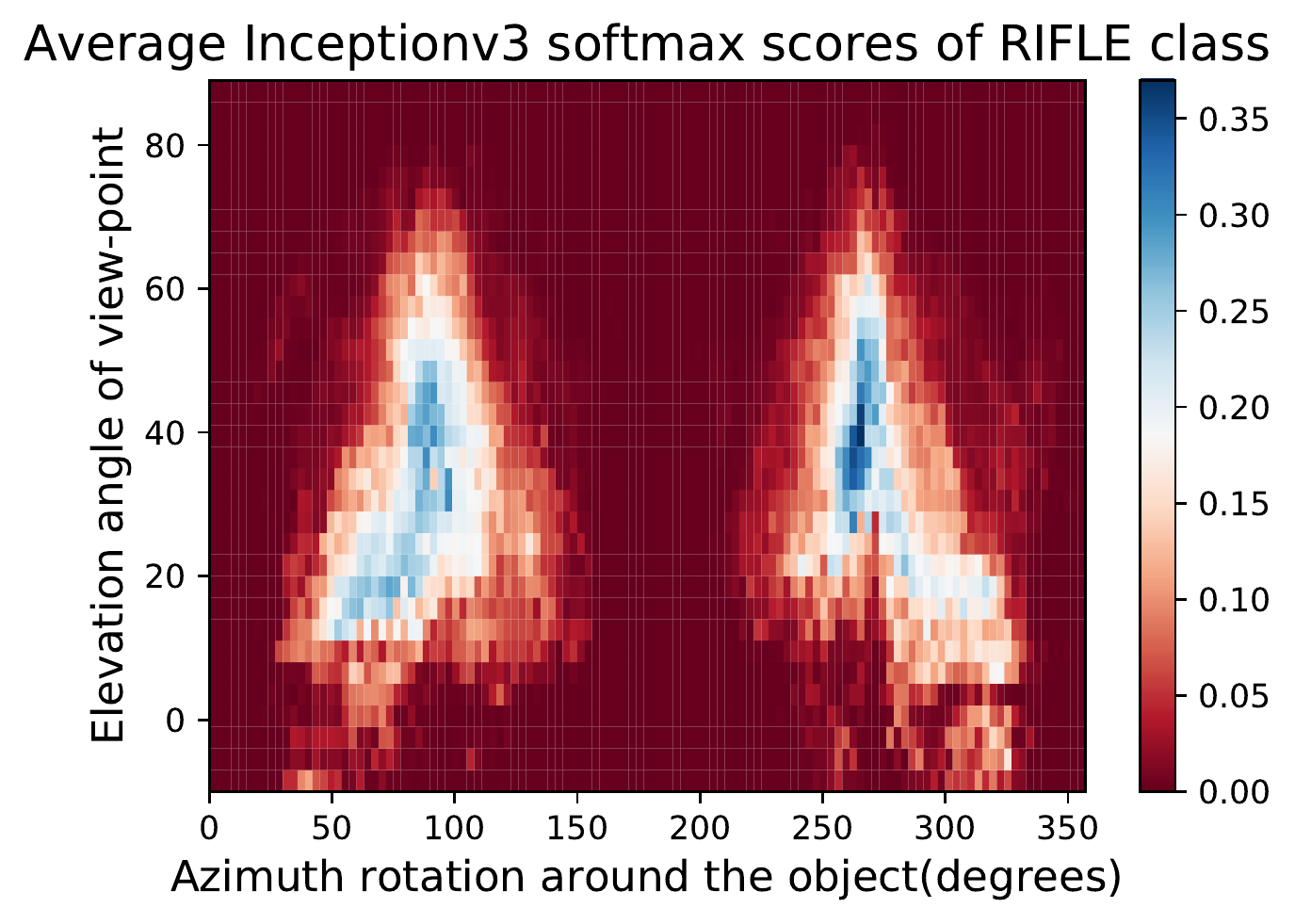}\\ \hline
\includegraphics[width = 0.24\linewidth]{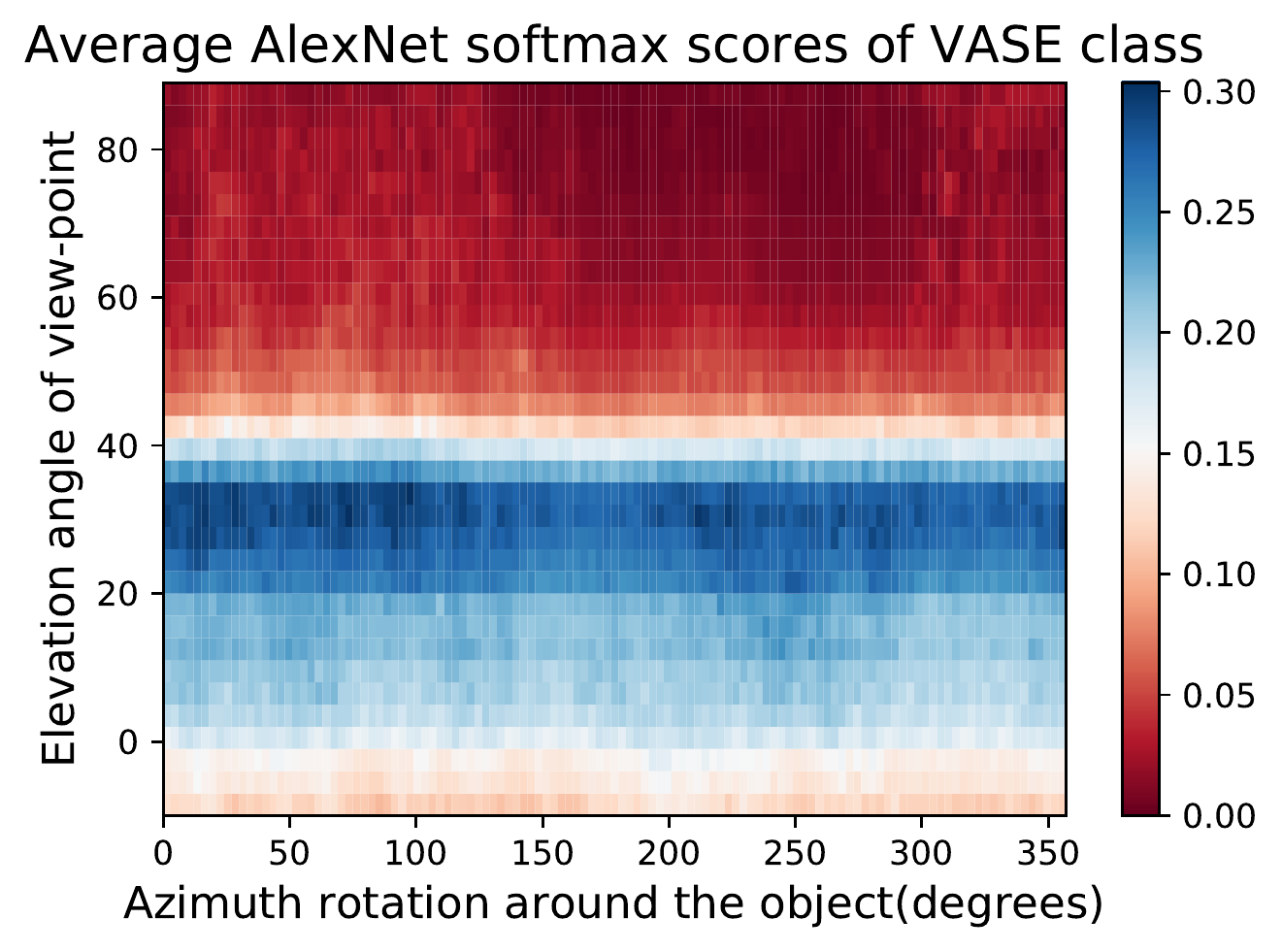}&
\includegraphics[width = 0.24\linewidth]{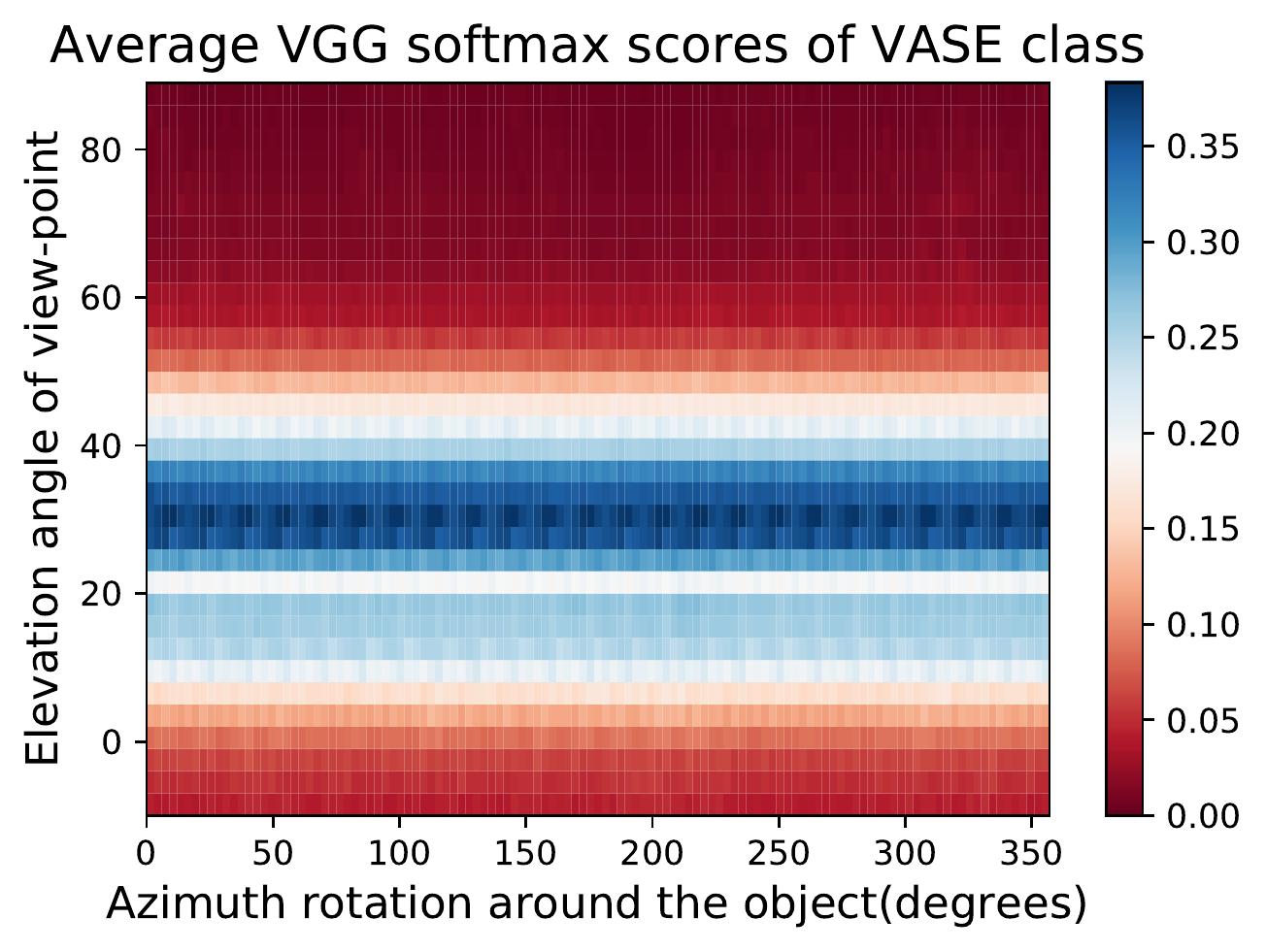}&
\includegraphics[width = 0.24\linewidth]{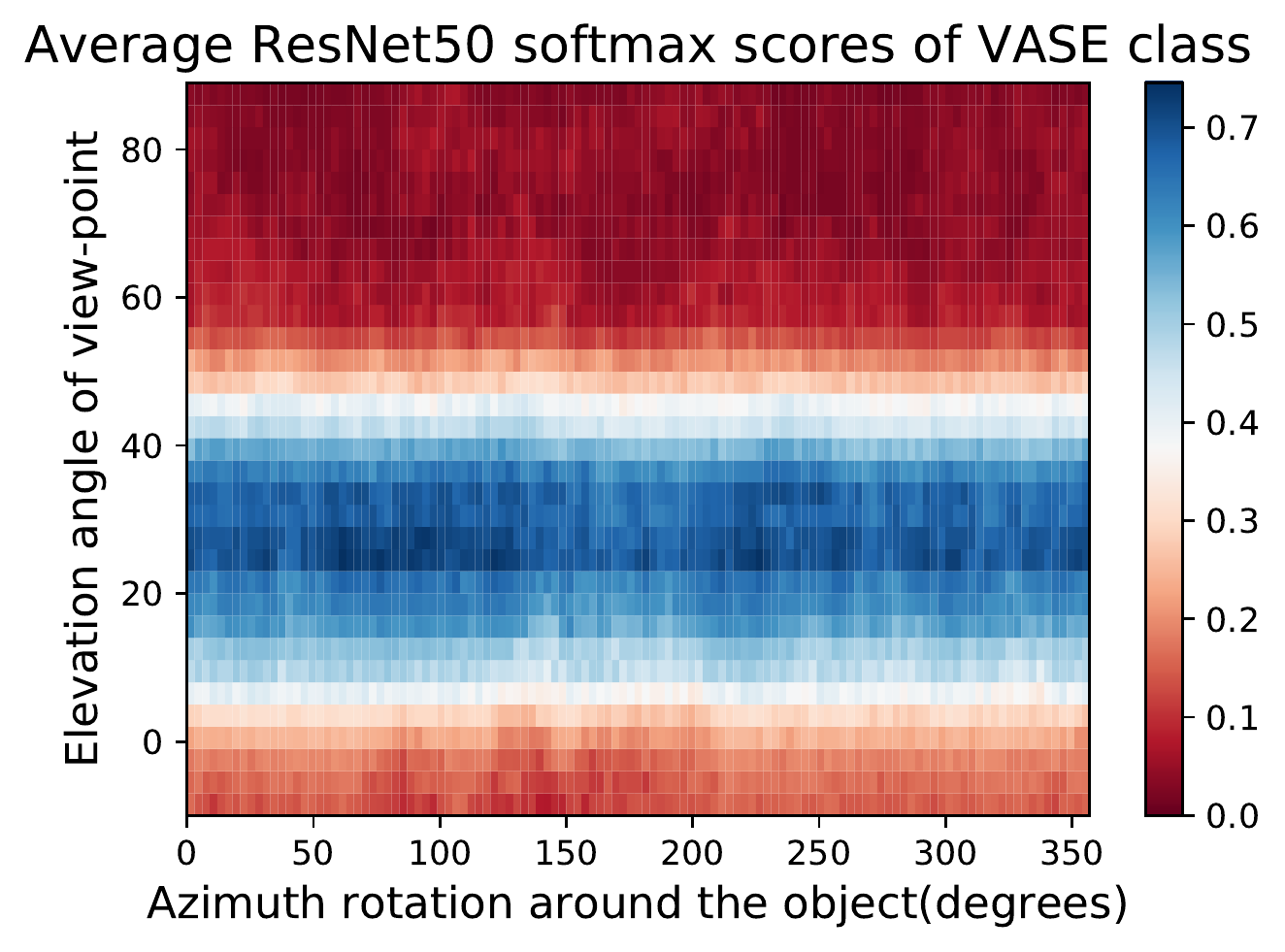}&
\includegraphics[width = 0.24\linewidth]{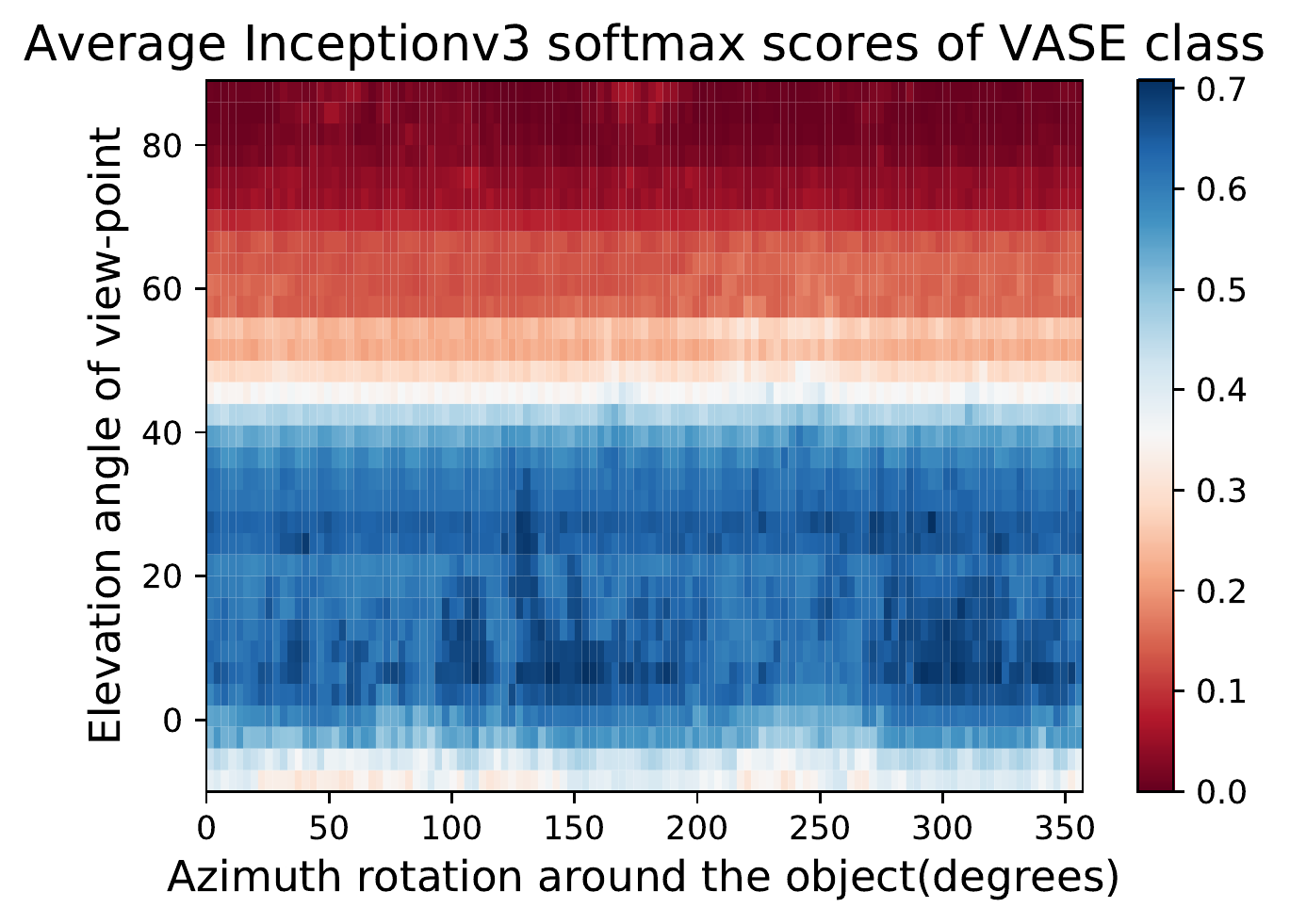}\\ \hline
\includegraphics[width = 0.24\linewidth]{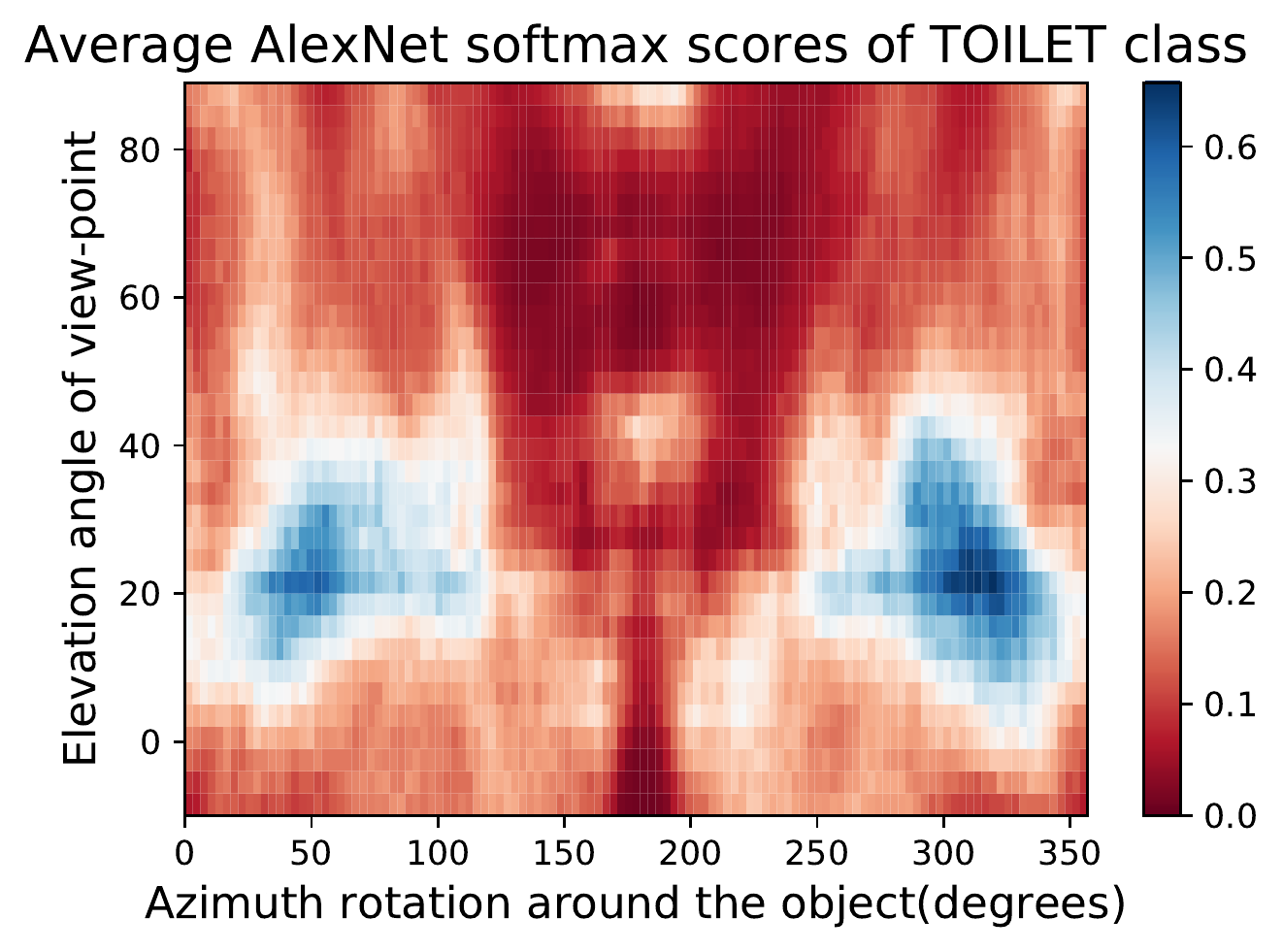}&
\includegraphics[width = 0.24\linewidth]{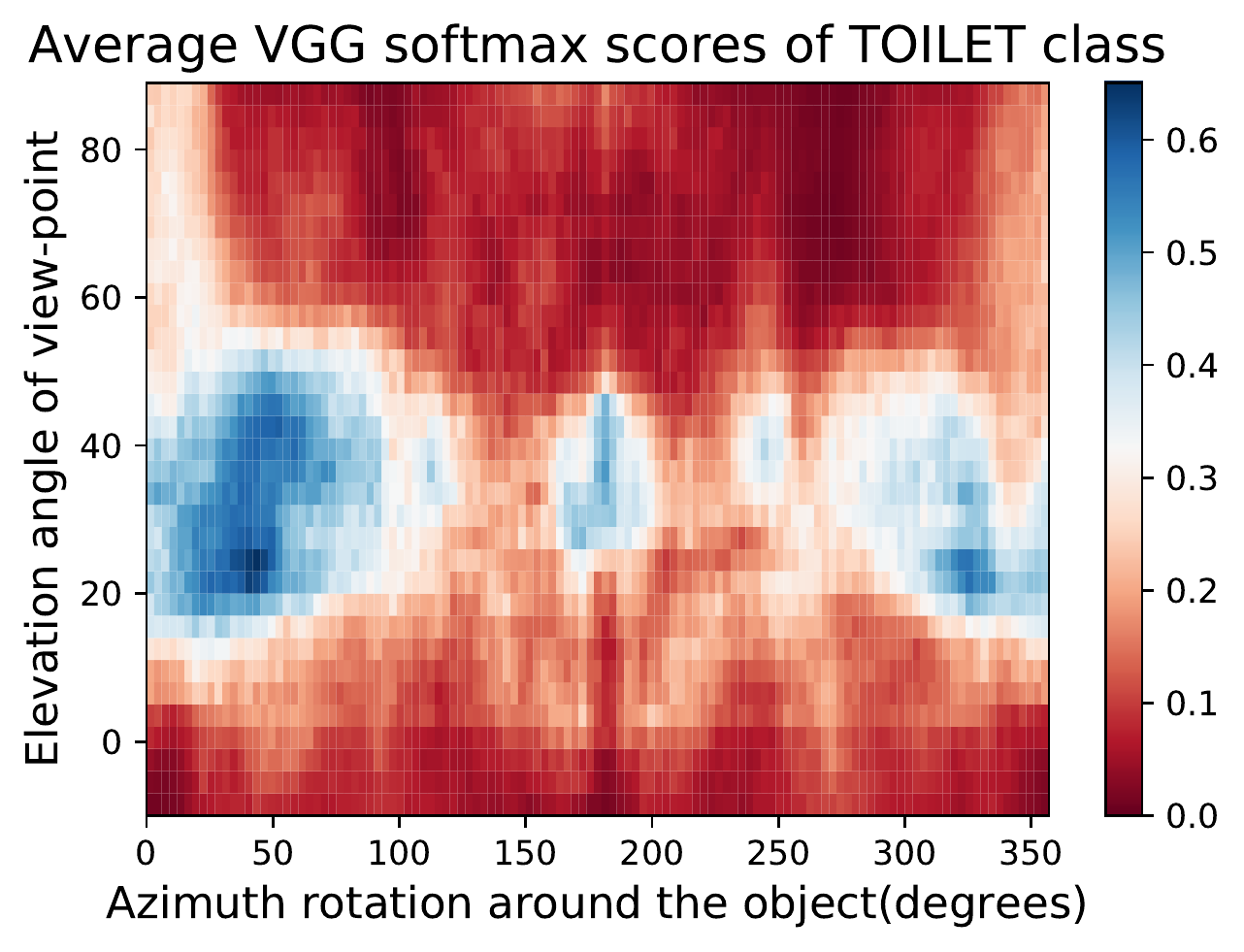}&
\includegraphics[width = 0.24\linewidth]{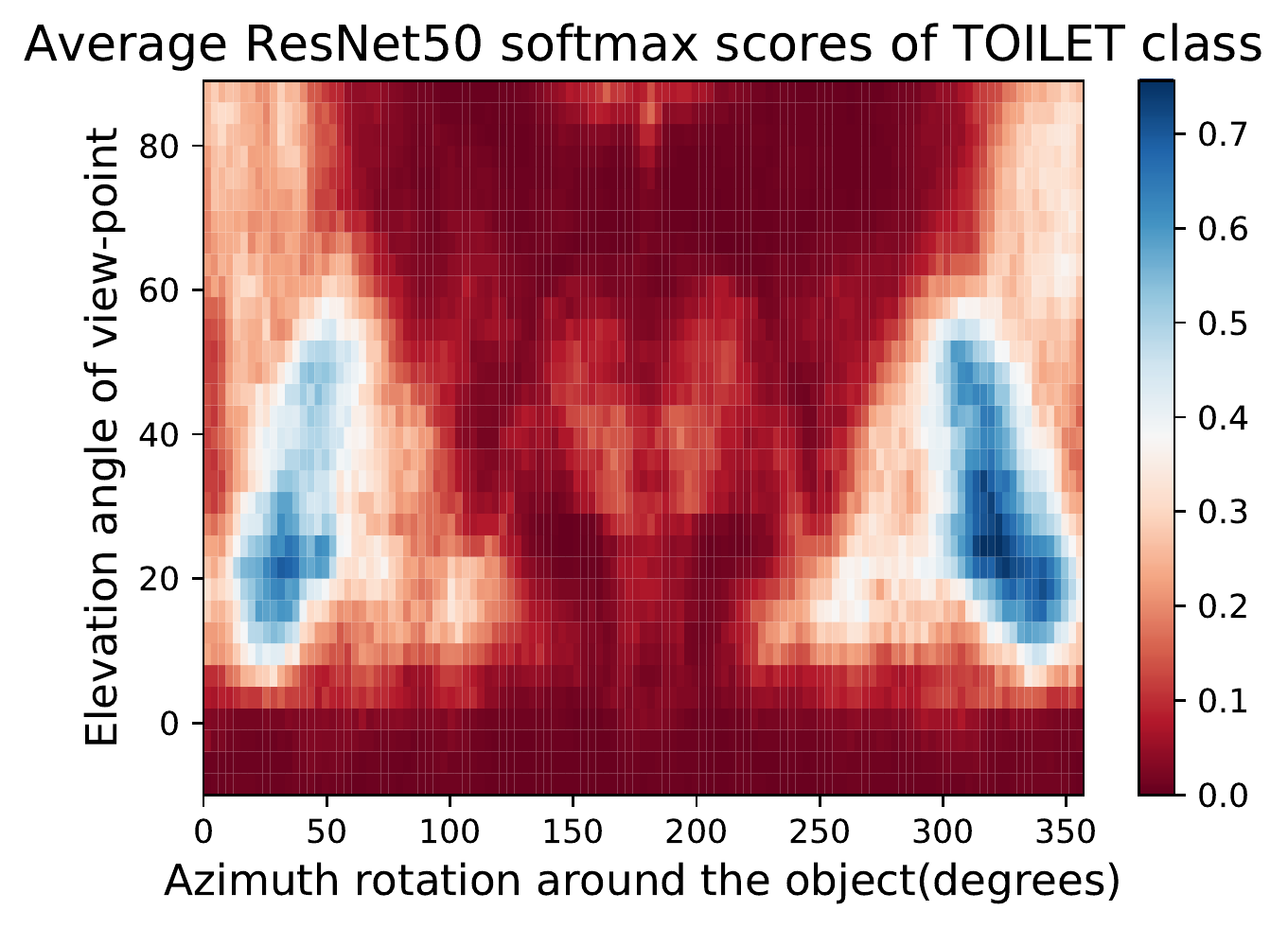}&
\includegraphics[width = 0.24\linewidth]{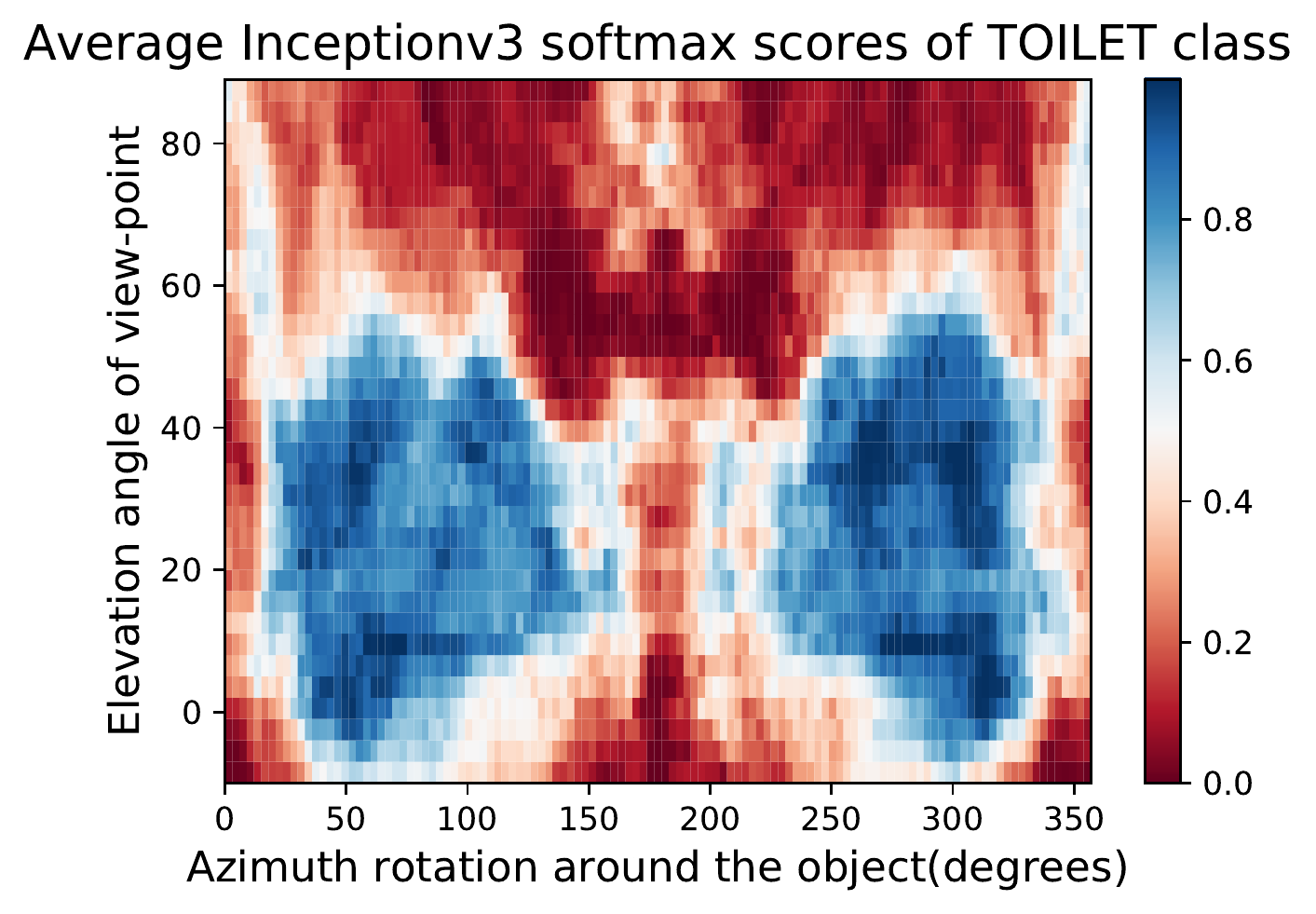} \\
\hline
\end{tabular}
   \caption{\small \textbf{2D Network Semantic Maps NMS-II}: visualizing 2D Semantic Robustness profile for different networks averaged over 10 different shapes. Every row is different class. observe that different DNNs profiles differ depending on the training , accuracy , and network architectures that all result in a unique ''signatures" for the DNN on that class.}
   \vspace{-8pt}
   \label{fig:nsm2d-2}
\end{figure*}

\subsection{Convergence of the Region Finding Algorithms}
Here we show how when we apply the region detection algorithms; the naive detect the smallest region while the OIR formulations detect bigger more general robust region. This result happens even with different initial points; they always converge to the same bounds of that robust region of the semantic maps. \figLabel{\ref{fig:converge}} show 4 different initializations for 1D case  along with predicted regions. In \figLabel{\ref{fig:conv1},\ref{fig:conv2}} shows the bounds evolving during the optimization of the three algorithms (naive , OIR\_B and OIR\_W ) for 500 steps.
\begin{figure*}[h]
\centering
\tabcolsep=0.03cm
\includegraphics[width = \textwidth]{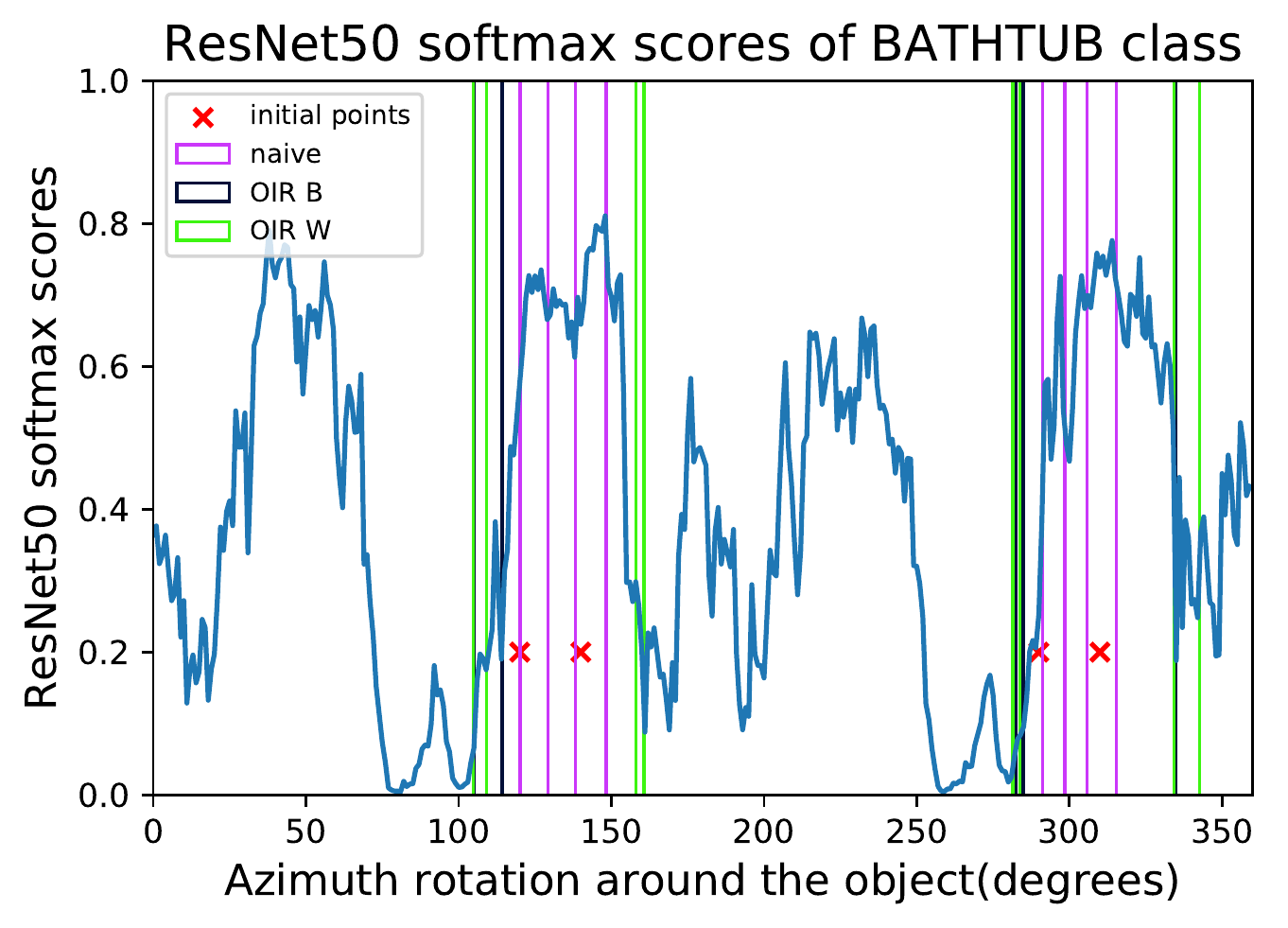}
   \caption{\small \textbf{Robust Region Detection with different initializations}: visualizing the Robust Regions Bounds found by the three algorithms for four initial points. We can see that the naive produce different bounds of the same region for different initializations while OIR detect the same region regardless of initialization.}
   \vspace{-8pt}
   \label{fig:converge}
\end{figure*}
\begin{figure*}[h]
\centering
\tabcolsep=0.03cm
   \begin{tabular}{c|c}  \hline
   \textbf{initialization = 120} & \textbf{initialization = 140} \\  \hline 
   \multicolumn{2}{c}{\textbf{Naive algorithm}} \\   
\includegraphics[width = 0.49\linewidth]{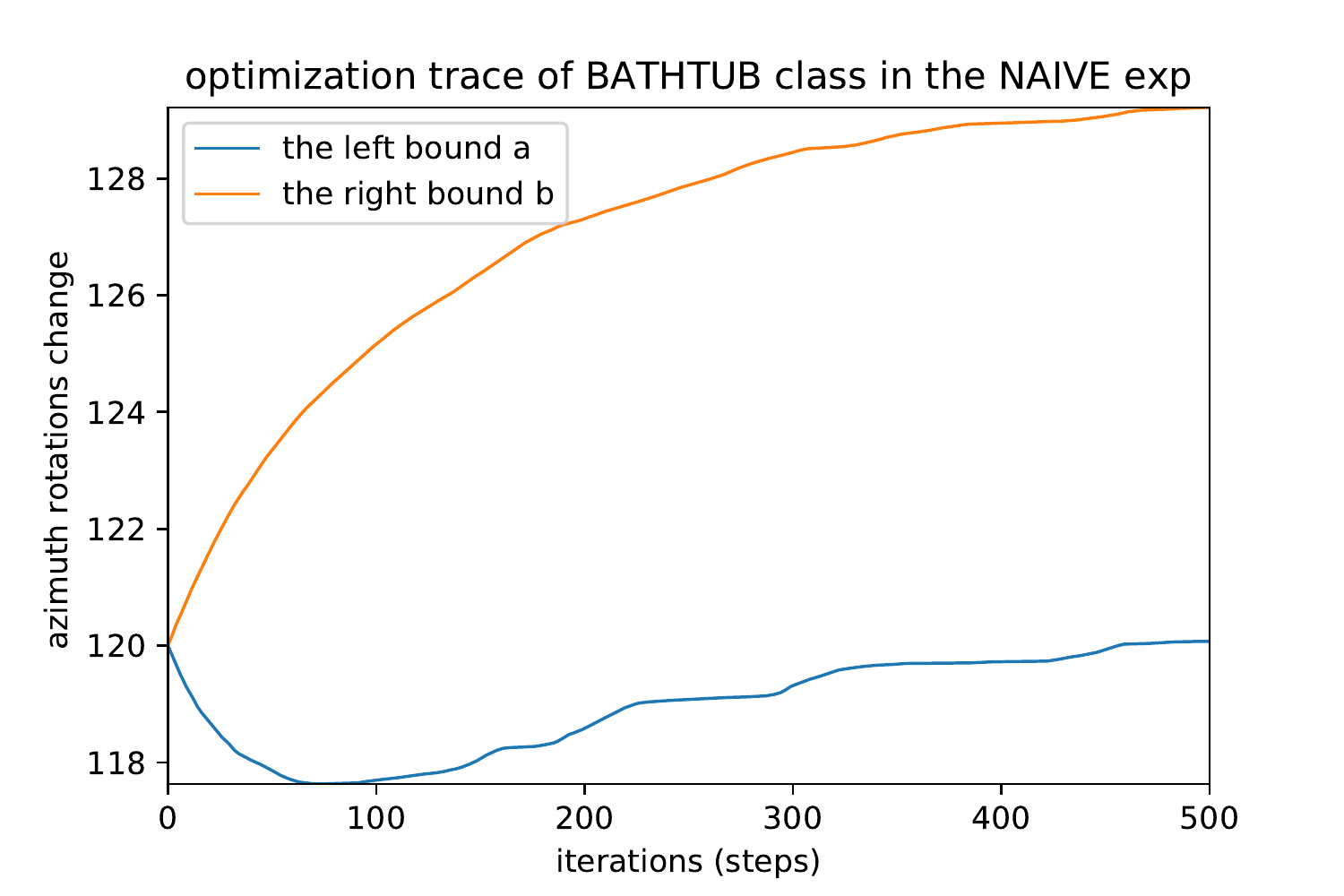} &
\includegraphics[width = 0.49\linewidth]{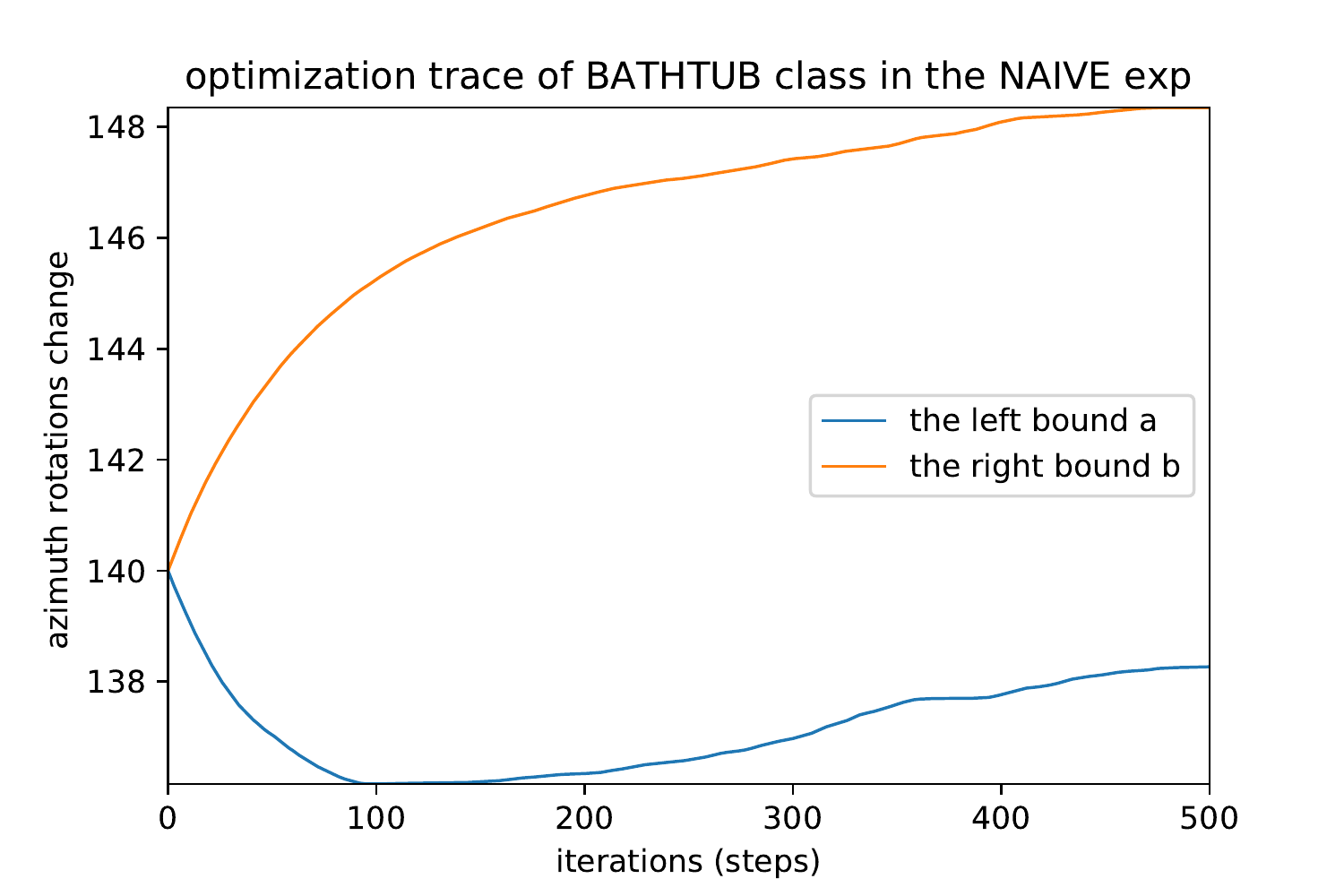} \\   \hline
\multicolumn{2}{c}{\textbf{Black-Box OIR algorithm}} \\  
\includegraphics[width = 0.49\linewidth]{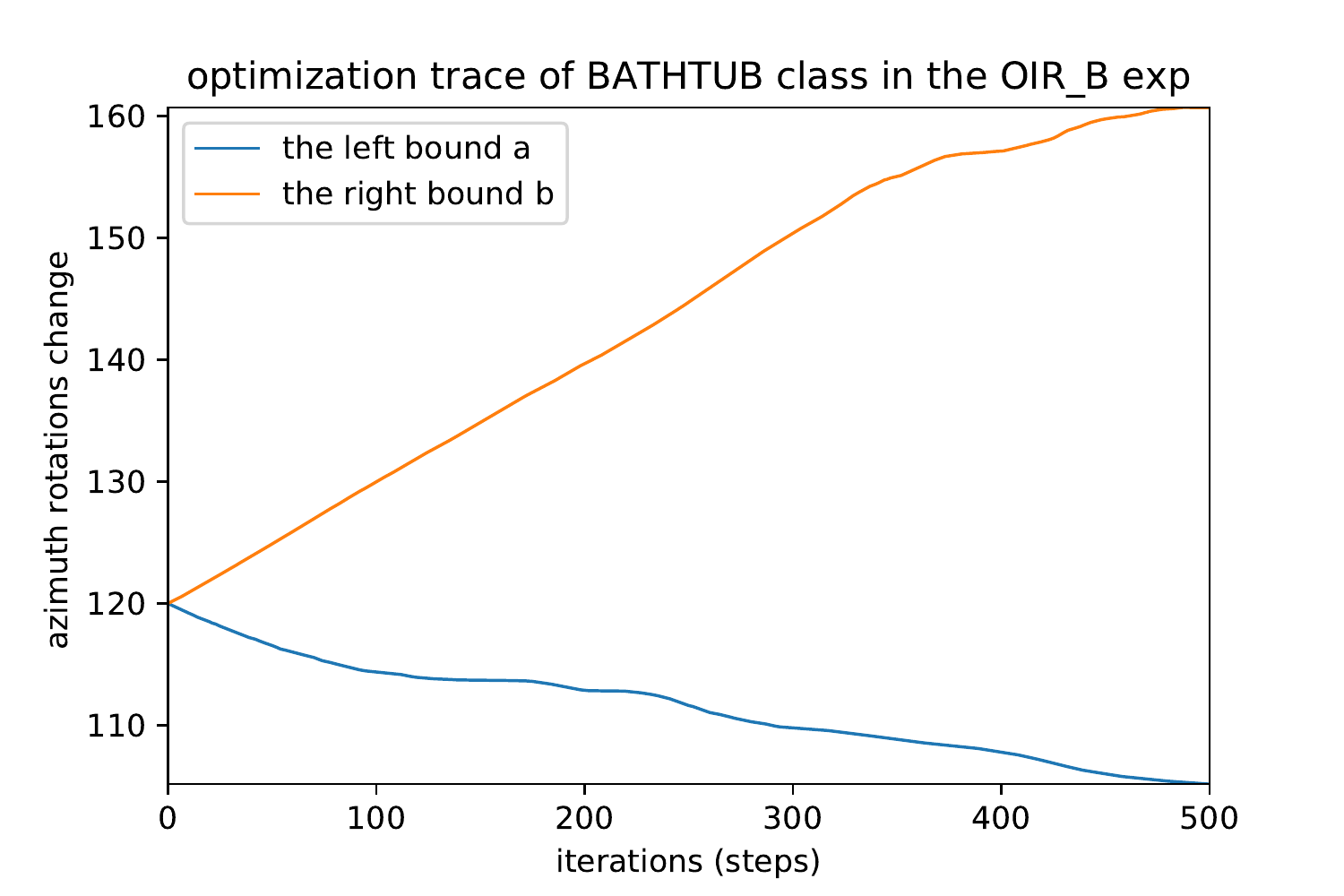} &
\includegraphics[width = 0.49\linewidth]{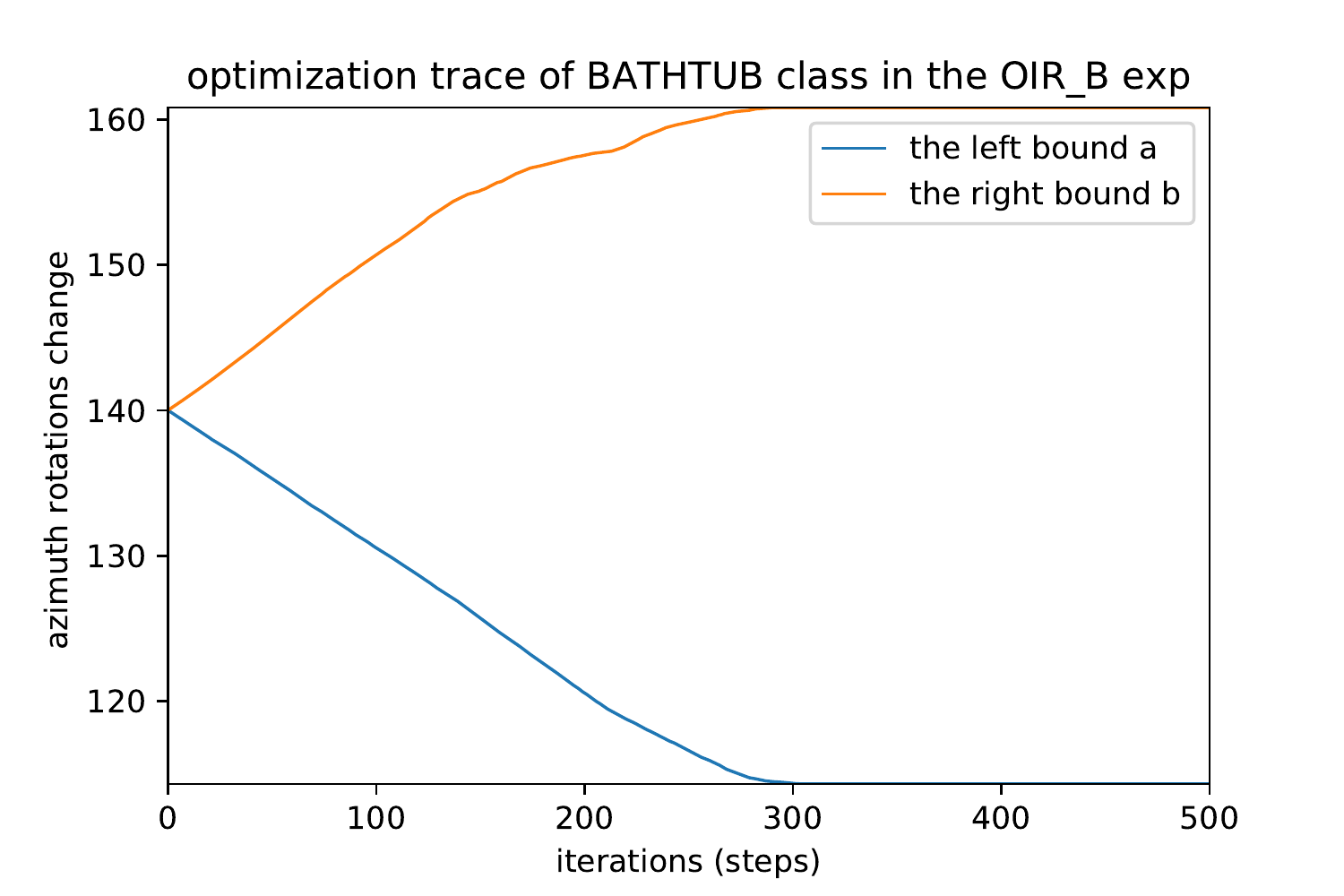} \\   \hline
\multicolumn{2}{c}{\textbf{White-Box OIR algorithm}} \\  
\includegraphics[width = 0.49\linewidth]{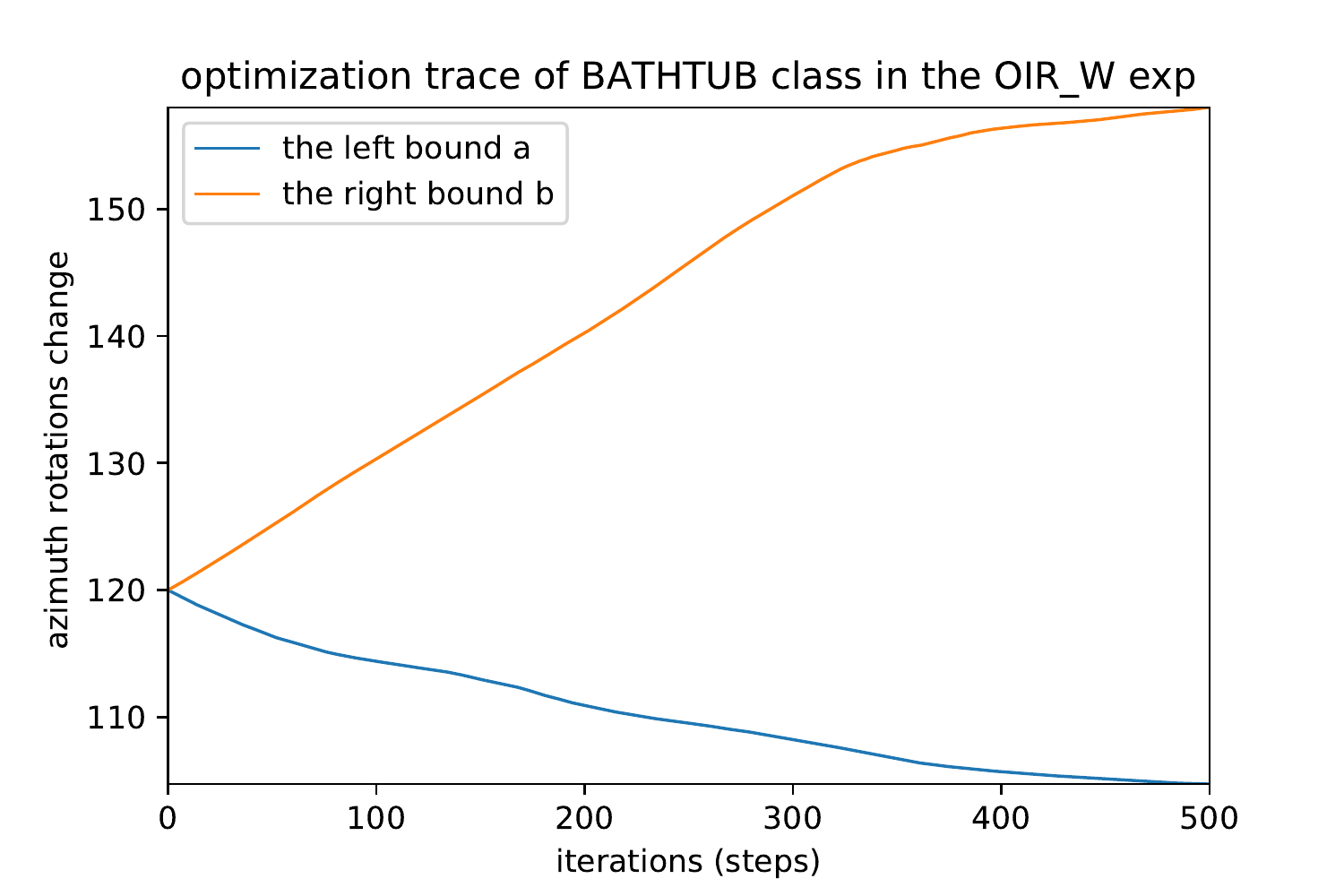} &
\includegraphics[width = 0.49\linewidth]{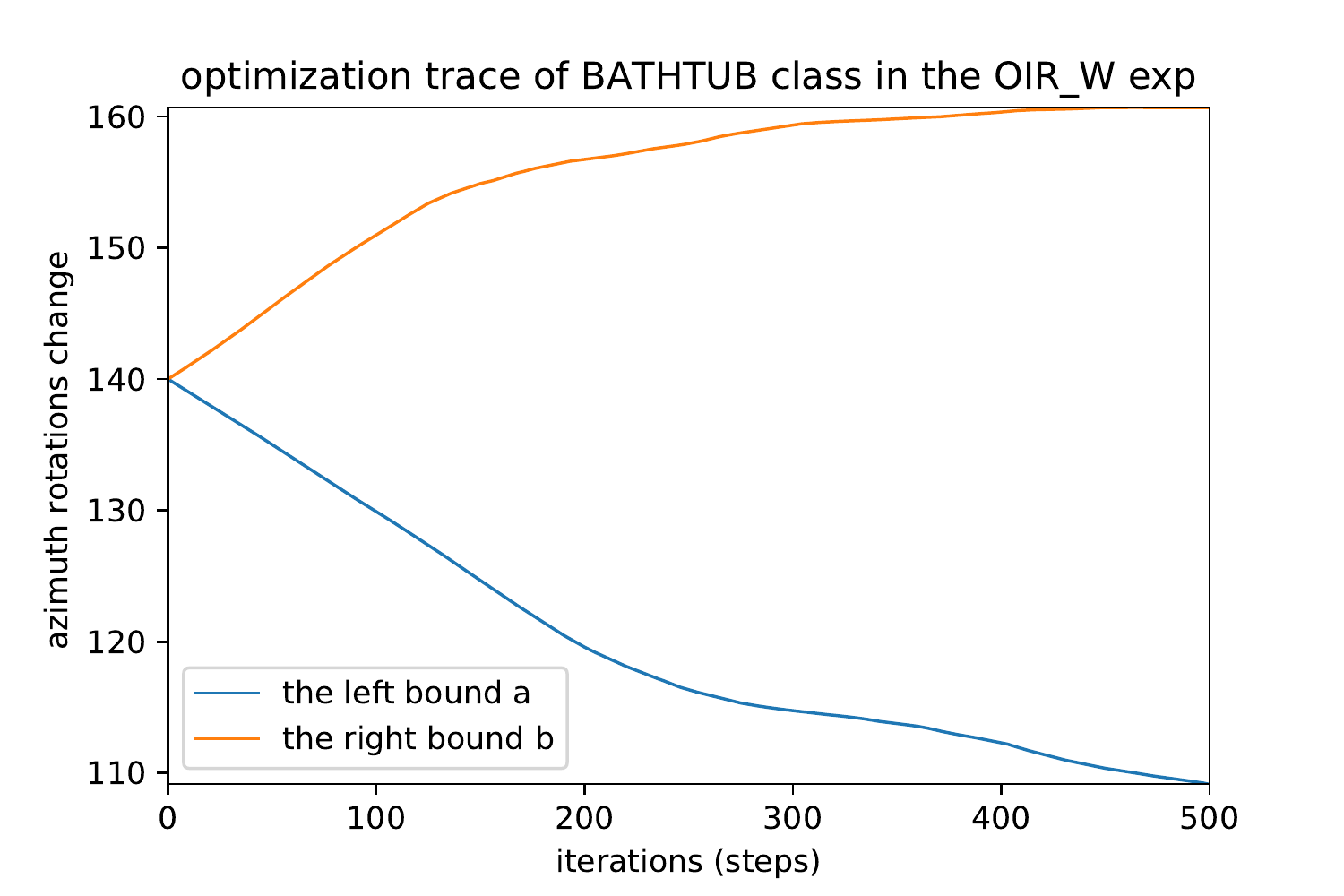} \\   \hline
\end{tabular}
   \caption{\small \textbf{Robust Region Bounds Growing I}: visualizing the bounds growing Using different algorithms from two different initializations (120 and 140) in \figLabel{\ref{fig:converge}}. We can see that OIR formulations converge to the same bounds of the robust region regardless of the initialization, which indicates effectiveness in detecting these regions, unlike the naive approach which can stop in a local optimum. }
   \vspace{-8pt}
   \label{fig:conv1}
\end{figure*}

\begin{figure*}[h]
\centering
\tabcolsep=0.03cm
   \begin{tabular}{c|c}  \hline
   \textbf{initialization = 290} & \textbf{initialization = 310} \\  \hline
      \multicolumn{2}{c}{\textbf{Naive algorithm}} \\   
\includegraphics[width = 0.49\linewidth]{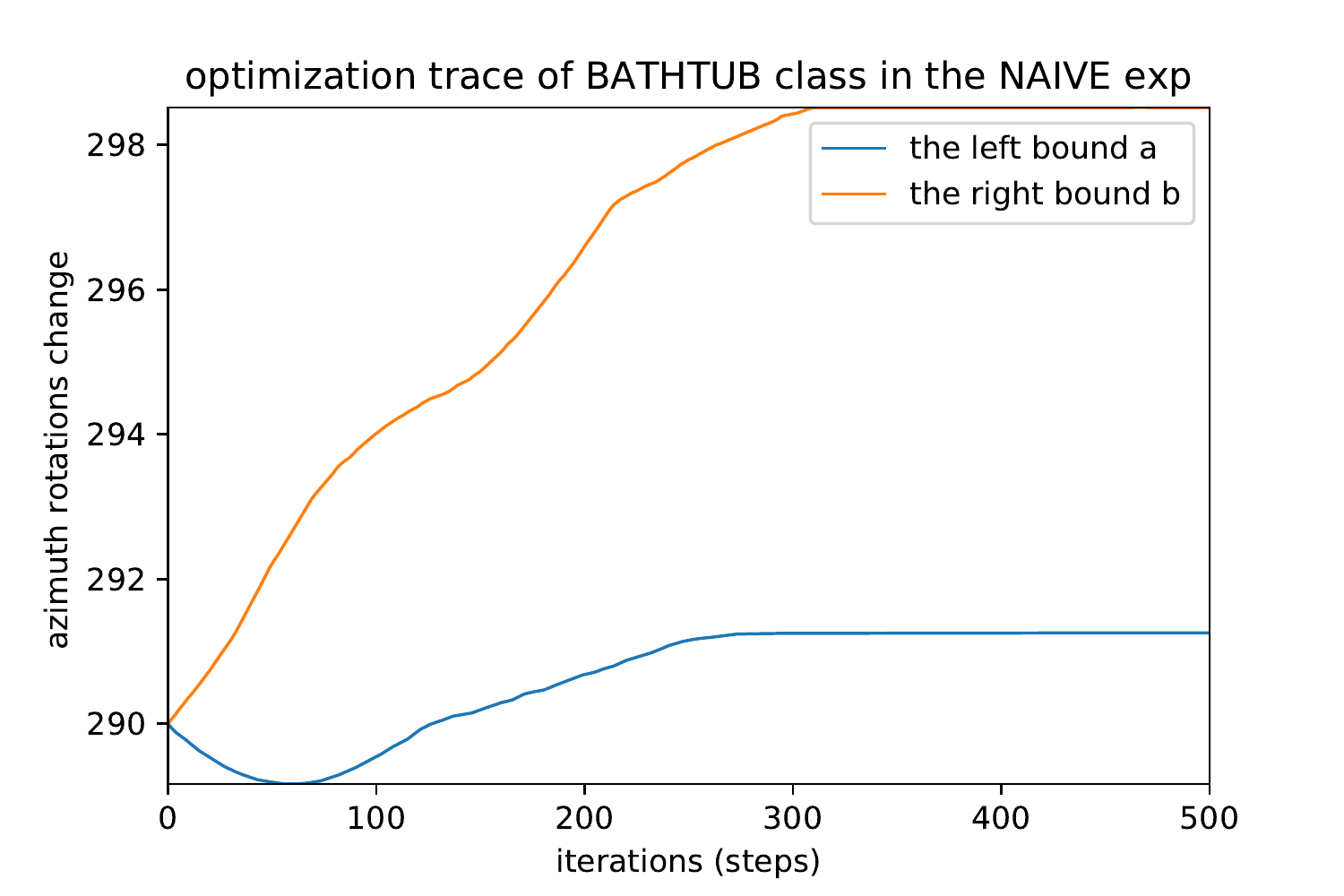} &
\includegraphics[width = 0.49\linewidth]{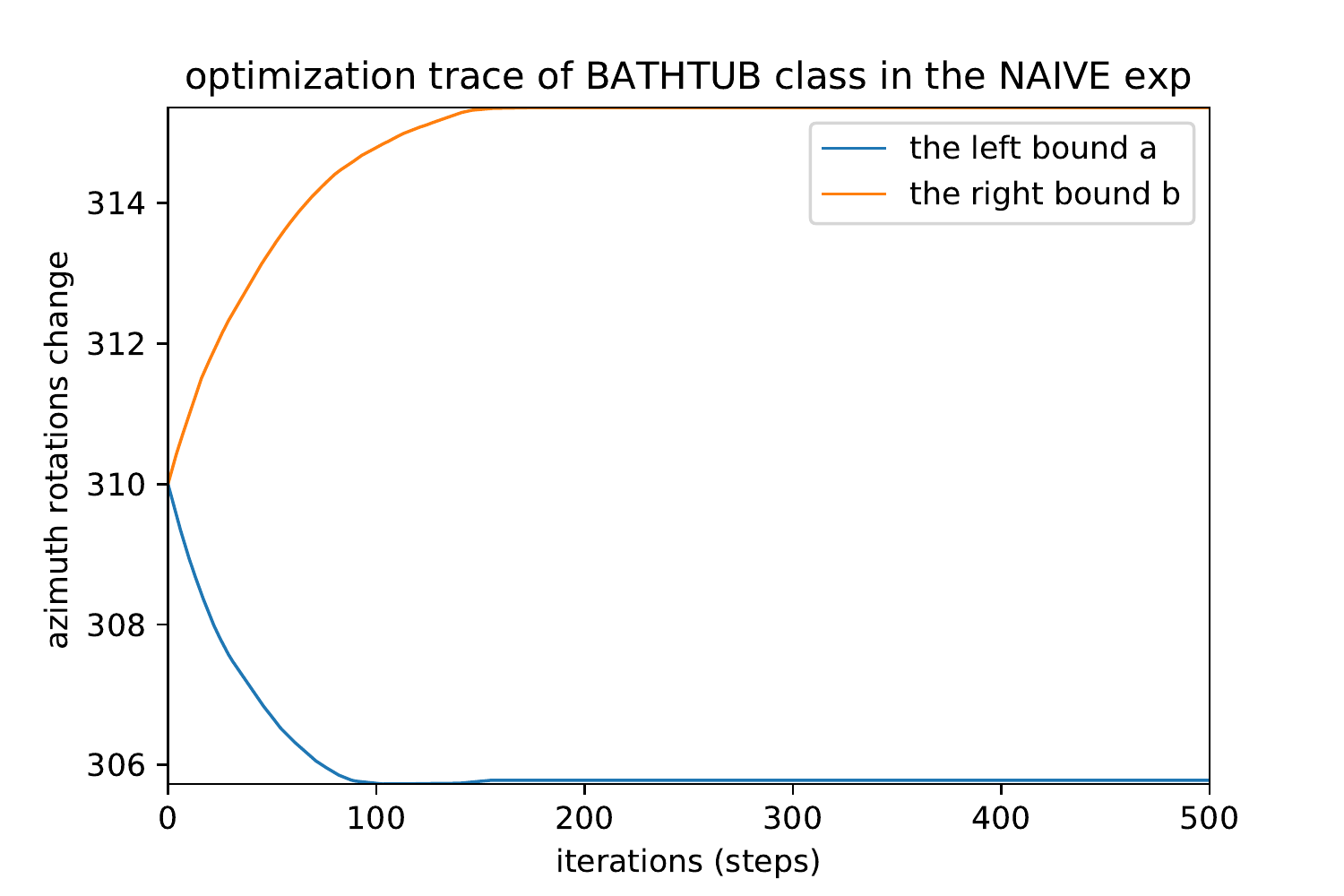} \\   \hline
\multicolumn{2}{c}{\textbf{Black-Box OIR algorithm}} \\  
\includegraphics[width = 0.49\linewidth]{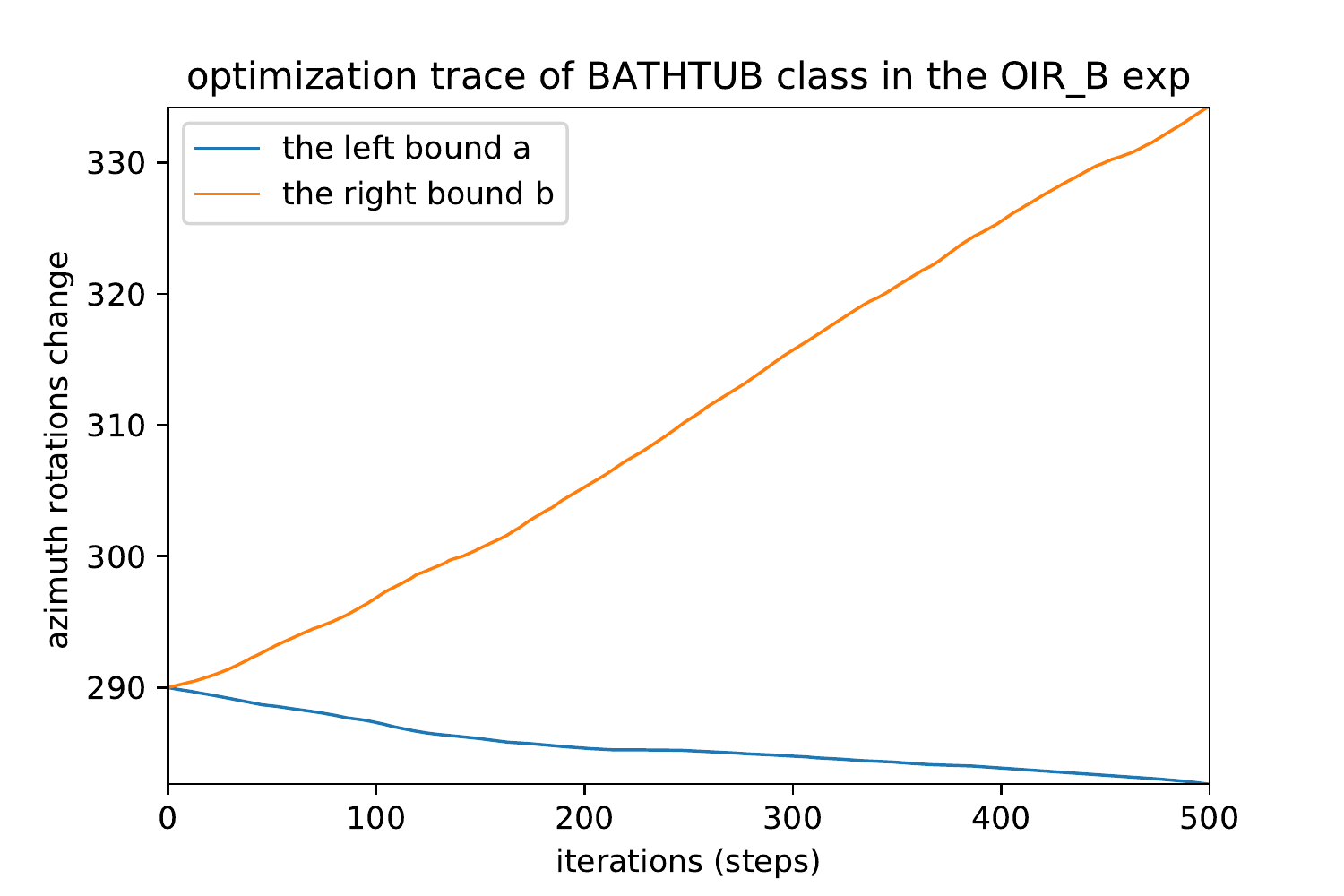} &
\includegraphics[width = 0.49\linewidth]{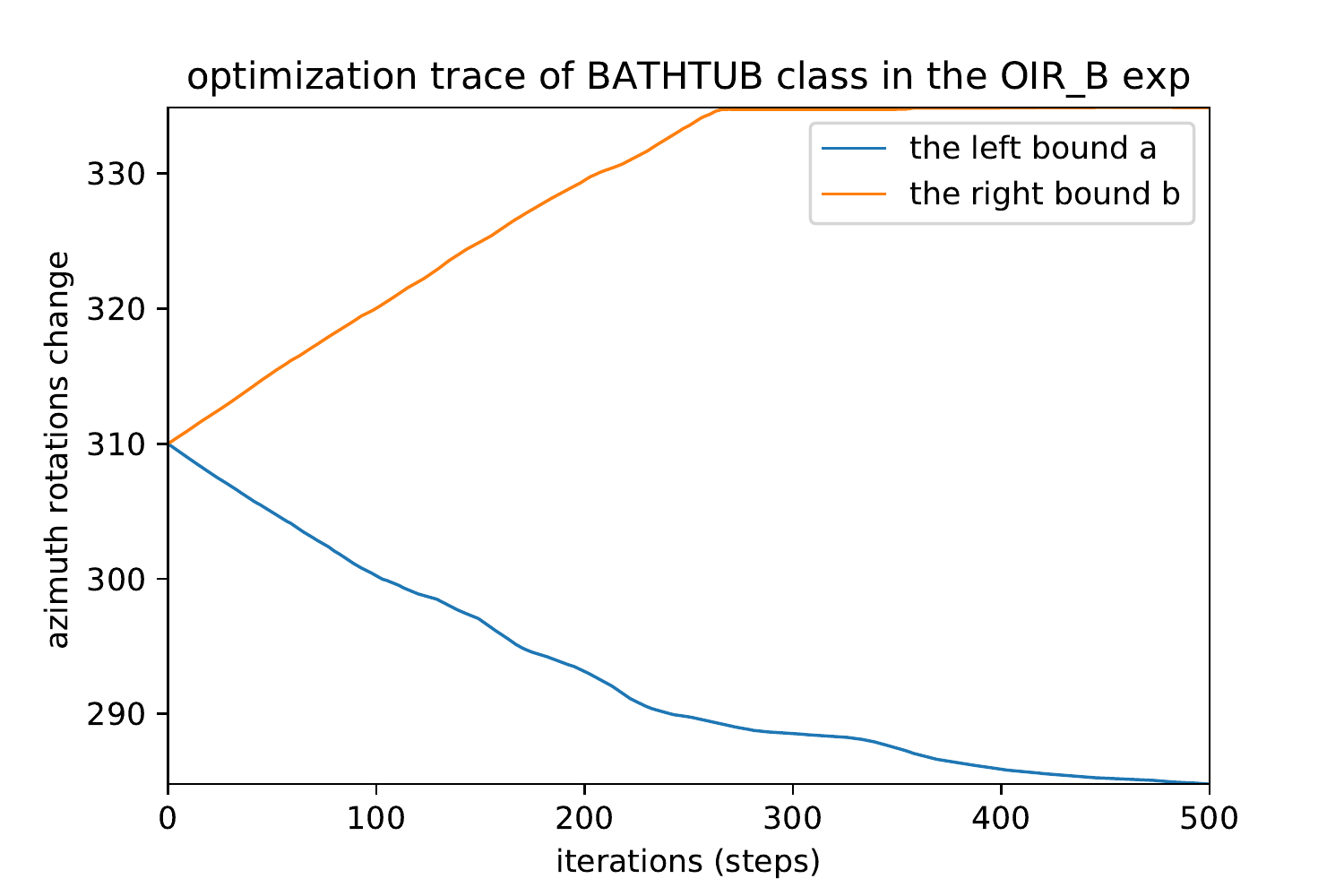} \\   \hline
\multicolumn{2}{c}{\textbf{White-Box OIR algorithm}} \\ 
\includegraphics[width = 0.49\linewidth]{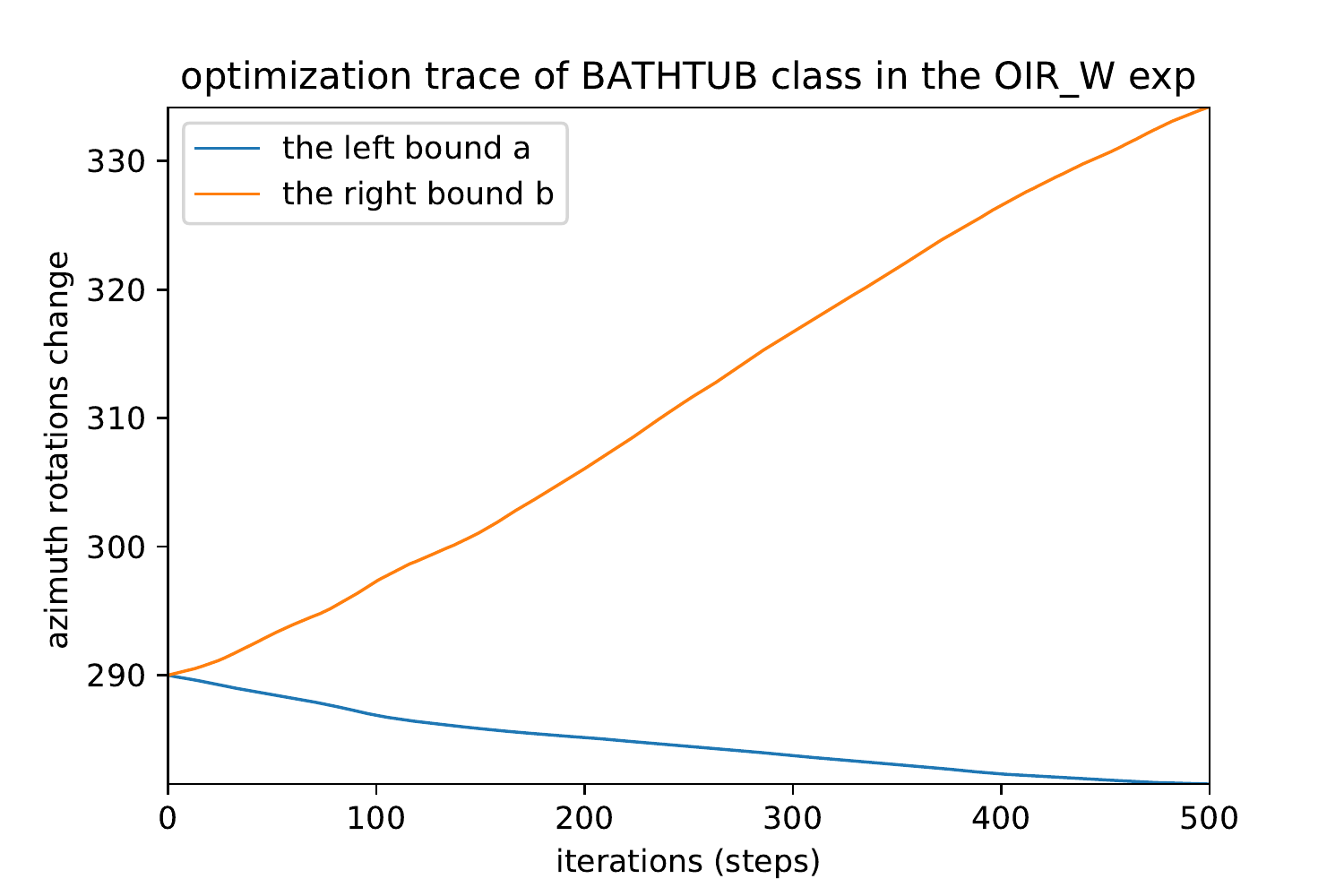} &
\includegraphics[width = 0.49\linewidth]{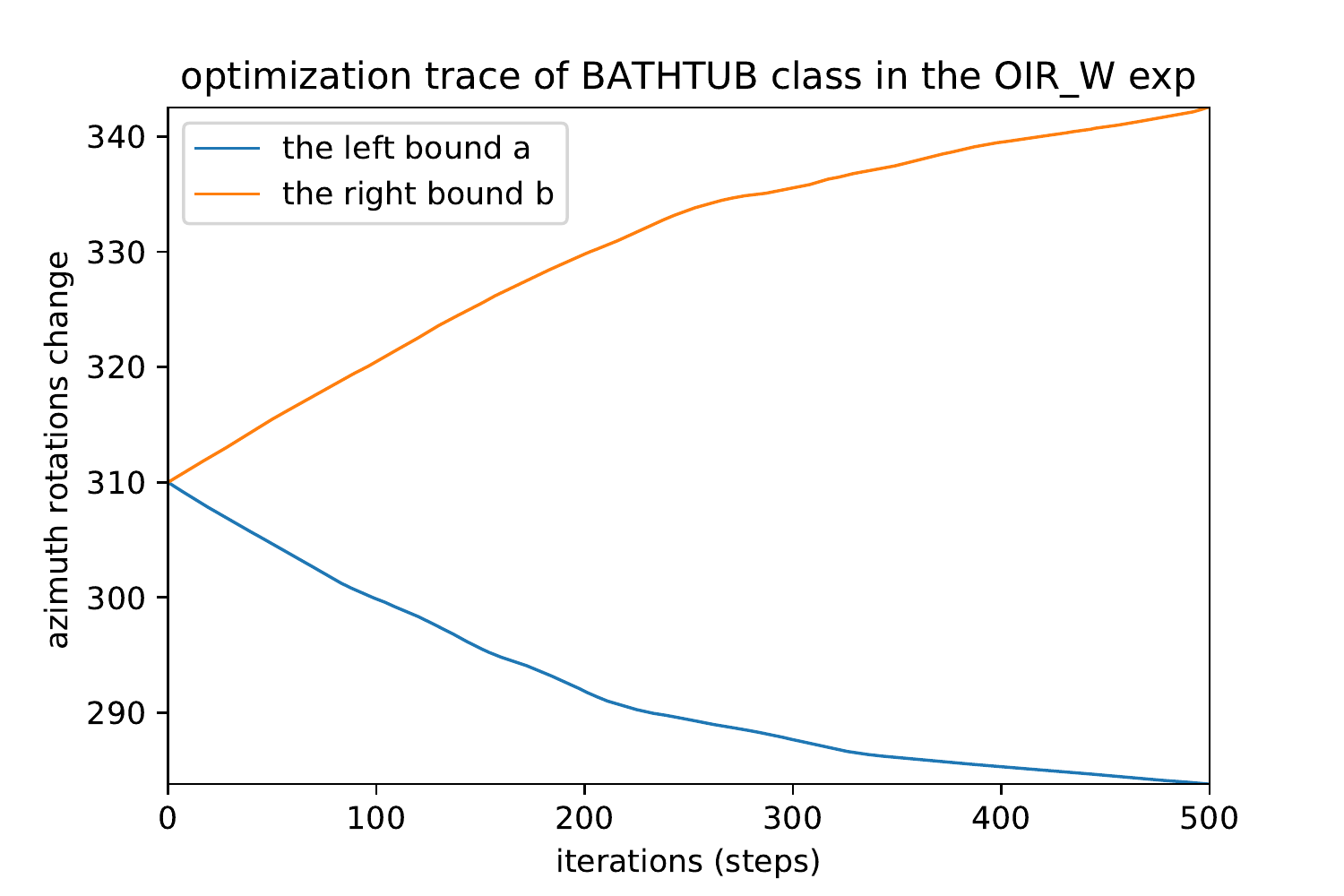} \\   \hline
\end{tabular}
   \caption{\small \textbf{Robust Region Bounds Growing I}: visualizing the bounds growing Using different algorithms from two different initializations (290, 310) in \figLabel{\ref{fig:converge}}. We can see that OIR formulations converge to the same bounds of the robust region regardless of the initialization, which indicates effectiveness in detecting these regions, unlike the naive approach which can stop in a local optimum. }
   \vspace{-8pt}
   \label{fig:conv2}
\end{figure*}

\subsection{Examples of Found Regions ( with Example Renderings)}
In \figLabel{\ref{fig:ex1},\ref{fig:ex2}} we provide examples of 2D regions found with the three algorithms along with renderings of the shapes from the robust regions detected.

\begin{figure*}[h]
\centering
\tabcolsep=0.03cm
   \begin{tabular}[!t]{c|c}  \hline
   \textbf{Detected Robust Regions} & \textbf{Renderings from Robust Regions} \\  \hline
\includegraphics[width = 0.55\linewidth]{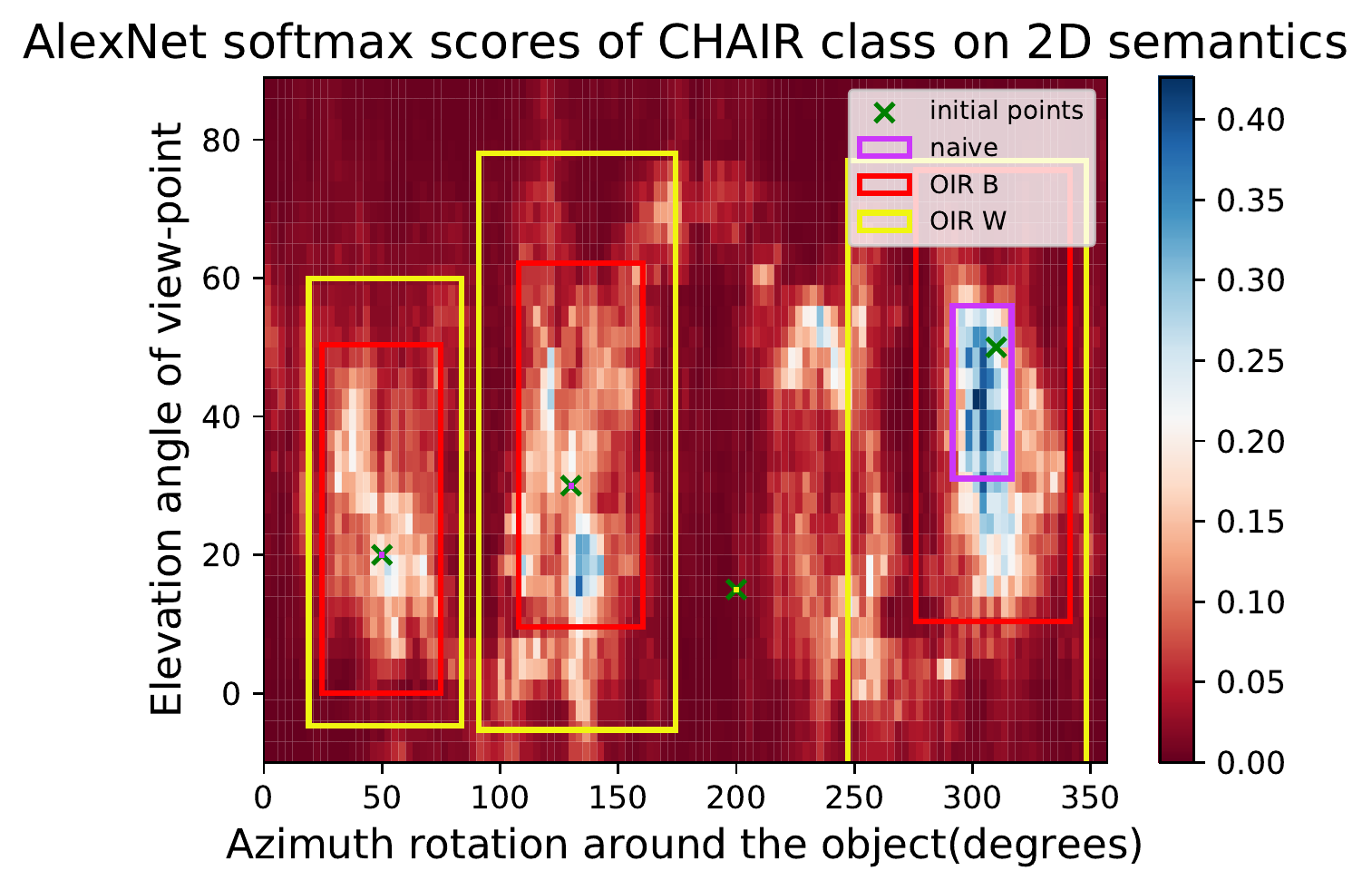} &
 \includegraphics[width = 0.43\linewidth]{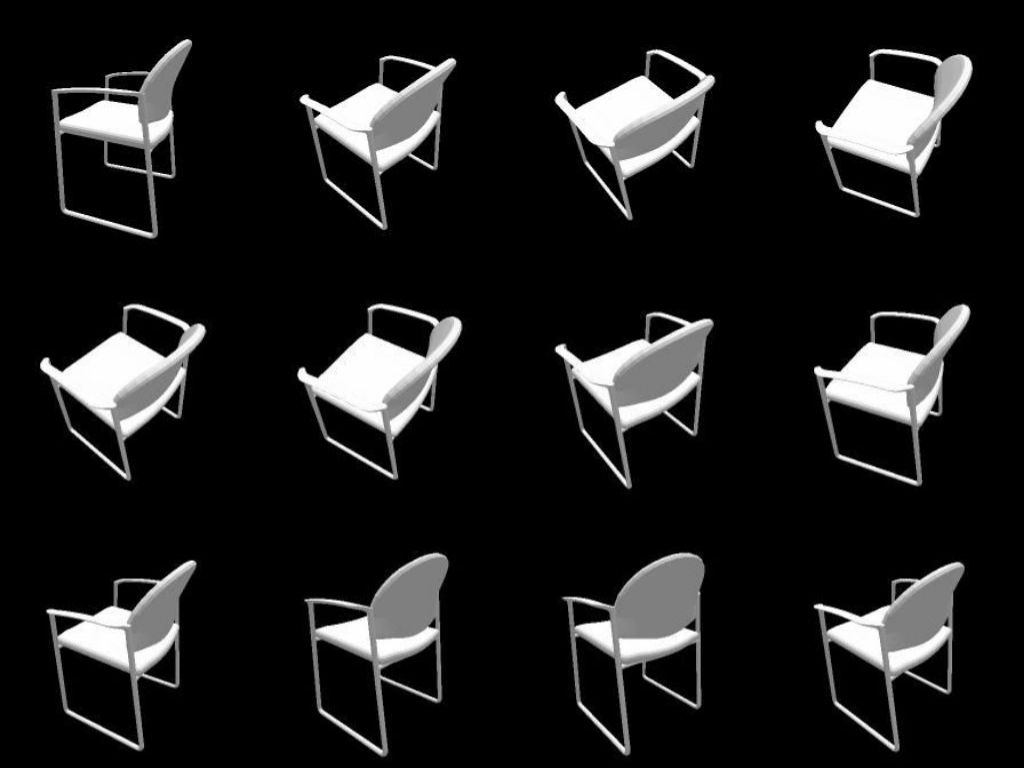}\\ \hline 
 \includegraphics[width = 0.55\linewidth]{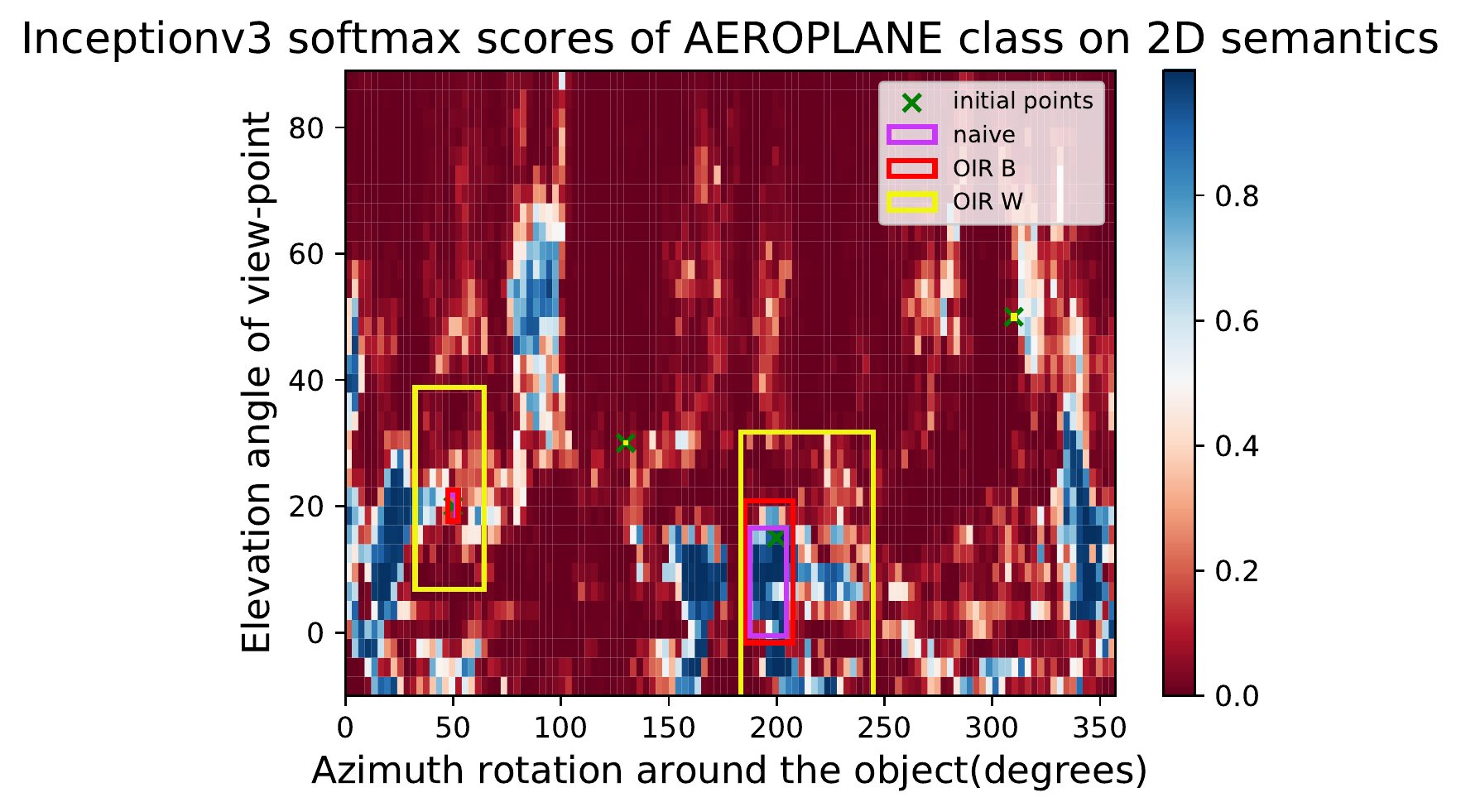} & 
 \includegraphics[width = 0.43\linewidth]{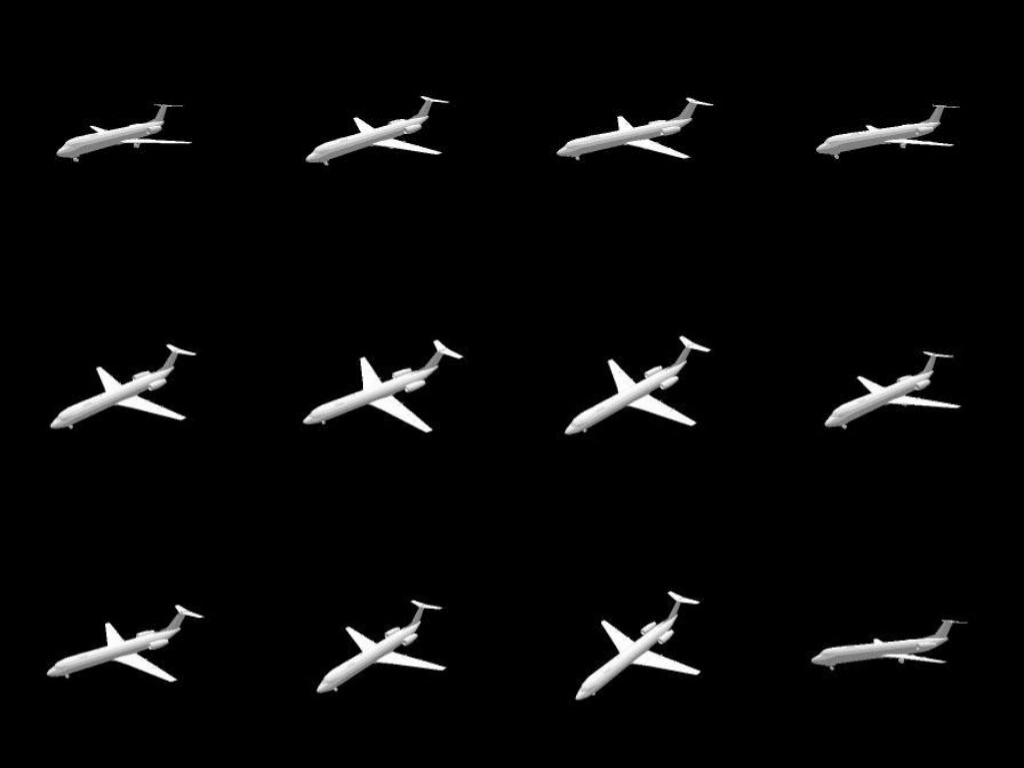}\\ \hline 
\includegraphics[width = 0.55\linewidth]{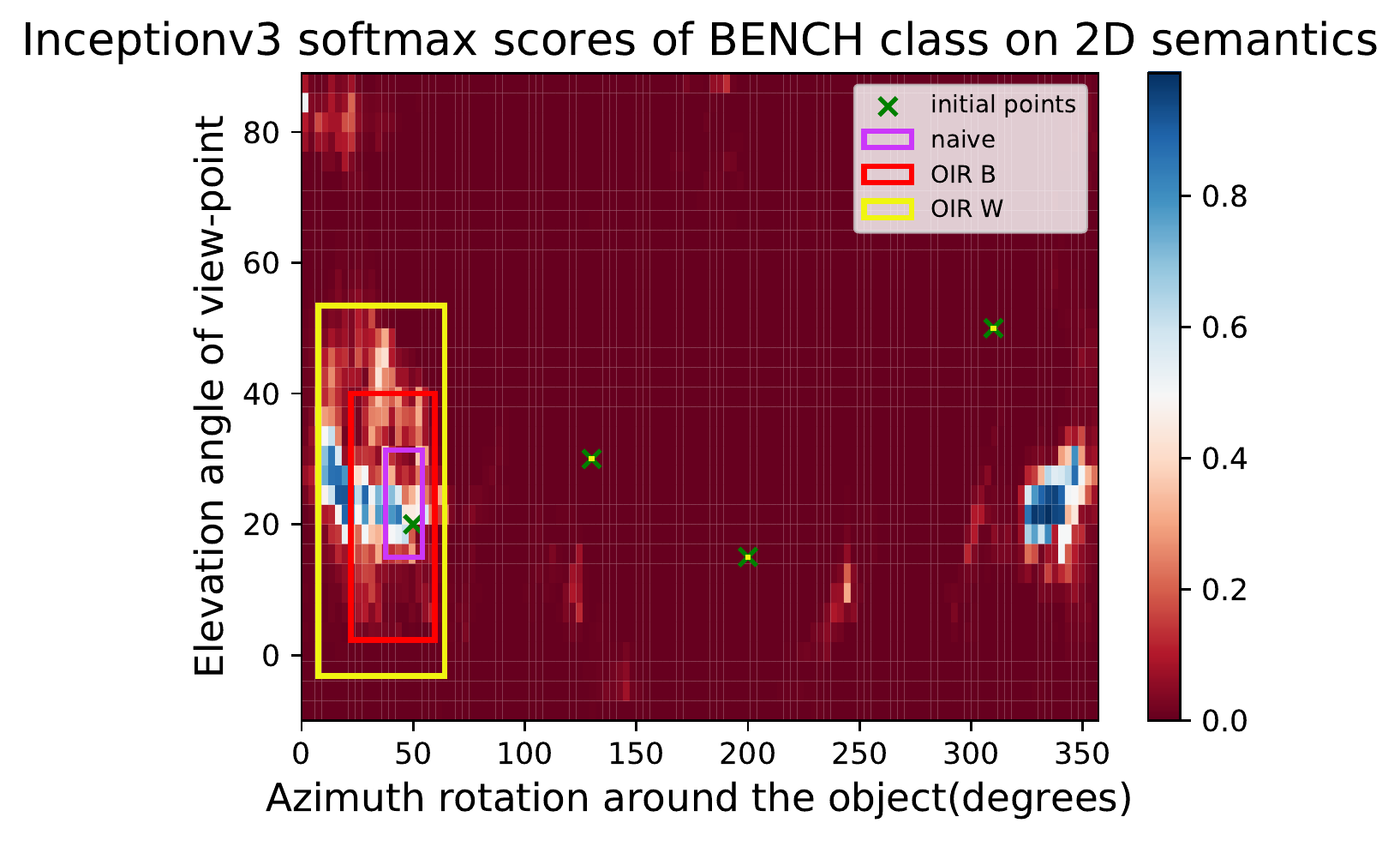} & 
 \includegraphics[width = 0.43\linewidth]{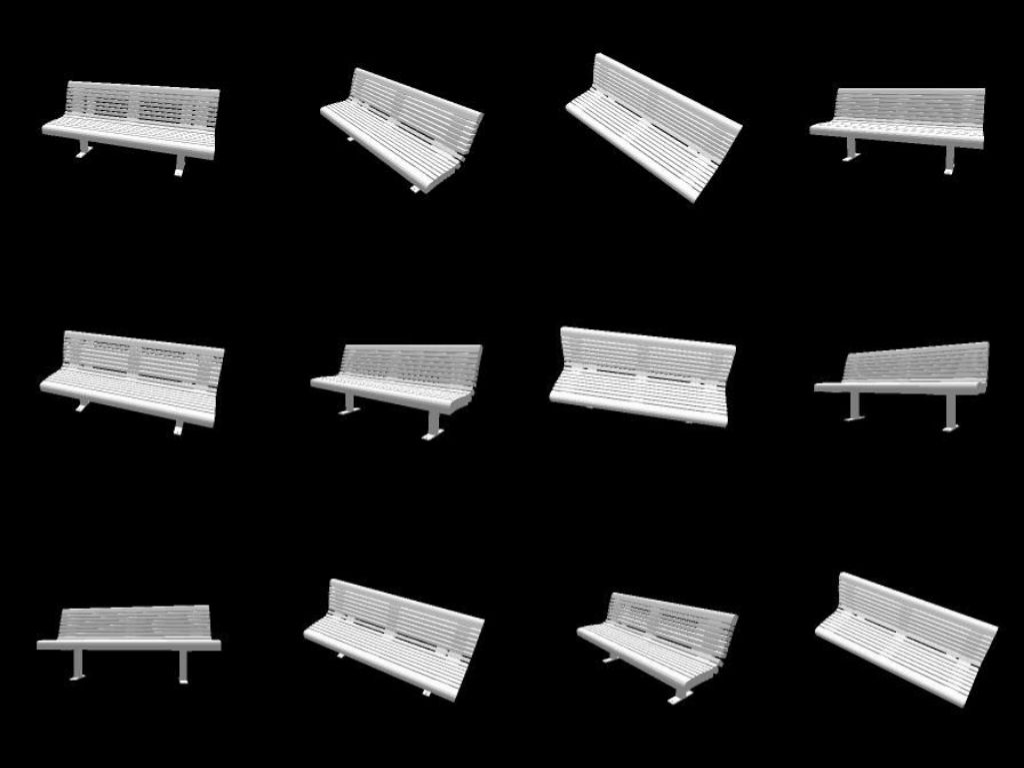}\\ \hline 
\includegraphics[width = 0.55\linewidth]{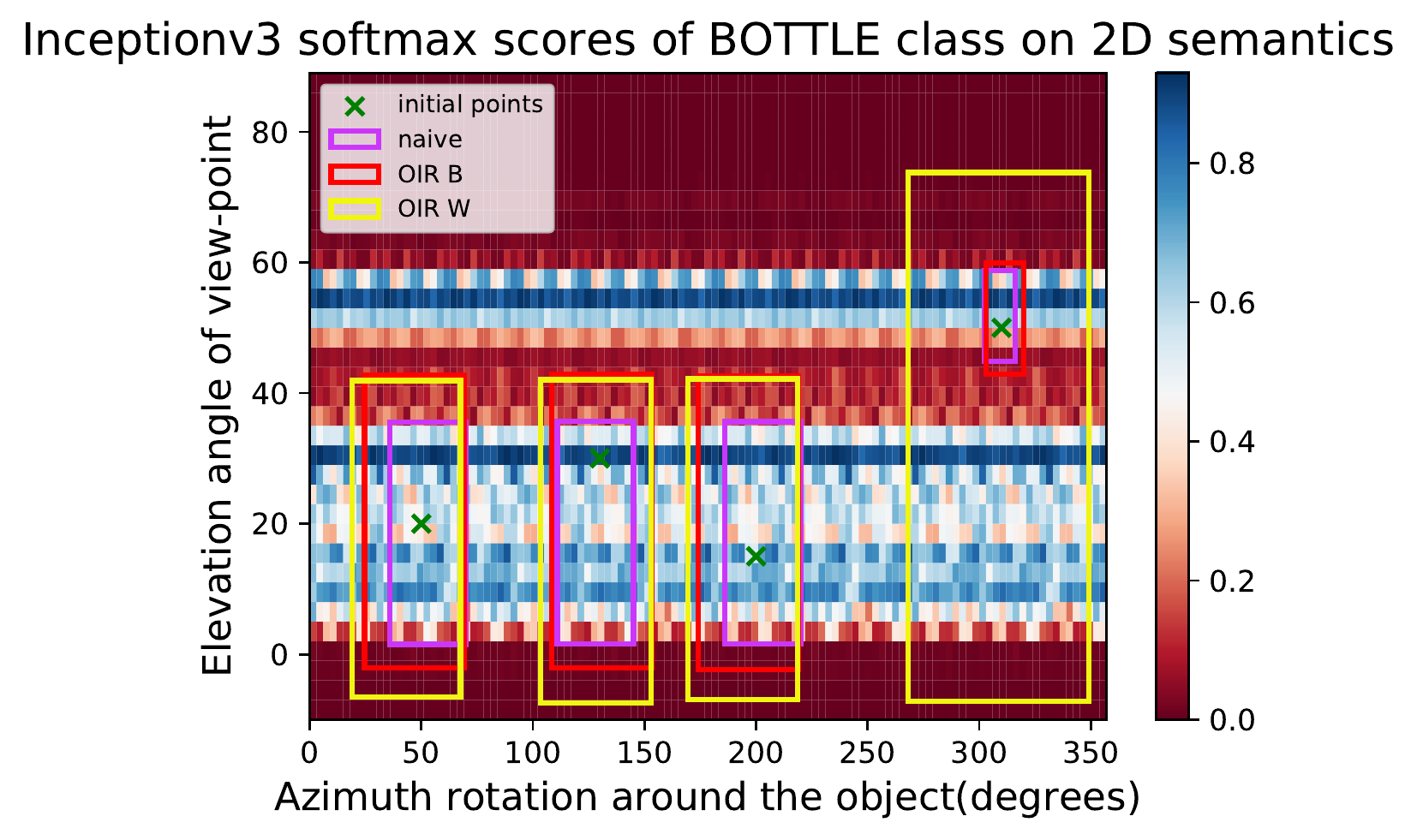}  &
 \includegraphics[width = 0.43\linewidth]{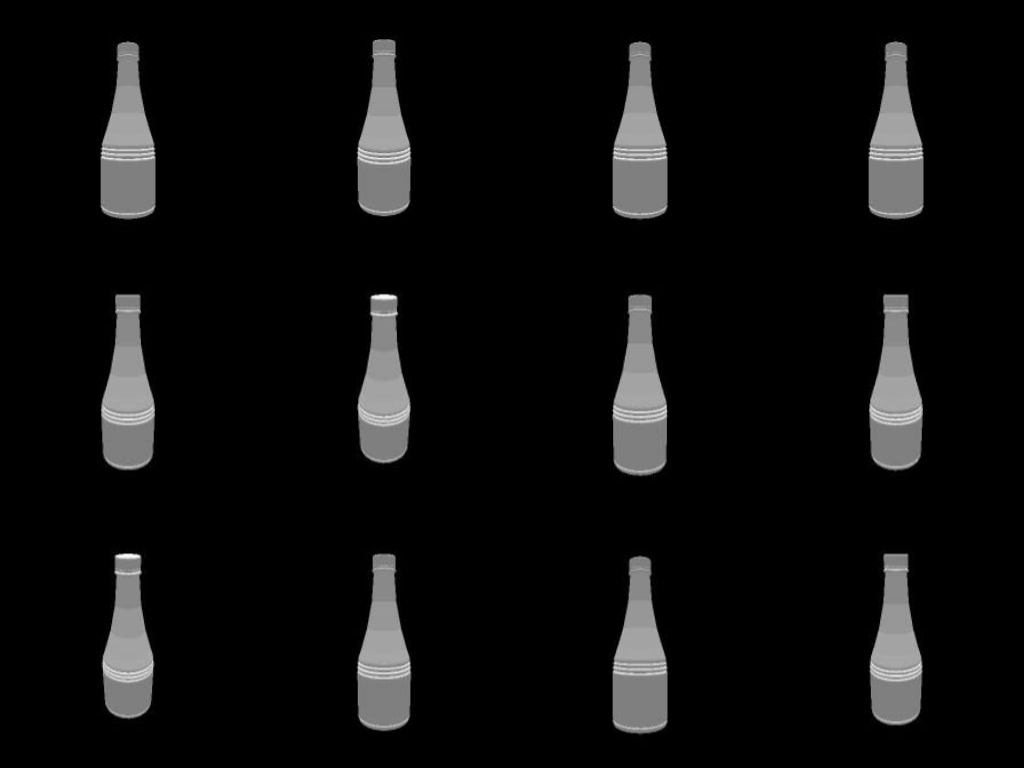}\\ \hline 
\end{tabular}
   \caption{\small \textbf{Qualitative Examples of Robust Regions I}: visualizing different runs of the algorithm to find robust regions along with different renderings from inside these regions for those specific shapes used in the experiments. }
  \vspace{-8pt}
   \label{fig:ex1}
\end{figure*}

\begin{figure*}[h]
\centering
\tabcolsep=0.03cm
   \begin{tabular}{c|c}  \hline
   \textbf{Detected Robust Regions} & \textbf{Renderings from Robust Regions} \\  \hline
\includegraphics[width = 0.55\linewidth]{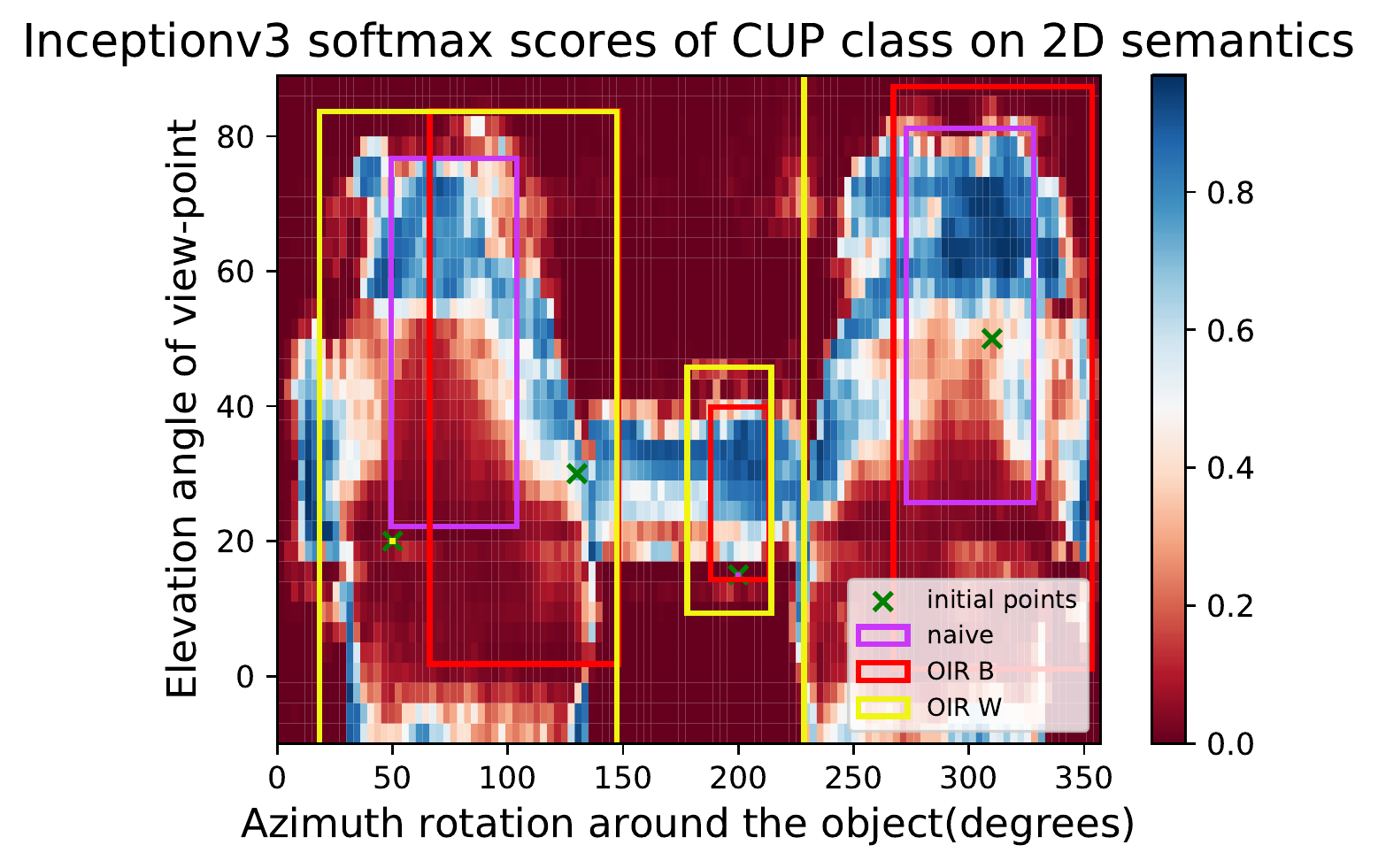} & 
 \includegraphics[width = 0.43\linewidth]{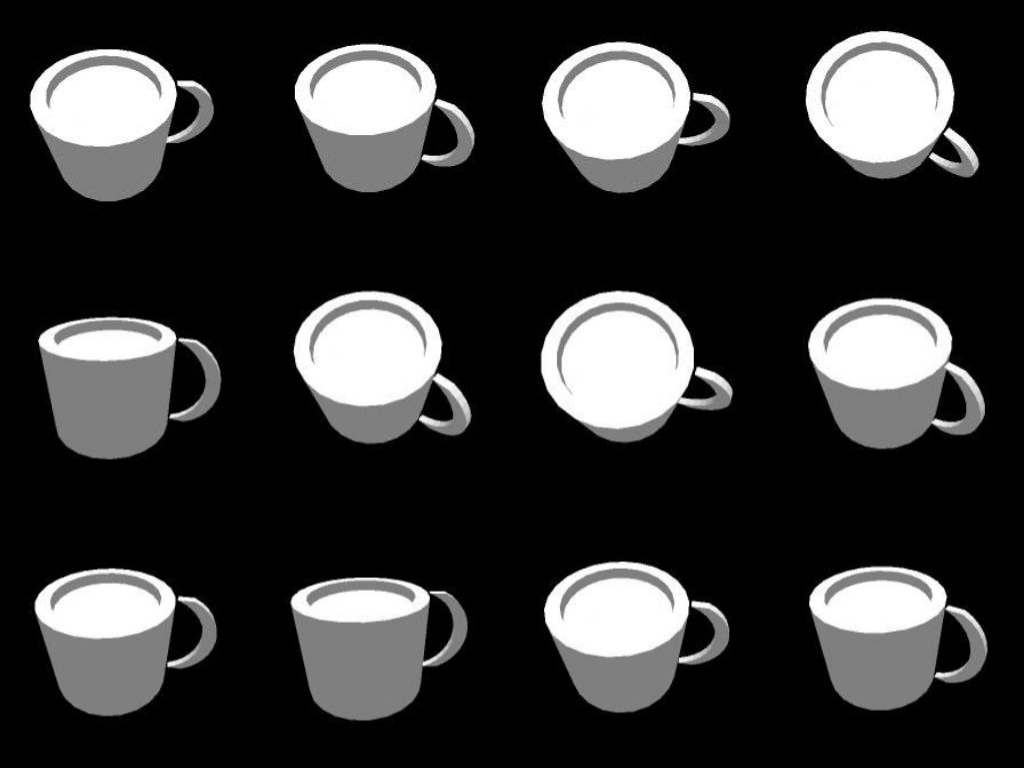}\\ \hline 
\includegraphics[width = 0.55\linewidth]{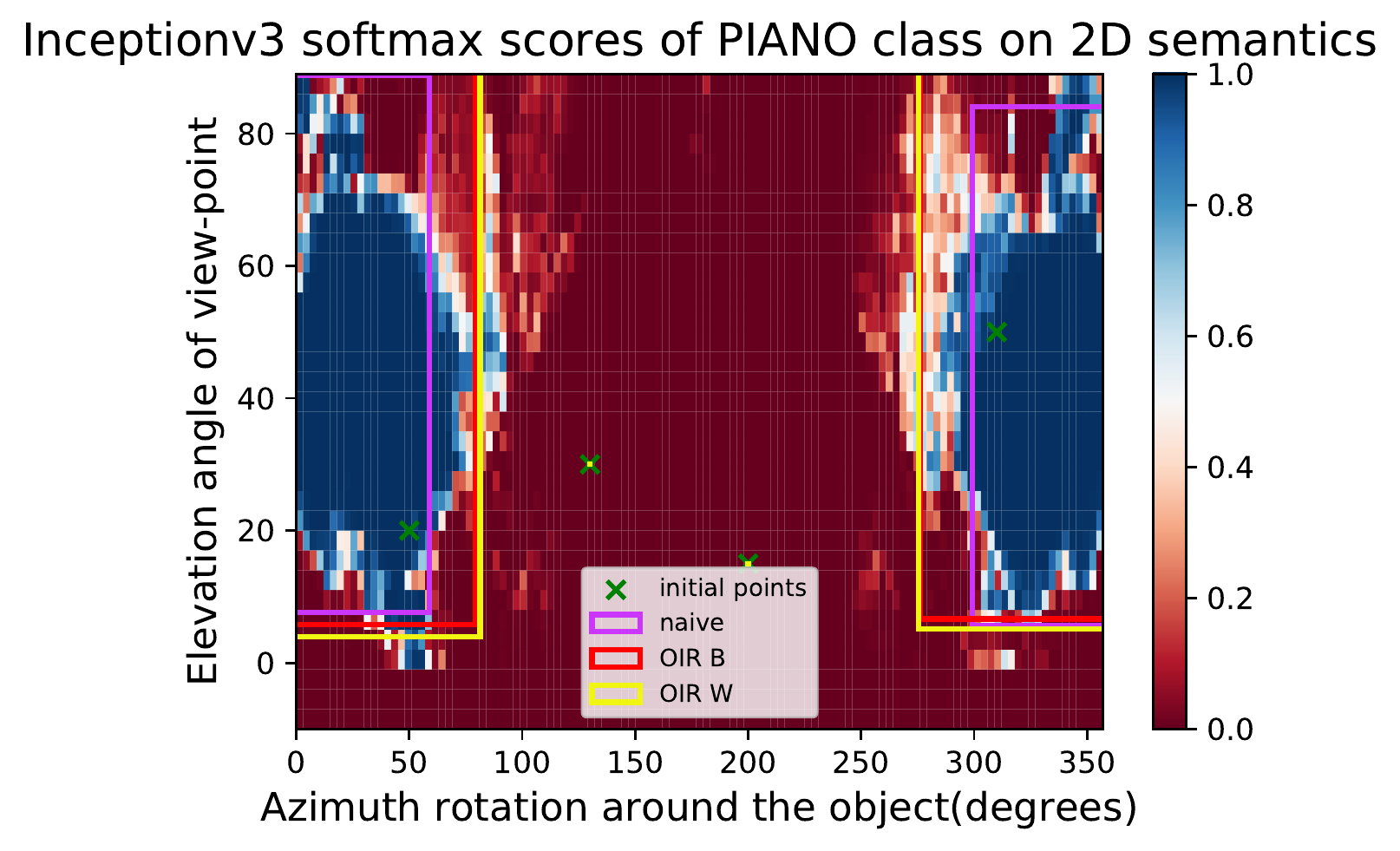} &  
\includegraphics[width = 0.43\linewidth]{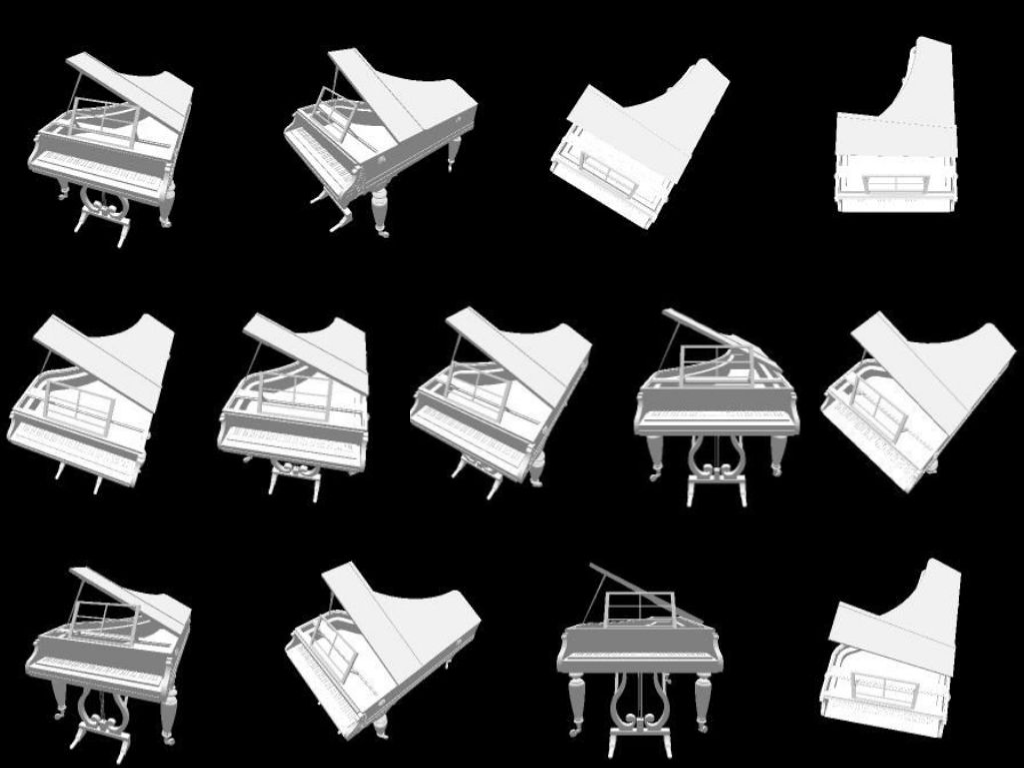}\\ \hline 
 \includegraphics[width = 0.55\linewidth]{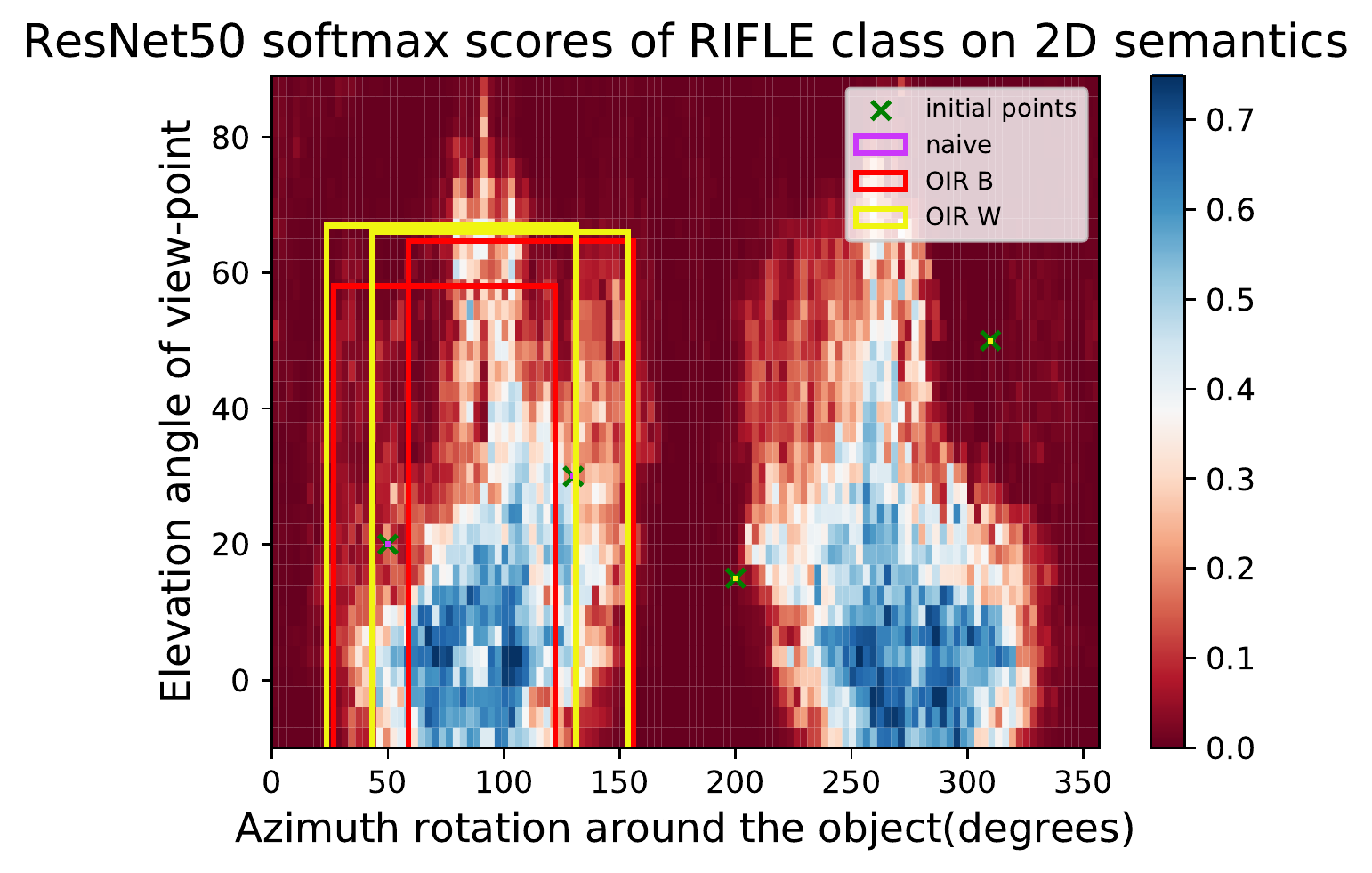} & 
  \includegraphics[width = 0.43\linewidth]{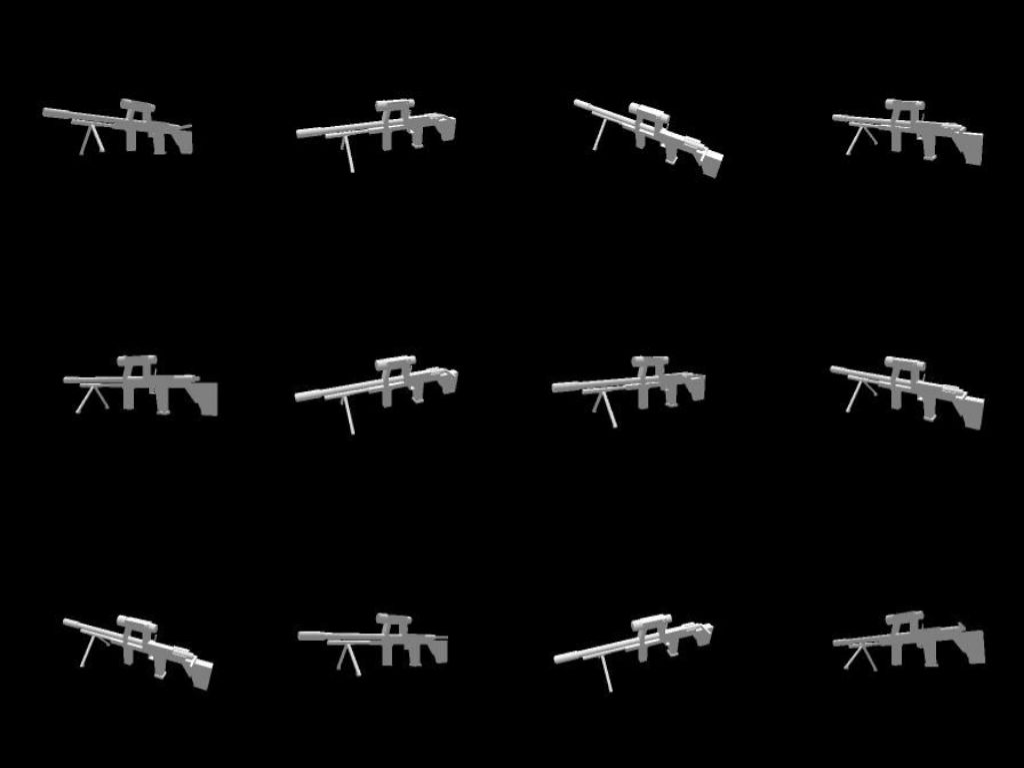}\\ \hline 
\end{tabular}
   \caption{\small \textbf{Qualitative Examples of Robust Regions II}: visualizing different runs of the algorithm to find robust regions along with different renderings from inside these regions for those specific shapes used in the experiments.}
   \vspace{-8pt}
   \label{fig:ex2}
\end{figure*}

\subsection{Analyzing Semantic Data Bias in ImageNet}
In \figLabel{\ref{fig:dsm1},\ref{fig:dsm2}}, we visualizing semantic data bias in the common training dataset (\ie ImageNet \cite{IMAGENET}) by averaging the Networks Semantic Maps (NSM) of different networks and on different shapes, Different classes have a different semantic bias in ImageNet as clearly shown in the maps above. These places of high confidence probably reveal the areas where an abundance of examples exists in ImageNet, while holes convey scarcity of such examples in ImageNet corresponding class.
\begin{figure*}[h]
\centering
\tabcolsep=0.03cm
   \begin{tabular}{c|c}
\includegraphics[width = 0.49\linewidth]{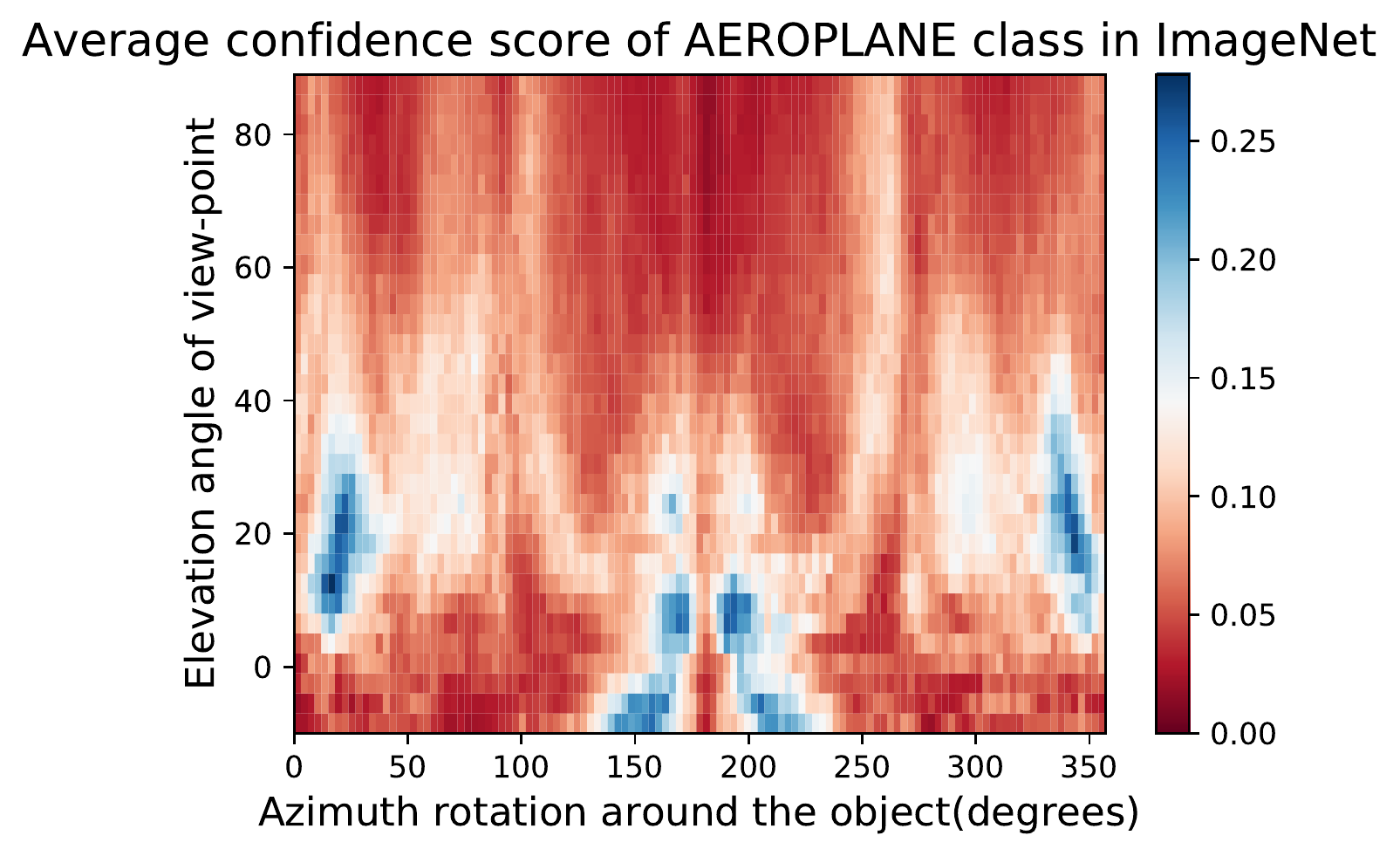}&
\includegraphics[width = 0.49\linewidth]{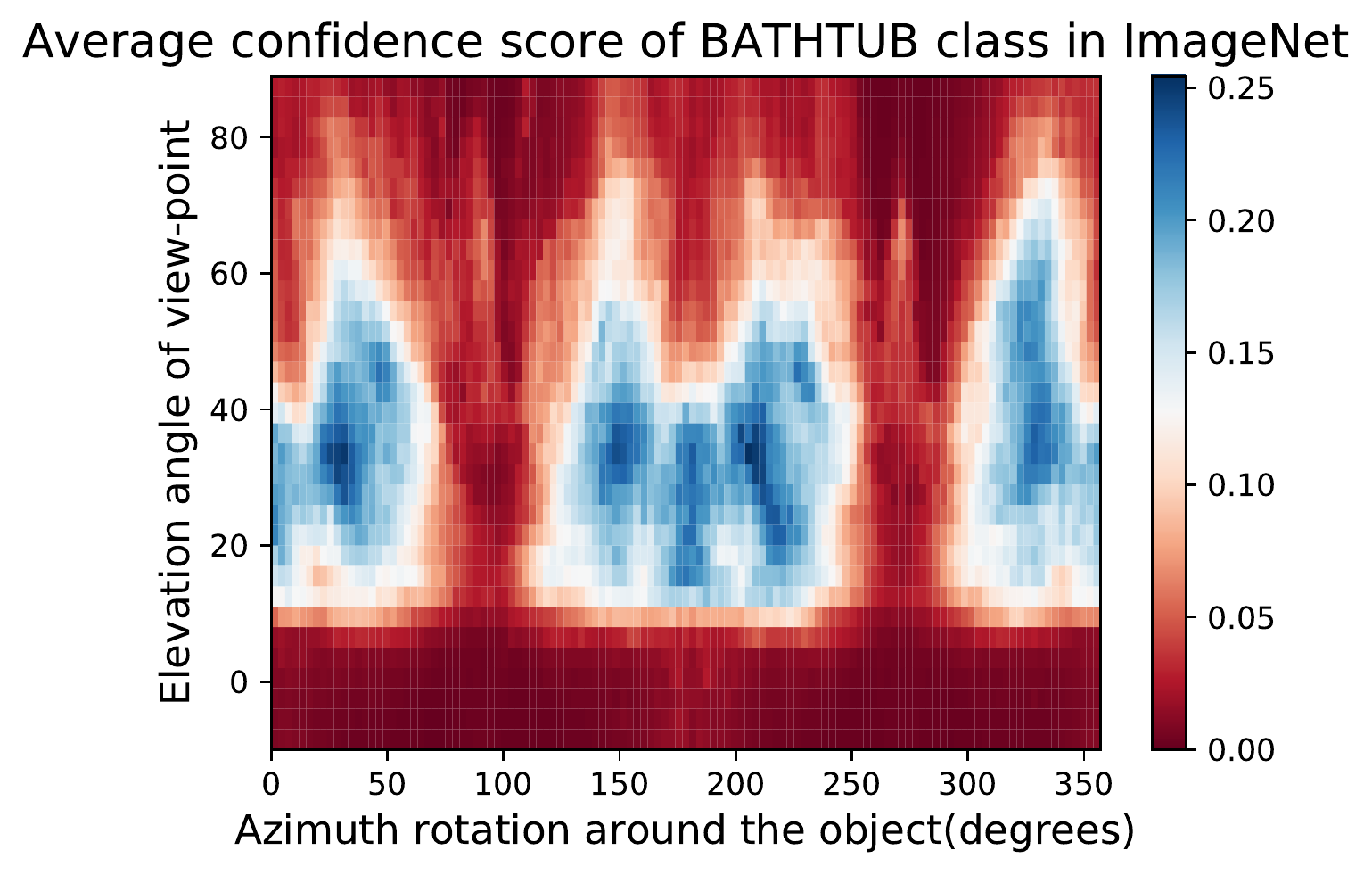}\\\hline
\includegraphics[width = 0.49\linewidth]{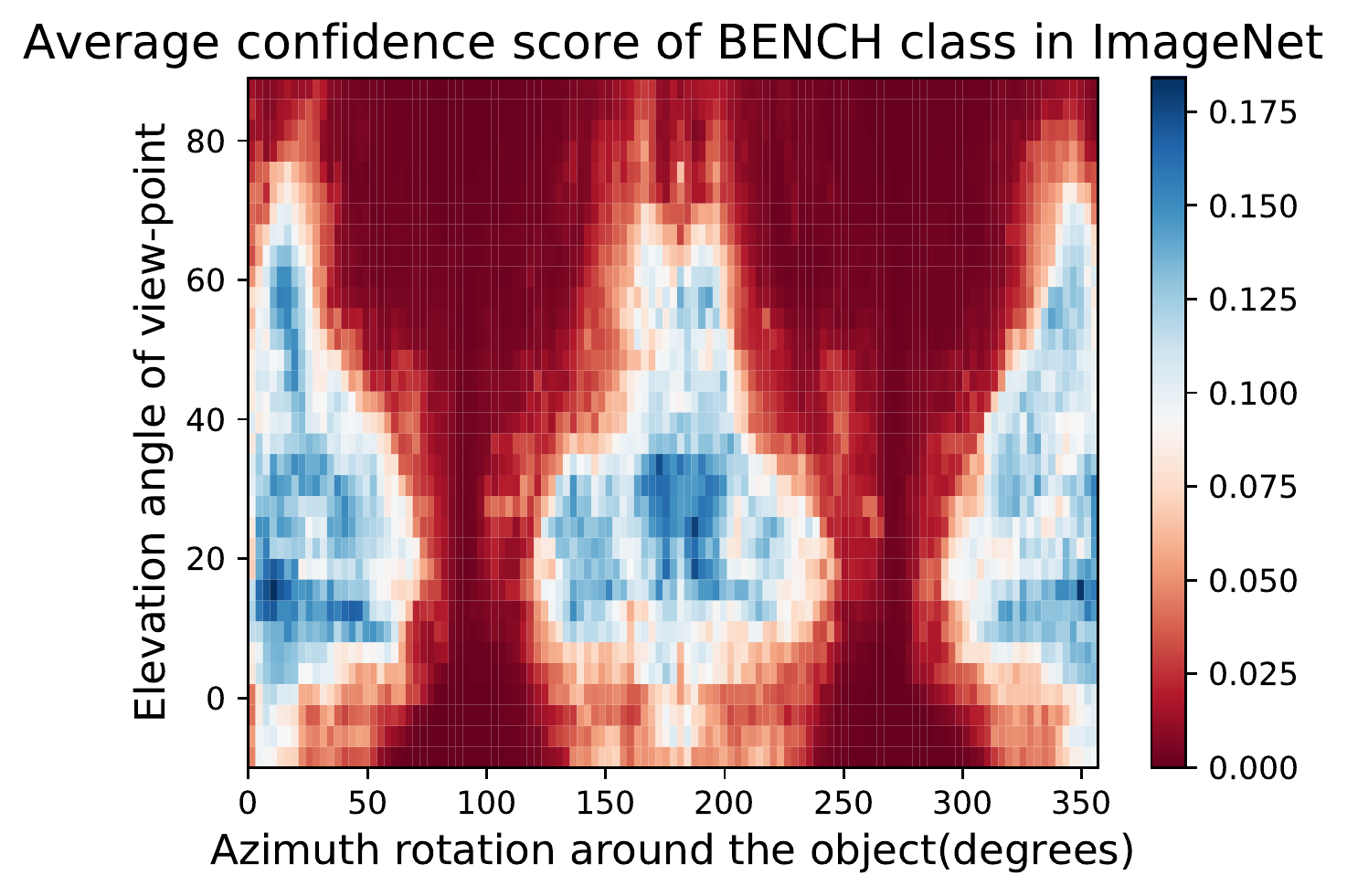}&
\includegraphics[width = 0.49\linewidth]{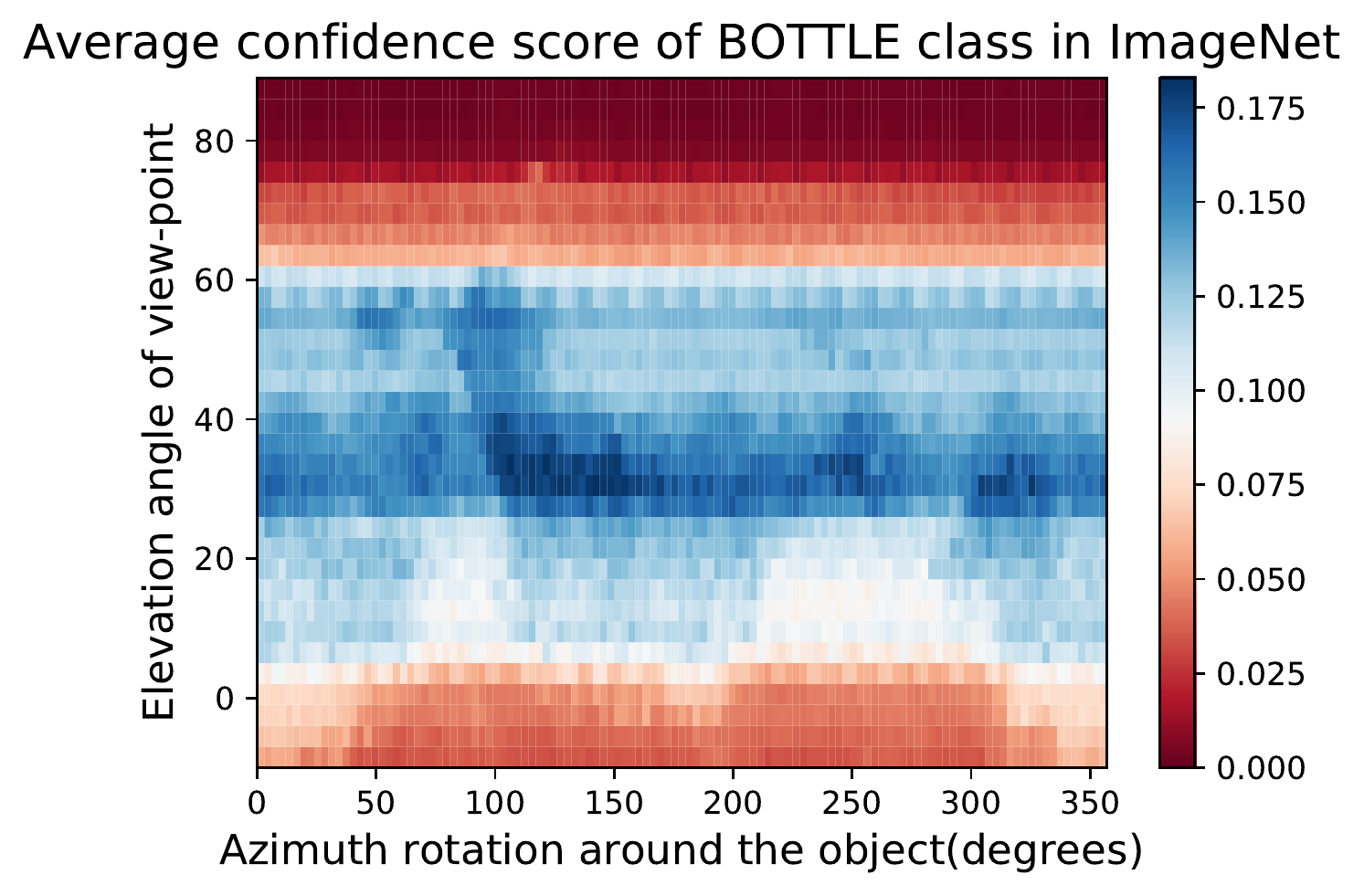}
\end{tabular}
   \caption{\small \textbf{Data Semantic Maps DSM-I}: visualizing Semantic Data Bias in the common training dataset (\ie ImageNet \cite{IMAGENET}) by averaging the Networks Semantic Maps (NSM) of different networks and on different shapes, Different classes have different semantic bias in ImageNet as clearly shown in the maps above. The symmetry in the maps are attributed to the 3D symmetry of the objects. }
   \vspace{-8pt}
   \label{fig:dsm1}
\end{figure*}

\begin{figure*}[]
\centering
\tabcolsep=0.03cm
   \begin{tabular}{c|c}
\includegraphics[width = 0.49\linewidth]{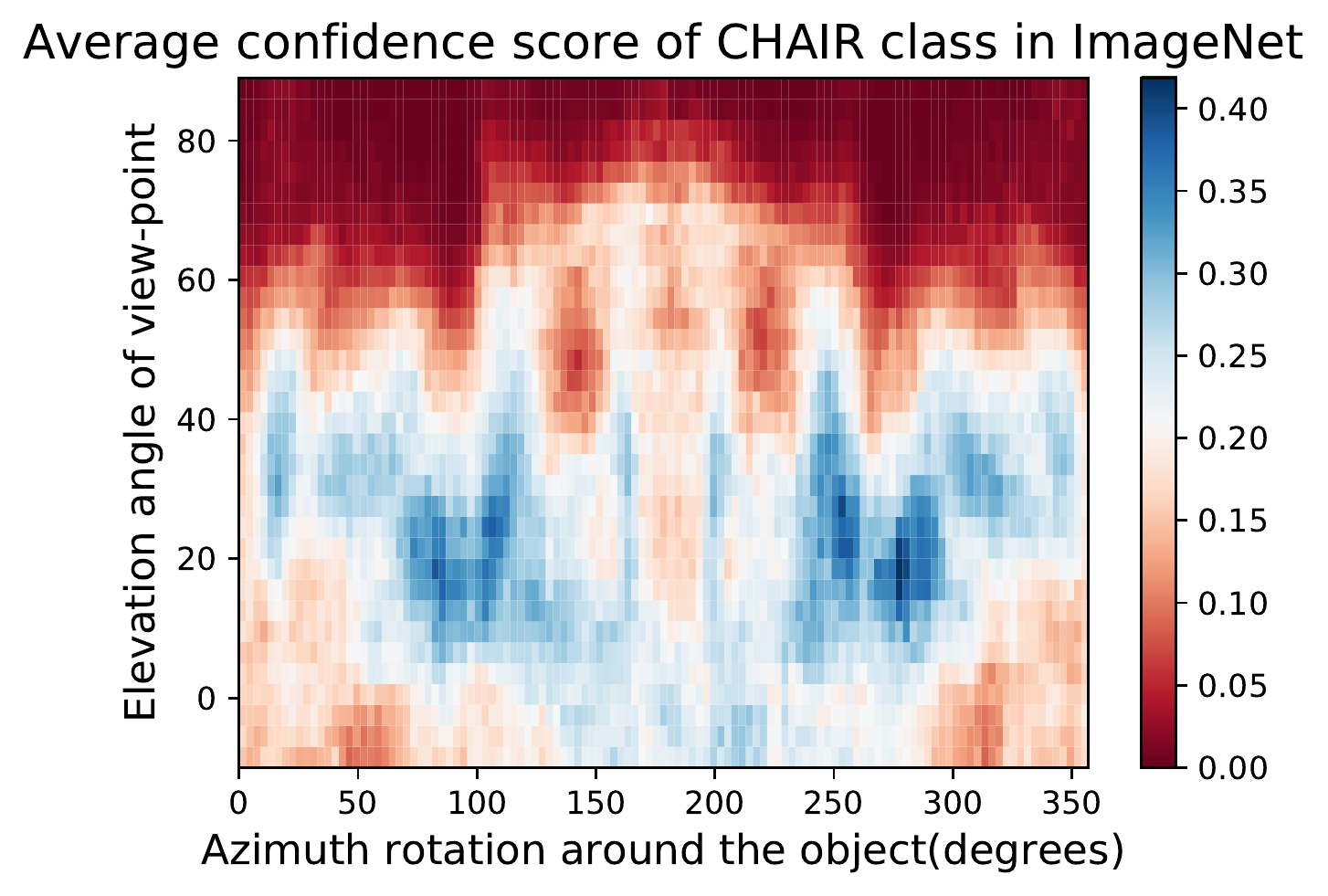}&
\includegraphics[width = 0.49\linewidth]{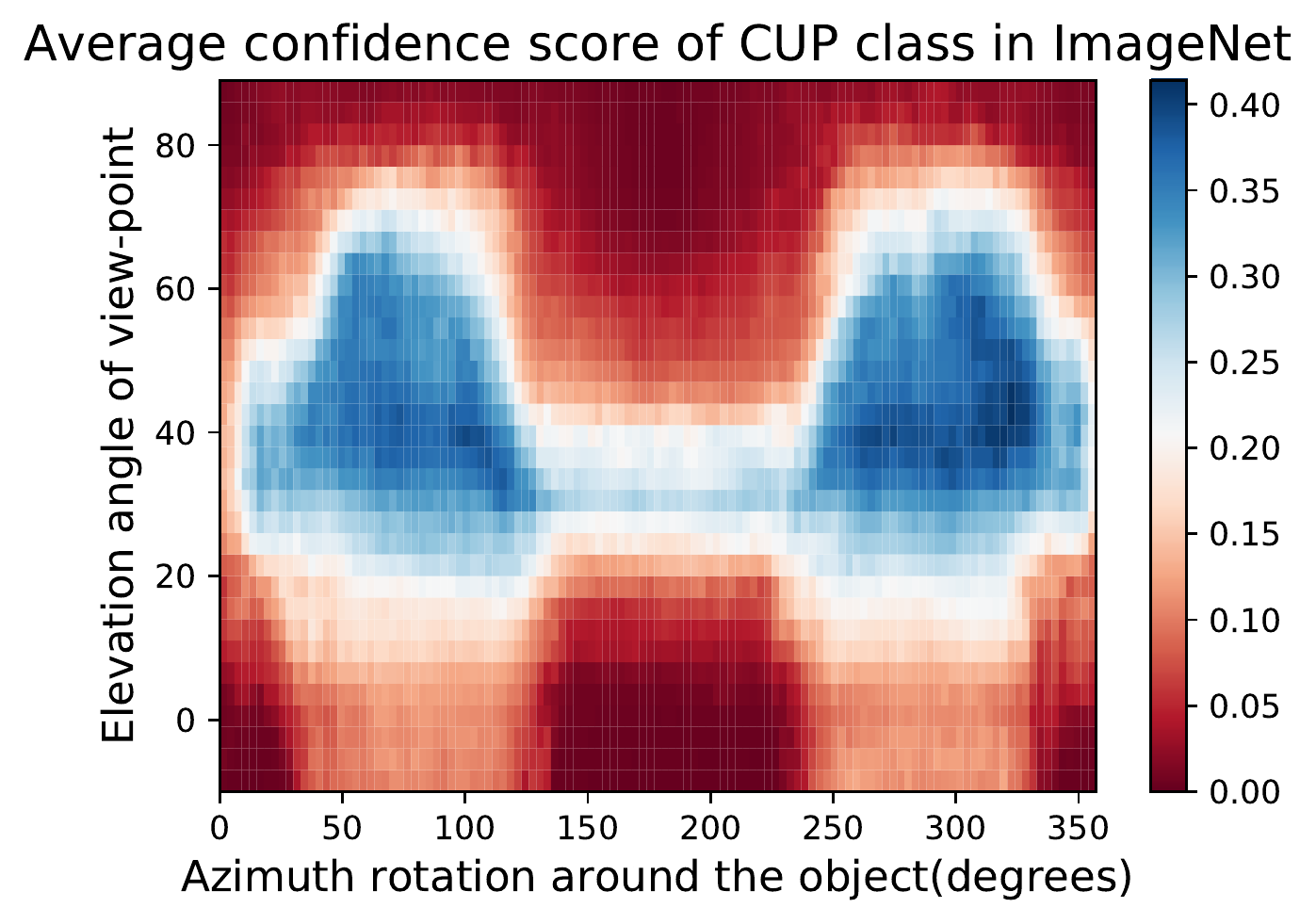}\\\hline
\includegraphics[width = 0.49\linewidth]{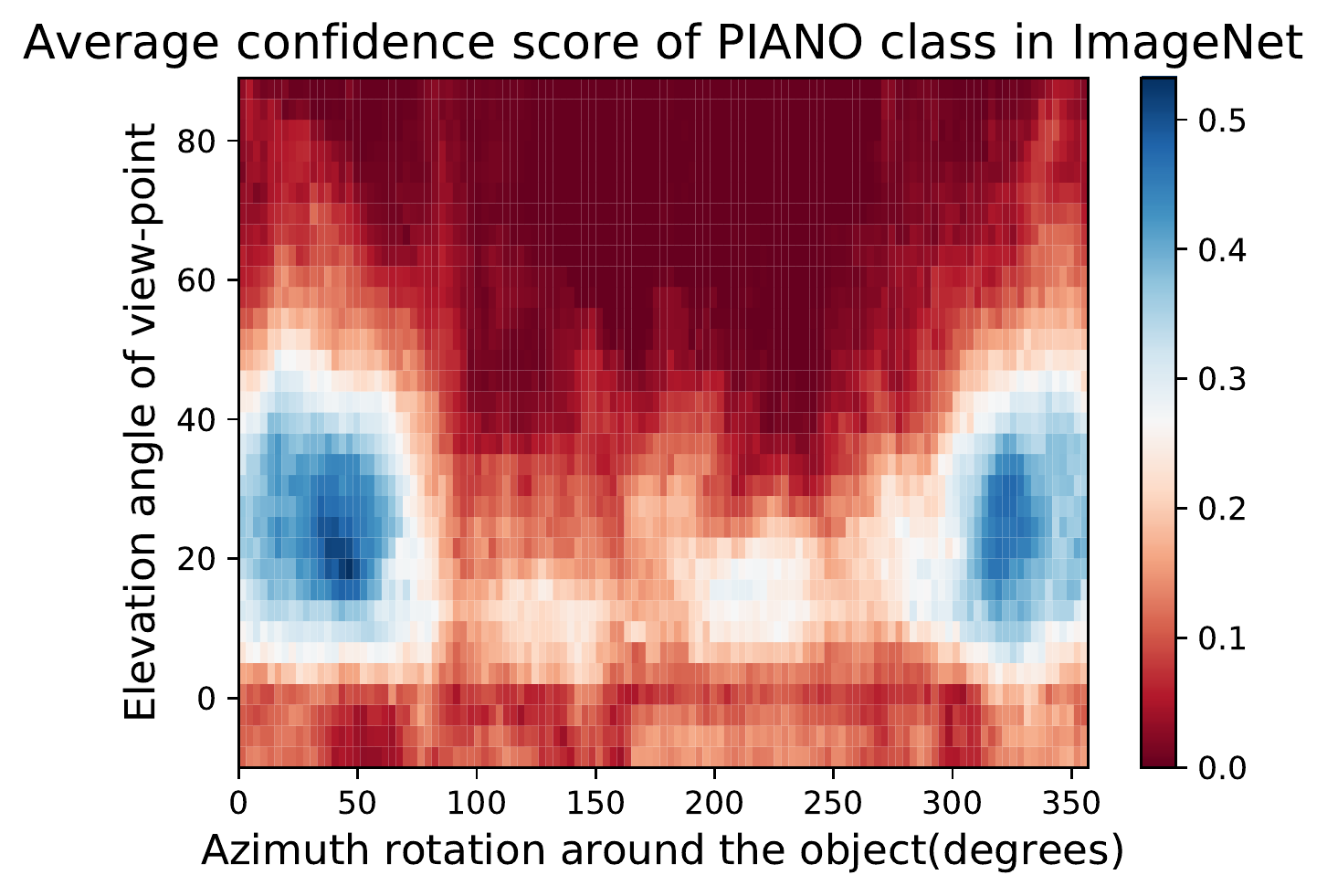}&
\includegraphics[width = 0.49\linewidth]{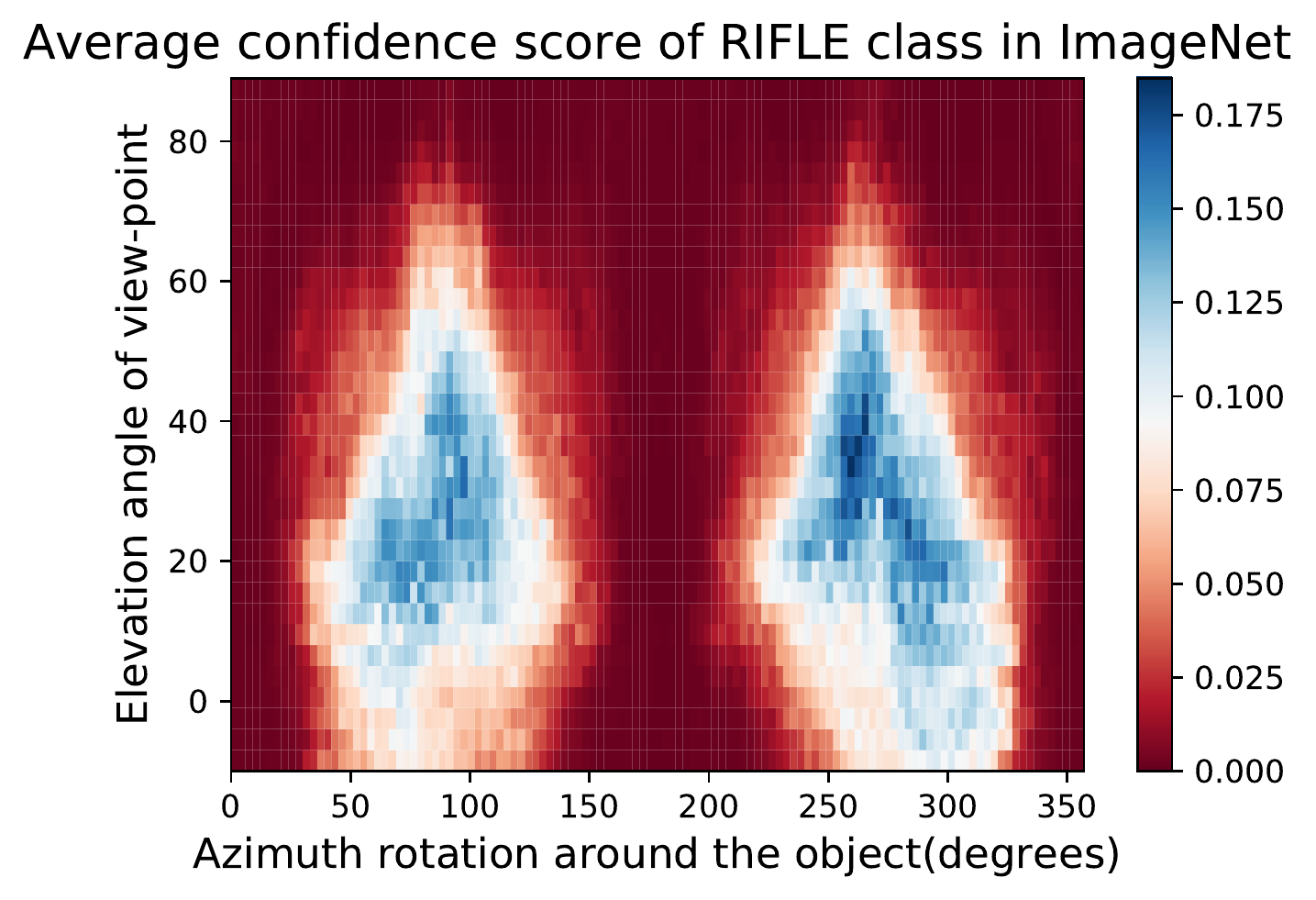}\\\hline
\includegraphics[width = 0.49\linewidth]{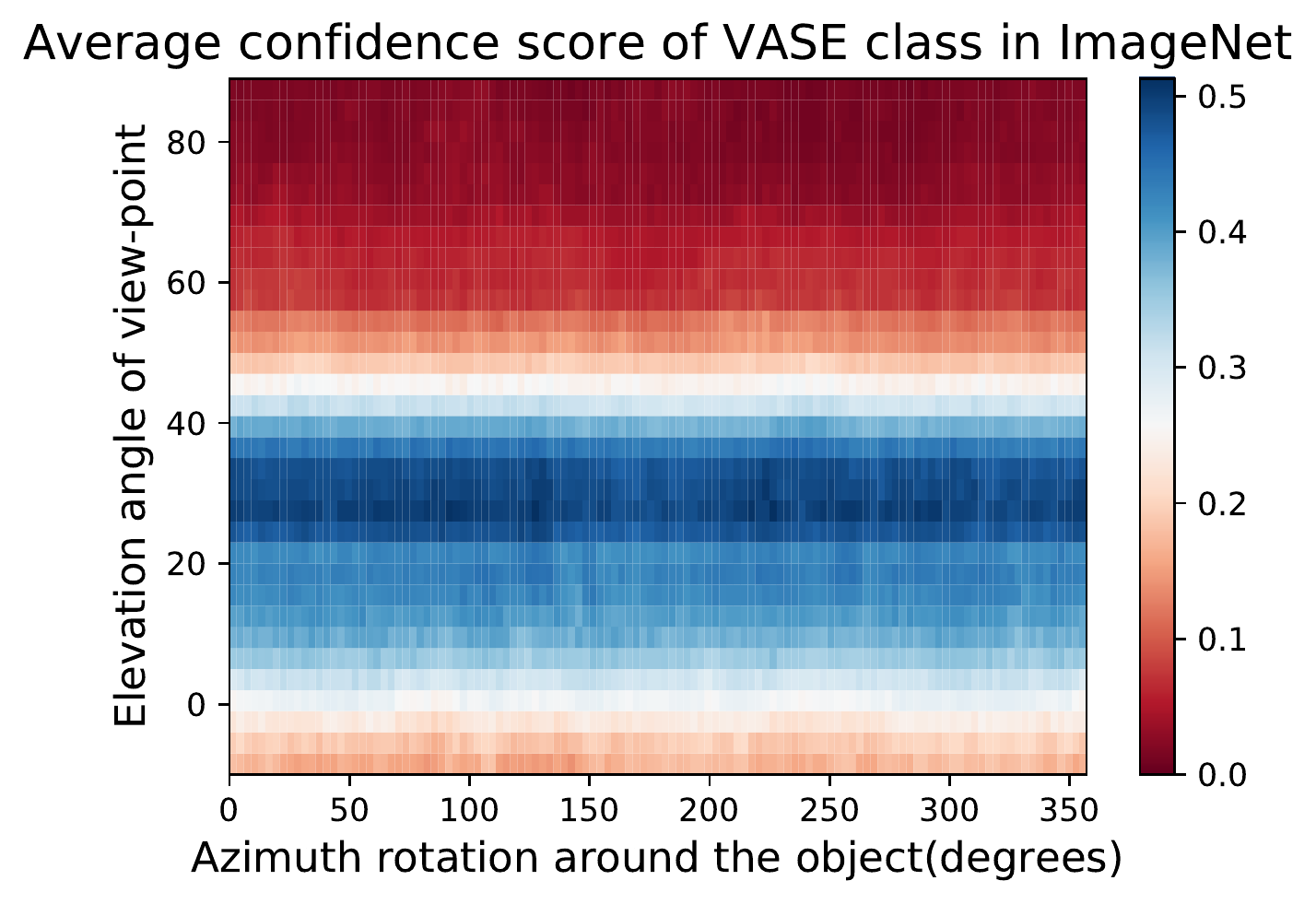}&
\includegraphics[width = 0.49\linewidth]{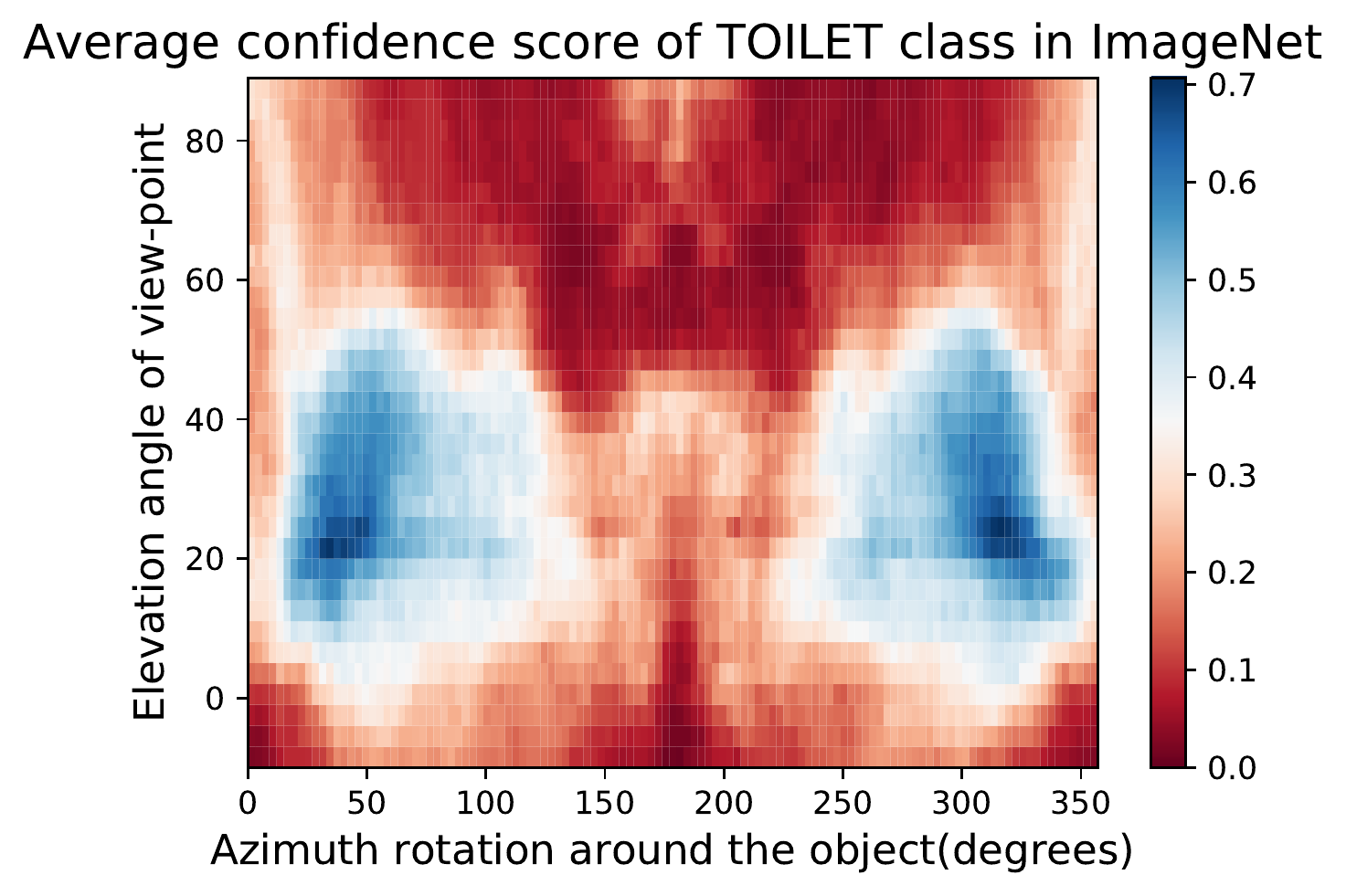}
\end{tabular}
   \caption{\small \textbf{Data Semantic Maps DSM-II}: visualizing Semantic Data Bias in the common training dataset (\ie ImageNet \cite{IMAGENET}) By averaging the Networks Semantic Maps (NSM) of different networks and on different shapes, Different classes have different semantic bias in ImageNet as clearly shown in the maps above. The symmetry in the maps are attributed to the 3D symmetry of the objects.}
   \vspace{-8pt}
   \label{fig:dsm2}
\end{figure*}

\clearpage
\clearpage
\section{Detailed Derivations of the Update Directions of the Bounds}
\subsection{Defining Robustness as an Operator}
In our case we consider a more general case where er are interested in the $\mathbf{u} \in \Omega \subset \mathbb{R}^{n}$ , a hidden latent parameter that generate the image and is passes to scene generator (\eg a renderer function $\mathbf{R}$) that takes the parameter $\mathbf{u}$ and a an object shape $\mathbf{S}$ of a class that is identified by classifier $\mathbf{C}$. $\Omega$ is the continuous semantic space for the parameters that we intend to study. The renderer creates the image $\mathbf{x} \in \mathbb{R}^{d}$, and then we study the behavior of a classifier $\mathbf{C}$ of that image across multiple shapes and multiple famous DNNs. Now, this function of interest is defined as follows. 
\begin{equation}
\begin{aligned} 
 f(\mathbf{u}) = \mathbf{C}_{z}(\mathbf{R}(\mathbf{S}_{z},\mathbf{u})) ~, ~~ 0\leq  f(\mathbf{u}) \leq 1
\label{eq:f-sup}
\end{aligned}
\end{equation}
Where $z$ is a class label of interest of study, and we observe the network score for that class by rendering a shape $\mathbf{S}_{z}$ of the same class. The shape and class labels are constants, and only the parameters vary for $f$. The robust-region-finding operator is then defined as follows 
\begin{equation}
\begin{aligned} 
& \mathbf{\Phi}_{\text{robust}}(f(\mathbf{u}),\mathbf{S}_{z},\mathbf{u}_{0}) = \mathbb{D} = \{\mathbf{u}: \mathbf{a} \leq \mathbf{u} \leq \mathbf{b}\} \\
  \text{s.t.}&~~ \mathbb{E}_{\mathbf{u}\sim \mathbb{D}} [f(\mathbf{u})] \ge 1-\epsilon_{m}~, ~~ \mathbf{u}_{0} \in \mathbb{D} ~, ~ \text{VAR}[f(\mathbf{u})] \le \epsilon_{v}
\label{eq:phi-rob-sup}
\end{aligned}
\end{equation}

where the left and right bounds of $\mathbb{D}$ are $\mathbf{a} = [a_{1},a_{2},...,a_{n}]$ and $\mathbf{b} = [b_{1},b_{2},...,b_{n}]$  respectively. The two samll thresholds $\epsilon_{m},\epsilon_{v}$ are to insure high performance and low variance of the DNN network in that robust region. We can define the opposite operator which is to find adversarial regions like follows :
\begin{equation}
\begin{aligned} 
& \mathbf{\Phi}_{\text{adv}}(f(\mathbf{u}),\mathbf{S}_{z},\mathbf{u}_{0}) = \mathbb{D} = \{\mathbf{u}: \mathbf{a} \leq \mathbf{u} \leq \mathbf{b}\} \\
 & \text{s.t.}~~ \mathbb{E}_{\mathbf{u}\sim \mathbb{D}} [f(\mathbf{u})] \leq \epsilon_{m}~, ~~ \mathbf{u}_{0} \in \mathbb{D} ~, ~ \text{VAR}[f(\mathbf{u})] \ge \epsilon_{v}
\label{eq:phi-adv-sup}
\end{aligned}
\end{equation}

We can show clearly that $\mathbf{\Phi}_{\text{adv}}$ and $\mathbf{\Phi}_{\text{robust}}$ are related as follows 
\begin{equation}
\begin{aligned} 
& \mathbf{\Phi}_{\text{adv}}(f(\mathbf{u}),\mathbf{S}_{z},\mathbf{u}_{0}) = \mathbf{\Phi}_{\text{robust}}(1-f(\mathbf{u}),\mathbf{S}_{z},\mathbf{u}_{0})
\label{eq:phi-adv-robust-sup}
\end{aligned}
\end{equation}
So we can just focus our attentions on $\mathbf{\Phi}_{\text{robust}}$ to find robust regions , and the adversarial regions follows directly from \eqLabel{\ref{eq:phi-adv-robust-sup}}.
\subsection{Divergence of the Bounds}
To develop an algorithm for $\mathbf{\Phi}$, we deploy the idea by \cite{ioc} which focus on maximizing the inner area of the function in the region and fitting the bounds to grow the region bounds. As we will show , maximizing the region by maximizing the integral can lead to divergence , as follows :
\begin{lemma} \label{thm:integral}
Let $f$ be a continuous scalar function $f: \mathbb{R}^{1} \rightarrow \mathbb{R}^{1} $, and let $L$ be the function defining the definite integral of $f$ in terms of the two integral bounds , \ie $L(a,b) = \int_{a}^{b}f(u)du$. Then, to maximize $L$, $-f(a)$ and $f(b)$ are valid ascent directions for the two bounds $a,b$ respectively. 
\end{lemma}
\begin{proof}
a direction $\mathbf{p}$ is an ascent direction of objective $L$ if it satisfies the inequality $\mathbf{p}^{T}\nabla L \geq 0$.\cite{Boyd}.\\ To find $\pd{l}{a} =  \pd{ }{a}\int_{a}^{b}f(u)du $, we use Leibniz rule from the fundamental theorem of calculus which states that $\frac { d } { d x } \int _ { a ( x ) } ^ { b ( x ) } f ( u ) d u = f ( b ( x ) ) \frac { d } { d x }b ( x ) - f ( a ( x ) )\frac { d } { d x } a ( x )$ \\
Therefore, $\pd{ }{a}\int_{a}^{b}f(u)du = f ( b ) \times 0 - f ( a ) \times 1 = - f ( a )$. Similarly, $\pd{ }{b}\int_{a}^{b}f(u)du = f ( b )$. By picking $\mathbf{p} = [-f(a), f(b)]^{T}$, then $\mathbf{p}^{T}\nabla L = f(a)^{2} + f(b)^{2} \geq 0$. This proves that $\mathbf{p}$ is a valid ascent direction for objective $L$. 
\end{proof}

\begin{theorem} \label{thm:unbounded}
Let $f$ be a positive continuous scalar function $f: \mathbb{R}^{1} \rightarrow (0,1) $, and let $L$ be the function defining the definite integral of $f$ in terms of the two integral bounds , \ie $L(a,b) = \int_{a}^{b}f(u)du$. Then, following the ascent direction in Lemma \ref{thm:integral} can diverge the bounds if followed in a gradient ascent technique with fixed learning rate  . 
\end{theorem}
\begin{proof}
If we follow the direction $= [-f(a), f(b)]^{T}$ , with a fixed learning rate $\eta$, then the update rules for $a,b$ will be as follws. $a_{k} = a_{k-1} - \eta f(a) , ~ b_{k} = b_{k-1} + \eta f(b) $. for initial points $a_{0},b_{0}$, then $a_{k}= a_{0} - \eta \sum_{i=0}^{k}f(\text{Area}_{\text{in}})$, and $b_{k}= b_{0} + \eta \sum_{i=0}^{k}f(b_{i})$ . We can see now if $f(u) = c, 0<c<1$, then as $k \rightarrow \infty$, the bounds $a_{k} \rightarrow - \infty, ~ b_{k} \rightarrow \infty$. This leads to the claimed divergence. 
\end{proof} \vspace{-8pt}
To solve the issue of bounds diverging we propose the following formulations for one dimensional bounds , and then we extend them to n- dimensions , which allows for finding the n-dimensional semantic robust/adversarial regions that are similar to figure 2 . some of the formulations are black-box in nature ( they dint need the gradient of the function $f $ in order to update the current estimates of the bound ) while others  . 

\subsection{Naive Approach} 
\begin{equation}
\begin{aligned} 
L = -\text{Area}_{\text{in}} + \frac{\lambda}{2} \left| b-a\right|_{2}^{2}
\label{eq:loss-naive-sup}
\end{aligned}
\end{equation}
using Libeniz rule as in Lemma \ref{thm:integral}, we get the following update steps for the objective L :
\begin{equation}
\begin{aligned} 
\pd{L}{a} =& f(a) - \lambda (b-a)\\
\pd{L}{b} =& - f(b) + \lambda (b-a)
\label{eq:update-naive-1-sup}
\end{aligned}
\end{equation}
where $\lambda$ is regularizing the boundaries not to extend too much in the case the function evaluation wa positive all the time.

\mysection{Extension to n-dimension:}
Lets start by $n=2$. Now, we have $f: \mathbb{R}^{2} \rightarrow (0,1)$, then we define the loss integral as a function of four bounds of a rectangular region as follows.
\begin{equation}
\begin{aligned} 
&L(a_{1},a_{2},b_{1},b_{2}) = \\ 
&- \int_{a_{2}}^{b_{2}}\int_{a_{1}}^{b_{1}}f(u,v)dvdu + \frac{\lambda}{2} \left| b_1 -a_1 \right|_{2}^{2} + \frac{\lambda}{2} \left| b_2 -a_2 \right|_{2}^{2} 
\label{eq:naive-integration2}
\end{aligned}
\end{equation}
We apply the trapezoidal approximation on the loss to obtain the following expression. 
\begin{equation}
\begin{aligned} 
&L(a_{1},a_{2},b_{1},b_{2})~ \approx ~\frac{\lambda}{2} \left| b_1 -a_1 \right|_{2}^{2} + \frac{\lambda}{2} \left| b_2 -a_2 \right|_{2}^{2} \\ 
& ~~- \frac{(b_{1}-a_{1})(b_{2}-a_{2})}{4}(~f(a_{1},a_{2})+f(b_{1},a_{2})+\\ &~~~~~~~~~~~~~~~~~~~~~~~~~~~~~~~~~~~~f(a_{1},b_{2})+f(b_{1},b_{2})~) 
\label{eq:naive-integration3}
\end{aligned}
\end{equation}
to find the update direction for the first bound $a_1$ by taking the partial derivative of the function in \eqLabel{\ref{eq:naive-integration2}} we get the following update direction for $a_1$ along with its trapezoidal approximation in order to able to compute it during the optimization:
\begin{equation}
\begin{aligned} 
&\pd{L}{a_{1}} =  \int_{a_{2}}^{b_{2}}f(a_1,v)dv ~~-~ \lambda (b_1-a_1)  \\  
&\approx \frac{(b_{2}-a_{2})}{2}\Big(~f(a_{1},a_{2}) + f(a_{1},b_{2}))~\Big) ~-~ \lambda (b_1-a_1)
\label{eq:update-naive-2-sup}
\end{aligned}
\end{equation}
Doing similar steps to the first bound for the other three bounds we obtain the full update directions for ($a_1 ,a_2 ,b_1 ,b_2$)
\begin{equation}
\begin{aligned} 
&\pd{L}{a_{1}} \approx \frac{(b_{2}-a_{2})}{2}\Big(~f(a_{1},a_{2}) + f(a_{1},b_{2}))~\Big) ~-~ \lambda (b_1-a_1) \\
&\pd{L}{b_{1}} \approx -\frac{(b_{2}-a_{2})}{2}\Big(~f(b_{1},a_{2}) + f(b_{1},b_{2}))~\Big) ~+~ \lambda (b_1-a_1)\\
&\pd{L}{a_{2}} \approx \frac{(b_{1}-a_{1})}{2}\Big(~f(a_{1},a_{2}) + f(b_{1},a_{2}))~\Big) ~-~ \lambda (b_2-a_2)\\
&\pd{L}{b_{2}} \approx -\frac{(b_{1}-a_{1})}{2}\Big(~f(a_{1},b_{2}) + f(b_{1},b_{2}))~\Big) ~+~ \lambda (b_2-a_2)
\label{eq:update-naive-3-sup}
\end{aligned}
\end{equation}
Now, for $f: \mathbb{R}^{n} \rightarrow (0,1)$, we define the inner region hyper-rectangle as before $\mathbb{D} = \{\mathbf{x}: \mathbf{a} \leq \mathbf{x} \leq \mathbf{b}\}$.Here , we assume the size of the region is positive at every dimension , \ie $\mathbf{r} =  \mathbf{b} -  \mathbf{a} > \mathbf{0} $. The volume of the region $\mathbb{D}$ normalized by exponent of dimension $n$ is expressed as follows
\begin{equation}
\begin{aligned} 
\text{volume}(\mathbb{D}) = \triangle = \frac{1}{2^{n}}\prod_{i=1}^{n}\mathbf{r}_{i} 
\label{eq:n-vol-sup}
\end{aligned}
\end{equation}
The region $\mathbb{D}$ can also be defined in terms of the matrix $\mathbf{D}$ of all the corner points $\{\mathbf{d}^{i}\}_{i=1}^{2^{n}}$ as follows.

\begin{equation}
\begin{aligned} 
\text{corners}&(\mathbb{D}) = \mathbf{D}_{n\times 2^{n}} = \left[\mathbf{d}^{1} | \mathbf{d}^{2} |.. | \mathbf{d}^{2^{n}}\right] \\
&\mathbf{D} = \mathbf{1}^{T}\mathbf{a}~ +~ \mathbf{M}^{T} \odot (\mathbf{1}^{T}\mathbf{r})
\label{eq:n-corners-sup}
\end{aligned}
\end{equation}

where $\mathbf{1}$ is the all-ones vector of size $2^{n}$, $\odot$ is the Hadamard product of matrices (element-wise) , and $\mathbf{M}$ is a constant  masking matrix defined as the matrix of binary numbers of n bits that range from 0 to $2{n} - 1 $ defined as follows.
\begin{equation}
\begin{aligned} 
\mathbf{M}_{n\times 2^{n}} = \left[\mathbf{m}^{0} | \mathbf{m}^{1} |.. | \mathbf{m}^{2^{n}-1}\right] ~, ~ \text{where}~~ \mathbf{m}^{i} = \text{binary}_{n}(i)
\label{eq:n-mask-sup}
\end{aligned}
\end{equation}
We define the function vector as the vector $\mathbf{f}_{\mathbb{D}}$ of all function evaluations at all corner points of $\mathbb{D}$
\begin{equation}
\begin{aligned} 
\mathbf{f}_{\mathbb{D}} &= \left[f(\mathbf{d}^{1}), f(\mathbf{d}^{2}),...,f(\mathbf{d}^{2^{n}}) \right]^{T} , ~ \mathbf{d^{i}} = \mathbf{D}_{:,i}
\label{eq:n-function-sup}
\end{aligned}
\end{equation}
We follow similar steps as in $n=2$ and obtain the following loss expressions and update directions :
\begin{equation}
\begin{aligned} 
L(\mathbf{a},\mathbf{b}) &= - \idotsint_\mathbb{D} f(u_1,\dots,u_n) \,du_1 \dots du_n  + \frac{\lambda}{2} \left| \mathbf{r}\right|^{2}\\ 
&\approx~ -\triangle\mathbf{1}^{T}\mathbf{f}_{\mathbb{D}} ~+~ \frac{\lambda}{2} \left| \mathbf{r}\right|^{2}\\ 
\nabla_{\mathbf{a}}L  &\approx~ 2\triangle\text{diag}^{-1}(\mathbf{r}) \overline{\mathbf{M}}\mathbf{f}_{\mathbb{D}} + \lambda \mathbf{r}\\
\nabla_{\mathbf{b}}L  &\approx~ -2\triangle\text{diag}^{-1}(\mathbf{r}) \mathbf{M}\mathbf{f}_{\mathbb{D}} - \lambda \mathbf{r}
\label{eq:n-loss-update-naive-sup}
\end{aligned}
\end{equation}

\subsection{Outer-Inner Ratio Loss (OIR)}
We introduce an outer region $A,B$ with bigger area that contains the small region $(a,b)$. We follow the following assumption to insure that outer area is always positive.
\begin{equation}
\begin{aligned} 
 A =  a - \alpha \frac{b-a}{2} ,  B =  b + \alpha \frac{b-a}{2}  
\label{eq:fixed-assumption}
\end{aligned}
\end{equation}
where $\alpha$ is the small boundary factor of the outer area to inner area. We formulate the problem as a ratio of outer over inner area and we try to make this ratio as close as possible to 0 . 
$L =  \frac{\text{Area}_{\text{out}}}{\text{Area}_{\text{in}}}  $
By using DencklBeck technique for solving non-linear fractional programming problems \cite{dinckl}. Using their formulation to transform $L$ as follows.
\begin{equation}
\begin{aligned} 
L &= \frac{\text{Area}_{\text{out}}}{\text{Area}_{\text{in}}} ~ =~ \text{Area}_{\text{out}} ~-~ \lambda ~ \text{Area}_{\text{in}} \\
&= \int_{A}^{B}f(a)du ~ -~ \int_{a}^{b}f(a)du ~ -~ \lambda ~\int_{a}^{b}f(a)du
\label{eq:loss-oir-sup}
\end{aligned}
\end{equation}
where $\lambda^{*} = \frac{\text{Area}_{\text{out}}^{*}}{\text{Area}_{\text{in}}^{*}}$ is the DencklBeck factor and it is equal to the small objective best achieved.

\mysection{Black-Box (OIR\_B)}

Here we set $\lambda = 1$ to simplify the problem. This yields the following expression of the loss 
\begin{equation}
\begin{aligned} 
L &=  \text{Area}_{\text{out}} - \text{Area}_{\text{in}} \\
&= \int_{A}^{a}f(u)du + \int_{b}^{B}f(u)du - \int_{a}^{b}f(u)du  \\
 &= \int_{A}^{B}f(u)du  - 2\int_{a}^{b}f(u)du \\
 &= \int_{a - \alpha \frac{b-a}{2}}^{b + \alpha \frac{b-a}{2} }f(u)du  - 2\int_{a}^{b}f(u)du\label{eq:oir-b-loss}
\end{aligned}
\end{equation}
using Libeniz ruloe as in Lemma \ref{thm:integral}, we get the following update steps for the objective L :
\begin{equation}
\begin{aligned} 
\pd{L}{a} =& -(1+ \frac{\alpha}{2})f(A) - \frac{\alpha}{2}f(B) + 2f(a) \\
\pd{L}{b} =& (1+ \frac{\alpha}{2})f(B) + \frac{\alpha}{2}f(A) - 2f(b)
\label{eq:update-ioc-1}
\end{aligned}
\end{equation}

\mysection{Extension to n-dimension:}
Lets start by $n=2$. Now, we have $f: \mathbb{R}^{2} \rightarrow (0,1)$, and with the following constrains on the outer region.
\begin{equation}
\begin{aligned} 
 A_1 =  a_{1}-\frac{b_{1}-a_{1}}{2} ,~~  B_1 =  b_{1}+\frac{b_{1}-a_{1}}{2} \\
  A_2 =  a_{2}-\frac{b_{2}-a_{2}}{2} ,~~  B_1 =  b_{2}+\frac{b_{2}-a_{2}}{2}
\label{eq:fixed-assumption-2}
\end{aligned}
\end{equation}
we define the loss integral as a function of four bounds of a rectangular region as follows.
\begin{equation}
\begin{aligned} 
&L(a_{1},a_{2},b_{1},b_{2}) =\\ &  \int_{A_2}^{B_2}\int_{A_1}^{B_1}f(u,v)dvdu ~- ~2  \int_{a_{2}}^{b_{2}}\int_{a_{1}}^{b_{1}}f(u,v)dvdu 
\label{eq:outer-integration2}
\end{aligned}
\end{equation}
We apply the trapezoidal approximation on the loss to obtain the following expression. 
\begin{equation}
\begin{aligned} 
&L(a_{1},a_{2},b_{1},b_{2}) \approx \\ 
& \frac{(B_{1}-A_{1})(B_{2}-A_{2})}{4}(~f(A_{1},A_{2})+f(B_{1},A_{2})+ \\ &~~~~~~~~~~~~~~~~~~~~~~~~~~~~~~~~~~~~f(A_{1},B_{2})+f(B_{1},B_{2})~)\\
&~~- \frac{(b_{1}-a_{1})(b_{2}-a_{2})}{2}(~f(a_{1},a_{2})+f(b_{1},a_{2})+\\ &~~~~~~~~~~~~~~~~~~~~~~~~~~~~~~~~~~~~f(a_{1},b_{2})+f(b_{1},b_{2})~) \\
&= \frac{(b_{1}-a_{1})(b_{2}-a_{2})}{4}( \\
&~~~~(1+ \alpha )^2 (~f(A_{1},A_{2})+f(B_{1},A_{2})+ \\ &~~~~~~~~~~~~~~~~~~~~~~~~~~~~~~~~~~~~f(A_{1},B_{2})+f(B_{1},B_{2})~)\\
&~~~~- 2(~f(a_{1},a_{2})+f(b_{1},a_{2})+f(a_{1},b_{2})+f(b_{1},b_{2})~)~~)
\label{eq:outer-integration3}
\end{aligned}
\end{equation}
to find the update direction for the first bound $a_1$ by taking the partial derivative of the function in \eqLabel{\ref{eq:outer-integration2}} we get the following update direction for $a_1$ along with its trapezoidal approximation in order to able to compute it during the optimization:
\begin{equation}
\begin{aligned} 
&\pd{L}{a_{1}} = -(1+\frac \alpha2)\int_{A_2}^{B_2}\Big(f(A_1,v) + \frac \alpha2 f(B_1,v) \Big)dv  \\ & +~2  \int_{a_{2}}^{b_{2}}f(a_1,v)dv  \\  
&\approx \frac{(b_{2}-a_{2})}{2}( \\
&~~~~-(1+ \alpha )(~(1+ \frac \alpha2 )(f(A_{1},A_{2}) + f(A_{1},B_{2}))\\ &~~~~~~~~~~~~~~~~~~~~~~~~~ +\frac \alpha2 ~~(f(B_{1},A_{2}) + f(B_{1},B_{2}))~)\\
&~~~~+ 2\Big(~f(a_{1},a_{2}) + f(a_{1},b_{2}))~\Big)
\label{eq:update-outer-2}
\end{aligned}
\end{equation}
Doing similar steps to the first bound for the other three bounds we obtain the full update directions for ($a_1 ,a_2 ,b_1 ,b_2$)
\begin{equation}
\begin{aligned} 
&\pd{L}{a_{1}} \approx \frac{(b_{2}-a_{2})}{2}( \\
&~~~~-(1+ \alpha )(~(1+ \frac \alpha2 )(f(A_{1},A_{2}) + f(A_{1},B_{2}))\\ &~~~~~~~~~~~~~~~~~~~~~~~~~ +\frac \alpha2 ~~(f(B_{1},A_{2}) + f(B_{1},B_{2}))~)\\
&~~~~+ 2\Big(~f(a_{1},a_{2}) + f(a_{1},b_{2}))~\Big)
\label{eq:update-outer-2-1}
\end{aligned}
\end{equation}
\begin{equation}
\begin{aligned} 
&\pd{L}{b_{1}} \approx \frac{(b_{2}-a_{2})}{2}( \\
&~~~~(1+ \alpha )(~(1+ \frac \alpha2 )(f(B_{1},A_{2}) + f(B_{1},B_{2}))~)\\ &~~~~~~~~~~~~~~~~~~~~~~~~~ +\frac \alpha2 ~~(f(A_{1},A_{2}) + f(A_{1},B_{2}))\\
&~~~~- 2\Big(~f(b_{1},a_{2}) + f(b_{1},b_{2}))~\Big)
\label{eq:update-outer-2-2}
\end{aligned}
\end{equation}
\begin{equation}
\begin{aligned} 
&\pd{L}{a_{2}} \approx \frac{(b_{1}-a_{1})}{2}( \\
&~~~~-(1+ \alpha )(~(1+ \frac \alpha2 )(f(A_{1},A_{2}) + f(B_{1},A_{2}))\\ &~~~~~~~~~~~~~~~~~~~~~~~~~ +\frac \alpha2 ~~(f(A_{1},B_{2}) + f(B_{1},B_{2}))~)\\
&~~~~+ 2\Big(~f(a_{1},a_{2}) + f(b_{1},a_{2}))~\Big)
\label{eq:update-outer-2-3}
\end{aligned}
\end{equation}
\begin{equation}
\begin{aligned} 
&\pd{L}{b_{2}} \approx \frac{(b_{1}-a_{1})}{2}( \\
&~~~~(1+ \alpha )(~(1+ \frac \alpha2 )(f(A_{1},B_{2}) + f(B_{1},B_{2}))\\ &~~~~~~~~~~~~~~~~~~~~~~~~~ +\frac \alpha2 ~~(f(A_{1},A_{2}) + f(B_{1},A_{2}))~)\\
&~~~~- 2\Big(~f(a_{1},b_{2}) + f(b_{1},b_{2}))~\Big)
\label{eq:update-outer-2-4}
\end{aligned}
\end{equation}
Now, for $f: \mathbb{R}^{n} \rightarrow (0,1)$, we define the inner region hyper-rectangle as before $\mathbb{D} = \{\mathbf{x}: \mathbf{a} \leq \mathbf{x} \leq \mathbf{b}\}$, but now define an outer bigger region $\mathbb{Q}$ that include the smaller region $\mathbb{D} $ and defined as follows : $\mathbb{Q} = \{\mathbf{x}: \mathbf{a} - \frac{\alpha}{2}\mathbf{r} \leq \mathbf{x} \leq \mathbf{b} + \frac{\alpha}{2}\mathbf{r} \}$ , where $\mathbf{a},\mathbf{b},\mathbf{r}$ are defined as before, while $\alpha$ is defined as the boundary factor of the outer region in all the dimensions equivilantly. The inner and outer regions can also be defined in terms of the corner points as follows.
\begin{equation}
\begin{aligned} 
\text{corners}(\mathbb{D}) &= \mathbf{D}_{n\times 2^{n}} = \left[\mathbf{d}^{1} | \mathbf{d}^{2} |.. | \mathbf{d}^{2^{n}}\right] \\
\mathbf{D} &= \mathbf{1}^{T}\mathbf{a}~ +~ \mathbf{M}^{T} \odot (\mathbf{1}^{T}\mathbf{r}) \\
\text{corners}(\mathbb{Q}) &= \mathbf{Q}_{n\times 2^{n}} = \left[\mathbf{q}^{1} | \mathbf{q}^{2} |.. | \mathbf{q}^{2^{n}}\right] \\
\mathbf{Q} &= \mathbf{1}^{T}(\mathbf{a} - \frac{\alpha}{2}\mathbf{r})~ +~ (1+\alpha)\mathbf{M}^{T} \odot (\mathbf{1}^{T}\mathbf{r}) 
\label{eq:n-corners2-sup}
\end{aligned}
\end{equation}
where $\mathbf{1}$ is the all-ones vector of size $2^{n}$, $\odot$ is the Hadamard product of matrices (elemnt-wise) , and $\mathbf{M}$ is a constant  masking matrix defined in \eqLabel{\ref{eq:n-mask-sup}}. Now we define two function vectors evaluated at all possible inner and outer corner points respectively.
\begin{equation}
\begin{aligned} 
\mathbf{f}_{\mathbb{D}} &= \left[f(\mathbf{d}^{1}), f(\mathbf{d}^{2}),...,f(\mathbf{d}^{2^{n}}) \right]^{T} , ~ \mathbf{d^{i}} = \mathbf{D}_{:,i}\\
\mathbf{f}_{\mathbb{Q}} &= \left[f(\mathbf{q}^{1}), f(\mathbf{q}^{2}),...,f(\mathbf{q}^{2^{n}}) \right]^{T} , ~ \mathbf{c^{i}} = \mathbf{Q}_{:,i}
\label{eq:n-function-outer-sup}
\end{aligned}
\end{equation}
Now the loss and update directions for the n-dimensional case becomes as follows .
\begin{equation}
\begin{aligned} 
L(\mathbf{a},\mathbf{b}) &= \idotsint_\mathbb{Q} f(u_1,\dots,u_n) \,du_1 \dots du_n \\ 
&- 2 \idotsint_\mathbb{D} f(u_1,\dots,u_n) \,du_1 \dots du_n \\ 
&\approx \triangle\Big((1+\alpha)^{n} \mathbf{1}^{T}\mathbf{f}_{\mathbb{Q}} ~-~ 2 ~ \mathbf{1}^{T}\mathbf{f}_{\mathbb{D}}\Big)\\ 
\nabla_{\mathbf{a}}L  \approx &2\triangle\text{diag}^{-1}(\mathbf{r}) \Big(2\overline{\mathbf{M}}\mathbf{f}_{\mathbb{D}} ~-~ \overline{\mathbf{M}}_{\mathbb{Q}}\mathbf{f}_{\mathbb{Q}}  \Big) \\
\nabla_{\mathbf{b}}L  \approx &2\triangle\text{diag}^{-1}(\mathbf{r}) \Big(-2\mathbf{M}\mathbf{f}_{\mathbb{D}} ~+~ \mathbf{M}_{\mathbb{Q}}\mathbf{f}_{\mathbb{Q}}  \Big)
\label{eq:n-loss-update-outer-sup}
\end{aligned}
\end{equation}
wheere  $\overline{\mathbf{M}}_{\mathbb{Q}}$ is the outer region mask defined as follows. 
\begin{equation}
\begin{aligned} 
\overline{\mathbf{M}}_{\mathbb{Q}} = &~  (1+\alpha)^{n-1}\Big((1+\frac{\alpha}{2})\overline{\mathbf{M}}+\frac{\alpha}{2}\mathbf{M}\Big)\\ 
\mathbf{M}_{\mathbb{Q}} = &~  (1+\alpha)^{n-1}\Big((1+\frac{\alpha}{2})\mathbf{M}+\frac{\alpha}{2}\overline{\mathbf{M}}\Big)
\label{eq:n-mask-outer-sup}
\end{aligned}
\end{equation}

\mysection{White-Box OIR (OIR\_W})
 The following formulation is white-box in nature ( it need the gradient of the function $f $ in order to update the current estimates of the bound ) this is useful when the function in hand is differntiable ( \eg DNN ) , to obtain more intelligent regions, rather then the regions surrounded by near 0 values of the function $f$. We set $\lambda = \frac{\alpha}{\beta}$ in \eqLabel{\ref{eq:loss-oir-sup}}, where $\alpha$ is the small boundary factor of the outer area, $\beta$ is the emphasis factor (we will show later how it determines the emphasis on the function vs the gradient). Hence, the objective in \eqLabel{\ref{eq:loss-oir-sup}} becomes :
\begin{equation}
\begin{aligned} 
&\argmin_{a,b} L = \argmin_{a,b}~ \text{Area}_{\text{out}} - \lambda ~ \text{Area}_{\text{in}} \\
=& \argmin_{a,b}~ \int_{A}^{a}f(u)du + \int_{b}^{B}f(u)du - \frac{\alpha}{\beta} \int_{a}^{b}f(u)du \\
=& \argmin_{a,b}~ \frac{\beta}{\alpha} \int_{a - \alpha \frac{b-a}{2}}^{b + \alpha \frac{b-a}{2} }f(u)du  - (1+\frac{\beta}{\alpha})\int_{a}^{b}f(u)du \\ 
\label{eq:loss-oir2-sup}
\end{aligned}
\end{equation}
using Libeniz ruloe as in Lemma \ref{thm:integral}, we get the following derivatives of the bound $a$ :
\begin{equation}
\begin{aligned} 
\pd{L}{a} &= \frac{\beta}{\alpha}\Big(f(a) - f\Big(a - \alpha \frac{b-a}{2}\Big) \Big) \\&~- \frac{\beta}{2}f\Big(b + \alpha \frac{b-a}{2}\Big) ~ - \frac{\beta}{2}f\Big(a - \alpha \frac{b-a}{2}\Big) + f(a) 
\label{eq:update-oir-1}
\end{aligned}
\end{equation}
now since $\lambda^{*}$ should be small for the optimal objective , then as $\lambda \rightarrow 0 ,~ \alpha \rightarrow 0$ and hence the derivative in \eqLabel{\ref{eq:update-oir-1}} becomes the following.  
\begin{equation}
\begin{aligned} 
\lim_{\alpha \to 0}\pd{L}{a} &= \lim_{\alpha \to 0} \beta\frac{\Big(f(a) - f\Big(a - \alpha \frac{b-a}{2}\Big) \Big)}{\alpha} \\&~- \frac{\beta}{2}f(b) ~- \frac{\beta}{2}f(a) + f(a) 
\label{eq:update-oir-2}
\end{aligned}
\end{equation}
We can see that the first term is proportional to the derivative of $f$ at $a$, and hence the expression becomes :
\begin{equation}
\begin{aligned} 
\lim_{\alpha \to 0}\pd{L}{a} &= \frac{\beta}{2}\Big((b-a)f^{\prime}(a) + f(b)\Big) ~+~ (1-\frac{\beta}{2})f(a) 
\label{eq:update-oir-3-sup}
\end{aligned}
\end{equation}
we can see that the update rule depends on the function value and the derivative of $f$ at the boundary $a$ with $\beta$ controlling the dependence. Similarly for the boundary $b$ we can see the following direction
\begin{equation}
\begin{aligned} 
\lim_{\alpha \to 0}\pd{L}{b} &= \frac{\beta}{2}\Big((b-a)f^{\prime}(b) + f(a)\Big) ~-~ (1-\frac{\beta}{2})f(b) 
\label{eq:update-oir-4}
\end{aligned}
\end{equation}
If $\beta \rightarrow 0 $, the update directions in \eqLabel{\ref{eq:update-oir-3-sup},\ref{eq:update-oir-4}} collapse to the unregularized naive update in \eqLabel{\ref{eq:update-naive-1-sup}} .\\

\mysection{Extension to n-dimension}.
Lets start by $n=2$. Now, we have $f: \mathbb{R}^{2} \rightarrow (0,1)$, and with the following constrains on the outer region.
\begin{equation}
\begin{aligned} 
 A_1 =  a_{1}-\frac{b_{1}-a_{1}}{2} ,~~  B_1 =  b_{1}+\frac{b_{1}-a_{1}}{2} \\
  A_2 =  a_{2}-\frac{b_{2}-a_{2}}{2} ,~~  B_1 =  b_{2}+\frac{b_{2}-a_{2}}{2}
\label{eq:fixed-assumption-2-w}
\end{aligned}
\end{equation}
we follow similar approach as in \eqLabel{\ref{eq:loss-oir2-sup}} to obtain the following expression.
\begin{equation}
\begin{aligned} 
&L(a_{1},a_{2},b_{1},b_{2}) =\\ &  \frac{\beta}{\alpha} \int_{A_2}^{B_2}\int_{A_1}^{B_1}f(u,v)dvdu ~- ~(1+\frac{\beta}{\alpha}) \int_{a_{2}}^{b_{2}}\int_{a_{1}}^{b_{1}}f(u,v)dvdu 
\label{eq:outer-integration2-w}
\end{aligned}
\end{equation}
We apply the trapezoidal approximation on the loss to obtain the following approximation. 
\begin{equation}
\begin{aligned} 
&L(a_{1},a_{2},b_{1},b_{2}) \approx -1 + (1+ \alpha )^2  \\ 
& \frac{(~f(A_{1},A_{2})+f(B_{1},A_{2})+f(A_{1},B_{2})+f(B_{1},B_{2})~)}{(~f(a_{1},a_{2})+f(b_{1},a_{2})+f(a_{1},b_{2})+f(b_{1},b_{2})~)}
\label{eq:outer-integration3-w}
\end{aligned}
\end{equation}
to find the update direction for the first bound $a_1$ by taking the partial derivative of the function in \eqLabel{\ref{eq:outer-integration2}} we get the fol owing update direction for $a_1$ along with its trapezoidal approximation in order to able to compute it during the optimization:
\begin{equation}
\begin{aligned} 
&\pd{L}{a_{1}} = -\frac{\beta}{\alpha}(1+\frac \alpha2)\int_{A_2}^{B_2}\Big(f(A_1,v) + \frac \alpha2 f(B_1,v) \Big)dv  \\ & ~~~~~~~~~~~~+~(1+\frac{\beta}{\alpha})  \int_{a_{2}}^{b_{2}}f(a_1,v)dv  \\  
&\approx \frac{(b_{2}-a_{2})}{2}( \\
&~~~~-(1+ \alpha )(~(\frac{\beta}{\alpha}+ \frac \beta2 )(f(A_{1},A_{2}) + f(A_{1},B_{2}))\\ &~~~~~~~~~~~~~~~~~~~~~~~~~ +\frac \beta2 ~~(f(B_{1},A_{2}) + f(B_{1},B_{2}))~)\\
&~~~~+ (1+\frac{\beta}{\alpha})\Big(~f(a_{1},a_{2}) + f(a_{1},b_{2}))~\Big)
\label{eq:update-outer-2-w}
\end{aligned}
\end{equation}
grouping the terms which are divided by $\alpha$ together and then taking the limit of $\alpha \rightarrow \infty $ (as explained in the 1-d case ), we get the following expressions.
\begin{equation}
\begin{aligned} 
& \lim_{\alpha \to 0} \pd{L}{a_{1}} \approx \frac{(b_{2}-a_{2})}{2}( \\
&\lim_{\alpha \to 0} \frac{\beta}{\alpha}(~(~f(a_{1},a_{2}) + f(a_{1},b_{2})) - (f(A_{1},A_{2}) + f(A_{1},B_{2})) ~)\\ &~~~~-\lim_{\alpha \to 0}\frac{3\beta}{2} ~~(f(A_{1},A_{2}) + f(A_{1},B_{2}) ~) \\ &~~~~-\lim_{\alpha \to 0} \frac \beta2 (~f(B_{1},A_{2}) + f(B_{1},B_{2})~)\\
&~~~~+ \Big(~f(a_{1},a_{2}) + f(a_{1},b_{2})~)~\Big)
\label{eq:update-outer-2-w-2}
\end{aligned}
\end{equation}
Noting that the first term is related to the directional derivatives of $f$, we get the following limit expression
\begin{equation}
\begin{aligned} 
& \lim_{\alpha \to 0} \pd{L}{a_{1}} \approx \frac{(b_{2}-a_{2})}{2}( \\
&\beta(~~\nabla f(a_1,a_2)\cdot \bigl( \begin{smallmatrix} \frac{b_1 -a_1 }{2}\\ \frac{b_2 -a_2}{2} \end{smallmatrix} \bigr) + \nabla f(a_1,b_2)\cdot \bigl( \begin{smallmatrix} \frac{b_1 -a_1 }{2}\\ -\frac{b_2 -a_2}{2} \end{smallmatrix} \bigr)~~)\\ &~~~~+\Big(1-\frac{3\beta}{2}\Big)               \Big(~f(a_{1},a_{2}) + f(a_{1},b_{2})~)~\Big) \\ &~~~~- \frac \beta2 (~f(b_{1},a_{2}) + f(b_{1},b_{2})~)
\label{eq:update-outer-2-w-3}
\end{aligned}
\end{equation}
Doing similar steps to the first bound for the other three bounds we obtain the full update directions for ($a_1 ,a_2 ,b_1 ,b_2$)
\begin{equation}
\begin{aligned} 
& \lim_{\alpha \to 0} \pd{L}{a_{1}} \approx \frac{(b_{2}-a_{2})}{2}( \\
&\beta(~~\nabla f(a_1,a_2)\cdot \bigl( \begin{smallmatrix} (b_1 -a_1)\\ (b_2 -a_2) \end{smallmatrix} \bigr) + \nabla f(a_1,b_2)\cdot \bigl( \begin{smallmatrix} (b_1 -a_1)\\ -(b_2 -a_2) \end{smallmatrix} \bigr)~~)\\ &~~~~+\Big(1-\frac{3\beta}{2}\Big)               \Big(~f(a_{1},a_{2}) + f(a_{1},b_{2})~)~\Big) \\ &~~~~- \frac \beta2 (~f(b_{1},a_{2}) + f(b_{1},b_{2})~)
\label{eq:update-outer-2-1-w}
\end{aligned}
\end{equation}
\begin{equation}
\begin{aligned} 
& \lim_{\alpha \to 0} \pd{L}{b_{1}} \approx \frac{(b_{2}-a_{2})}{2}( \\
&\beta(~~\nabla f(b_1,a_2)\cdot \bigl( \begin{smallmatrix} (b_1 -a_1)\\ -(b_2 -a_2) \end{smallmatrix} \bigr) + \nabla f(b_1,b_2)\cdot \bigl( \begin{smallmatrix} (b_1 -a_1)\\ (b_2 -a_2) \end{smallmatrix} \bigr)~~)\\ &~~~~-\Big(1-\frac{3\beta}{2}\Big) \Big(~f(b_{1},a_{2}) + f(b_{1},b_{2})~)~\Big) \\ &~~~~+ \frac \beta2 (~f(a_{1},a_{2}) + f(a_{1},b_{2})~)
\label{eq:update-outer-2-2-w}
\end{aligned}
\end{equation}
\begin{equation}
\begin{aligned} 
& \lim_{\alpha \to 0} \pd{L}{a_{2}} \approx \frac{(b_{2}-a_{2})}{2}( \\
&\beta(~~\nabla f(a_1,a_2)\cdot \bigl( \begin{smallmatrix} (b_1 -a_1)\\ (b_2 -a_2) \end{smallmatrix} \bigr) + \nabla f(b_1,a_2)\cdot \bigl( \begin{smallmatrix} -(b_1 -a_1)\\ (b_2 -a_2) \end{smallmatrix} \bigr)~~)\\ &~~~~+\Big(1-\frac{3\beta}{2}\Big)               \Big(~f(a_{1},a_{2}) + f(b_{1},a_{2})~)~\Big) \\ &~~~~- \frac \beta2 (~f(a_{1},b_{2}) + f(b_{1},b_{2})~)
\label{eq:update-outer-2-3-w}
\end{aligned}
\end{equation}
\begin{equation}
\begin{aligned} 
& \lim_{\alpha \to 0} \pd{L}{b_{2}} \approx \frac{(b_{2}-a_{2})}{2}( \\
&\beta(~~\nabla f(a_1,b_2)\cdot \bigl( \begin{smallmatrix} -(b_1 -a_1)\\ (b_2 -a_2) \end{smallmatrix} \bigr) + \nabla f(a_2,b_2)\cdot \bigl( \begin{smallmatrix} (b_1 -a_1)\\ (b_2 -a_2) \end{smallmatrix} \bigr)~~)\\ &~~~~-\Big(1-\frac{3\beta}{2}\Big)               \Big(~f(a_{1},b_{2}) + f(b_{1},b_{2})~)~\Big) \\ &~~~~+ \frac \beta2 (~f(a_{1},a_{2}) + f(b_{1},a_{2})~)
\label{eq:update-outer-2-4-w}
\end{aligned}
\end{equation}
Now, for $f: \mathbb{R}^{n} \rightarrow (0,1)$Following previous notations we have the following expressions for the loss and update directions for the bound 
\begin{equation}
\begin{aligned} 
L(\mathbf{a},\mathbf{b})  \approx&~~ \frac{(1+\alpha)^{n} \mathbf{1}^{T}\mathbf{f}_{\mathbb{Q}}}{\mathbf{1}^{T}\mathbf{f}_{\mathbb{D}}} ~-~ 1\\ 
\nabla_{\mathbf{a}}L  \approx & ~\triangle  \Big(\text{diag}^{-1}(\mathbf{r})\overline{\mathbf{M}}_{\mathbb{D}}\mathbf{f}_{\mathbb{D}} ~+~ \beta\text{diag}(\overline{\mathbf{M}}\mathbf{G}_{\mathbb{D}})~+ \beta \overline{\mathbf{s}}  \Big)  \\
\nabla_{\mathbf{b}}L  \approx & ~\triangle  \Big(- \text{diag}^{-1}(\mathbf{r})\mathbf{M}_{\mathbb{D}}\mathbf{f}_{\mathbb{D}} ~+~ \beta\text{diag}(\mathbf{M}\mathbf{G}_{\mathbb{D}})~+ \beta \mathbf{s}  \Big)
\label{eq:n-loss-update-grad-sup}
\end{aligned}
\end{equation}
where the mask is the special mask 
\begin{equation}
\begin{aligned} 
 \overline{\mathbf{M}}_{\mathbb{D}} =&~ \Big( \gamma_n \overline{\mathbf{M}} ~-~\beta \mathbf{M}  \Big) \\
  \mathbf{M}_{\mathbb{D}} =&~ \Big( \gamma_n \mathbf{M} ~-~\beta \overline{\mathbf{M}}  \Big) \\
  \gamma_n =&~ 2~-~\beta(2n-1) 
\label{eq:n-mask-grad-sup}
\end{aligned}
\end{equation}
Where diag(.) is the diagonal matrix of the vector argument or the diagonal vector of the matrix argument. $\mathbf{s}$ is a weighted sum of the gradient from other dominions ($i \neq k$) contributing to the update direction of dimension $k$, where $k \in \{1 , 2 ,...,n\}$.
\begin{equation}
\begin{aligned} 
\mathbf{s}_{k} &= \frac{1}{\mathbf{r}_{k}}\sum_{i=1, i\neq k }^{n}\mathbf{r}_{i}( (\overline{\mathbf{M}}_{i,:} - \mathbf{M}_{i,:})\odot \overline{\mathbf{M}}_{k,:} ) \mathbf{G}_{:,i} \\
\overline{\mathbf{s}}_{k} &= \frac{1}{\mathbf{r}_{k}}\sum_{i=1, i\neq k }^{n}\mathbf{r}_{i}( ( \mathbf{M}_{i,:} - \overline{\mathbf{M}}_{i,:} )\odot \mathbf{M}_{k,:} ) \mathbf{G}_{:,i} \\
&~~~k \in \{1 , 2 ,...,n\}
\label{eq:n-update-grad-selection-sup}
\end{aligned}
\end{equation}

\subsection{Trapezoidal Approximation Formulation} 
Here we use the trapezoidal approximation of the integral, a first-order approximation from Newton-Cortes formulas for numerical integration \cite{numerical}. The rule states that a definite integral can be approximated as follows: 
\begin{equation}
\begin{aligned} 
\int_{a}^{b}f(u)du \approx (b-a)\frac{f(a)+f(b)}{2}
\label{eq:trapezoidal-integration}
\end{aligned}
\end{equation}
asymptotic error estimate is given by $-\frac { ( b - a ) ^ { 2 } } { 48 } \left[ f ^ { \prime } ( b ) - f ^ { \prime } ( a ) \right] + \mathcal{O} \Big( 0.125 \Big)$. So as long the derivatives are bounded by some lipschitz constant $\mathbb{L}$, then the error becomes bounded by the following $|\text{error}| \leq \mathbb{L}( b - a ) ^ { 2 }  $. The regularized loss of interest in \eqLabel{\ref{eq:loss-naive-sup}} becomes the following .
\begin{equation}
\begin{aligned} 
L &= -\text{Area}_{\text{in}} + \lambda \left| b-a\right|_{2}^{2}\\ &\approx -(b-a)\frac{f(a)+f(b)}{2} + \lambda \left| b-a\right|_{2}^{2}  
\label{eq:loss-trap-1}
\end{aligned}
\end{equation}
taking the derivative of $L$ approximation directly with respect to these bounds , yields the following update directions which are different from the expressions in \eqLabel{\ref{eq:update-naive-1-sup}}
\begin{equation}
\begin{aligned} 
\pd{L}{a} = -\frac{b - a}{2}f^{\prime}(a)~ +~\frac{f(a) + f(b)}{2}  ~- \lambda (b-a)\\
\pd{L}{b} = -\frac{b - a}{2}f^{\prime}(b) ~-~\frac{f(a) + f(b)}{2}  ~+ \lambda (b-a)\\
\label{eq:update-trap-1}
\end{aligned}
\end{equation}
note that it needs the first derivative $f^{\prime}(.)$ of the function $f$ evaluated at the bound to update that bound.

\mysection{Extension to n-dimensions}\\
Lets start by $n=2$. Now, we have $f: \mathbb{R}^{2} \rightarrow (0,1)$, and we define the loss integral as a function of four bounds of a rectangler region and apply the trapezoidal approximation as follows.
\begin{equation}
\begin{aligned} 
&L(a_{1},a_{2},b_{1},b_{2}) = - \int_{a_{1}}^{b_{1}}\int_{a_{2}}^{b_{2}}f(u,v)dvdu \\&~~~~~~~~~~~~~~~~~~~~~~~~+ \frac{\lambda}{2} \left| b_1 -a_1 \right|_{2}^{2} + \frac{\lambda}{2} \left| b_2 -a_2 \right|_{2}^{2}  \\
&\approx ~~~~ \frac{\lambda}{2} \left| b_1 -a_1 \right|_{2}^{2} + \frac{\lambda}{2} \left| b_2 -a_2 \right|_{2}^{2}  \\&~~~~~- \frac{(b_{1}-a_{1})(b_{2}-a_{2})}{4}\Big( ~f(a_{1},a_{2})+f(b_{1},a_{2})+\\ &~~~~~~~~~~~~~~~~~~~~~~~~~~~~~~~~~~~~f(a_{1},b_{2})+f(b_{1},b_{2})~ \Big)
\label{eq:trapezoidal-integration2}
\end{aligned}
\end{equation}
Then following similar steps as in the one-dimensional case we can obtain the following update directions for the four bounds 
\begin{equation}
\begin{aligned} 
~~~~&\pd{L}{a_{1}} = (b_{1}-a_{1})(b_{2}-a_{2})\frac{f^{\prime}_{1}(a_{1},a_{2}) + f^{\prime}_{1}(a_{1},b_{2})}{4}\\ &-(b_{2} - a_{2})\frac{f(a_{1},a_{2}) + f(a_{1},b_{2}) + f(b_{1},a_{2}) + f(b_{1},b_{2}) }{4} \\ &-~ \lambda (b_1-a_1)
\label{eq:update-trap-2-1}
\end{aligned}
\end{equation}
\begin{equation}
\begin{aligned} 
~~~~&\pd{L}{a_{2}} = (b_{1}-a_{1})(b_{2}-a_{2})\frac{f^{\prime}_{2}(b_{1},a_{2}) + f^{\prime}_{2}(b_{1},b_{2})}{4}\\ &-(b_{2} - a_{2})\frac{f(a_{1},a_{2}) + f(a_{1},b_{2}) + f(b_{1},a_{2}) + f(b_{1},b_{2}) }{4}  \\ &-~ \lambda (b_2-a_2)
\label{eq:update-trap-2-2}
\end{aligned}
\end{equation}
\begin{equation}
\begin{aligned} 
~~~~&\pd{L}{b_{1}} = (b_{1}-a_{1})(b_{2}-a_{2})\frac{f^{\prime}_{1}(a_{1},a_{2}) + f^{\prime}_{1}(b_{1},a_{2})}{4}\\ &+(b_{1} - a_{1})\frac{f(a_{1},a_{2}) + f(a_{1},b_{2}) + f(b_{1},a_{2}) + f(b_{1},b_{2}) }{4}  \\ &+~ \lambda (b_1-a_1)
\label{eq:update-trap-2-3}
\end{aligned}
\end{equation}
\begin{equation}
\begin{aligned} 
~~~~&\pd{L}{b_{2}} = (b_{1}-a_{1})(b_{2}-a_{2})\frac{f^{\prime}_{2}(a_{1},b_{2}) + f^{\prime}_{2}(b_{1},b_{2})}{4}\\ &+(b_{1} - a_{1})\frac{f(a_{1},a_{2}) + f(a_{1},b_{2}) + f(b_{1},a_{2}) + f(b_{1},b_{2}) }{4}  \\ &+~ \lambda (b_2-a_2)
\label{eq:update-trap-2-4}
\end{aligned}
\end{equation}
Where $f^{\prime}_{1}(.) = \pd{f(u,v)}{u}, f^{\prime}_{2}(.) = \pd{f(u,v)}{v}$.
extending the 2-dimensional to general n-dimensions is straight forward. For $f: \mathbb{R}^{n} \rightarrow (0,1)$  , we define the following. Let the left bound vector be $\mathbf{a} = [a_{1},a_{2},...,a_{n}]$ and the right bound vector $\mathbf{b} = [b_{1},b_{2},...,b_{n}]$ define the n-dimensional hyper-rectangle region of interest. The region is then definesd as follows : $\mathbb{D} = \{\mathbf{x}: \mathbf{a} \leq \mathbf{x} \leq \mathbf{b}\}$, Here, we assume the size of the region is positive at every dimension , \ie $\mathbf{r} =  \mathbf{b} -  \mathbf{a} > \mathbf{0} $. The volume of the region $\mathbb{D}$ normalized by exponent of dimension $n$ is expressed as in \eqLabel{\ref{eq:n-vol-sup}}, the region  $\mathbb{D}$ is defined as in \eqLabel{\ref{eq:n-corners-sup}}, $\mathbf{M}$ is defined as in \eqLabel{\ref{eq:n-mask-sup}}, and $\mathbf{f}_{\mathbb{D}}$ is defined as in \eqLabel{\ref{eq:n-function-sup}}.
now we can see that the loss integral in \eqLabel{\ref{eq:trapezoidal-integration2}} becomes as follows .
\begin{equation}
\begin{aligned} 
L(\mathbf{a},\mathbf{b}) &= \idotsint_\mathbb{D} f(u_1,\dots,u_n) \,du_1 \dots du_n + \frac{\lambda}{2} \left| \mathbf{r}\right|^{2} \\ 
&\approx \triangle \mathbf{1}^{T}\mathbf{f}_{\mathbb{D}} ~+~~ \frac{\lambda}{2} \left| \mathbf{r}\right|^{2}
\label{eq:n-loss-trap}
\end{aligned}
\end{equation}
Similarly to \eqLabel{\ref{eq:n-function-sup}}, we define the gradient matrix $\mathbf{G}_{\mathbb{D}}$ as the matrix of all gradient vectors evaluated at all corner points of $\mathbb{D}$
\begin{equation}
\begin{aligned} 
\mathbf{G}_{\mathbb{D}} &= \left[\nabla f(\mathbf{d}^{1})~|~\nabla f(\mathbf{d}^{2})~|~...~|~\nabla f(\mathbf{d}^{2^{n}}) \right]^{T} 
\label{eq:n-gradient-sup}
\end{aligned}
\end{equation}

The update directions for the left bound $\mathbf{a}$ and the right bound $\mathbf{b}$ becomes as follows by the trapezoid approximation
\begin{equation}
\begin{aligned} 
\nabla_{\mathbf{a}}L  &\approx \triangle \Big(\text{diag}(\overline{\mathbf{M}}\mathbf{G}_{\mathbb{D}}) ~+~ \mathbf{1}^{T}\mathbf{f}_{\mathbb{D}} ~\text{diag}^{-1}(\mathbf{r})\mathbf{1}  \Big)+ \lambda \mathbf{r} \\
\nabla_{\mathbf{b}}L  &\approx \triangle \Big(\text{diag}(\mathbf{M}\mathbf{G}_{\mathbb{D}}) ~-~ \mathbf{1}^{T}\mathbf{f}_{\mathbb{D}} ~\text{diag}^{-1}(\mathbf{r})\mathbf{1}  \Big) - \lambda \mathbf{r}
\label{eq:n-update-trap}
\end{aligned}
\end{equation}
Where $\overline{\mathbf{M}}$ is the complement of the binary mask matrix $\mathbf{M}$.

\clearpage 
\section{Analysis}
\subsection{Detected Robust Regions}
We apply the algorithms on two semantic parameters ( the azimuth angle of the camera around the object, and the elevation angle around the object) that are regularly used in the literature \cite{sada,strike}. When we use one parameter ( the azimuth), we fix the elevation to $\ang{35}$. We used two instead of larger numbers because it is easier to verify and visualize 2D, unlike higher dimensions. Also, the complexity of sampling increase exponentially with dimensionality for those algorithms ( albeit being much better than grid sampling, see Table \ref{tbl:complexity-sup}). The regions in \figLabel{\ref{fig:ex1},\ref{fig:ex2}} look vertical rectangles most of the time. This is because the scale of the horizontal-axis (0,360) is much smaller than the vertical axis (-10,90), so most regions are squares but looks rectangles because of figure scales. 

\subsection{hyper-parameters}
How to select the hyper parameters in all the above algorithms? The answer is not straight forward. For the $\lambda$ in the naive approach, it is set experimentally by trying different values and using the one which detects some regions that known to behave robustly. The values we found for this are $\lambda = 0.07\sim0.1$. The learning rate $\eta$ is easier to detect with observing the convergence and depends on the full range of study. A rule of thumb is to make 0.001of the full range. For the OIR formulations, we have the boundary factor $\alpha$ which we set to 0.05. A rule of thumb is to set it to be between  $ 0.5 \sim 1/N$, where $N$ is the number of samples needed for that dimensions to adequately characterize the space. In our case $N=180$, so $1/180 \approx 0.005$. The only hyperparameter with mathematically established bound is the emphasis factor of the OIR\_W formulation $\beta$. The bound shown in Table \ref{tbl:complexity-sup} which is $0 \leq \beta \leq \frac{2}{2n-1}$ can be shown as follows. We start from \eqLabel{\ref{eq:n-mask-grad-sup}}. This expression is the actual expression for the special masks ( we apologize in the typos in the main paper ). As we can see, the most important term is $\gamma_n$. IT dictates how the function at the boundaries determine the next move of the bounds. Here $\gamma_n$ should always be positive to ensure the correct direction of the movement for positive function evaluation.
    \begin{equation}
\begin{aligned} 
 \gamma_n ~~~~~&> 0 \\
 2-(2n-1)\beta  &> 0  \\
 \beta~~~~~~~ &< \frac{2}{2n-2}
\label{eq:gamma}
\end{aligned}
\end{equation}
\subsection{Detest}
The data set used is collected from ShapeNet \cite{shapenet} and consists of 10 classes and 10 shapes eaach that are all identified by at least ResNet50 trained on ImagNet. This criteria is important to obtain valid NSM and DSM. The classes are \big['aeroplane',"bathtub",'bench', 'bottle','chair',"cup","piano",'rifle','vase',"toilet"\big]. Part of the dataset is shown in \figLabel{\ref{fig:conv1},\ref{fig:conv2}}. 4 shapes faced difficulty of rendering during the SRVR experiments , therefore , they were replaced by another shapes from the same class.

\subsection{Possible Future Directions}
We analyze DNNs from a semantic lens and show how more confident networks tend to create adversarial semantic regions inside highly confident regions. We developed a bottom-up approach to analyzing the networks semantically by growing adversarial regions, which scales well with dimensionality and we use it to benchmark the semantic robustness of DNNs. We aim to investigate how to use the insights we gain from this work to develop and train semantically robust networks from the start while maintaining accuracy. Another direct extension of our work is to develop large scale semantic robustness challenge where we label these robust/adversarial regions in the semantic space and release some of them to allow for training. Then, we test the trained models on the hidden test set to measure robustness while reporting the accuracy on ImageNet validation set to make sure that the features of the model did not get affected by the robust training.

\begin{table}[t]
\footnotesize
\setlength{\tabcolsep}{6pt} %
\renewcommand{\arraystretch}{1} %
\centering
\resizebox{\hsize}{!}{
\begin{tabular}{c||c|c|c|c|c|c} 
\toprule
\specialcell{\textbf{Analysis}\\ \textbf{Approach}} & \textbf{Paradigm}& \specialcell{\textbf{Total} \\ \textbf{Sampling} \\ \textbf{complexity} } & \specialcell{\textbf{Black} \\\textbf{-box}\\ \textbf{Functions} } & \specialcell{\textbf{Forward} \\ \textbf{pass} \\\textbf{/step} } & \specialcell{\textbf{Backward} \\ \textbf{pass} \\\textbf{/step} } & \specialcell{\textbf{Identification} \\ \textbf{Capabaility} } \\
\midrule
\textbf{Grid Sampling} &top-down &\specialcell{ $\mathcal{O}(N^{n})$\\$ N \gg 2$  }  & \textcolor{green}{\checkmark} & - & - & \specialcell{Fully identifies the \\semantic map of DNN}\\ \hline
\textbf{Naive} & bottom-up &$\mathcal{O}(2^{n})$  & \textcolor{green}{\checkmark}& $2^{n}$ &0 & \specialcell{finds strong robust\\ regions only around $\mathbf{u}_{0}$}  \\ \hline
\textbf{OIR\_B}& bottom-up &$\mathcal{O}(2^{n+1})$ & \textcolor{green}{\checkmark} & $2^{n+1}$ &0 & \specialcell{finds strong and\\ week robust regions\\ around $\mathbf{u}_{0}$ }  \\ \hline
\textbf{OIR\_W}& bottom-up &$\mathcal{O}(2^{n})$ & \textcolor{red}{\xmark}  & $2^{n}$& $2^{n}$ & \specialcell{finds strong and\\ week robust regions\\ around $\mathbf{u}_{0}$ }  \\ 
 \bottomrule
\end{tabular}
}
\vspace{-4pt}
\caption{\small \textbf{Semantic Analysis Techniques}: comparing different approaches to analyse the semantic robustness of DNN.}
\vspace{-4pt}
\label{tbl:complexity-sup}
\end{table}
\newpage
\begin{algorithm}[h] 
\caption{Robust n-dimensional Region Finding for Black-Box DNNs by Outer-Inner Ratios}\label{alg: black-sup}
\small
\SetAlgoLined
  \textbf{Requires: } Senatic Function of a DNN $f(\mathbf{u})$ in \eqLabel{\ref{eq:f-sup}}, initial semantic parameter $\mathbf{u}_{0}$, number of iterations T , learning rate $\eta$ , object shape $\mathbf{S}_{z}$ of class label $z$, boundary factor $\alpha$, smalll $\epsilon$ \\
   Form constant binary matrices $\mathbf{M}, \overline{\mathbf{M}},\mathbf{M}_{\mathbb{Q}},\overline{\mathbf{M}_{\mathbb{Q}}}, \mathbf{M}_{\mathbb{D}},\overline{\mathbf{M}_{\mathbb{D}}} $ \\
   Initialize bounds $\mathbf{a}_{0}\leftarrow \mathbf{u}_{0} - \epsilon \mathbf{1} $, $\mathbf{b}_{0} \leftarrow \mathbf{u}_{0}+- \epsilon \mathbf{1}$ \\
    $\mathbf{r}_{0} \leftarrow \mathbf{a}_{0}-\mathbf{b}_{0} $ , update region volume $ \triangle_{0} $ as in \eqLabel{\ref{eq:n-vol-sup}}\\
  \For{$t \leftarrow 1$ \KwTo $T$}{
   form the all-corners function vectors ${f}_{\mathbb{D}},{f}_{\mathbb{Q}}$ as in \eqLabel{\ref{eq:n-function-outer-sup}}\\
    $\nabla_{\mathbf{a}}L  \leftarrow 2\triangle_{t-1}\text{diag}^{-1}(\mathbf{r}_{t-1}) \Big(2\overline{\mathbf{M}}\mathbf{f}_{\mathbb{D}} ~-~ \overline{\mathbf{M}}_{\mathbb{Q}}\mathbf{f}_{\mathbb{Q}}  \Big)$ \\
$\nabla_{\mathbf{b}}L  \leftarrow 2\triangle_{t-1}\text{diag}^{-1}(\mathbf{r}_{t-1}) \Big(-2\mathbf{M}\mathbf{f}_{\mathbb{D}} ~+~ \mathbf{M}_{\mathbb{Q}}\mathbf{f}_{\mathbb{Q}}  \Big)$\\
    update bounds: $\mathbf{a}_{t}\leftarrow \mathbf{a}_{t-1} - \eta \nabla_{\mathbf{a}}L$, $\mathbf{b}_{t}\leftarrow \mathbf{b}_{t-1} - \eta \nabla_{\mathbf{b}}L$ \\
     $\mathbf{r}_{t} \leftarrow \mathbf{a}_{t}-\mathbf{b}_{t} $ , update region volume $ \triangle_{t} $ as in \eqLabel{\ref{eq:n-vol-sup}}
    }
    \textbf{Returns: }robust region bounds: $ \mathbf{a}_{T},\mathbf{b}_{T}$ .
\end{algorithm}

\begin{algorithm}[!b] 
\caption{Robust n-dimensional Region Finding for White-Box DNNs by Outer-Inner Ratios}\label{alg: white-sup}
\small
\SetAlgoLined
  \textbf{Requires: }  Senatic Function of a DNN $f(\mathbf{u})$ in \eqLabel{\ref{eq:f-sup}}, initial semantic parameter $\mathbf{u}_{0}$, , learning rate $\eta$ , object shape $\mathbf{S}_{z}$ of class label $z$, emphasis factor $\beta$, smalll $\epsilon$ \\
   Form constant binary matrices $\mathbf{M}, \overline{\mathbf{M}}, \mathbf{M}_{\mathbb{D}},\overline{\mathbf{M}_{\mathbb{D}}} $ \\
   Initialize bounds $\mathbf{a}_{0}\leftarrow \mathbf{u}_{0} - \epsilon \mathbf{1} $, $\mathbf{b}_{0} \leftarrow \mathbf{u}_{0}+- \epsilon \mathbf{1}$ \\
    $\mathbf{r}_{0} \leftarrow \mathbf{a}_{0}-\mathbf{b}_{0} $ , update region volume $ \triangle_{0} $ as in \eqLabel{\ref{eq:n-vol-sup}}\\
  \For{$t \leftarrow 1$ \KwTo $T$}{
   form the all-corners function vector ${f}_{\mathbb{D}}$ as in \eqLabel{\ref{eq:n-function-outer-sup}}\\
   form the all-corners gradients matrix $\mathbf{G}_{\mathbb{D}}$ as in \eqLabel{\ref{eq:n-gradient-sup}}\\ 
   form the gradient selection vectors $\mathbf{s} ,\overline{\mathbf{s}}$ as in \eqLabel{\ref{eq:n-update-grad-selection-sup}}
   $\nabla_{\mathbf{a}}L \leftarrow  \triangle_{t-1} \Big(\text{diag}^{-1}(\mathbf{r}_{t-1})\overline{\mathbf{M}}_{\mathbb{D}}\mathbf{f}_{\mathbb{D}} + \beta\text{diag}(\overline{\mathbf{M}}\mathbf{G}_{\mathbb{D}}+ \beta \overline{\mathbf{s}}  \Big)$  \\
$\nabla_{\mathbf{b}}L  \leftarrow  \triangle_{t-1} \Big(- \text{diag}^{-1}(\mathbf{r}_{t-1})\mathbf{M}_{\mathbb{D}}\mathbf{f}_{\mathbb{D}} + \beta\text{diag}(\mathbf{M}\mathbf{G}_{\mathbb{D}})+ \beta \mathbf{s}  \Big)$\\
    update bounds: $\mathbf{a}_{t}\leftarrow \mathbf{a}_{t-1} - \eta \nabla_{\mathbf{a}}L$, $\mathbf{b}_{t}\leftarrow \mathbf{b}_{t-1} - \eta \nabla_{\mathbf{b}}L$ \\
     $\mathbf{r}_{t} \leftarrow \mathbf{a}_{t}-\mathbf{b}_{t} $ , update region volume $ \triangle_{t} $ as in \eqLabel{\ref{eq:n-vol-sup}}
    }
    \textbf{Returns: }robust region bounds: $ \mathbf{a}_{T},\mathbf{b}_{T}$ .
\end{algorithm}

\end{document}